\documentclass[onefignum,onetabnum]{siamonline190516}

\usepackage{hyperref}
\usepackage{amsmath}
\usepackage{amsfonts}
\usepackage{amssymb}
\usepackage{graphicx}
\usepackage{color}
\usepackage{enumerate}
\usepackage{bm}
\usepackage{bbm}
\usepackage{enumitem}
\usepackage{comment}
\usepackage{cancel}
\usepackage{soul}
\usepackage{setspace}
\usepackage{hyperref}
\usepackage{epstopdf}
\usepackage{algorithmic}
\usepackage{lipsum}                     % Dummytext
\usepackage{xargs}                      % Use more than one optional parameter in a new commands
\usepackage{xr}
\usepackage{adjustbox}

\newtheorem{assumption}[theorem]{Assumption}

\usepackage[colorinlistoftodos,prependcaption,textsize=tiny]{todonotes}
\newcommandx{\unsure}[2][1=]{\todo[linecolor=red,backgroundcolor=red!25,bordercolor=red,#1]{#2}}
\newcommandx{\change}[2][1=]{\todo[linecolor=blue,backgroundcolor=blue!25,bordercolor=blue,#1]{#2}}
\newcommandx{\info}[2][1=]{\todo[linecolor=olive,backgroundcolor=olive!25,bordercolor=black,#1]{#2}}
\newcommandx{\improvement}[2][1=]{\todo[linecolor=teal,backgroundcolor=teal!25,bordercolor=teal,#1]{#2}}

\newcommandx{\thiswillnotshow}[2][1=]{\todo[disable,#1]{#2}}

\usepackage{color}

\newcommand*{\nm}[1]{{\left\|#1\right \|}}
\newcommand*{\nmu}[1]{{\|#1\|}}

\newcommand{\ben}[1]{\begin{equation} #1 \end{equation}}
\newcommand{\bes}[1]{\begin{equation*} #1 \end{equation*}}

\newcommand{\eas}[1]{\begin{align*} #1 \end{align*}}

\newcommand*{\bbC}{\mathbb{C}}
\newcommand*{\bbN}{\mathbb{N}}
\newcommand*{\bbR}{\mathbb{R}}

\def\D{\,\mathrm{d}}

\newcommand{\be}{\begin{equation}}
\newcommand{\ee}{\end{equation}}
\newcommand{\bea}{\begin{eqnarray}}
\newcommand{\eea}{\end{eqnarray}}
\newcommand{\beas}{\begin{eqnarray*}}
\newcommand{\eeas}{\end{eqnarray*}}

\newcommand{\bmnu}{{\ensuremath{\bm \nu}}}

\newcommand{\cU}{\ensuremath{\mathcal{U}}}

\newcommand{\cO}{\ensuremath{\mathcal{O}}}

\newcommand{\sbmatrix}{\begin{bmatrix}}
\newcommand{\fbmatrix}{\end{bmatrix}}

\def\cE{{\cal E}}
\def\cA{{\cal A}}
\def\cF{{\cal F}}
\def\cJ{{\cal J}}

\def\cB{{\cal B}}

\def\cL{{\cal L}}
\def\cV{{\cal V}}
\def\cU{{\cal U}}

\def\N{{\mathbb N}}

\def\R{{\mathbb R}}

\def\bc{{\bm c}}

\def\bx{{\bm x}}

\def\bz{{\bm z}}

\def\bA{{\bm A}}

\def\({\Bigl (}
\def\){\Bigr )}
\def\[{\Bigl [}
\def\]{\Bigr ]}
\def \<{\langle}
\def \>{\rangle}

\DeclareMathOperator*{\supp}{supp}

\newcommand{\Loeve}{Lo{\`e}ve }

\newcommand{\epsrel}{\varepsilon_{\textnormal{rel}}}

\graphicspath{{./figures/}}

\usepackage{lipsum}
\usepackage{amsfonts}
\usepackage{graphicx}
\usepackage{epstopdf}
\usepackage{algorithmic}
\ifpdf
  \DeclareGraphicsExtensions{.eps,.pdf,.png,.jpg}
\else
  \DeclareGraphicsExtensions{.eps}
\fi

\usepackage{enumitem}
\setlist[enumerate]{leftmargin=.5in}
\setlist[itemize]{leftmargin=.5in}

\newsiamremark{remark}{Remark}
\newsiamremark{hypothesis}{Hypothesis}
\crefname{hypothesis}{Hypothesis}{Hypotheses}
\newsiamthm{claim}{Claim}

\headers{The gap between theory and practice}{B. Adcock and N. Dexter}

\title{The gap between theory and practice in function approximation with deep neural networks\thanks{Submitted to the editors January 16th, 2020.
\funding{
N.D.\ acknowledges the support of the PIMS Postdoctoral Fellowship program. This work was supported by the PIMS CRG ``High-dimensional Data Analysis'', SFU's Big Data Initiative ``Next Big Question" Fund and by NSERC through grant R611675.}}}

\author{Ben~Adcock and Nick Dexter\thanks{Simon Fraser University, 8888 University Drive, Burnaby, BC V5A 1S6, Canada (\tt{ben\_adcock@sfu.ca}, \tt{nicholas\_dexter@sfu.ca})}}

\usepackage{amsopn}

\usepackage{enumerate}
\usepackage[nocompress]{cite}

\ifpdf
\hypersetup{
  pdftitle={The gap between theory and practice in function approximation with deep neural networks},
  pdfauthor={B. Adcock and N. Dexter}
}
\fi

\begin{document}

\maketitle

% REQUIRED
\begin{abstract}
Deep learning (DL) is transforming whole industries as complicated decision-making processes are being automated by {\em deep neural networks} (DNNs) trained on real-world data.  Driven in part by a rapidly-expanding literature on DNN approximation theory showing that DNNs can approximate a rich variety of functions, these tools are increasingly being considered for problems in scientific computing. Yet, unlike more traditional algorithms in this field, relatively little is known about DNNs from the principles of numerical analysis, namely, stability, accuracy, computational efficiency and sample complexity. In this paper we first introduce a computational framework for examining DNNs in practice, and then use it to study their empirical performance with regard to these issues. We examine the performance of DNNs of different widths and depths on a variety of test functions in various dimensions, including smooth and piecewise smooth functions. We also compare DL against best-in-class methods for smooth function approximation based on compressed sensing.  Our main conclusion from these experiments is that there is a crucial gap between the approximation theory of DNNs and their practical performance, with trained DNNs performing relatively poorly on functions for which there are strong approximation results (e.g.\ smooth functions), yet performing well in comparison to best-in-class methods for other functions. To analyze this gap further, we then provide some theoretical insights. We establish a \textit{practical existence theorem}, which asserts the existence of a DNN architecture and training procedure that offers the same performance as compressed sensing. This result establishes a key theoretical benchmark. It demonstrates that the gap can be closed, albeit via a DNN approximation strategy which is guaranteed to perform as well as, but no better than, current best-in-class schemes. Nevertheless, it demonstrates the promise of practical DNN approximation, by highlighting the potential for developing better schemes through the careful design of DNN architectures and training strategies.
\end{abstract}

% REQUIRED

% REQUIRED

% !TEX root = ./MLFA.tex

\reversemarginpar

%----------------------------------------------------
\section{Introduction}
\label{sec:introduction}
%----------------------------------------------------

The past decade has seen an explosion of interest in the field of machine learning, largely due to the impressive results achieved with DNNs. Breakthroughs have been obtained on large classes of historically-challenging problems, including: speech recognition \cite{Dahl2012,Hinton2012} and natural language processing \cite{Wu2016}, image classification \cite{Krizhevsky2012,Simonyan2015}, game intelligence \cite{SilverD2017Mtgo}, and autonomous vehicles \cite{FarabetClement2012SPwM}.
As DNNs have shown such promise in these real-world applications, a trend has developed in the scientific computing community towards applying them to problems in mathematical modelling and computational science. 
Recent studies have focused on applications ranging from image reconstruction tasks in medical imaging \cite{ArridgeEtAlACTA},
 discovering underlying {\em partial differential equation} (PDE) dynamics \cite{Rudye1602614} and approximating solutions of PDEs \cite{Weinan2018,EWeinan2017DLNM} to complex mathematical modeling, prediction, and classification tasks in physics \cite{CarrasquillaJuan2016Mlpo}, biology \cite{TarcaAdiL2007MLaI, SommerChristoph2013Mlic, ZielinskiBartosz2017Dlat}, and engineering \cite{LOYER2016109,TAFFESE20171}. 

Simultaneously, the broader applied mathematics community has taken interest in the approximation capabilities of NNs \cite{Yarotsky2017,Bach2017,Petersen2018,Beck2019}. The earliest results in this direction \cite{Cybenko1989,HornikKurt1989Mfna} established that even a single hidden layer fully-connected NN has universal approximation capability: so long as the number of nodes in the hidden layer are allowed to grow unbounded, such architectures are able to approximate any Borel measurable function on a compact domain to arbitrary uniform accuracy.
 More recent works have studied the connection between expressiveness of DNNs and their depth \cite{Liang2017,Lu2020,Yarotsky2018}, while others have established connections between DNNs and other methods of approximation, e.g., sparse grids \cite{Montanelli2017}, splines \cite{Unser2019}, polynomials \cite{Schwab2017,Daws2019b}, and ``$h,p$''-finite elements \cite{Opschoor2019}. A plethora of results now exist concerning the approximation power of DNNs for different function spaces -- e.g.\ Sobolev spaces \cite{Guhring2019}, bandlimited functions \cite{Montanelli2017}, analytic functions \cite{e2018exponential, Opschoor2019Legendre}, Barron functions \cite{e2019barron}, 
 cartoon-like functions \cite{Grohs2016}, H\"older spaces \cite{Shen2019} -- and tasks in scientific computing, such as approximation of high-dimensional functions \cite{Schwab2017,Li2019}  and PDEs \cite{Grohs2018,Berner2018,kutyniok2019theoretical}, dimensionality reduction \cite{Zhang2019}, and methods for DEs \cite{Lu2017,Weinan2018}. Theoretically, these works establish best-in-class approximation properties of DNNs for many problems.
Concurrently, some works have sought to address the construction of DNN approximations directly, although typically without theoretical guarantees on trainability. For example \cite{Dereventsov2019b,Fokina2019} employ ideas from greedy methods in order to do this, \cite{Dereventsov2019a} derives a formula for integral representations of shallow ReLU networks, and \cite{Daws2019a} construct networks directly approximating the Legendre basis, using this as an initialization point for training. On the other hand, works such as \cite{pmlr-v107-cyr20a} reinterpret DNN approximation as approximation in adaptive bases, which can be learned via training data. 

\subsection{Challenges} Yet despite the impressive empirical and theoretical results achieved in the broader DL community, there is concern that methods based on DNNs do not currently meet the usual rigorous standards for algorithms in computational science \cite{sciMLDOE}. While the aforementioned theoretical results assert the \textit{expressibility} of the class of DNNs -- that is, the \textit{existence} of a DNN of a given architecture that achieves a desired rate of convergence for a given problem -- they say little about their practical performance when trained by modern approaches in DL. If such techniques are to achieve widespread adoption in scientific computing, it is vital they be understood through the lens of numerical analysis, namely, (i) stability, (ii) accuracy, (iii) sample complexity, (iv) curse of dimensionality and (v) computational cost.

{\bf (i) Stability.}\ Recently, researchers have begun to question the stability properties of DNNs \cite{MDMohsen2016Uap, SzegedyChristian2013Ipon, FawziAlhussein2017TRoD}.
A series of works have demonstrated that DNNs trained on tasks such as image classification are vulnerable to misclassification when provided images with small ``adversarial'' perturbations \cite{MDMohsen2015Dasa} and can even completely fail on image reconstruction tasks in the presence of small structural changes in the data \cite{Antun2019,GottschlingTroublesome}. 
As deep learning is increasingly being applied towards critical problems in healthcare, e.g., DeepMind's recent work on machine-assisted diagnostic imaging in retinal disease \cite{DeFauw2018}, many have questioned the ethics of applying tools whose stability properties are not fully understood to such problems.

{\bf (ii) Accuracy.}\
Over the past 5 years, many works have been published on the classes of functions, e.g., analytic or piece-wise continuous, that can be approximated by DNNs of a given size with a certain rate of convergence. These results are constructive, often showing the existence of a DNN emulating another approximation scheme, e.g. polynomials, for which convergence rates have already been established. 
While such results provide a useful benchmark for DNN expressivity, they do not suggest methods for training DNNs that reliably achieve the tolerances required in computational science applications. 

{\bf (iii) Sample complexity.}\  
Areas in which DL has seen the greatest success include problems in supervised learning such as image classification. In such settings, DNNs are trained on large sets of labeled images, yielding a model capable of predicting labels for unseen images. Popular datasets for DL competitions include the ImageNet database which contains 14 million hand-annotated images of more than 20,000 categories of subjects \cite{imagenet_cvpr09}. In contrast, problems in computational science are often relatively \textit{data-starved}, e.g.\ applications in \textit{uncertainty quantification} (UQ) which involve computing a quantity of interest from sampled solutions of a parameterized PDE \cite{Gunzburger2014}. As each sample involves the discretization and solution of a PDE, which may require thousands of degrees of freedom to accurately resolve, there is great attention paid in such problems to minimizing the required number of samples \cite{Adcock2016,dexter2018mixed}.

{\bf (iv) Curse of dimensionality.}\  Many modern problems in scientific computing involve high dimensionality. High-dimensional PDEs occur in numerous applications, and parametrized PDEs in UQ applications often involve tens to hundreds of variables. Recent works have shown that certain DNNs have the expressive capabilities to mitigate the curse of dimensionality to the same extent as current best-in-class schemes \cite{Montanelli2017,kutyniok2019theoretical,e2018exponential,Opschoor2019Legendre,Schwab2017,Grohs2018,Berner2018}. Yet, as noted, this does not assert these rates can be achieved via training. Moreover, the curse of dimensionality is an important consideration in the sample complexity, as the cost of obtaining samples can often dominate the overall cost.
A recent numerical study has shown that approximation quality degrades with increasing dimension for approximating solutions of high-dimesional parameterized PDEs \cite{geist2020numerical}. However the observed scaling is not exponential in the dimension $d$, but dependent on the complexity of the underlying problem. Understanding this scaling with respect to the sample complexity is crucial to applying these methods in computational science applications.

{\bf (v) Computational cost.}\
By far, the largest barrier to entry for DL research is the cost of training. DNNs are typically trained on {\em graphics processing units} (GPUs), and a single GPU can cost thousands of US dollars. In many industry applications, models are trained on hundreds of these specialized cards. In addition, the training process itself is very energy-intensive, and can produce a large amount of excess CO$_2$ emissions\footnote{A recent study estimated the cost of training a natural language processing model for 274,000 GPU hours at between \$942,000-\$3,300,000 USD, meanwhile producing an excess of 626,000 lbs of CO$_2$, or the equivalent of 5 cars output over their expected lifespan \cite{strubell2019energy}.}.  
Even a small reduction in computational cost can yield large cost savings and greater access to resources for researchers.

In the near term, it seems likely that any DL implementation will pay a price in computational cost. Hence there needs to be a clear understanding of the benefits vis-a-vis properties (i)--(iv) above.  The study of these concerns is the broad purpose of this paper.

\subsection{Contributions}
Our main objective is to examine practical DNN approximation on problems motivated by scientific computing. In many applications in computational science, the core task involves approximating a function $f:\cU \to \R$, with domain $\cU\subset \R^d$ where $d\geq1$ (often $d \gg 1$). Hence our main aim is to examine the performance of DL on practical function approximation through the five considerations (i)--(v).  Our main contributions are:

{\bf 1.}\ We develop a computational framework for examining the practical capabilities of DNNs in scientific computing, based on the rigorous testing principles of numerical analysis. We provide clear practical guidance on training DNNs for function approximation problems.

{\bf 2.}\ We conduct perhaps the first comprehensive empirical study of the performance of training fully-connected feedforward ReLU DNNs on standard function classes considered in numerical analysis, namely, (piecewise) smooth functions on bounded domains. We compare performance over a range of dimensions, examining the capability of DL for mitigating the curse of dimensionality. We examine the effect of network architecture (depth and width) on both ease of training and approximation performance of the trained DNNs.
We also make a clear empirical comparison between DL and current best-in-class approximation schemes for smooth function approximation. The latter is based on polynomial approximation via \textit{compressed sensing} (CS) \cite{Adcock2016}, which (as we also show) achieves exponential rates of convergence for analytic functions in arbitrarily-many dimensions, see Section \ref{ss:expconvpoly} for more details.

{\bf 3.}\ We present theoretical analysis that compares the performance of DL to that of CS for smooth function approximation.  In particular, we establish a novel \textit{practical existence theorem}, which asserts the existence of a DNN architecture and training procedure (based on minimizing a certain cost function) that attains the same exponential rate of convergence as CS for analytic function approximation with the same sample complexity. 

\subsection{Conclusions}\label{sec:introconc}
The primary conclusion of this work is the following. While it is increasingly well understood that DNNs have substantial expressive power for problems relevant to scientific computing, there remains a large gap between expressivity and practical performance achieved with standard methods of training.  Surprisingly, trained DNNs can perform very badly on functions for which there are strong expressivity results, such as smooth functions in high dimensions and piecewise smooth functions. Yet, on other examples, they are competitive with current best-in-class schemes based on CS. We also draw the following conclusions based on our experimental results training fully-connected feedforward ReLU DNNs: 

{\bf 1.}\ The accuracy of trained DNNs is limited by the precision used, but is typically nowhere near machine epsilon despite training to such tolerances in the loss. In this work, we perform experiments in both single and double precision. Yet, in both cases, it is typically impossible to get beyond four digits of accuracy even when approximating extremely smooth functions.
In contrast, a combination of Legendre polynomial approximation in the {\em hyperbolic cross} subspace with weighted $\ell_1$-minimization and a specific choice of weights motivated by smooth function approximation can reliably achieve six or seven digits of accuracy on such problems.
Since our theoretical contribution shows that DNNs should be capable of obtaining these results (albeit with a different initialization and training procedure), this fact suggests a limitation of standard training methods in obtaining more accurate results.

{\bf 2.}\ The training process itself is also highly sensitive to the parameterization of the solvers and initialization of the weights and biases of the DNNs. After extensive testing, we chose the {\tt Adam} optimizer with exponentially decaying learning rate over a variety of optimizers including standard {\tt SGD}, finding empirically that this strategy can help to mitigate some of the challenges of solving the non-convex optimization problem of training, e.g., non-monotonic decrease in the loss, and slow convergence to minimizers.
We also initialize our networks with symmetric uniform or normal distributions with small variance, finding larger variances can lead to failure in training or slow convergence.
These choices combined lead to a training process that is largely stable and minimizes the {\em probability of failure} in training over a wide range of architectures. However, failures can still occur and are often a consequence of choosing a network that is either too shallow and narrow (where failure occurs immediately and the error stagnates), or too wide and deep (where the resulting network achieves order machine epsilon tolerance in the loss, but massively overfits exhibiting numerical artifacts).

{\bf 3.}\  Generally speaking, deeper architectures (which are sufficiently wide) are both easier to train and perform better than shallower architectures. However, this is clearly highly-dependent on the regularity and dimensionality of the target function for smooth-function approximation problems. 
The width of the network also plays an important role; we find networks with width 5-10 times larger than the depth perform better.
While this trend appears to be general, some of our results on less-smooth and piecewise continuous functions indicate that it is not universal.

{\bf 4.}\ While much of the success of DL has been in the field of classification, surprisingly performance of trained DNNs on simple piecewise constant functions is relatively poor in comparison to smooth functions, and adversely affected by the curse of dimensionality. 
Yet, DNNs do approximate such functions to some accuracy, unlike polynomial-based CS techniques. This highlights the flexibility of the DL approach.

We also draw the following theoretical conclusions:

{\bf 5.} For analytic function approximation, a certain DNN strategy based on emulating polynomials via a suitable DNN can achieve the same guaranteed performance (up to constants) as best-in-class schemes based on CS.

{\bf 6.} However, such scheme will not typically offer any superior performance. In particular, it is not \textit{flexible}. It only perform well on the function classes on which CS performs well, and perform correspondingly poorly on other, e.g.\ piecewise smooth functions. This is in contrast to the experimental results in this paper, which show that standard feedforward architectures and training via the $\ell^2$-loss can approximate both smooth and nonsmooth functions, and expressibility results, which show that DNN architectures can approximate functions from a range of different function classes.

\subsection{Outlook} This paper raises and seeks to answer the following question: is DL a useful tool for problems in scientific computing? Some of the above conclusions may appear rather negative in this regard, certainly in comparison to the positive impression given by the plethora of expressivity results on DNNs. Let us raise several caveats. First, this study considers one particular setup: namely, fully-connected, ReLU networks of constant hidden layer widths, in combination with the $\ell^2$-loss function. There are almost countless variations on this setup, some of which will undoubtedly perform better. These variations include different activation functions (e.g.\ sigmoid, hyperbolic tangent), different architectures (e.g.\ ResNets, sparsely-connected layers, convolutional layers) and different loss functions (e.g.\ those incorporating regularization). We elected to use this setup based on standards in the literature; for instance, most expressibility results consider ReLU activations. Given difficulties and intense computational resources required for training, it is beyond the scope of this first work to methodically compare all possible setups. Second, the practical existence theorem we prove shows that a careful choice of architecture and cost function can allow DNNs to offer similar performance to state-of-the-art techniques. While it comes at the price of flexibility, this result nonetheless provides a theoretical benchmark for DNN approximations. It also demonstrates the potential of DNN strategies for eventually outperforming such methods, and indicates that theory-inspired architecture and training strategy design as a way to achieve such improvements. The extent to which this can (provably) be done while maintaining flexibility -- i.e.\ the ability to approximate different function classes with the same DNN procedure -- is an enticing challenge for future work.

\subsection{Outline} The outline of the remainder of this paper is as follows. In \S \ref{sec:framework} we introduce the approximation problem, DNNs, DL and CS. In \S \ref{sec:testingsetup} we describe the experimental setup, including details of the training procedure used. Our numerical results are found in \S \ref{sec:experiments}. Finally, in \S \ref{sec:theory1} we present our theoretical results. Additional information for this paper is contained in the Supplementary Materials. Code accompanying the computational framework is available at \url{https://github.com/ndexter/MLFA}.

% !TEX root = ./MLFA.tex

%----------------------------------------------------
\section{Framework} 
\label{sec:framework}
%----------------------------------------------------

In this section, we first describe the function approximation problem, and then introduce DNNs, DL, polynomial approximation and CS.

\subsection{Problem formulation}

Throughout this paper, we consider the unit cube in $d$ dimensions, $\cU = (-1,1)^d$, equipped with the uniform probability measure $\D \varrho = 2^{-d} \D \bm{x}$, where $\D \bm{x}$ is the Lebesgue measure and $\bm{x} = (x_1,\ldots,x_d)$ is the $d$-dimensional variable.  
Let $L^2(\cU)$ denote the space of real-valued square-integrable functions on $\cU$ with respect to $\varrho$. Our objective is to approximate an unknown function $f \in L^2(\cU)$ from samples.
These samples are generated by simple Monte Carlo sampling: we draw $\bm{x}_1,\ldots,\bm{x}_m$ randomly and independently from the measure $\varrho$.  Hence the approximation problem we aim to solve is
\ben{
\label{eq:approx_problem}
\mbox{Given the measurements $\{ (\bm{x}_i,f(\bm{x}_i)) \}^{m}_{i=1}$, approximate $f$.} \tag{AP}
} 
We note in passing that much of what follows in this paper can be extended to more general domains, sampling strategies and to functions taking values in other vector spaces (for instance, complex-valued functions, vector-valued functions, or even Hilbert-valued functions, as arise commonly in UQ applications \cite{dexter2018mixed}). We assume the above setup for ease of presentation.

We require several further pieces of notation.  We write $\nm{\cdot}_{L^2}$ for the $L^2$-norm with respect to $\varrho$. The space of essentially bounded functions on $\cU$ is denoted by $L^\infty(\cU)$ and its norm by $\nm{\cdot}_{L^{\infty}}$. We use $\bm{\nu} = (\nu_1,\ldots,\nu_d)$ to denote a  (multi)index of length $d$. If $0 < p < \infty$ and $\cF \subseteq \bbN^d_0$ is a finite or countable (multi)index set, we write $\ell^p(\cF)$ for the space of $\ell^p$-summable sequences $\bc = (c_\bmnu)_{\bmnu\in\cF}\subset \R$, i.e.\ those satisfying $\|\bc\|_{p} := \left( \sum_{\bmnu\in\cF} |c_\bmnu|^p \right)^{1/p} < \infty$. When $p = \infty$, we define $\ell^{\infty}(\cF)$ and $\nm{\cdot}_{\infty}$ in the usual way.

%----------------------------------------------------
\subsection{Deep Learning}
\label{subsec:DL}
%----------------------------------------------------

We now introduce DL. First, we recall the definition of a DNN:

\begin{definition}[Neural network]
\label{d:NN}
Let $L\in \N_0$ and $N_0, \ldots, N_{L+2} \in \N$. A map $\Phi: \R^{N_0} \to \R^{N_{L+2}}$ given by
\begin{align}
\label{eq:DNN_Phi}
\Phi(x) = \begin{cases}
\cA_1 (\rho(\cA_0(\bm{x}))), & L = 0 \\
\cA_{L+1} ( \rho ( \cA_{L} ( \rho ( \cdots \rho ( \cA_0 (\bm{x}) ) \cdots ) ) ) ), & L \ge 1
\end{cases}
\end{align}
with affine linear maps $\cA_l : \R^{N_{l}} \to \R^{N_{l+1}}$, $l = 0,\ldots,L+1$, and the activation function $\rho$ acting component-wise (i.e., $\rho(\bm{x}) := (\rho(x_1), \ldots, \rho(x_d))$ for $\bm{x} = (x_1,\ldots,x_d)$) is called a neural network (NN). The map $\cA_l$ corresponding to layer $l$ is given by $\cA_l (\bm{x}) = \bm{W}_l \bm{x} + \bm{b}_l$, where $\bm{W}_l \in \R^{N_{l+1} \times N_{l}}$ is the $l$th weight matrix and $\bm{b}_l \in \R^{N_{l+1}}$ the $l$th bias vector. We refer to $L$ as the depth of the network and $\max_{1\le l \le L+1} N_l$ as its width. 
\end{definition}

Informally, we consider a DNN as any NN with $L \geq 1$ hidden layers.
Definition \ref{d:NN} pertains to \textit{feedforward} DNNs. We do not consider more exotic constructions such as recurrent networks or ResNets in this paper. We also consider so-called \textit{fully connected} networks, meaning that the weights and biases can take arbitrary real values. The layers $l = 1,\ldots,L$ are referred to as \textit{hidden} layers. In our experiments later, we set their widths to be equal, $N_1 = \ldots = N_{L+1}$. Note that $N_0$ and $N_{L+2}$ are specified by the problem. In our case, $N_0 = d$ and $N_{L+2} = 1$.

There are numerous choices for the activation function $\rho$, and moreover, one may also choose different activation functions in different layers. Since it is popular both in theory and in practice, we use the \textit{rectified linear unit} (ReLU), defined by $\rho(x) = \max \{0,x\}$. We will also refer to a DNN architecture having ReLU activation function and $L$ hidden layers with $N$ nodes per layer as a ReLU $L\times N$ DNN.

The \textit{architecture} of a network is the specific choice of activation $\rho$ and parameters $L$ and $N_1,\ldots,N_{L+1}$. We denote the set of neural networks of a given architecture by $\mathcal{N}$.  Note that this family is parametrized by the weight matrices and biases. Selecting the right architecture for a given problem is a significant challenge. We discuss this topic further in \S \ref{sec:testingsetup} and \S \ref{sec:experiments}.

Given an unknown function $f \in L^2(\cU)$, \textit{training} is the process of computing a neural network $\Phi$ that approximates $f$ from the data $\{ (\bm{x}_i,f(\bm{x}_i)) \}^{m}_{i=1}$. This is normally achieved by minimizing a \textit{loss function} $\cL : \mathcal{N} \rightarrow \bbR$, i.e.\ we solve
\bes{
\textnormal{minimize}_{\Phi \in \mathcal{N}} \cL(\Phi),
} 
where $\mathcal{N}$ is the family of neural networks of the chosen architecture. Note that this is equivalent to a minimization problem for the weights $\bm{W}_l$ and biases $\bm{b}_l$.
A typical choice is the $\ell^2$-loss (also known as \textit{empirical risk}, \textit{mean squared loss}):
\begin{align}
\label{eq:ER_loss}
\mathcal{L}(\Phi) := \frac{1}{m} \sum_{i=1}^m \left( \Phi(\bm{x}_i) - f(\bm{x}_i) \right)^2.
\end{align}
We primarily use this loss function in this paper.
However, many other choices are possible.  For instance, it is common to add a regularization term to the loss function, e.g.\
\bes{
\mathcal{L}(\Phi) := \frac{1}{m} \sum_{i=1}^m \left( \Phi(\bm{x}_i) - f(\bm{x}_i) \right)^2 + \cJ(\Phi).
}
Here $\cJ : \mathcal{N} \rightarrow \bbR$ is chosen to promotes some desirable features of the network.  For instance, $\cJ$ may be a norm of the weight matrices, thus promoting small and/or sparse weights.

%----------------------------------------------------
\subsection{Polynomial approximation of smooth functions}\label{sec:polyapprox} 
%----------------------------------------------------
We now introduce the polynomial approximation schemes against which we compare DL for function approximation.
Polynomial approximation is a vast and classical topic. Yet, it has received renewed attention in the last several decades, motivated by applications in UQ where one seeks to approximate a smooth quantity of interest of a parametric PDE \cite{CohenDeVoreApproxPDEs}. The particular scheme we consider is based on orthogonal expansions in orthonormal polynomials in $L^2(\cU)$, i.e.\ multivariate Legendre polynomials.
The univariate, orthonormal Legendre polynomials on the $[-1,1]$ are defined by
\bes{
\psi_{\nu}(x) = \sqrt{2 \nu + 1} P_{\nu}(x),\qquad \nu \in \N_0,
}
where $P_{\nu}$ is the classical Legendre polynomial with normalization $P_{\nu}(1) = 1$.  The functions $\psi_{\nu}$ form an orthonormal basis of $L^2(-1,1)$.  When $d > 1$, we define the tensor orthonormal Legendre polynomials as
\bes{
\Psi_{\bm{\nu}}(\bm{x}) = \prod^{d}_{i=1} \psi_{\nu_i}(x_i),\qquad \bm{\nu} = (\nu_1,\ldots,\nu_d) \in \N^d_0,\ \bm{x} = (x_1,\ldots,x_d) \in \cU.
}
The set $\{ \Psi_{\bm{\nu}} \}_{\bm{\nu} \in \bbN^d_0}$ forms an orthonormal basis of $L^2(\cU)$. Hence any function $f \in L^2(\cU)$ has a convergent expansion
\ben{
\label{eq:f_inf_exp}
f = \sum_{\bm{\nu} \in \N^d_0} c_{\bm{\nu}} \Psi_{\bm{\nu}},
}
where $c_{\bm{\nu}} = \int_{\cU} f(\bm{x}) \Psi_{\bm{\nu}}(\bm{x}) \D \varrho(\bm{x})$ is the coefficient of $f$ with respect to $\Psi_{\bm{\nu}}$.  Note that the sequence $\bc = (c_{\bmnu})_{\bmnu\in\N^d_0}$ is an element of $\ell^2(\N^d_0)$, the space of square-summable sequences with indices in $\N^d_0$.  By Parseval's identity, $\nm{f}_{L^2(\cU)} = \nm{\bm{c}}_2$.

When $d = 1$, approximating a smooth function in the Legendre basis is typically achieved by truncating the expansion \eqref{eq:f_inf_exp} after its first $s$ terms, then using, for instance, least-squares to approximately recover the coefficients $c_0,\ldots,c_{s-1}$ from the measurements $\{ (x_i,f(x_i)) \}^{m}_{i=1}$.  In $d \geq 2$ dimensions, the situation becomes more complicated, since there are many different choices of index set $S$ of cardinality $s$ one might employ to truncate the expansion \eqref{eq:f_inf_exp}:
$
f \approx f_{S} = \sum_{\bm{\nu} \in S} c_{\bm{\nu}} \Psi_{\bm{\nu}}.
$
By Parseval's identity, the error
$
\nm{f - f_{S}} = \sqrt{\sum_{\bm{\nu} \in \N^d_0 \backslash S} | c_{\bm{\nu}} |^2 }
$
depends on the coefficients outside $S$. Hence, an \textit{a priori} choice of $S$ may have limited effectiveness, since it may fail to capture any anisotropic behaviour of $f$.  This naturally motivates the concept of \textit{best $s$-term approximation} \cite{CohenDeVoreApproxPDEs}.  In best $s$-term approximation, the index set $S$ is chosen so that it contains the multi-indices corresponding to the large $s$ coefficients $c_{\bm{\nu}}$ in absolute value.  This is a type of nonlinear approximation scheme \cite{DeVoreNLACTA}.  If $\tilde{f}_s$ denotes the best $s$-term approximation, the error satisfies
\bes{
\nmu{f - \tilde{f}_s }_{L^2} = \inf \left \{ \sqrt{\sum_{\bm{\nu} \notin S} | c_{\bm{\nu}} |^2 } : S \subset \N^d_0,\ |S| \leq s \right \}.
}
Under appropriate conditions (e.g.\ $f$ is analytic), this approximation converges exponentially fast in $s$ (see Theorem \ref{thm:LegExpOSZ}). In high dimensions this significantly improves over any linear approximation scheme based on a fixed, isotropic choice of $S$. We discuss this further in \S \ref{ss:expconvpoly}.

\subsection{Polynomial approximation with compressed sensing}

Computing the best $s$-term approximation is on the face of it a daunting task.  In theory, it involves computing all infinitely-many of the coefficients $\bm{c}$, then selecting the largest $s$. This is of course intractable, and generally still computationally infeasible even if one limits oneself to computing a large, but finite number of coefficients. 

A solution is to use compressed sensing. Here one first selects a large, but finite multi-index set $\Lambda \subset \N^d_0$.  This set is generally assumed to contain the coefficients of some quasi-best $s$-term approximation, if not the coefficients of the true best $s$-term approximation itself.  For reasons that will be made clear in \S \ref{subsec:CS}, a reasonable choice is $\Lambda = \Lambda^{\mathrm{HC}}_{s}$, where
\ben{
\label{HCindex}
\Lambda^{\textnormal{HC}}_{s} = \left\{ \bmnu = (\nu_1,\ldots,\nu_d) \in \N^d_0: \prod_{j=1}^d (\nu_j + 1) \le s+1 \right\},
}
is the \textit{hyperbolic cross} index set of degree $s$.

Having chosen $\Lambda$, the finite vector $\bm{c}_{\Lambda} = (c_{\bm{\nu}})_{\bm{\nu} \in \Lambda}$ can now be assumed to be approximately sparse.  Next, one formulates the normalized measurement matrix and vector of measurements
\begin{equation}
\label{eq:A_f_def}
\bA = \left( \frac{\Psi_\nu(\bm{x}_i)}{\sqrt{m}} \right)_{\substack{1\leq i\leq m\\ \bmnu\in\cJ}}, \qquad {\bm f} = \left( \frac{f(\bm{x}_i)}{\sqrt{m}} \right)_{1\le i \le m}.
\end{equation}
Then one searches for an approximately sparse solution of the linear system $\bm{A} \bm{z} = \bm{f}$.  A standard means to do this is to solve the \textit{quadratically-constrained basis pursuit} problem
\ben{
\label{eq:CS_BPDN}
\textnormal{minimize}_{\bz\in\R^N} \| \bz\|_{1} \;\;\; \textnormal{s.t.}\;\;\; \| \bA \bz - \bm{f} \|_2 \le \eta,
}
for suitably chosen $\eta \geq 0$, or the \textit{unconstrained LASSO} problem
\begin{align}
\label{eq:CS_uncon_min}
\textnormal{minimize}_{\bz\in\R^n} \|\bz\|_{1} + \mu \|\bA \bz - \bm{f} \|_2^2,
\end{align}
for appropriately chosen $\mu>0$.  
A solution $\hat{\bm{c}} = (\hat{c}_{\bm{\nu}})_{\bm{\nu} \in \Lambda}$ of either problem yields an approximation $\hat{f} = \sum_{\bm{\nu} \in \Lambda} \hat{c}_{\bm{\nu}} \Psi_{\bm{\nu}}$ of $f$.

Unfortunately, simply promoting the sparsity of the polynomial coefficients via the $\ell^1$-norm is not sufficient to achieve favourable sample complexity bounds. Bounds on $m$ for $\ell^1$-norm based approaches can be exponential in the dimension $d$ \cite{Adcock2016}.  Fortunately, as considered in \cite{RW15,AdcockBen2018ICSa,ChkifaDexterTranWebster18}, this issue can be overcome by replacing the $\ell^1$-norm with a certain weighted $\ell^1$-norm.  For instance, instead of \eqref{eq:CS_BPDN} one now solves
\begin{align}
\label{eq:CS_BPDN_weighted}
\textnormal{minimize}_{\bz\in\R^N} \| \bz\|_{1,\bm{u}} \;\;\; \textnormal{s.t.}\;\;\; \| \bA \bz - \bm{f} \|_2 \le \eta.
\end{align}
Here $\bm{u} = (u_{\bm{\nu}})_{\bm{\nu} \in \Lambda}$ is a vector of weights and $\nm{\bz}_{1,\bm{u}} = \sum_{\bm{\nu} \in \Lambda} u_{\bm{\nu}} |z_{\bm{\nu}} |$.  As shown in \cite{AdcockBen2018ICSa}, an appropriate choice of weights is
\ben{
\label{eq:uweightsLeg}
u_{\bm{\nu}} = \nm{\Psi_{\bm{\nu}}}_{L^{\infty}} = \prod^{d}_{j=1} \sqrt{2 \nu_j + 1},\qquad \bm{\nu} = (\nu_1,\ldots,\nu_d).
}
For the remainder of this paper, we consider the CS polynomial approximation scheme $\hat{f} = \sum_{\bm{\nu} \in \Lambda} \hat{c}_{\bm{\nu}} \Psi_{\bm{\nu}}$, where $\hat{\bm{c}}$ is either a solution of \eqref{eq:CS_BPDN_weighted}, or for the sake of comparison, \eqref{eq:CS_BPDN}.

% !TEX root = ./MLFA.tex

%----------------------------------------------------
\section{Testing setup}
\label{sec:testingsetup}
%----------------------------------------------------

We now describe the testing setup for our experiments. Training DNNs requires careful choices of the optimization solver, initialization and optimization parameters. This section describes the choices we made to deliver consistent performance across a range of DNN architectures. In summary, we find the \texttt{Adam} ({\bf ada}ptive {\bf m}oments) optimizer with exponentially-decaying learning rate and a specific random initialization to deliver this performance, whereas other solvers such as \texttt{SGD} and other learning rate schedules perform less consistently.

%----------------------------------------------------
\subsection{Setup}
\label{sec:testing}
%----------------------------------------------------

We first summarize the main methodology:

{\bf (i) Implementation.}
Our framework has been implemented in a package called {\tt MLFA} ({\bf M}achine {\bf L}earning {\bf F}unction {\bf A}pproximation) in version 1.13 of Google's {\tt TensorFlow} software library \url{https://www.tensorflow.org/}, and is available on GitHub at \url{https://github.com/ndexter/MLFA}. 
Details about the set of features supported by {\tt MLFA} and data recorded by the code can also be found on the GitHub page.

{\bf (ii) Hardware.}
In the course of testing our DNN models, we observed improved accuracy on some of our test problems by initializing and training the networks in double precision.
Modern GPUs often support half, single, and double precision arithmetic, though many commonly-available GPUs are optimized to perform single precision computations much more quickly (see \S \ref{sec:singledoubleSM}). 
The majority of our computations were performed in single precision using the Tesla P100 GPUs on Compute Canada's Cedar compute cluster at Simon Fraser University\footnote{See \url{https://www.computecanada.ca/} and \url{https://docs.computecanada.ca/wiki/Cedar}}, though for some of our test problems we provide double precision results for a subset of the architectures considered for comparison.

{\bf (iii) Choice of architectures and initialization.} We consider fully-connected ReLU networks. This choice is inspired by the many theoretical existence results on such networks (see \S \ref{sec:introduction}).
There is a vast literature suggesting various strategies for designing architectures and initializing neural networks. For an introduction to these topics see, e.g., \cite[\S 5.2 \& 8.4]{DeepLearningBook}.
We recall the work \cite{Hanin2018} focusing on which DNN architectures result in exploding and vanishing gradients, a common problem in backpropagation which can result in failure during training.
Our empirical results confirm that choosing DNN architectures with a fixed number of nodes per layer $N$ and depth $L$ such that the ratio $\beta := L/N$ is small, e.g., $\beta \in (0.025,0.5)$, is an effective choice for training.
In \S \ref{subsec:initialization}, we study several popular strategies for initializing DNNs.
 Our results show that initializing the weights and biases to be normal random variables with mean 0 and variance 0.01 is an effective choice for many of the architectures and problems considered herein.
 We note that for the range of architectures studied, this choice results in weights and biases with variance smaller than many other popular choices, e.g., the strategies from \cite{Glorot2010,He2015}, and is sufficiently small to avoid the failure modes analyzed in \cite{HaninRolnick2018}.
 We also note that setting the variances smaller than 0.01 can result in networks which are more difficult to train, which can be attributed to the exponentially decaying length scales described in \cite[Theorem 1]{HaninRolnick2018}.
The seed used for the random number generators can also have an important effect on the initialization.
In this study, we initialize all of our networks from the same seed 0 for both {\tt TensorFlow} and {\tt NumPy}, in order to reduce the complexity of our experiments. 
A more comprehensive study of the average performance of DL would require averaging over a large set of seeds used in initialization. We leave such a study to a future work.

{\bf (iv) Optimizers for training and parameterization.}
\S \ref{subsec:convergence} compares the performance of a variety of solvers and learning rate schedules on the function approximation problem.
In practice, we find the {\tt Adam} optimizer \cite{kingma2014adam} with an exponentially decaying learning rate yields the most accurate results of the solvers tested in the least amount of training time. See \S \ref{sec:AdamSM} for implementation details. In single precision, the DNNs are trained for 50,000 epochs or to a tolerance of $\varepsilon_{\mathrm{tol}} = 5\times 10^{-7}$, while in double precision DNNs are trained for 200,000 epochs or to a tolerance of $\varepsilon_{\mathrm{tol}} = 5\times 10^{-16}$.
Due to the non-monotonic convergence of minimizing the non-convex loss \eqref{eq:ER_loss} with respect to the weights and biases, we checkpoint our partially trained networks once the training loss has been reduced to 1/8th of the previous checkpoint's training loss, saving the best result at the end of training. We then average our testing results over the final trained networks.

{\bf (v) Training data and design of experiments.}
To understand the average performance of DL on a variety of reconstruction tasks, we run 20 trials of solving \eqref{eq:approx_problem} with each of our DNN architectures and CS over a range of data sets of increasing size. We then average the testing error and run statistics over all trials for each data set in plotting.
We define a trial as one complete run of training a DNN, initialized as above, or solving a CS problem on a set of training data consisting of the values $\{(\bx_i,f(\bx_i))\}_{i=1}^{m_k}$.
To generate each data set of size $m_k$, with $0 < m_1 < m_2 < \cdots < m_{k_\mathrm{final}}$, we sample 20 i.i.d. sets of points $\{\bx_i\}_{i=1}^{m_k}$ from the uniform distribution on $(-1,1)^d$ and evaluate our target function $f$ at these points to form our training data.
Since we are interested in the sample complexity of the DL problem, for most of our examples we choose $m_{k_\mathrm{final}}$ to be relatively small, on the order of 1000 points.

{\bf (vi) Testing data and error metric.}
To study the generalization capabilities of DL, we use a common error metric for numerical analysis, the relative $L^2$ error
\begin{align}
\label{eq:rel_err}
\epsrel = \|f - \tilde{f}\|_{L^2}/\|f\|_{L^2},
\end{align}
where $\tilde{f}$ is an approximation obtained using either DL or CS. We purposefully choose this over other norms (e.g.\ the $L^{\infty}$- or $H^1$-norms) so that we can use the same error metric across all function classes, including both smooth and nonsmooth functions.
As we run multiple trials of our experiments, we compute the average of \eqref{eq:rel_err} over all of our trials in testing.
In contrast to the training data, we use deterministically generated points and function data for testing the relative $L^2$ errors. 
 More specifically, we compute an approximation to the $L^2$ integrals using a high order isotropic Clenshaw Curtis sparse grid quadrature \cite{Gerstner1998} rule, see, e.g., \cite{NTW08} for more details.
As opposed to the training data, we use a large set of testing points, e.g., ($d=1$) 65,537, ($d=2$) 311,297, ($d=4$) 643,073, and ($d=8$) 1,863,937 points, to ensure a good covering of our parameter space $\cU$ in computing these statistics.
For such moderate dimensional problems an isotropic rule is sufficient to study the generalization performance of both CS and DL.
For higher-dimensional instances of \eqref{eq:approx_problem}, the points and weights can be pre-computed and re-used in testing  multiple functions. We rely on the {\tt TASMANIAN} sparse grid toolkit \cite{stoyanov2015tasmanian,stoyanov2018adaptive,stoyanov2016dynamically} for the generation of these rules.

{\bf (vii) Compressed sensing.}
We solve the problems \eqref{eq:CS_BPDN} and \eqref{eq:CS_BPDN_weighted} (with weights \eqref{eq:uweightsLeg}) using the {\tt MATLAB} solver {\tt SPGL1} \cite{spgl1:2007,BergFriedlander:2008} in double precision.
The parameter $\eta$ is chosen as
\ben{
\label{eta_opt_choice}
\eta = \nm{\bm{A} \bm{c}_{\Lambda} - \bm{f}}_{2},\qquad \bm{c}_{\Lambda} = (c_{\bm{\nu}})_{\bm{\nu} \in \Lambda},
}
To compute $\bm{c}_{\Lambda}$ we use the same sparse grid rule described above in evaluating each coefficient $c_{\bm\nu} = \int_{\cU} f(\bm{x}) \Psi_{\bm{\nu}}(\bm{x}) \D \varrho(\bm{x})$.
See \S \ref{sec:CSSM} for further discussion.
We set $\Lambda$ as in \eqref{HCindex} to be the hyperbolic cross index set of degree $s$ chosen so that $\#(\Lambda^{\textrm{HC}}_s) \approx $ 3,000. 

%----------------------------------------------------
\subsection{Solvers}
\label{subsec:solvers}
%----------------------------------------------------

The choice of solver and parameterization of the solver are important factors in training DNN models to a desired tolerance in a reasonable amount of time. A common choice in many computational science applications is the {\tt Adam} optimizer, a variant of {\tt SGD} ({\bf s}tochastic {\bf g}radient {\bf d}escent) incorporating moment estimates of the gradient.
In the course of testing our implementation of the {\tt MLFA} package, we studied the effect of the solvers on the generalization error and time to train given a fixed budget of 50,000 epochs in single precision. Fig.\ \ref{fig:optimizer_comp} displays results for a variety of solvers and learning rate schedules.
There we observe comparable performance for the {\tt Adam} and {\tt RMSProp} algorithms in terms of accuracy, with {\tt Adagrad}, {\tt SGD}, and {\tt PGD} ({\bf p}roximal {\bf g}radient {\bf d}escent) performing the worst in both accuracy and computational cost. When comparing run times and accuracy for all methods tested, the {\tt Adam} optimizer achieves the best accuracy with the least computational cost.
We also include results obtained with the {\tt AdamW} optimizer \cite{Loshchilov2019}, which implements a decoupled weight decay (a form of $\ell^2$-regularization)
on the weights and biases. In Fig.\ \ref{fig:optimizer_comp} we observe that smaller values of the weight decay parameter $\lambda$ allow the {\tt AdamW} optimizer to achieve identical performance to {\tt Adam} with minimal overhead, but do not outperform standard {\tt Adam}, while larger values of $\lambda$ both decrease accuracy and increase run time.

\begin{figure}[ht]
\begin{center}
\includegraphics[width=0.23\paperwidth,clip=true,trim=0mm 0mm 0mm 0mm]{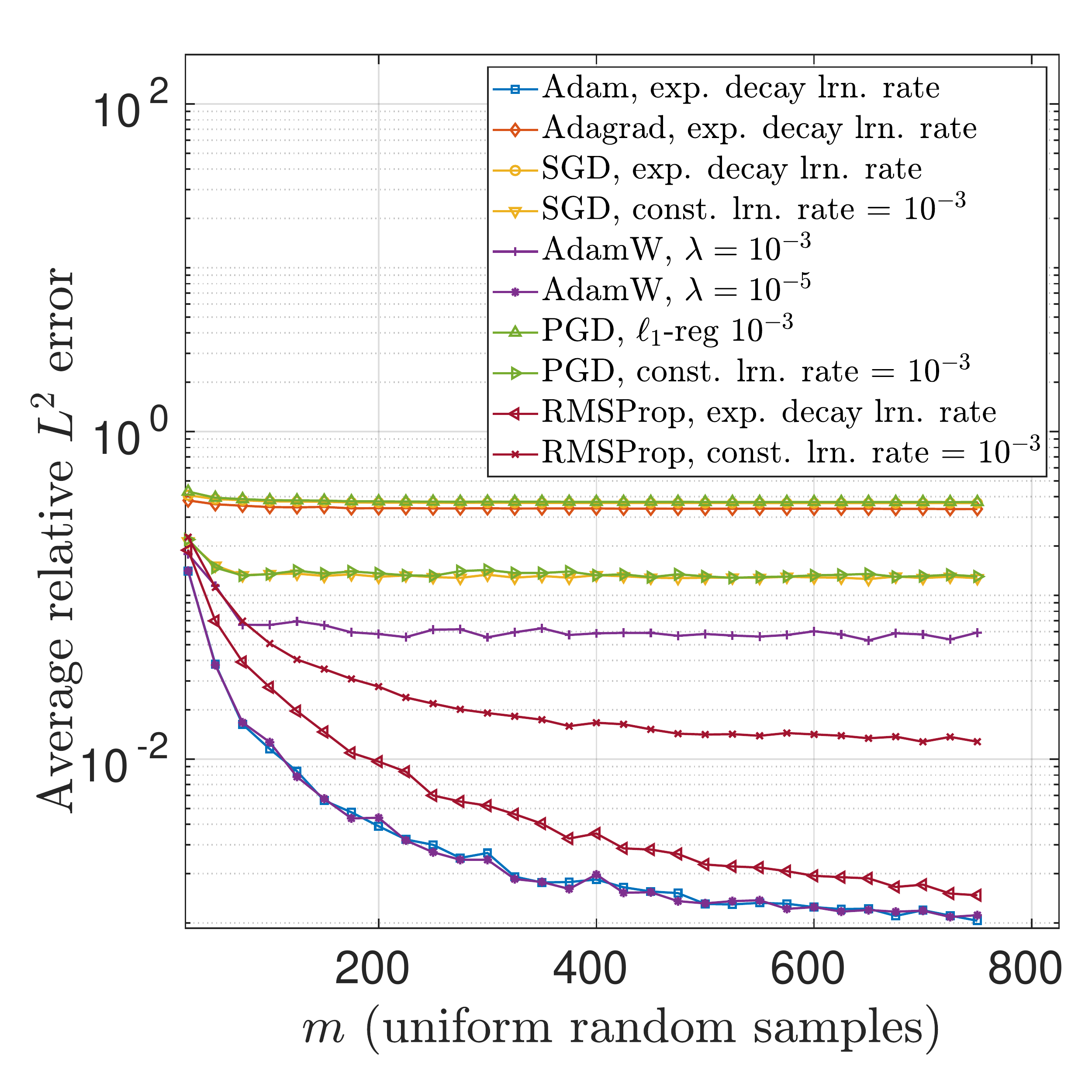}
\includegraphics[width=0.23\paperwidth,clip=true,trim=0mm 0mm 0mm 0mm]{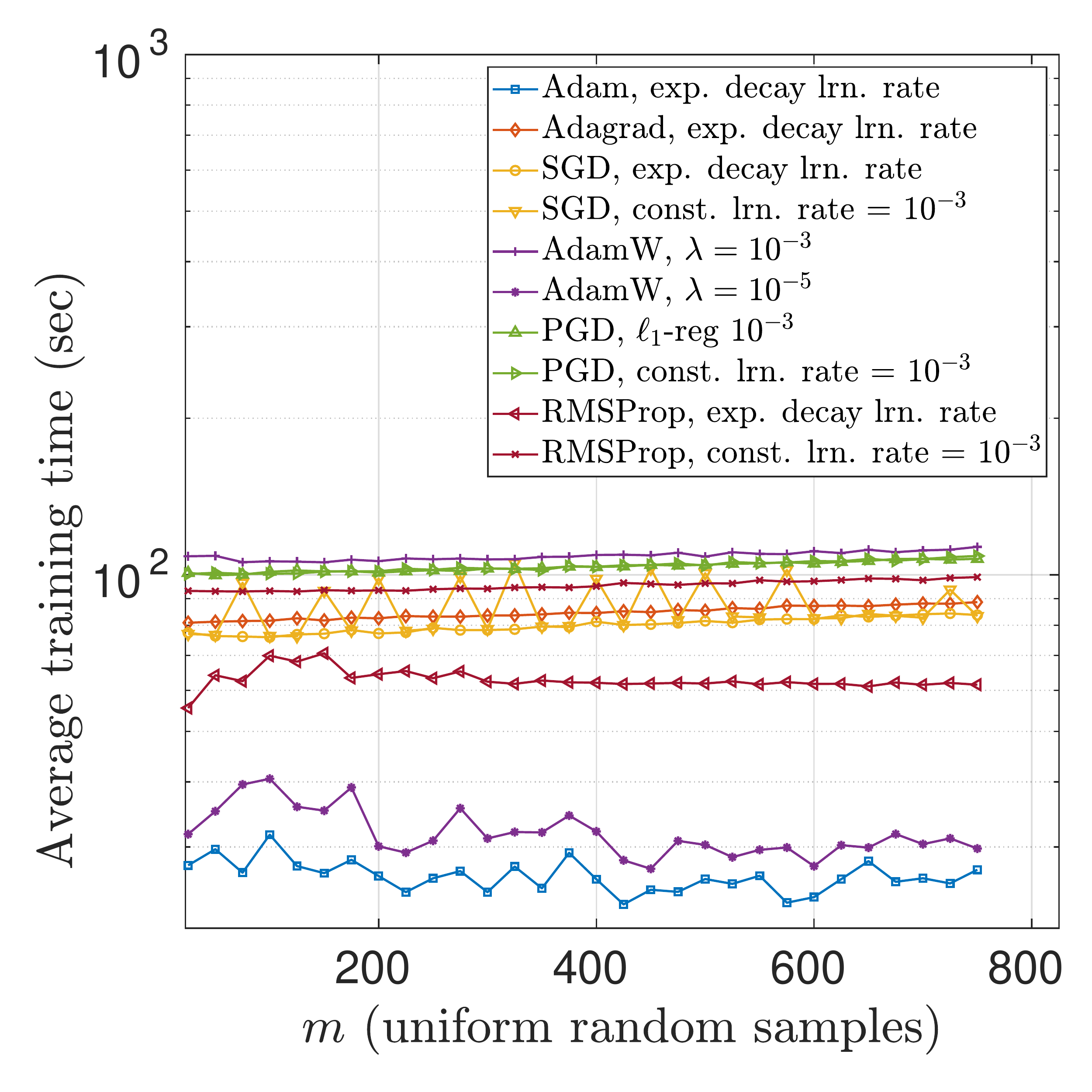}
\end{center}

\vspace{-2mm}
\caption{Comparison of {\bf(left)} average relative $L^2$ error and {\bf(right)} average time in training a set of 20 ReLU $10\times 100$ DNNs to approximate $f(x) = \log(\sin(10x) + 2) + \sin(x)$.}
\label{fig:optimizer_comp}
\end{figure}

For certain solvers, it is often the case that some of the trials of a given test may fail to achieve the desired loss tolerance before arriving at the final epoch. 
An example of this can be seen in the left plot of Fig.\ \ref{fig:typical_convergence} where none of the trials of {\tt SGD} with a constant learning rate of $10^{-3}$ were able to achieve the desired tolerance in 50,000 epochs of training.
The middle plot of Fig.\ \ref{fig:typical_convergence} displays the effect of using an exponentially decaying learning rate with {\tt SGD}, though we also observe there that none of trials trained with this learning rate schedule were able to achieve the tolerance in 50,000 epochs of training. 
On the other hand the right plot shows that, on the same problem, all 20 trials trained with the {\tt Adam} optimizer with the exponentially decaying learning rate were able to converge to the $5\times 10^{-7}$ loss tolerance in under 14,000 epochs of training.
We also compared learning rate schedules for the {\tt Adam} optimizer, but found the exponentially decaying learning rate to give consistently good results. See \S \ref{sec:AdamLearnRateSM}.

\begin{figure}[ht]
\begin{center}
\includegraphics[width=0.23\paperwidth,clip=true,trim=0mm 0mm 40mm 10mm]{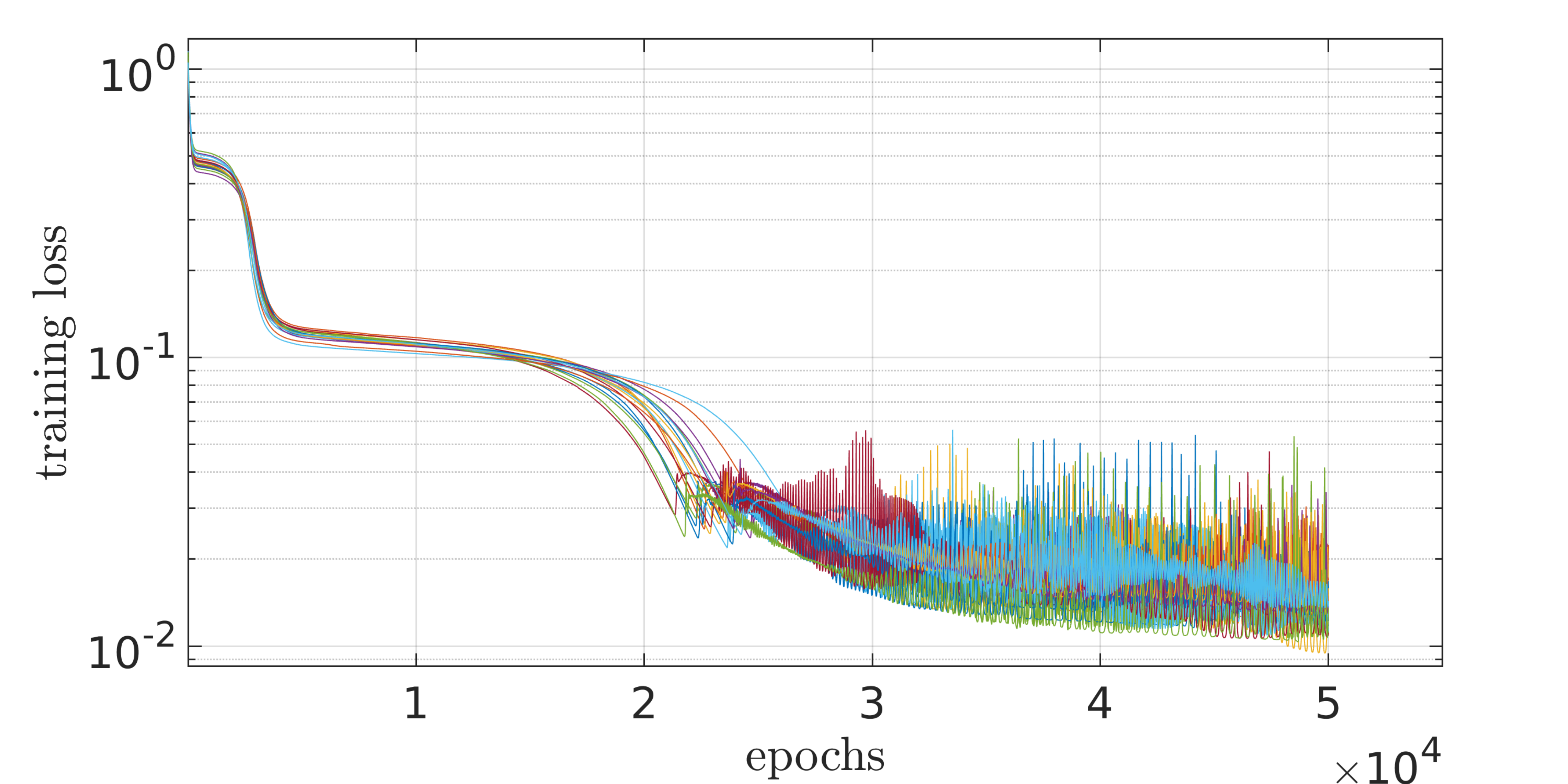}
\includegraphics[width=0.23\paperwidth,clip=true,trim=0mm 0mm 40mm 10mm]{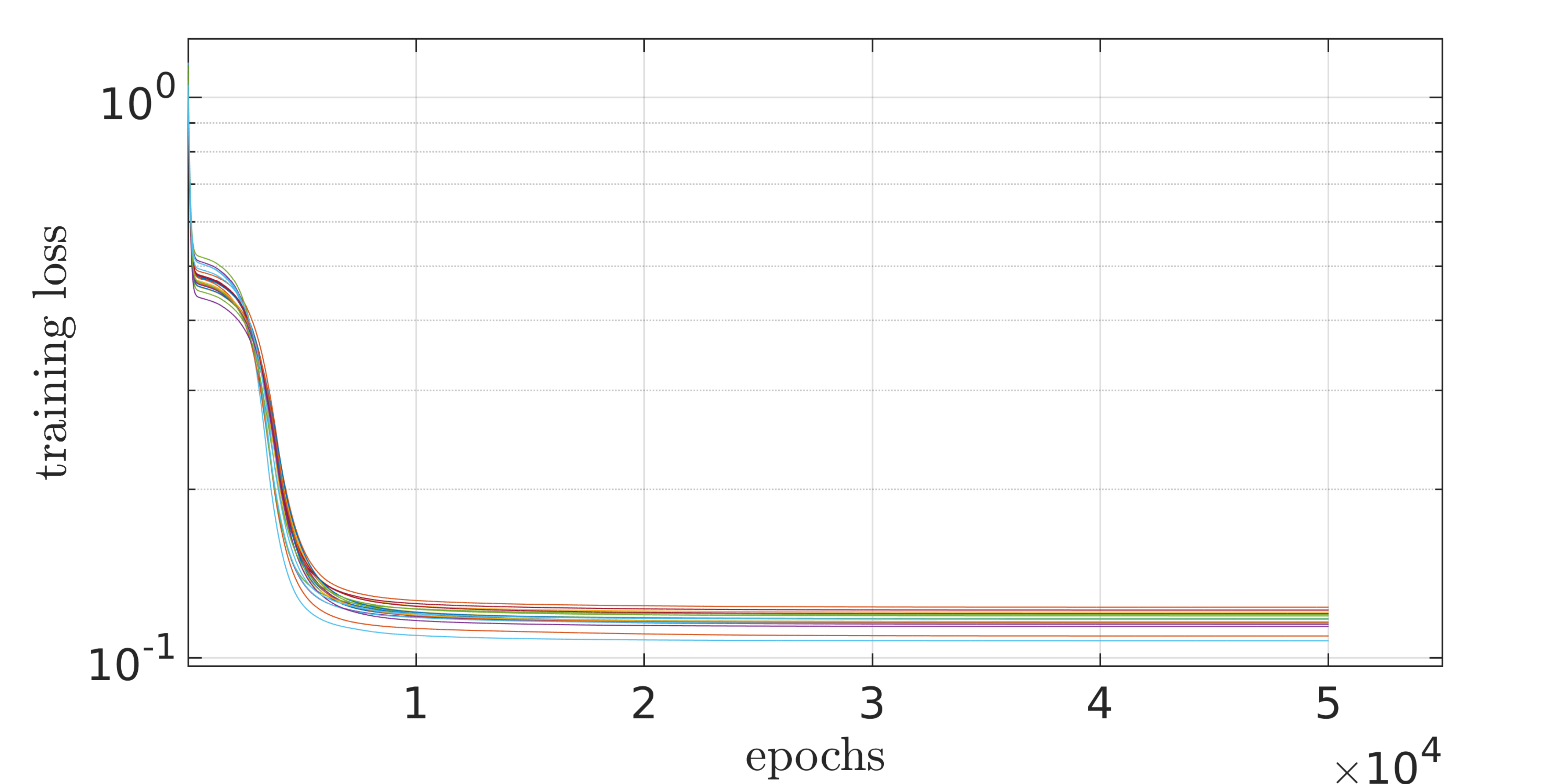}
\includegraphics[width=0.23\paperwidth,clip=true,trim=0mm 0mm 40mm 10mm]{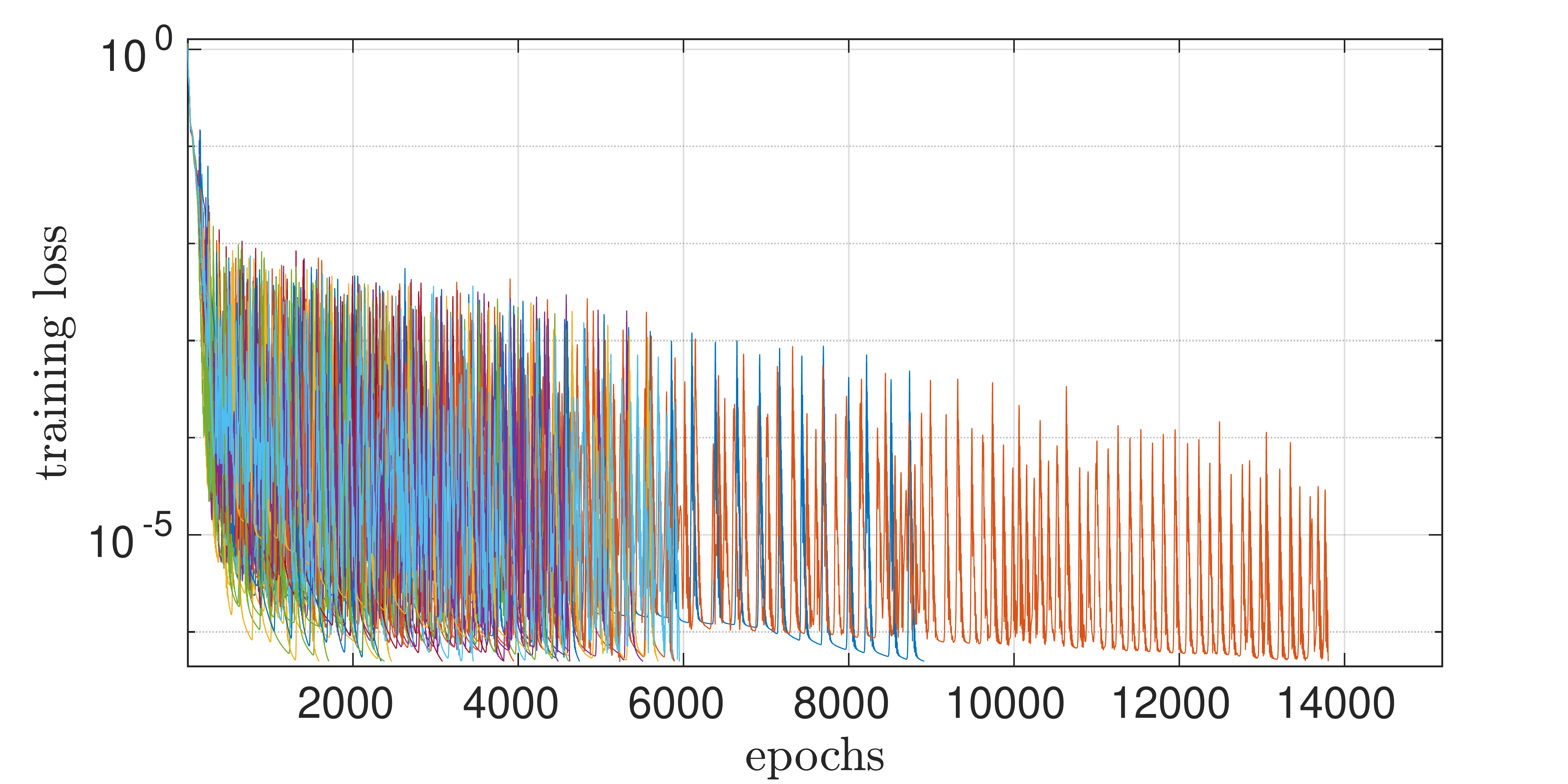}
\end{center}

\vspace{-2mm}
\caption{Examples of typical convergence in training a set of 20 ReLU $10\times 100$ DNNs on $m=750$ data points of $f(x) = \log(\sin(10x)+2)+\sin(x)$ with {\bf(left)} {\tt SGD} with a constant learning rate $\tau = 10^{-3}$ {\bf(middle)} {\tt SGD} with an exponentially decaying learning rate and {\bf(right)} {\tt Adam} with an exponentially decaying learning rate.}
\label{fig:typical_convergence}
\end{figure}

Previous studies on training DNNs suggest that the batch size can affect the convergence of the algorithms. Due to the small data set sizes considered herein over those in, e.g.\ computer vision, we have found that setting the batch size too small can result in longer training times (due to the increased transfer time between the CPU and GPU) and only marginal performance benefit. See \S \ref{sec:batchesSM}. We leave a more detailed study of the affects of batch size on the performance of trained DNNs in higher-dimensional problems to a future work.

%----------------------------------------------------
\subsection{Initialization and Precision}
\label{subsec:initialization}
%----------------------------------------------------

The strategy used to initialize the 
weights and biases of a DNN can have a large effect on the training process, with some initializations resulting in failure to train at all.
To explore the impact of this choice, we tested three different initialization strategies, all of which based on symmetric uniform or normal distributions with small variance. Empirically we find small variance to be an important factor in the success or failure of training.
The strategies tested are (i) normal with mean 0 and variance 0.01, (ii) normal with mean 0 and variance $2/N$, and (iii) uniform on $(-2/N,2/N)$.

Since we use constant layer widths in this work, i.e., $N_\ell = N$ for all hidden layers $\ell=1,\ldots,L+1$, we can compare these strategies to the popular Xavier and He initializations, see [45], as follows. The Xavier initialization sets the weights as mean zero normal random variables with variance $1/N_\ell$ where $N_\ell$ is the number of nodes on layer $\ell$, while the He initialization modifies the variance to $2/N_\ell$. 
Therefore strategy (ii) is similar to the He initialization, with the exception of the input and output layers for which strategy (ii) prescribes much smaller variance for problems with input and output dimension in the ranges considered in this work.
Also, for values of $N_\ell$ less than 100, the Xavier initialization results in variances larger than the variance of 0.01 used in strategy (i), and similarly for values of $N_\ell$ less than 200 for the He initialization.

Fig.\ 3 presents the results of these different initialization strategies for training three different architectures with {\tt Adam} on the piecewise continuous function approximation problem \eqref{eq:halfspace_func}.
There we see that strategy (i) performs near the best for the smaller ReLU $2\times 20$ and $3\times 30$ architectures, and about as well for the $5\times 50$ architecture as strategy (ii).

For the majority of our results, we focused on architectures with a small ratio $\beta = L/N$ (with $L$ the number of layers and $N$ the number of nodes per layer).
However, we also note that setting $\beta$ too small can result in DNNs which exhibit numerical instabilities after training, see Fig.\ \ref{fig:unstable_network}, despite achieving the training loss tolerance. 
We also remark that this can occur for networks with relatively small weights, see the right plot of Fig.\ \ref{fig:relu_exp_cos_func_weight_comp} which shows the weights and biases for the $20\times 800$ network remain on average bounded by $2$ as we increase the samples.

We also observed that for some problems, initializing and training our DNN models in double precision can lead to improved accuracy in testing the generalization performance, and that the improvement is more pronounced for deeper networks, see Fig.\ \ref{fig:precision_comparison}.
Due to the massive increase in computation time associated with training in double precision, a result of needing to train for more epochs to achieve the tolerance $5\times 10^{-16}$ and the increased time of double precision arithmetic, see item (ii) on Hardware in \S \ref{sec:testing}, we limit our comparisons of single vs. double precision results to a handful of problems and architectures.

\begin{figure}[ht]
\begin{center}
\includegraphics[width=0.23\paperwidth,clip=true,trim=2mm 0mm 0mm 0mm]{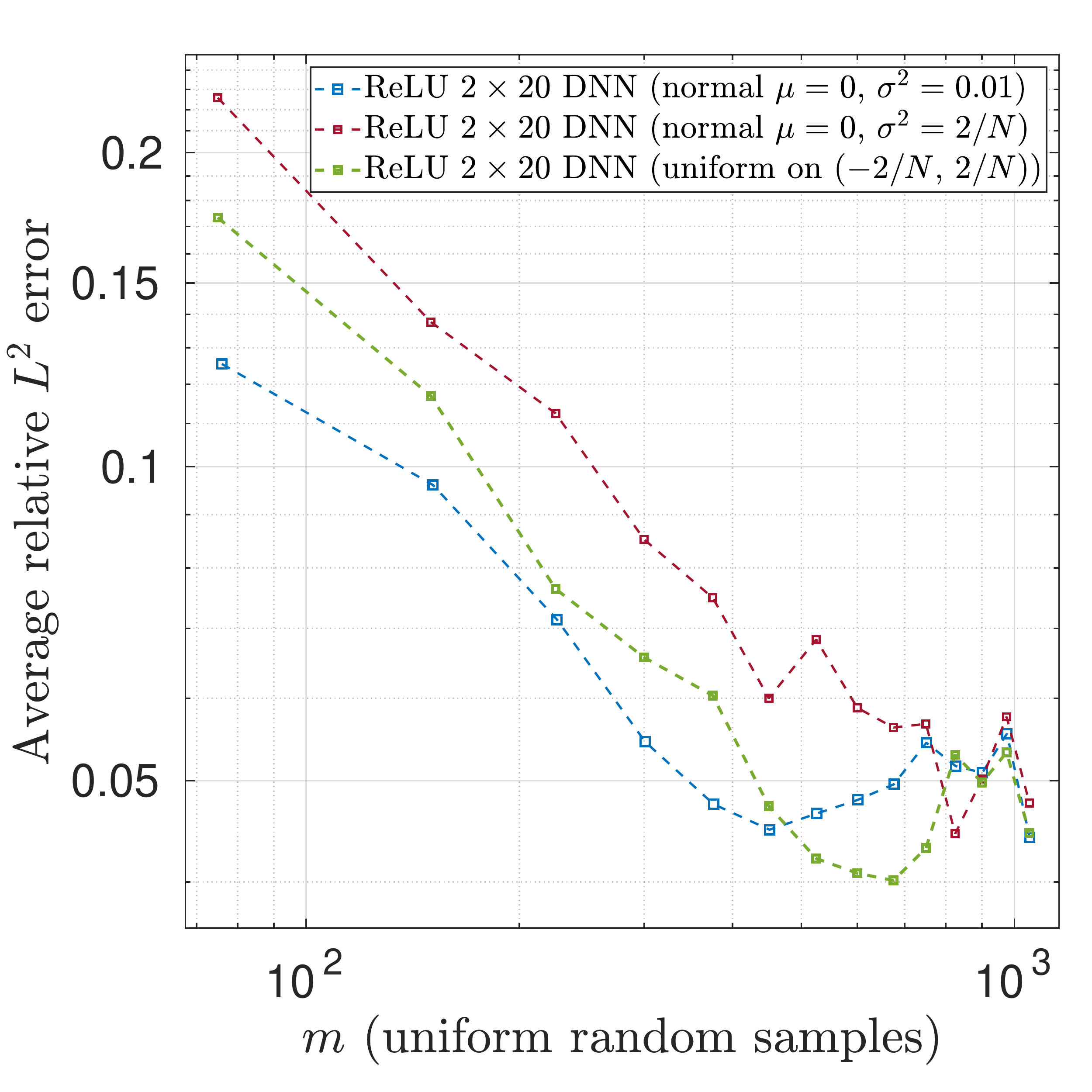}
\includegraphics[width=0.23\paperwidth,clip=true,trim=2mm 0mm 0mm 0mm]{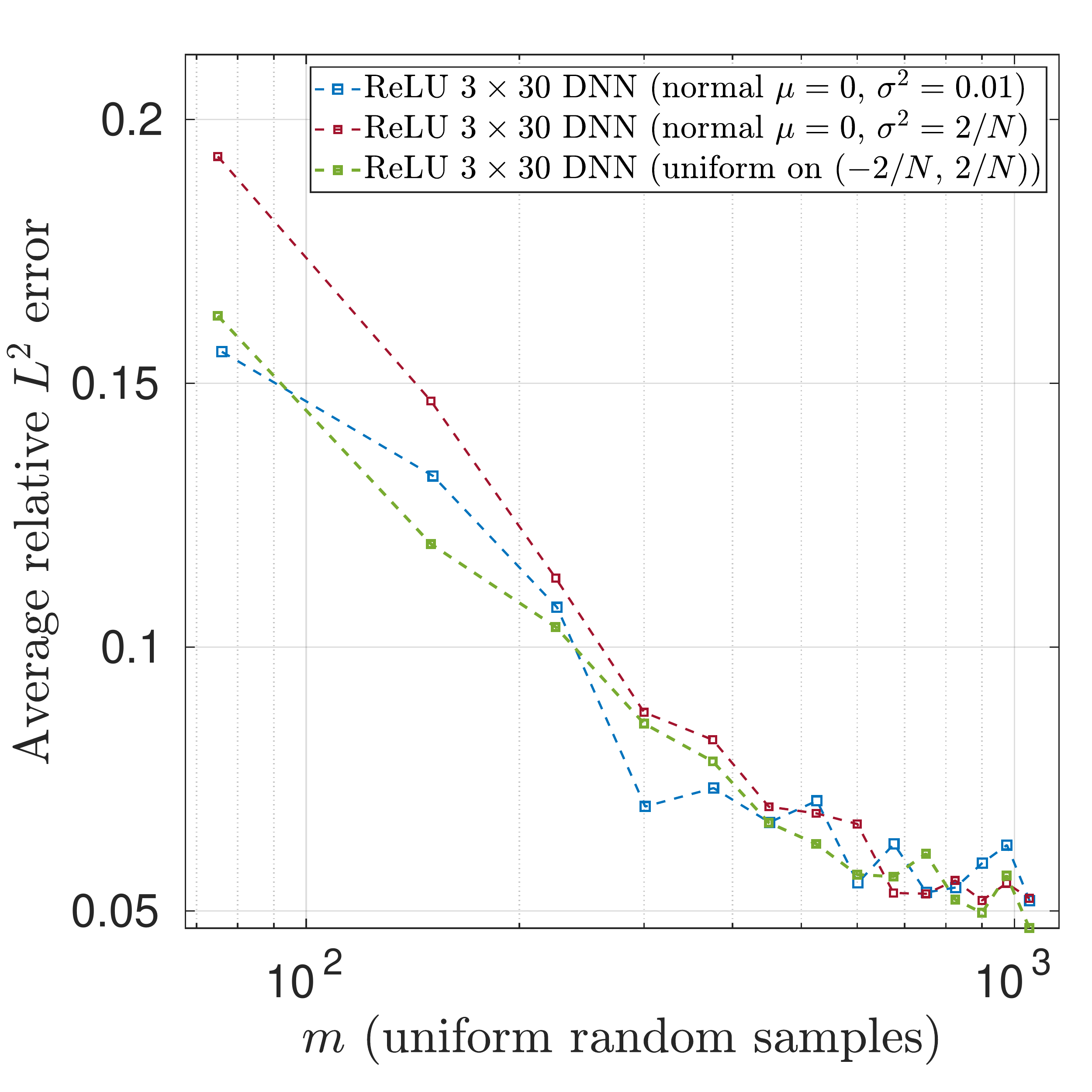}
\includegraphics[width=0.23\paperwidth,clip=true,trim=2mm 0mm 0mm 0mm]{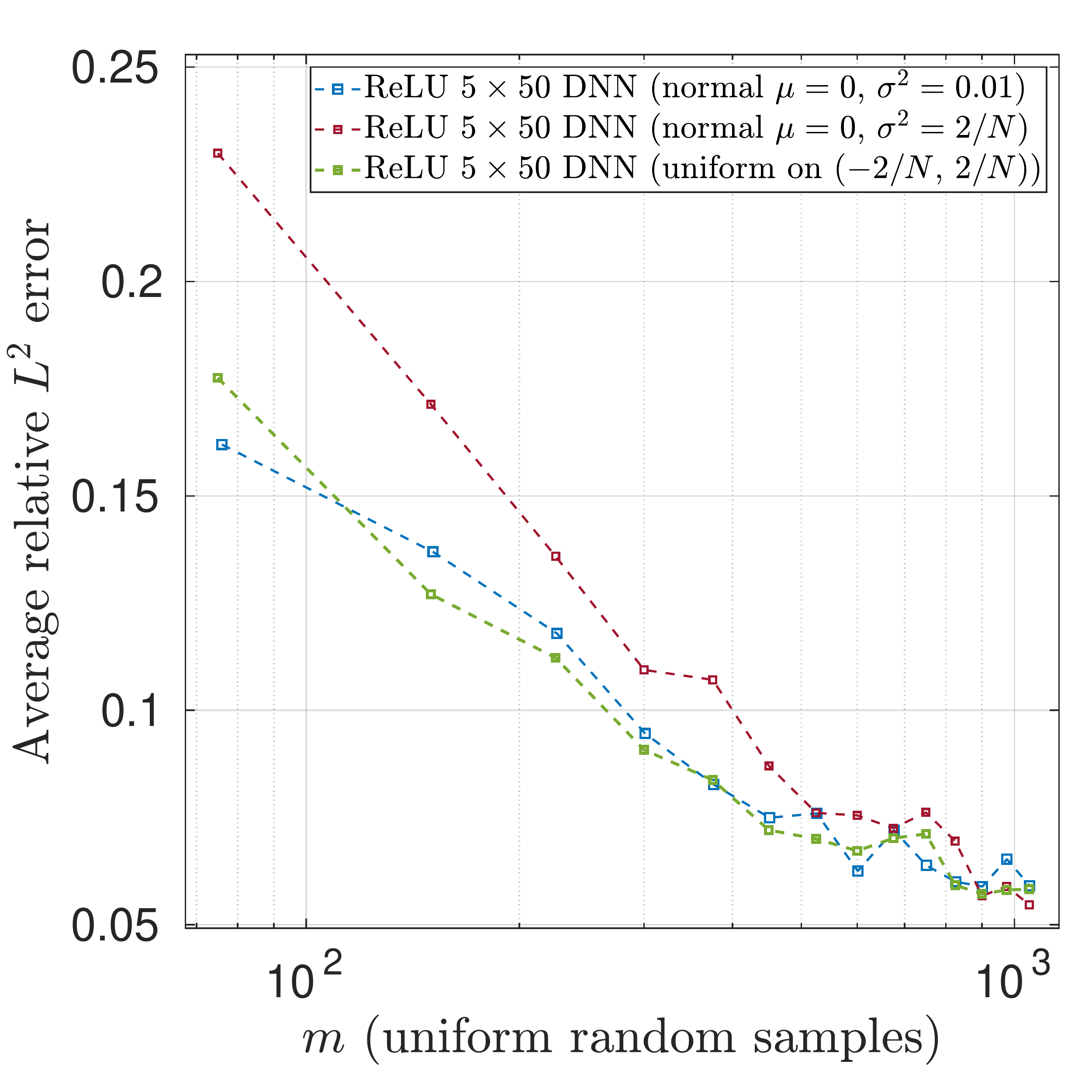}
\end{center}

\vspace{-2mm}
\caption{Comparison of strategies for initializing ReLU DNNs with $L$ hidden layers and $N$ nodes per hidden layer with {\bf(left)} $L=2$ and $N=20$, {\bf(middle)} $L=3$ and $N=30$, and {\bf(right)} $L=5$ and $N=50$. Results were obtained by training with the {\tt Adam} optimizer and an exponentially decaying learning rate schedule on function data generated from $f(x)$ from \eqref{eq:halfspace_func} with $d = 2$.}
\label{fig:init_comparison}
\end{figure}

\begin{figure}[ht]
\begin{center}
\includegraphics[width=0.17\paperwidth,clip=true,trim=0mm 0mm 0mm 0mm]{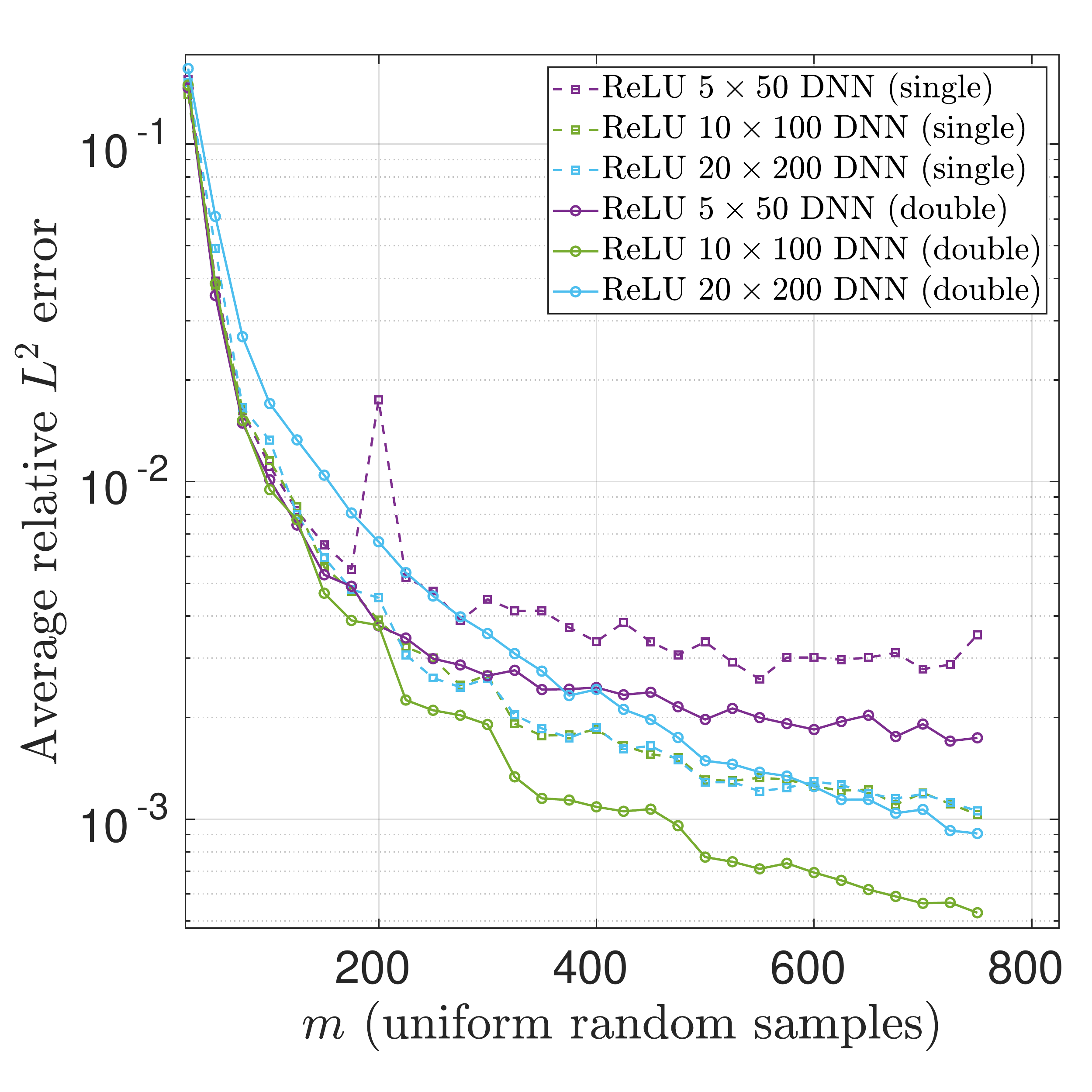} 
\includegraphics[width=0.17\paperwidth,clip=true,trim=0mm 0mm 0mm 0mm]{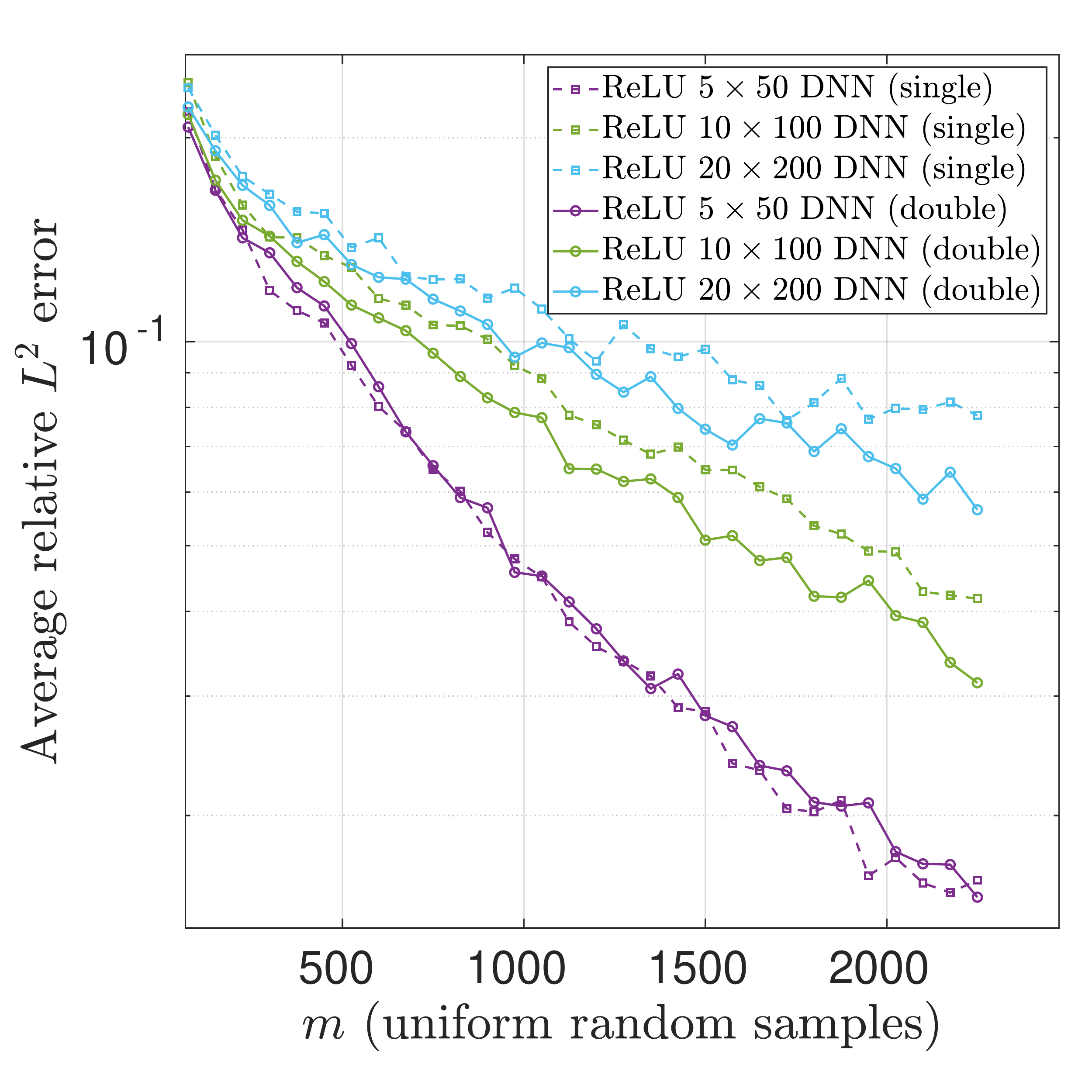}
\includegraphics[width=0.17\paperwidth,clip=true,trim=0mm 0mm 0mm 0mm]{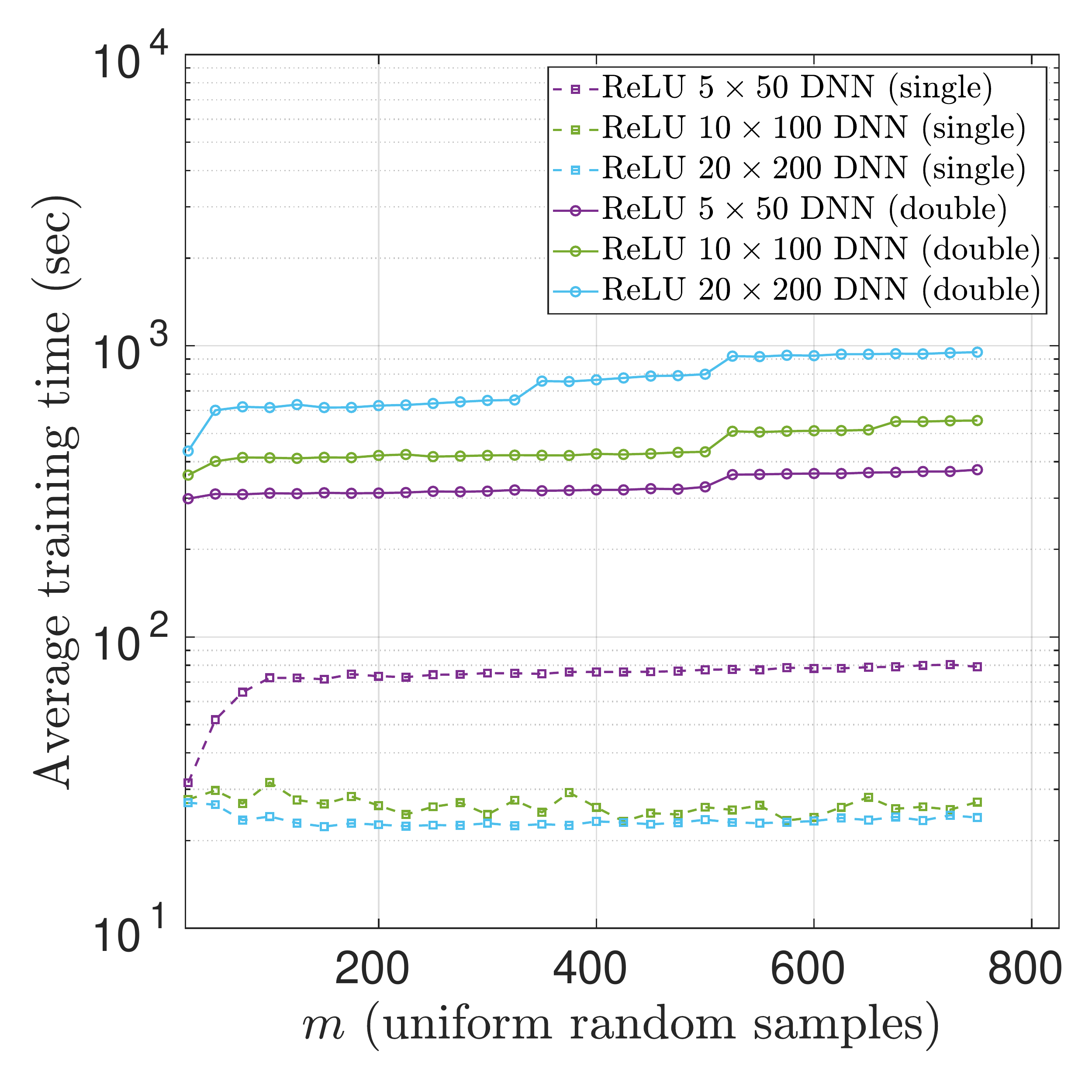}
\includegraphics[width=0.17\paperwidth,clip=true,trim=0mm 0mm 0mm 0mm]{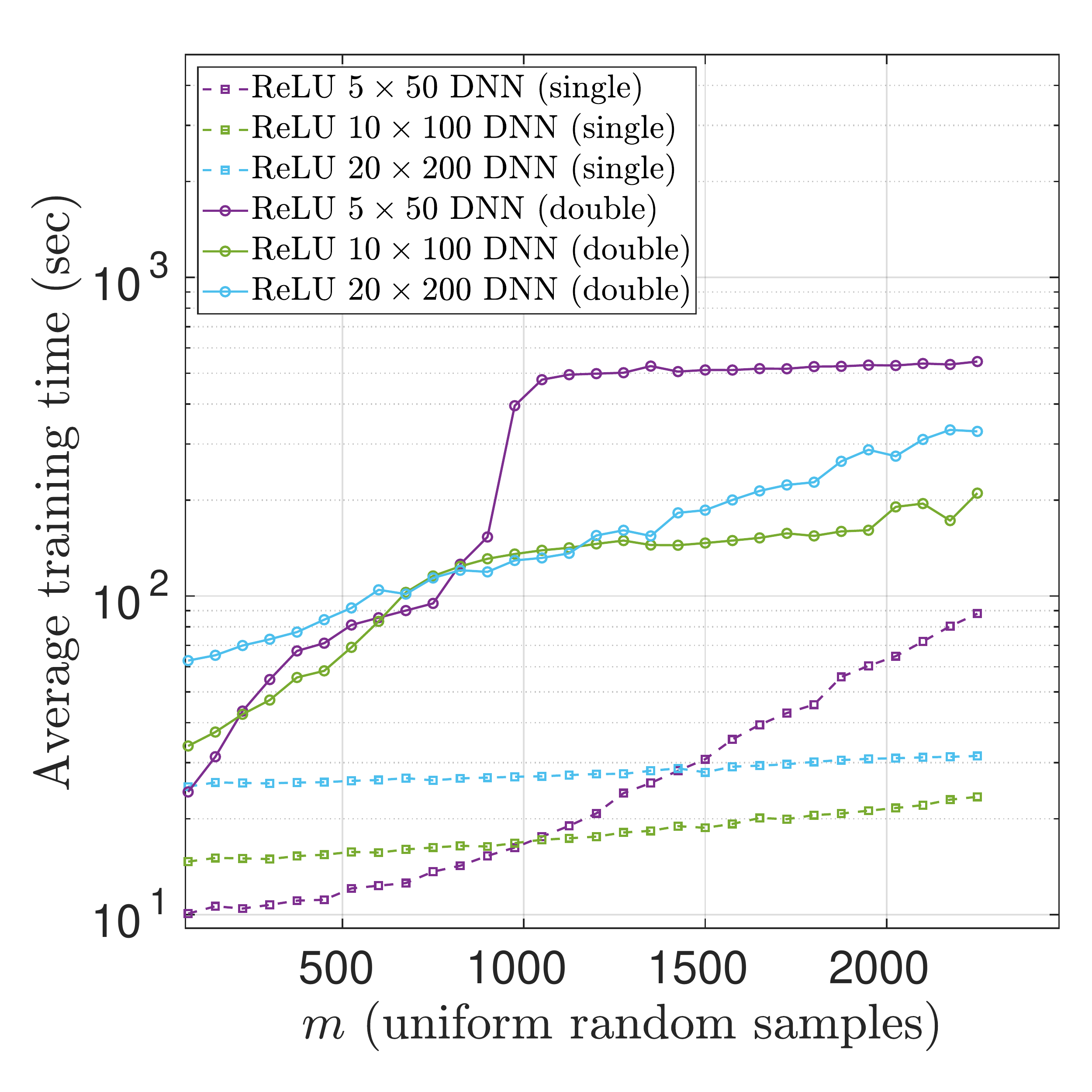}\\[-0.2cm]
\begin{tabular}{c c c c} 
\;\;\;\;{\footnotesize\bf(a)} & \,\,\;\quad\qquad\qquad\qquad {\footnotesize\bf(b)} & \,\,\;\quad\qquad\qquad\qquad {\footnotesize\bf(c)} & \,\;\quad\qquad\qquad\qquad {\footnotesize\bf(d)}
\end{tabular}
\end{center}

\vspace{-2mm}
\caption{Comparison of {\bf(a)} \& {\bf(b)} average relative $L^2$ error and {\bf(c)} \& {\bf(d)} average training time for DNNs initialized and trained in single vs. double precision on a {\bf(a)} \& {\bf(c)} one dimensional smooth function $f(x) = \log(\sin(10x)+2) + \sin(x)$ and {\bf(b)} \& {\bf(d)} eight-dimensional smooth function $f(x)$ from \eqref{eq:slower_decay_rational_func}.}
\label{fig:precision_comparison}
\end{figure}

%----------------------------------------------------
\subsection{Convergence of {\tt Adam} on function approximation problems}
\label{subsec:convergence}
%----------------------------------------------------

We now discuss the convergence of the {\tt Adam} optimizer. 
Our stopping criteria for testing the convergence depends only on the $\ell_2$ loss with respect to the training data and the current number of epochs of training. 
During the course of training, the process outlined above may result in networks which exhibit numerical instabilities or provide poor  performance, see Fig.\ \ref{fig:unstable_network}. 
In \S \ref{subsec:solvers}, we noted that some trials may not achieve the desired accuracy tolerance within our budget of 50,000 epochs in single precision and 200,000 epochs in double precision. This can also occur for some of the trials with the {\tt Adam} optimizer, see Fig.\ \ref{fig:nonmonotone_Adam_convergence}. 
Despite these issues, when viewing the average performance of the {\tt Adam} optimizer on such problems we often find that the majority of networks are numerically stable and approximate the function well, see Fig.\ \ref{fig:boxplot} which displays boxplots of the average relative $L^2$ error of trained ReLU DNNs with 10 hidden layers and 100 nodes per layer on a discontinuous two-dimensional function.
Note that the spread shown in the boxplots is somewhat large due to the discontinuity in the underlying function. On smoother functions (not shown), this spread is smaller.
We also remark that the observed performance improves with depth, see Fig.\ \ref{fig:logsin_plus_sin_K10_relu_plots} which displays the outputs of all 20 trials of a variety of ReLU DNN architectures trained on a smooth, one-dimensional function.

\begin{figure}[ht]
\begin{center}
\includegraphics[width=0.34\paperwidth,clip=true,trim=10mm 0mm 30mm 10mm]{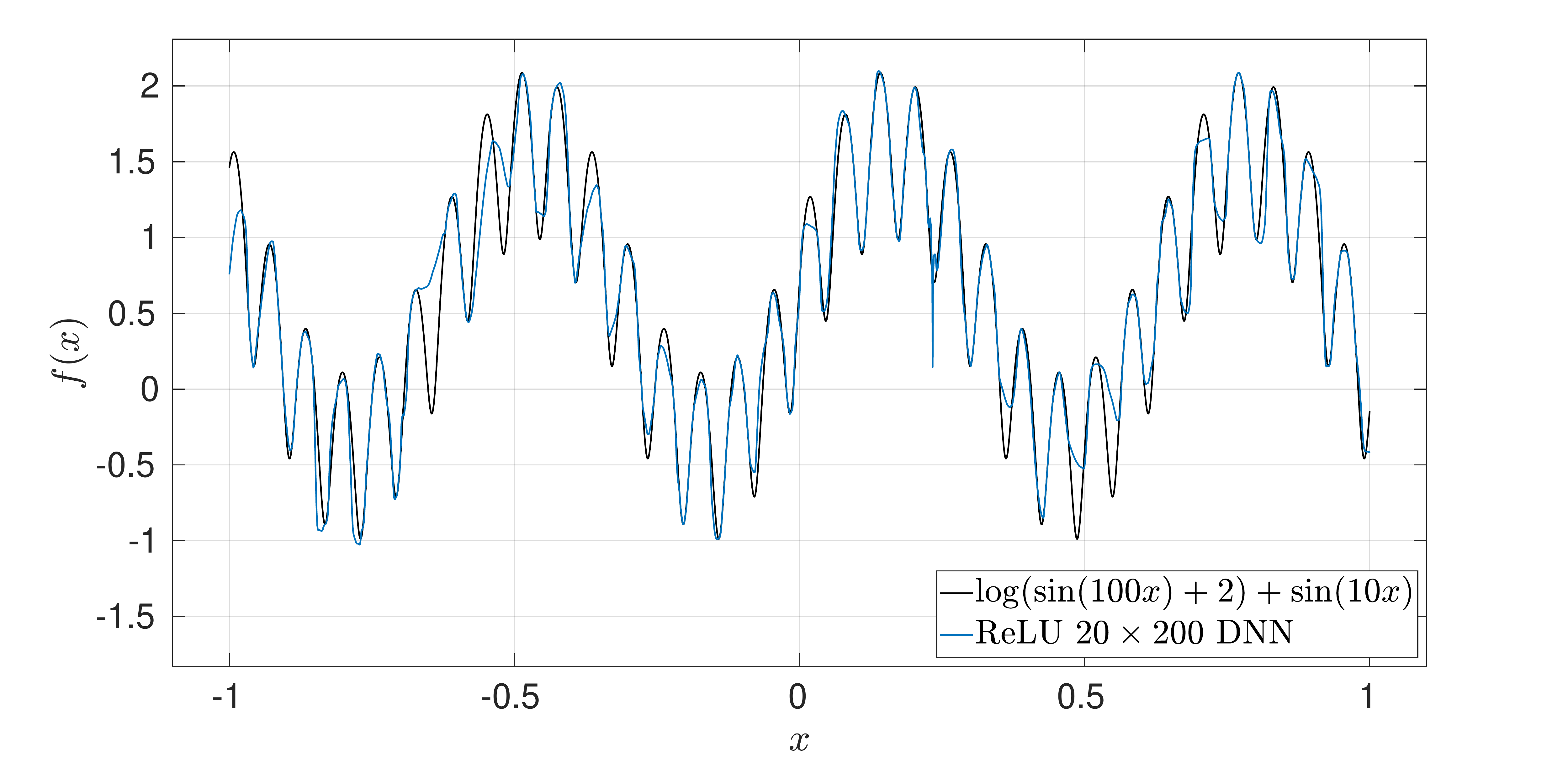}
\includegraphics[width=0.34\paperwidth,clip=true,trim=10mm 0mm 30mm 10mm]{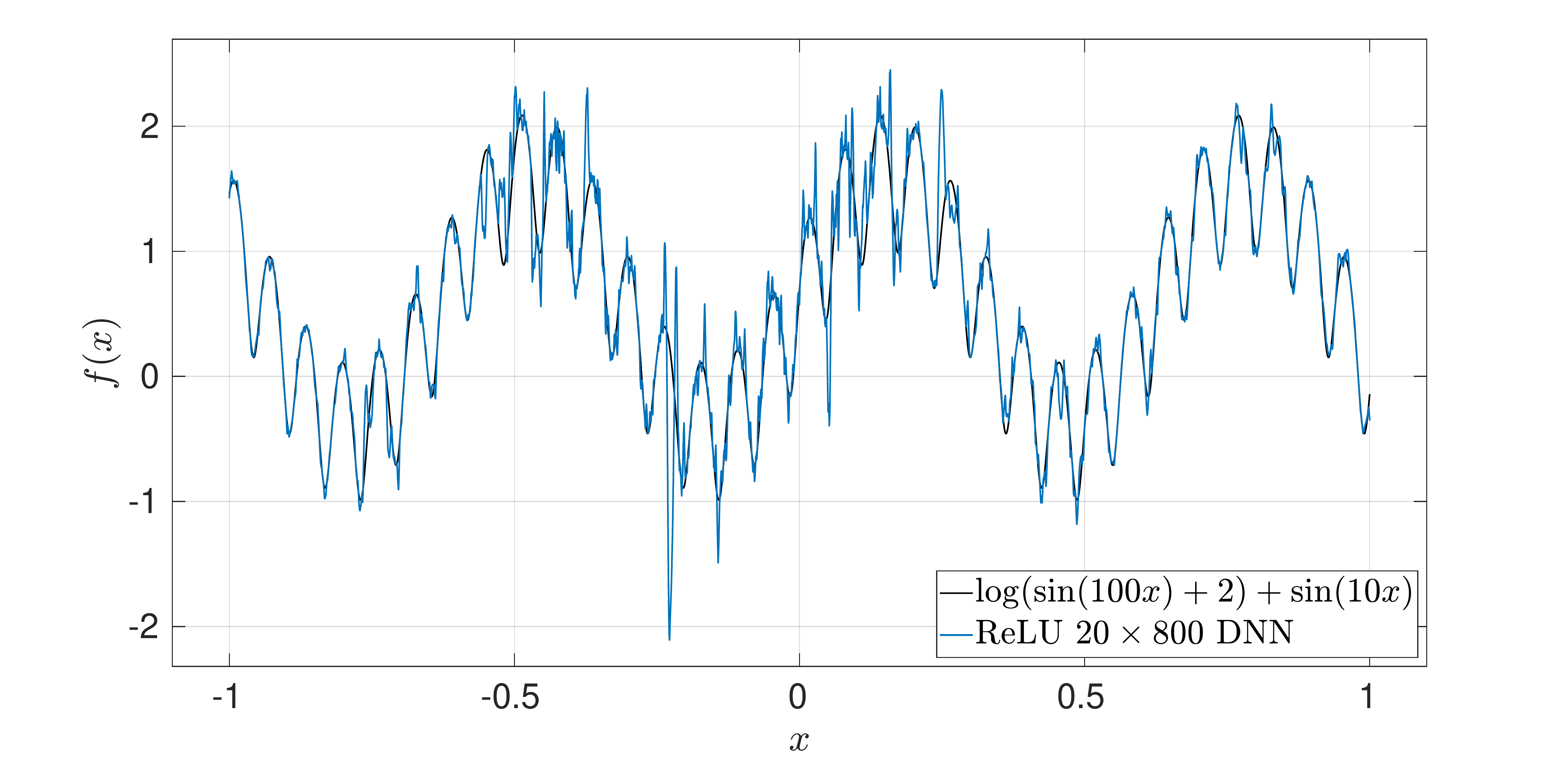}
\end{center}

\vspace{-2mm}
\caption{Unstable ReLU networks with $20$ hidden layers, trained in single precision to a loss tolerance of $5\times 10^{-7}$ on noiseless data. The {\bf(left)} network has $200$ nodes per layer and exhibits a numerical instability (a spike) near $x=0.2334$, while the {\bf(right)} network has $800$ nodes per layer and exhibits multiple instabilities throughout the domain, providing poor pointwise estimation of the target function.}
\label{fig:unstable_network}
\end{figure}

\begin{figure}[ht]
\begin{center}
\includegraphics[width=0.34\paperwidth,clip=true,trim=5mm 0mm 40mm 10mm]{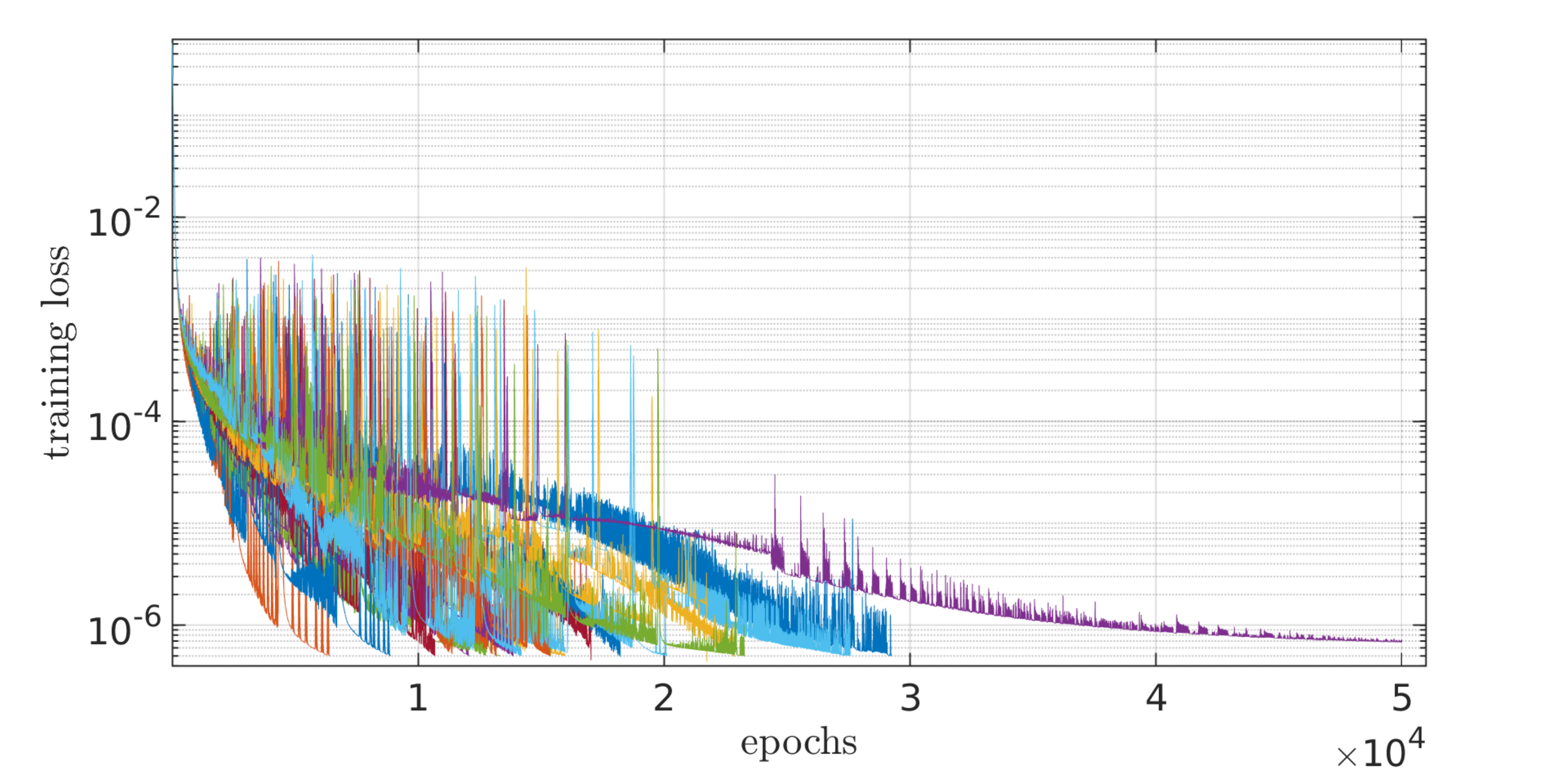}
\includegraphics[width=0.34\paperwidth,clip=true,trim=5mm 0mm 40mm 10mm]{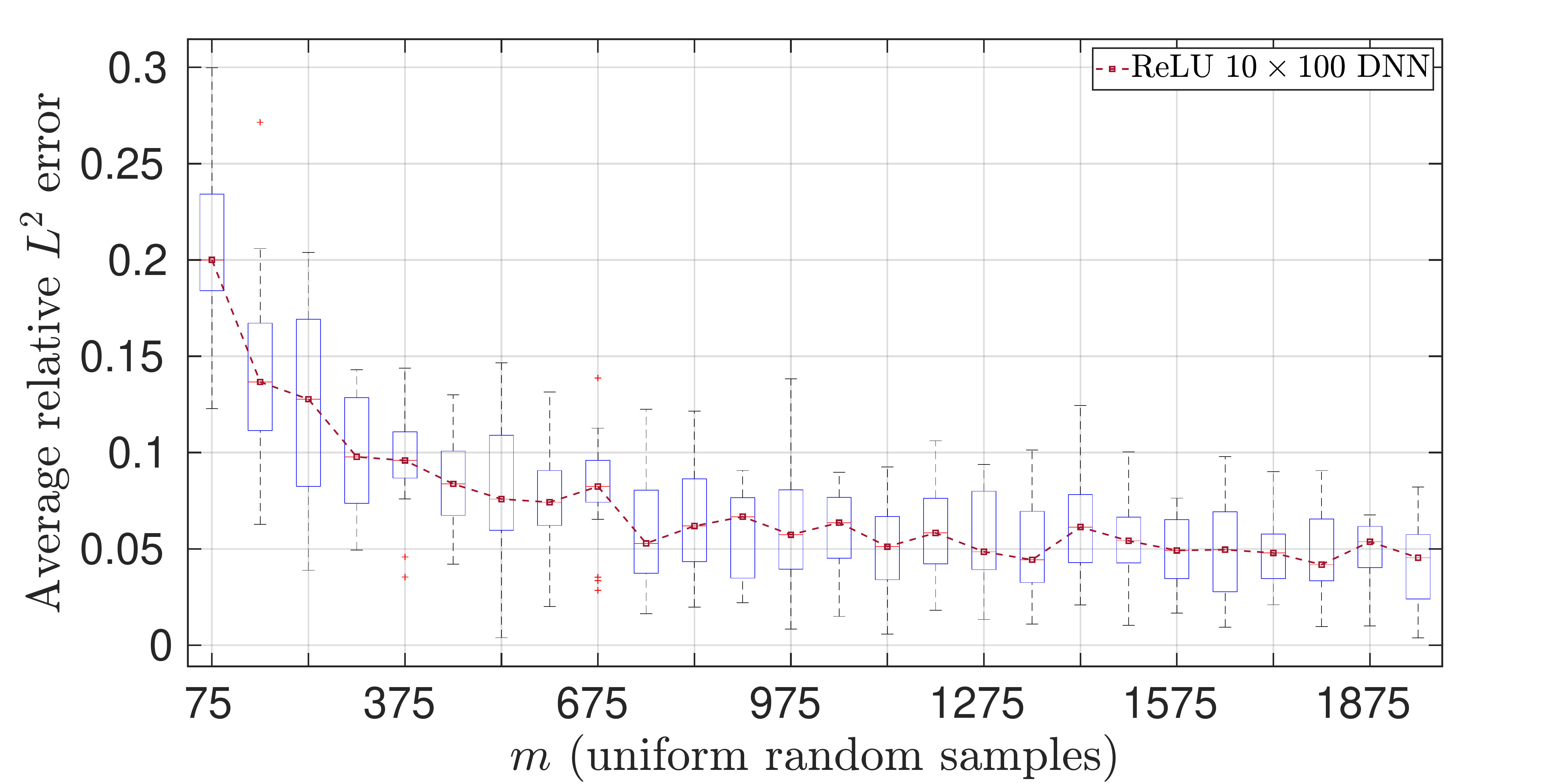}
\end{center}

\vspace{-2mm}
\caption{Training a $10$ layer DNN with $100$ nodes per layer with {\tt Adam} on the function \eqref{eq:halfspace_func} with $d = 2$. {\bf(left)} 
An example where one out of 20 trials (each plotted with a different color) is unable to complete within 50,000 epochs. {\bf(middle)} Boxplot of the average relative $L^2$ errors of 20 trials. For the boxplot, the tops and bottoms of the boxes represent the 25th and 75th percentiles, with the whiskers covering the most extreme datapoints and outliers (red plusses) plotted individually, see {\tt boxplot} from {\tt MATLAB} for more details.}
\label{fig:nonmonotone_Adam_convergence}
\label{fig:boxplot}
\end{figure}

\begin{figure}[ht]
\begin{center}
\includegraphics[width=0.23\paperwidth,clip=true,trim=20mm 0mm 40mm 0mm]{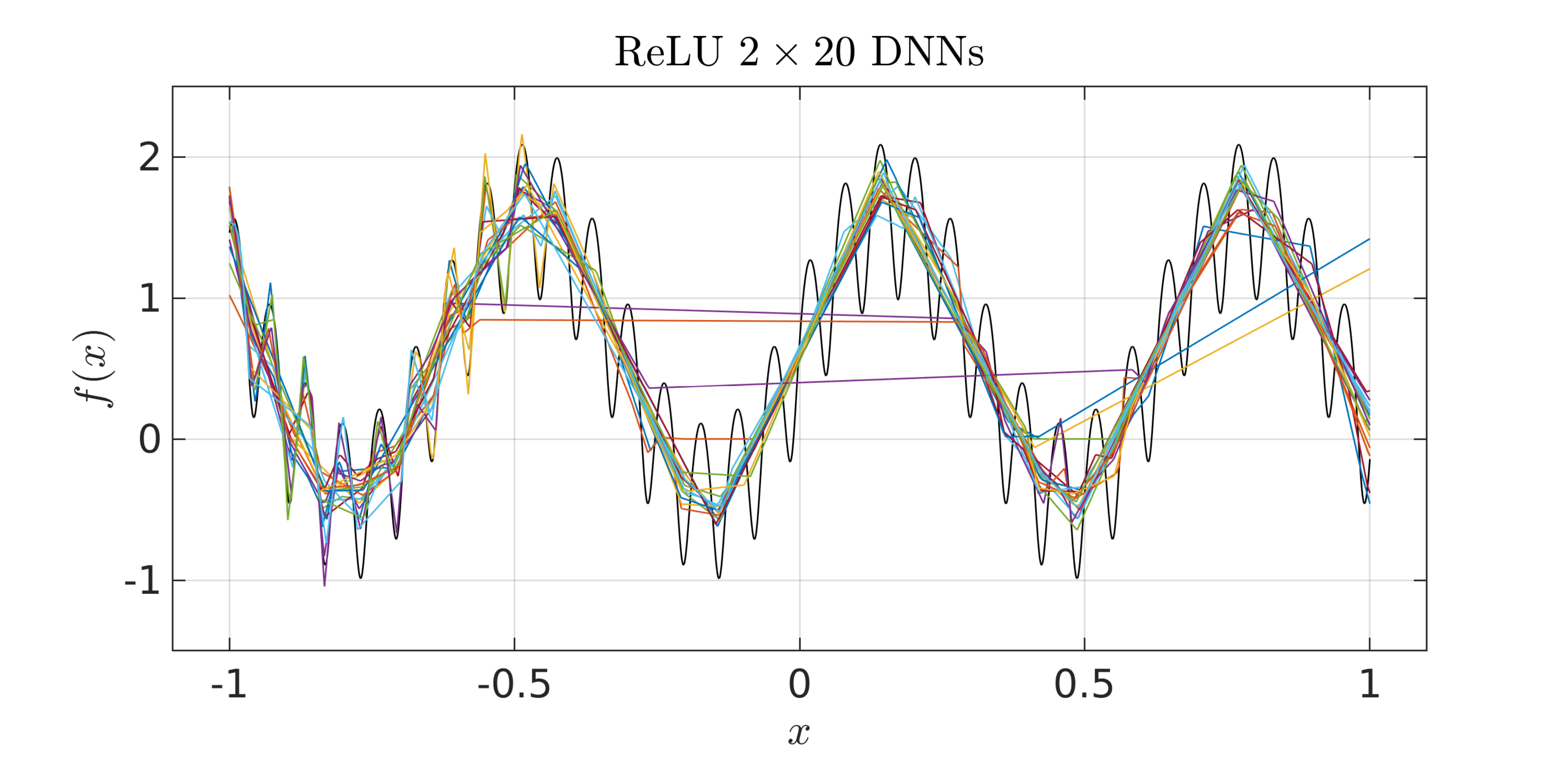}
\includegraphics[width=0.23\paperwidth,clip=true,trim=20mm 0mm 40mm 0mm]{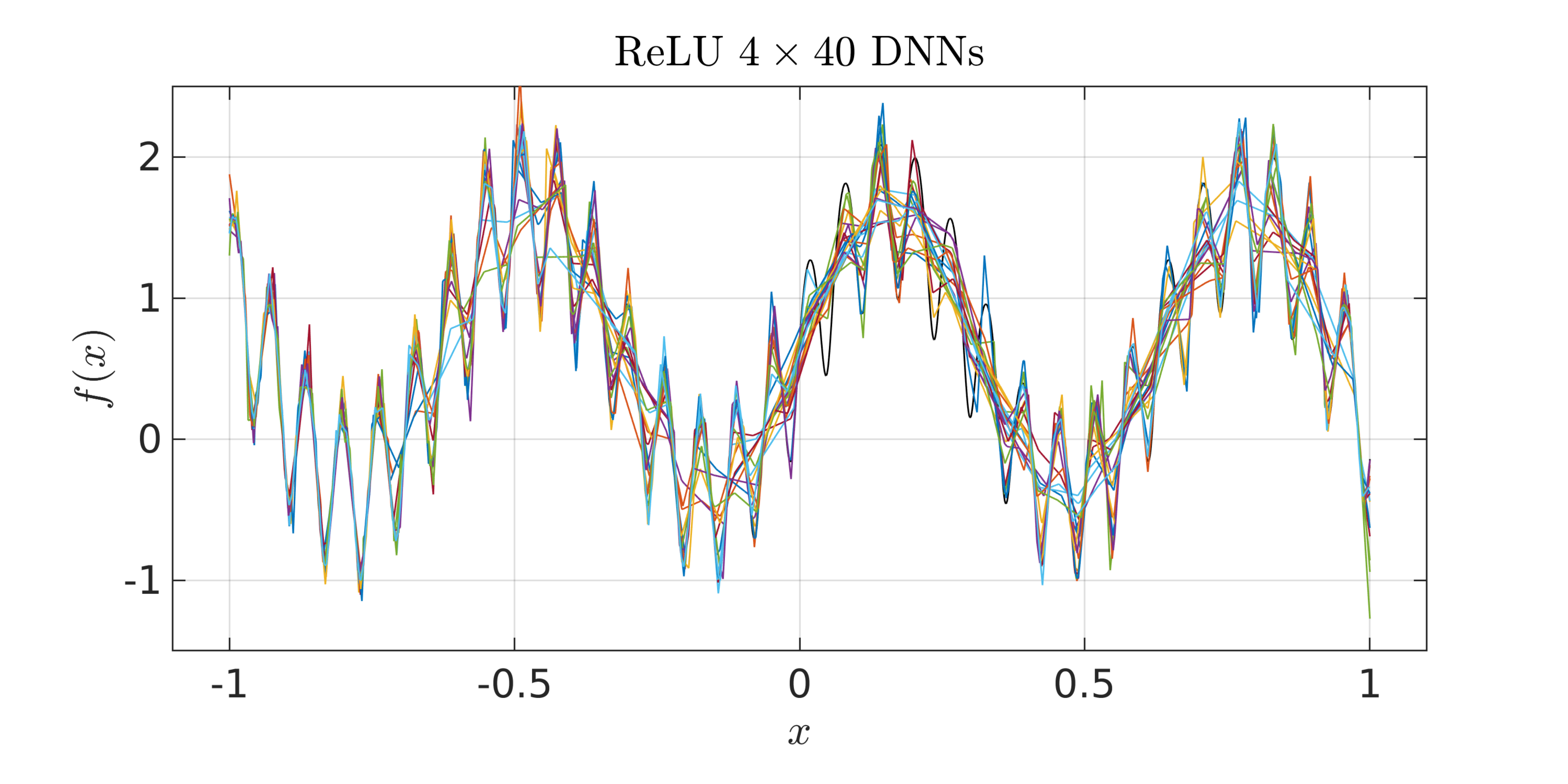} 
\includegraphics[width=0.23\paperwidth,clip=true,trim=20mm 0mm 40mm 0mm]{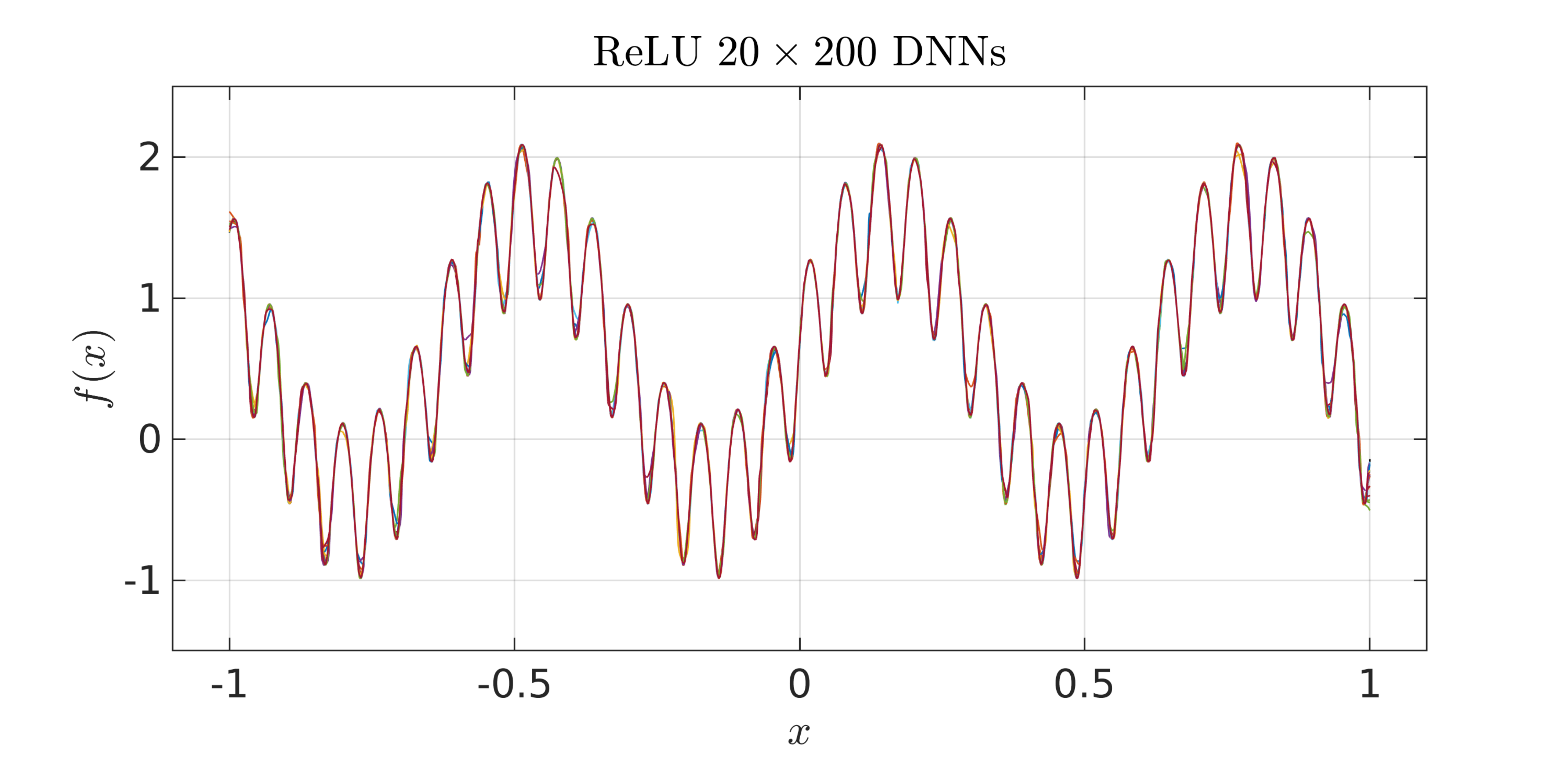}
\end{center}

\vspace{-2mm}
\caption{Comparison of outputs of a variety of ReLU DNN architectures trained with {\tt Adam} on 750 samples of the function $f(x) = \log(\sin(100x)+2)+\sin(10x)$ (plotted in black).}
\label{fig:logsin_plus_sin_K10_relu_plots}
\end{figure}

% !TEX root = ./MLFA.tex

\section{Numerical experiments}
\label{sec:experiments}

We now present our numerical experiments. We test on smooth functions of one or more variables, as well as piecewise continuous functions. The results presented in this section elaborate on the main conclusions in \S \ref{sec:introconc}.

%----------------------------------------------------
\subsection{Smooth one-dimensional functions}
\label{subsec:smooth_1d_numex}
%----------------------------------------------------

We first consider problems where the target function has analytic dependence on only one variable, with regularity controllable by a single parameter. Specifically we consider
\begin{align}
\label{1dsmoothfnK}
f(x) = \log(\sin(10 K x) + 2) + \sin(K x),
\end{align}
for values of $K=1$ and $K=10$. When $K=1$, $f$ is smooth and has rapidly decaying Legendre coefficients, making it ideal for reconstruction with CS techniques. When $K=10$, $f$ oscillates rapidly and has slowly decaying coefficients, leading to less efficient approximation by polynomial-based methods.
Fig.\ \ref{fig:1d_log_sin_func_1} displays results in the case of $K=1$, and the right plot compares the average relative $L^2$ errors of the CS and DL approaches with respect to the number of samples $m$ used in training.
We observe shallower networks fail to achieve a good approximation, while increasing depth and width leads to networks which are competitive with the unweighted and weighted CS approaches.
However, comparing the results obtained with the 10 and 20 hidden layer deep networks, we note diminishing returns in performance with increasing size.

Any number of factors can contribute to this phenomenon, including overfitting of the data, increased difficulty of training larger networks from our choice of random initialization, or development of numerical instabilities due to the accumulation of errors from standard sources, e.g., roundoff, overflow, and underflow.
We remark that improved results have been obtained on this function by training some of our networks in double precision, see Fig.\ \ref{fig:precision_comparison}. 
However we also observe that training the 20 hidden layer DNN in double precision did not improve, but actually decreased its accuracy, suggesting room for improvement in training and initializing larger DNNs in double precision.

\begin{figure}[ht]
\begin{center}
\includegraphics[width=0.23\paperwidth,clip=true,trim=0mm 0mm 0mm 0mm]{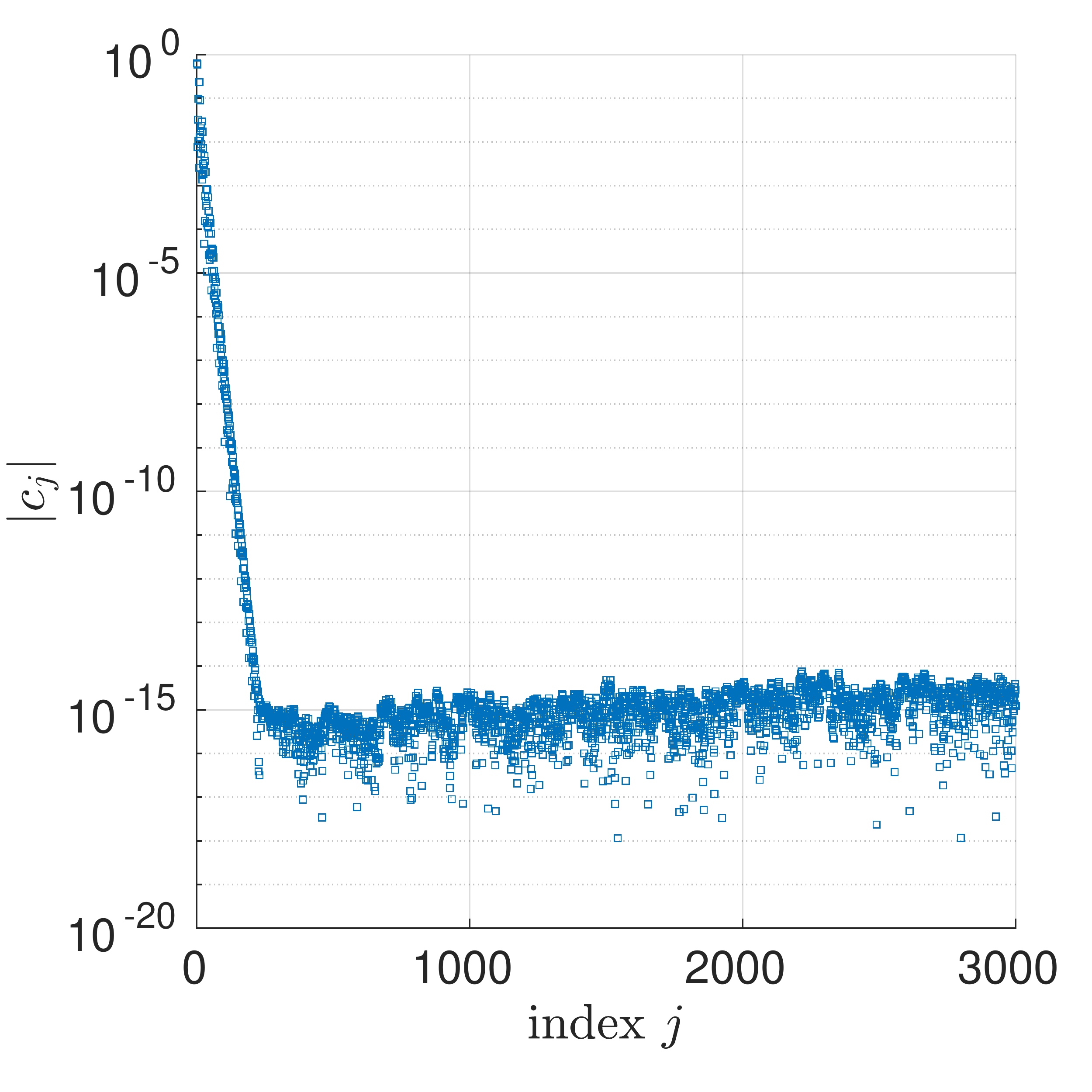}
\includegraphics[width=0.23\paperwidth,clip=true,trim=0mm 0mm 0mm 0mm]{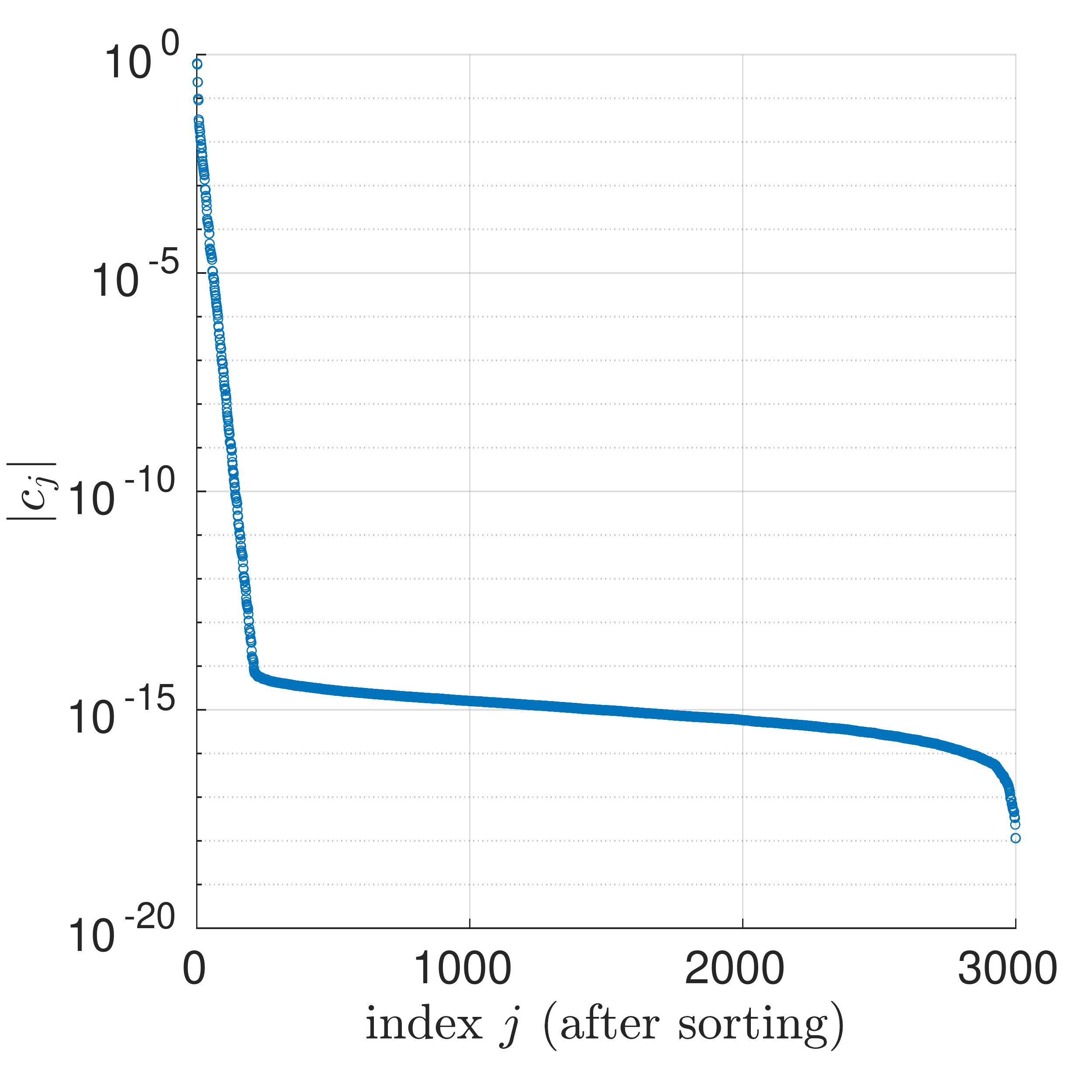}
\includegraphics[width=0.23\paperwidth,clip=true,trim=0mm 0mm 0mm 0mm]{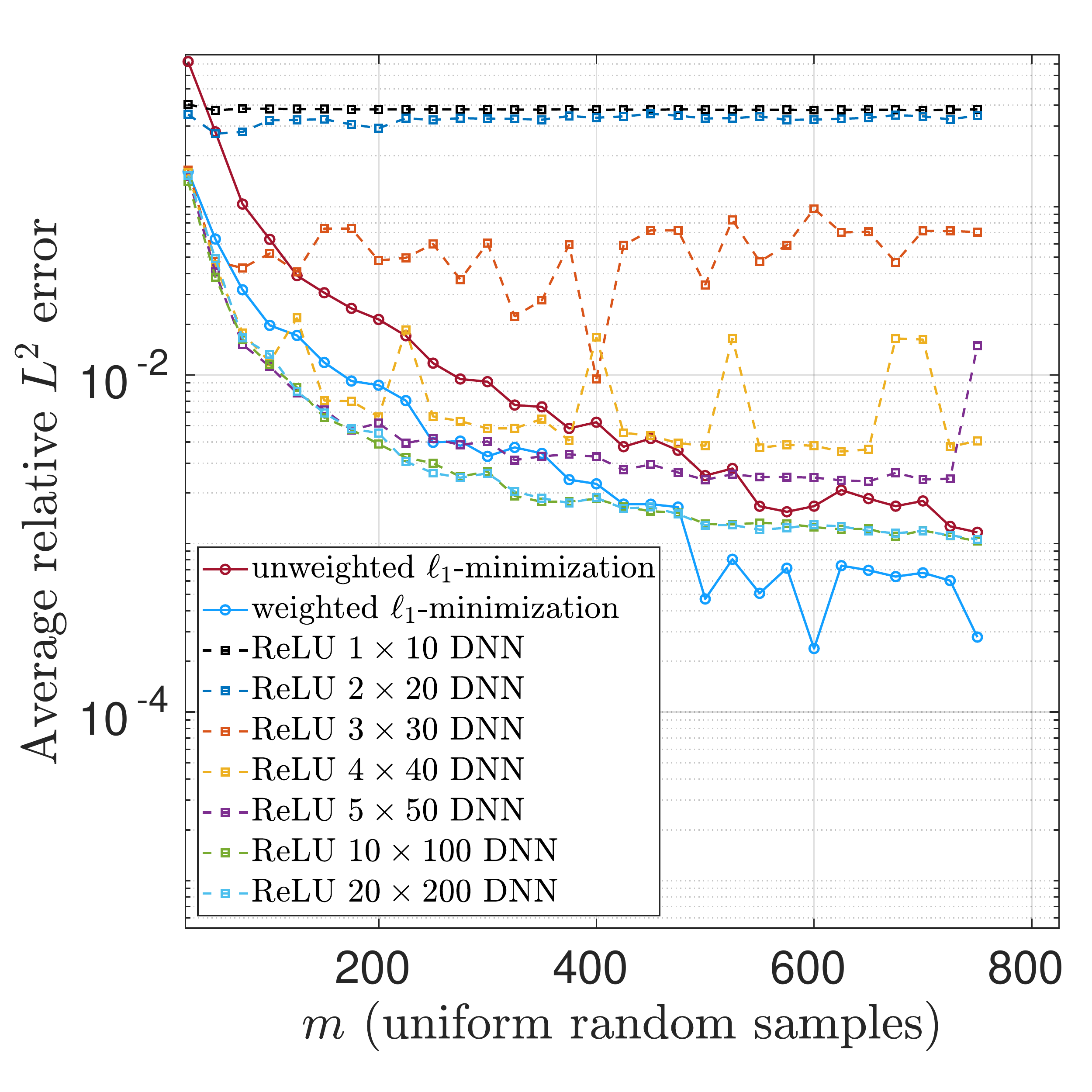}
\end{center}

\vspace{-2mm}
\caption{Legendre coefficients of $f(x) = \log( \sin(10 x) + 2) + \sin(x)$ sorted {\bf (left)} lexicographically and {\bf (center)} by decreasing magnitude. {\bf (right)} Average relative $L^2$ error v.s. number of samples used in training. CS approximations were computed with the Legendre basis of cardinality $n=3,000$.}
\label{fig:1d_log_sin_func_1}
\end{figure}

Fig.\ \ref{fig:1d_log_sin_func_2} displays the results of approximating $f$ when $K=10$, where the more rapid oscillation now leads to a degradation in performance for both methods.
We again observe improved accuracy by increasing depth and width as in the previous example. We also again observe the diminishing returns increasing the size of the network architectures from $10 \times 100$ to $20\times 200$, although the 20 hidden layer network provides the best accuracy of all tested DNNs after approximately 600 samples and achieves the same accuracy as the weighted CS approximations around 700 samples.

\begin{figure}[ht]
\begin{center}
\includegraphics[width=0.23\paperwidth,clip=true,trim=0mm 0mm 0mm 0mm]{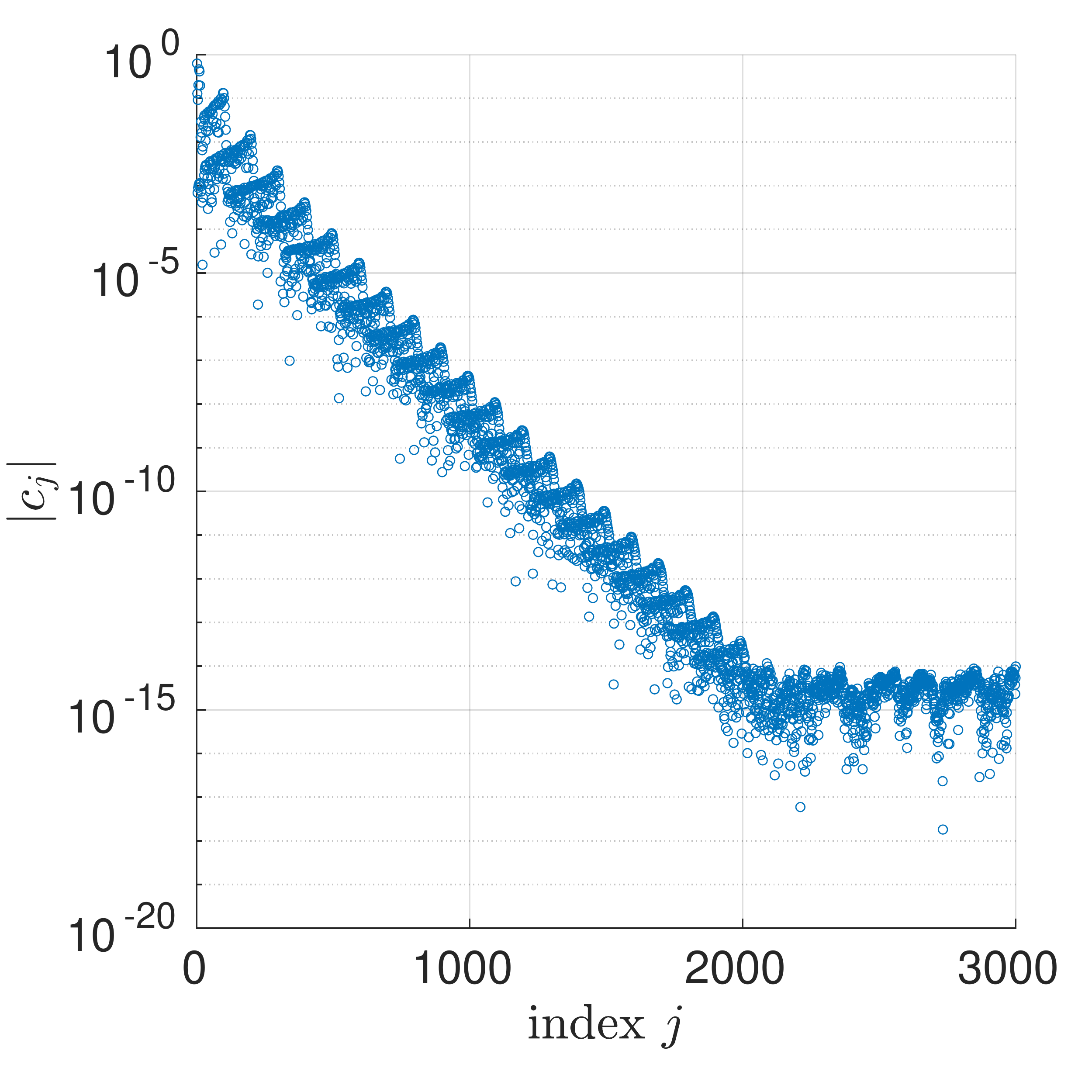}
\includegraphics[width=0.23\paperwidth,clip=true,trim=0mm 0mm 0mm 0mm]{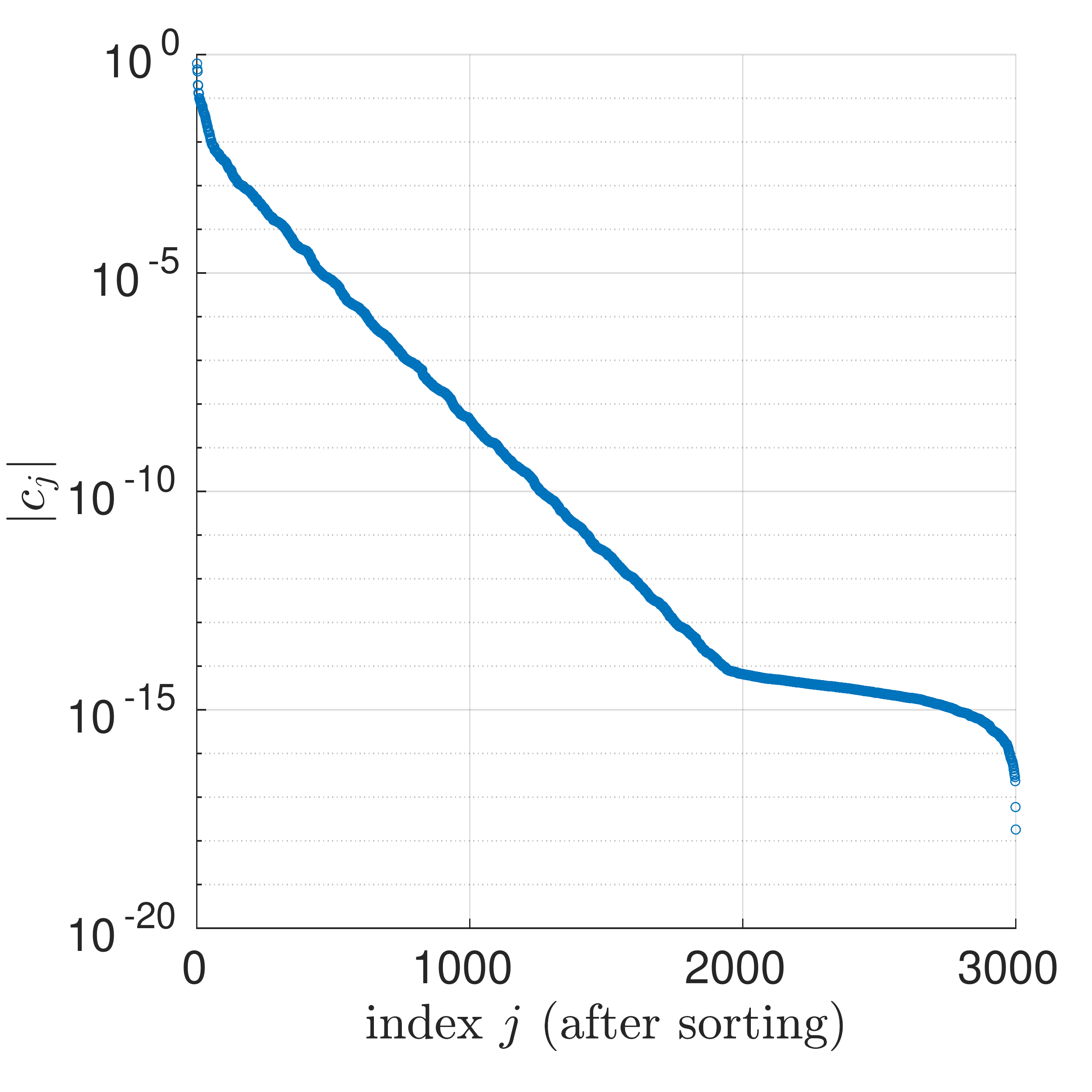}
\includegraphics[width=0.23\paperwidth,clip=true,trim=0mm 0mm 0mm 0mm]{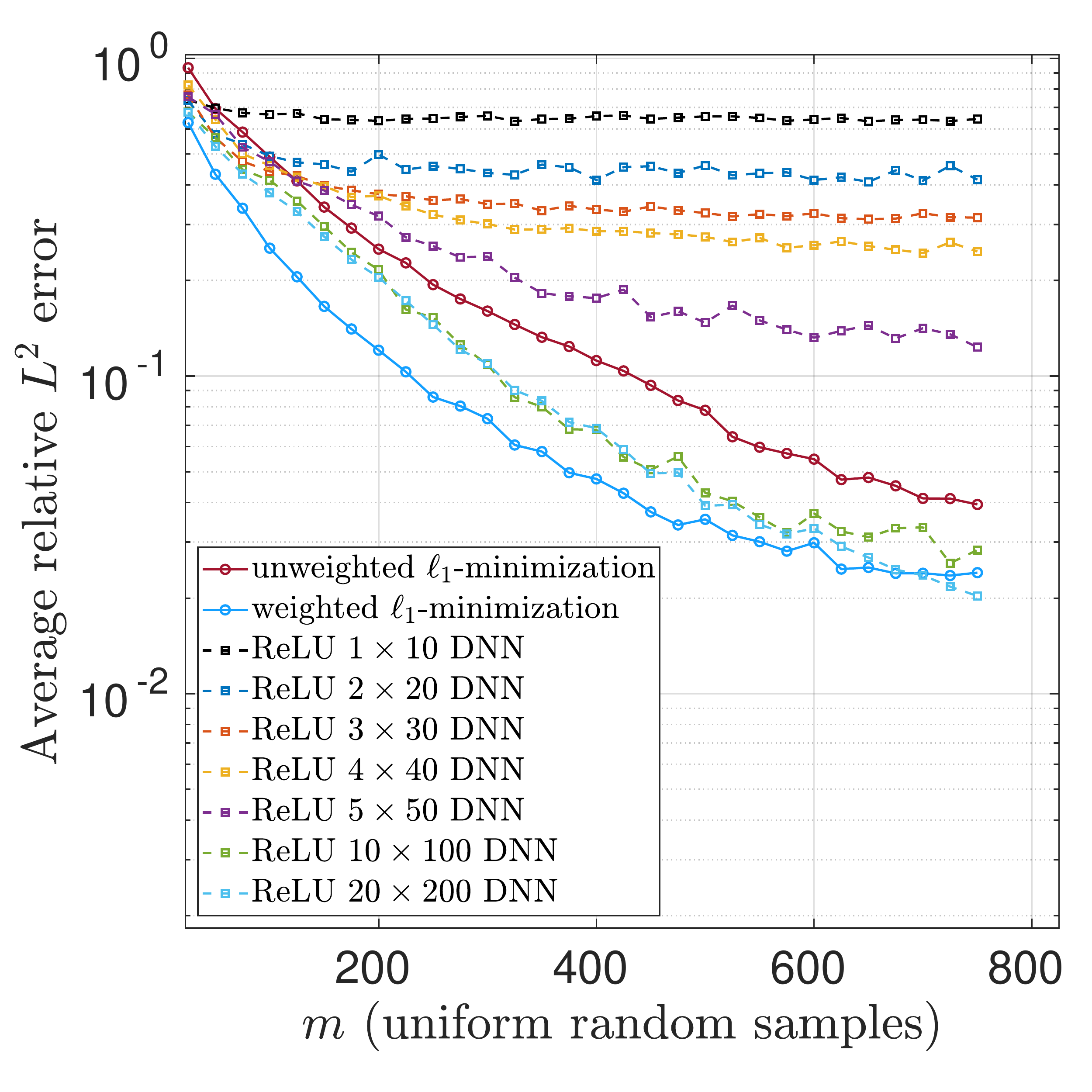}
\end{center}

\vspace{-2mm}
\caption{Legendre coefficients of $f(x) = \log( \sin(100 x) + 2) + \sin(10 x)$ sorted {\bf (left)} lexicographically and {\bf (center)} by decreasing magnitude. {\bf (right)} Average relative $L^2$ error v.s. number of samples used in training. CS approximations were computed with the Legendre basis of cardinality $n=3,000$.}
\label{fig:1d_log_sin_func_2}
\end{figure}

Fig.\ \ref{fig:1d_log_sin_func_2_beta_comp} displays the effect of choosing wider networks in approximating $f$ with $K=10$. There we see wider counterparts of the network architectures generally outperform narrower architectures with the same number of hidden layers. However, we also observe diminishing returns going from 10 to 20 hidden layers, and in the right plot we see the ReLU $20\times 800$ DNNs diverge, with the resulting trained networks exhibiting numerical artifacts as in the right plot of Fig.\ \ref{fig:unstable_network}.

\begin{figure}[ht]
\begin{center}
\includegraphics[width=0.23\paperwidth,clip=true,trim=0mm 0mm 0mm 0mm]{logsin_plus_sin_K=10_d=1_CSvDNNs.pdf}
\includegraphics[width=0.23\paperwidth,clip=true,trim=0mm 0mm 0mm 0mm]{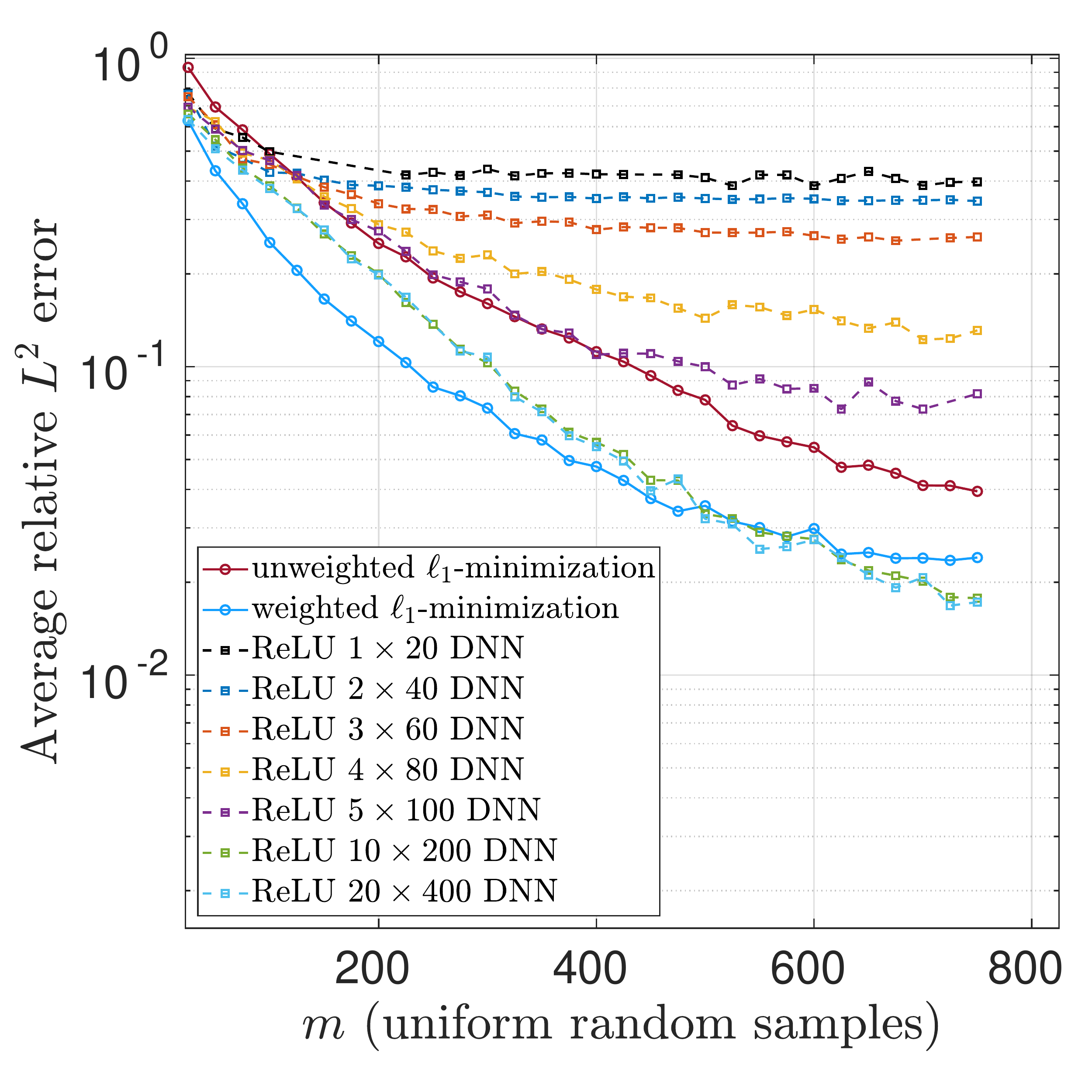}
\includegraphics[width=0.23\paperwidth,clip=true,trim=0mm 0mm 0mm 0mm]{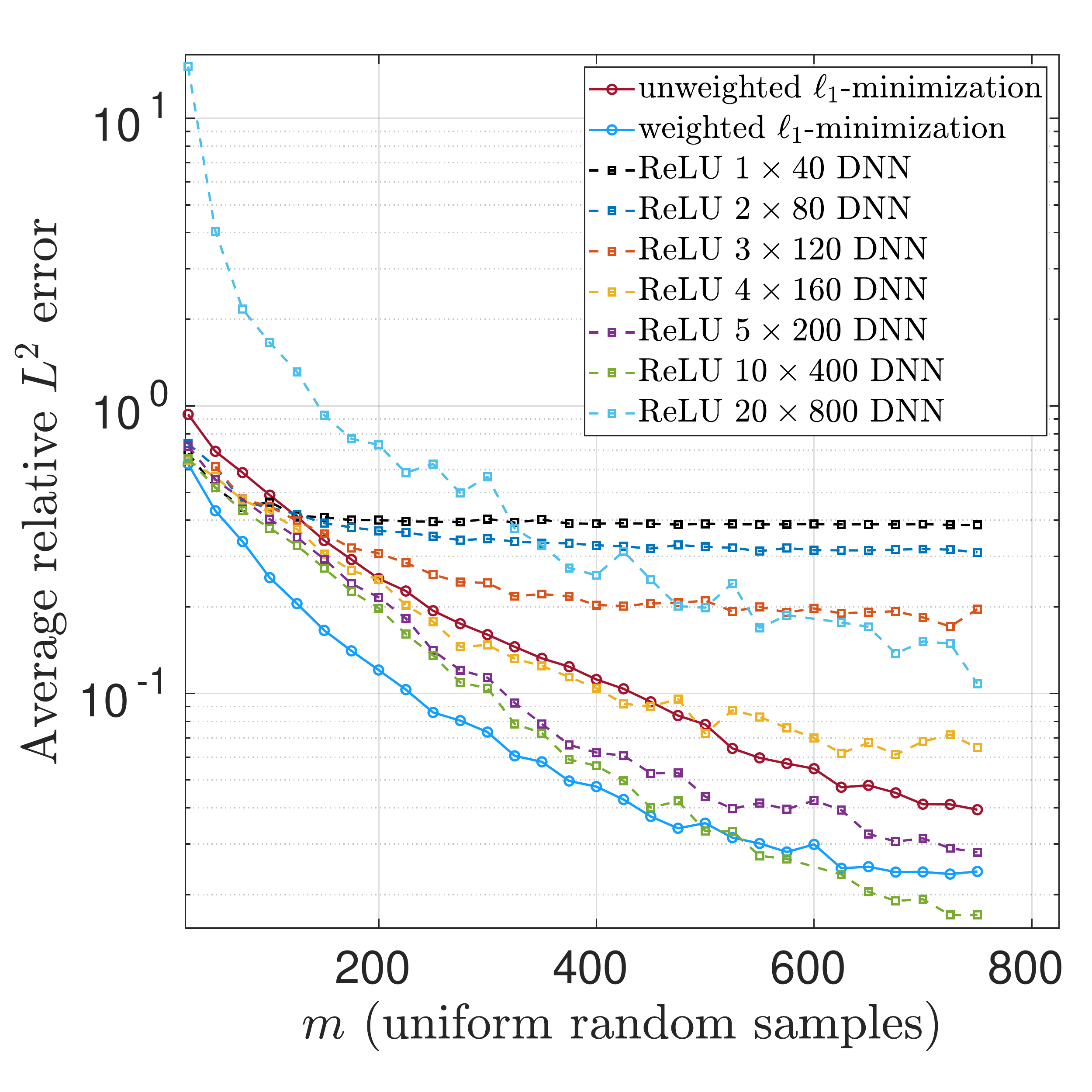}
\end{center}

\vspace{-2mm} \caption{Average relative $L^2$ error v.s. number of samples of $f(x) = \log(\sin(100x)+2)+\sin(10x)$ used in training DNNs with $L$ hidden layers and $N$ nodes per layer when $\beta = L/N$ is {\bf(left)} 0.1 {\bf(middle)} 0.05 and {\bf(right)} 0.025. CS approximations were computed with the Legendre basis of cardinality $n=3,000$.}
\label{fig:1d_log_sin_func_2_beta_comp}
\end{figure}

Fig.\ \ref{fig:1d_log_sin_func_K1_K10_weight_comp} compares the average absolute maximum of the weights and biases for ReLU DNNs trained on \eqref{1dsmoothfnK}. We observe in the case $K=10$, corresponding to the more oscillatory version of this function, on average the trained DNN architectures have larger weights and biases in magnitude when compared to those trained on the same function with $K=1$. Also, when comparing the trained weights of the architectures when the ratio $\beta = 0.1$ and $\beta= 0.025$ on the function with $K=10$, we note the similarity in the magnitudes of the weights despite the presence of severe artifacts for the wider networks that are not present for their narrower counterparts, see Fig.\ \ref{fig:unstable_network}.

\begin{figure}[ht]
\begin{center}
\includegraphics[width=0.23\paperwidth,clip=true,trim=0mm 0mm 0mm 0mm]{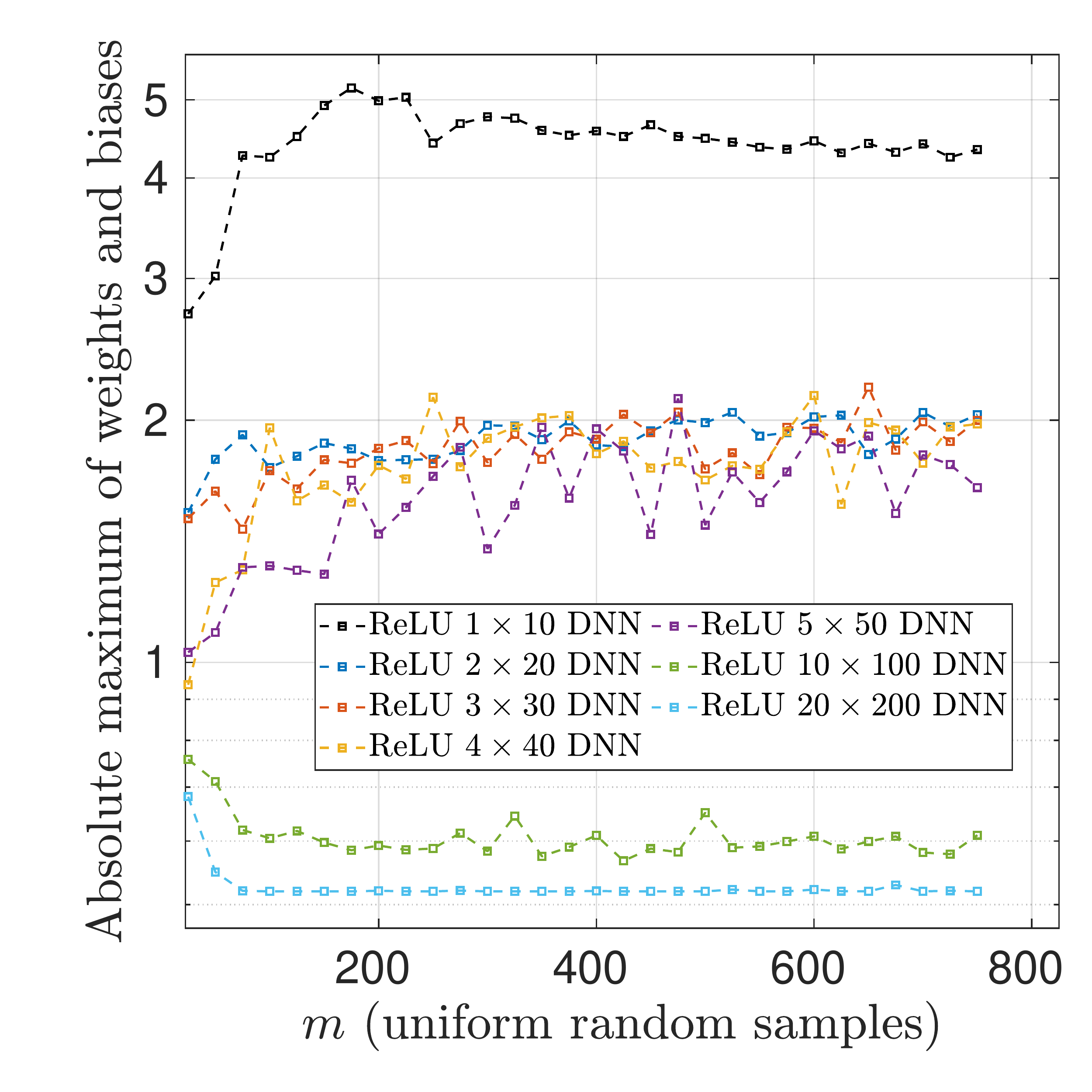}
\includegraphics[width=0.23\paperwidth,clip=true,trim=0mm 0mm 0mm 0mm]{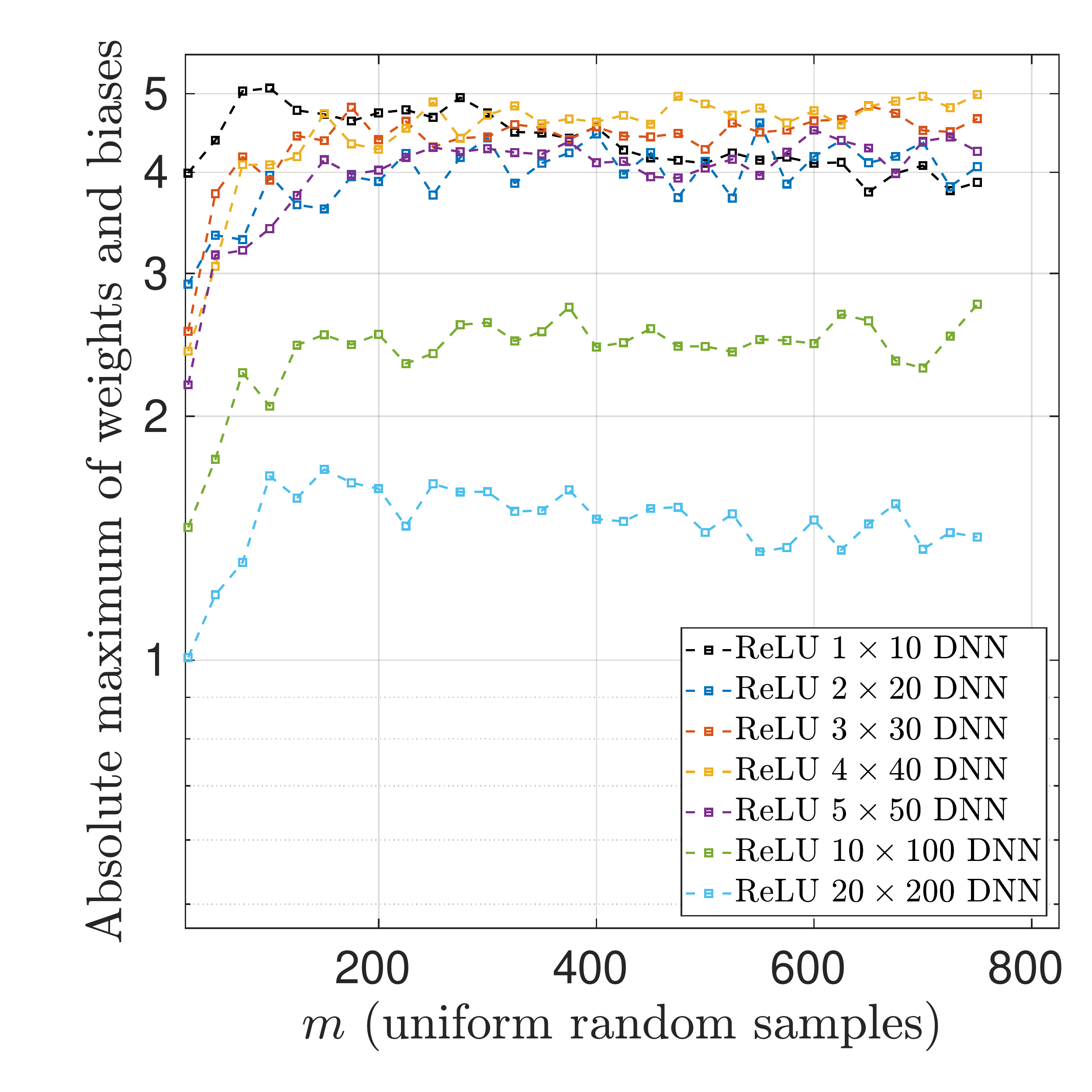}
\includegraphics[width=0.23\paperwidth,clip=true,trim=0mm 0mm 0mm 0mm]{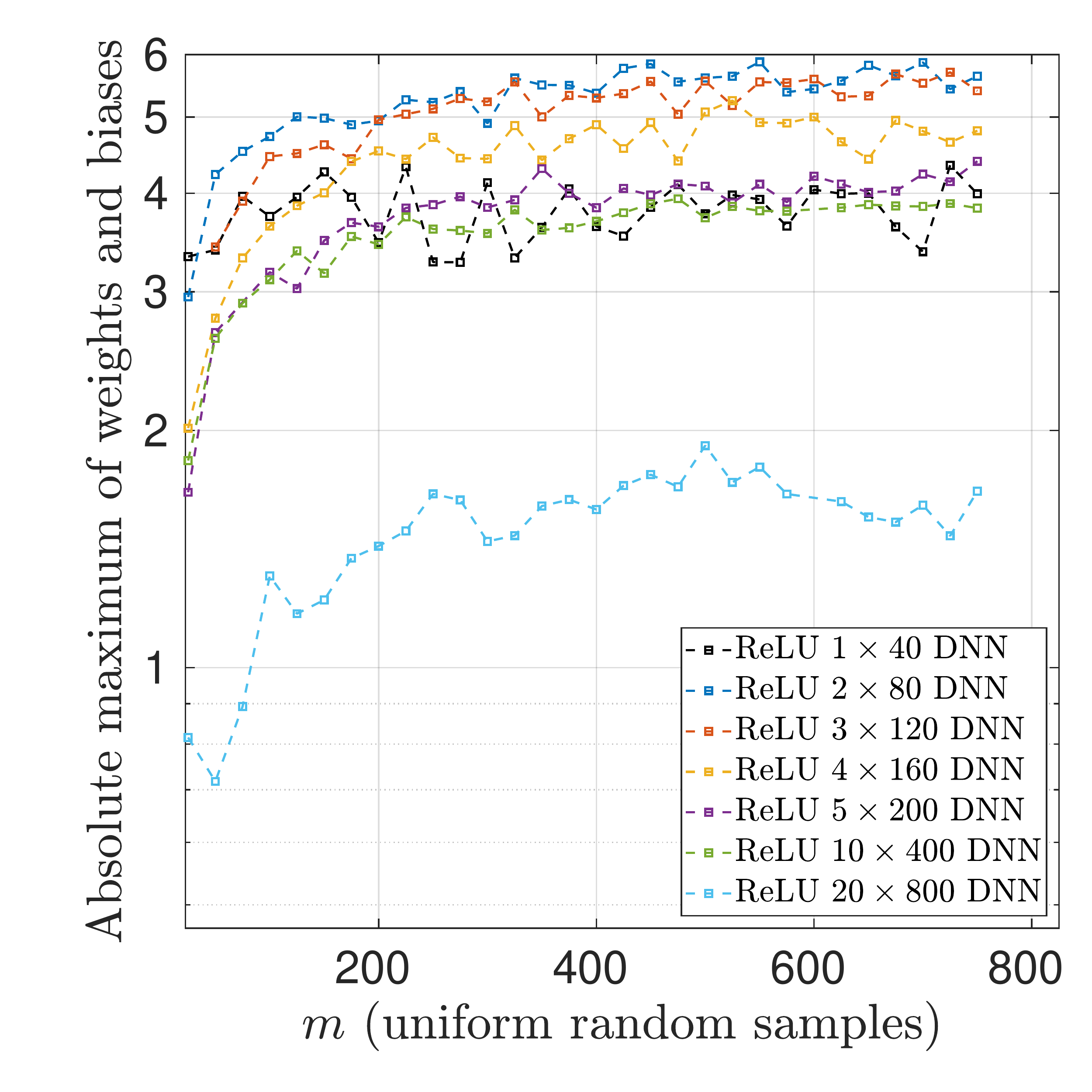}
\end{center}

\vspace{-2mm}
\caption{Average absolute maximum of weights and biases v.s. number of samples of $f(x) = \log(\sin(10Kx) + 2) + \sin(Kx)$ for {\bf(left)} $K = 1$ and $\beta = 0.1$, {\bf(middle)} $K = 10$ and $\beta = 0.1$, and {\bf(right)} $K=10$ and $\beta = 0.025$.}
\label{fig:1d_log_sin_func_K1_K10_weight_comp}
\end{figure}

%----------------------------------------------------
\subsection{Smooth higher-dimensional functions}
\label{subsec:smooth_high-d_numex}
%----------------------------------------------------

In Fig.\ \ref{fig:exp_cos_func}, we present results for the function 
\begin{align}
\label{eq:exp_cos_func}
f(x_1,\ldots,x_d) = \exp \left( - (\cos(x_1)+\ldots+\cos(x_d) ) / (8d) \right),
\end{align}
studied in \cite{ChkifaDexterTranWebster18}. 
This function has smooth, isotropic dependence on its parameters and rapidly decaying Legendre coefficients with increasing polynomial order, making it ideal for approximation with CS techniques.
Despite the analyticity of $f$, the DNNs are unable to obtain an approximation accurate beyond 3 digits while both CS approaches achieve nearly 8 digits of accuracy. 
This result highlights the perceived gap between theory and practice in this work: theoretical results suggest DNNs are capable of achieving the same performance as CS on this problem, but practically we never achieve these accuracies.

\begin{figure}[ht]
\begin{center}
\includegraphics[width=0.23\paperwidth,clip=true,trim=0mm 0mm 0mm 0mm]{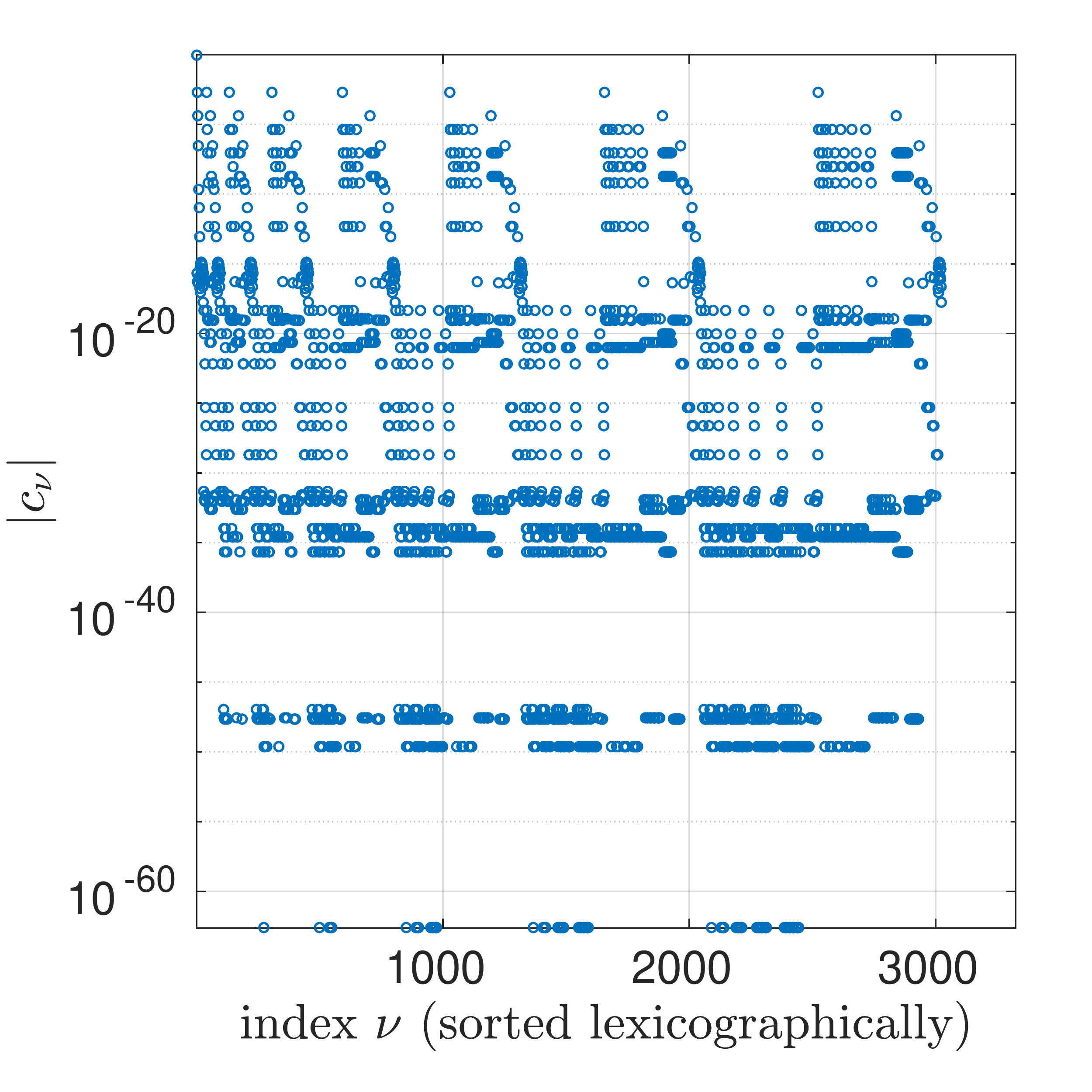}
\includegraphics[width=0.23\paperwidth,clip=true,trim=0mm 0mm 0mm 0mm]{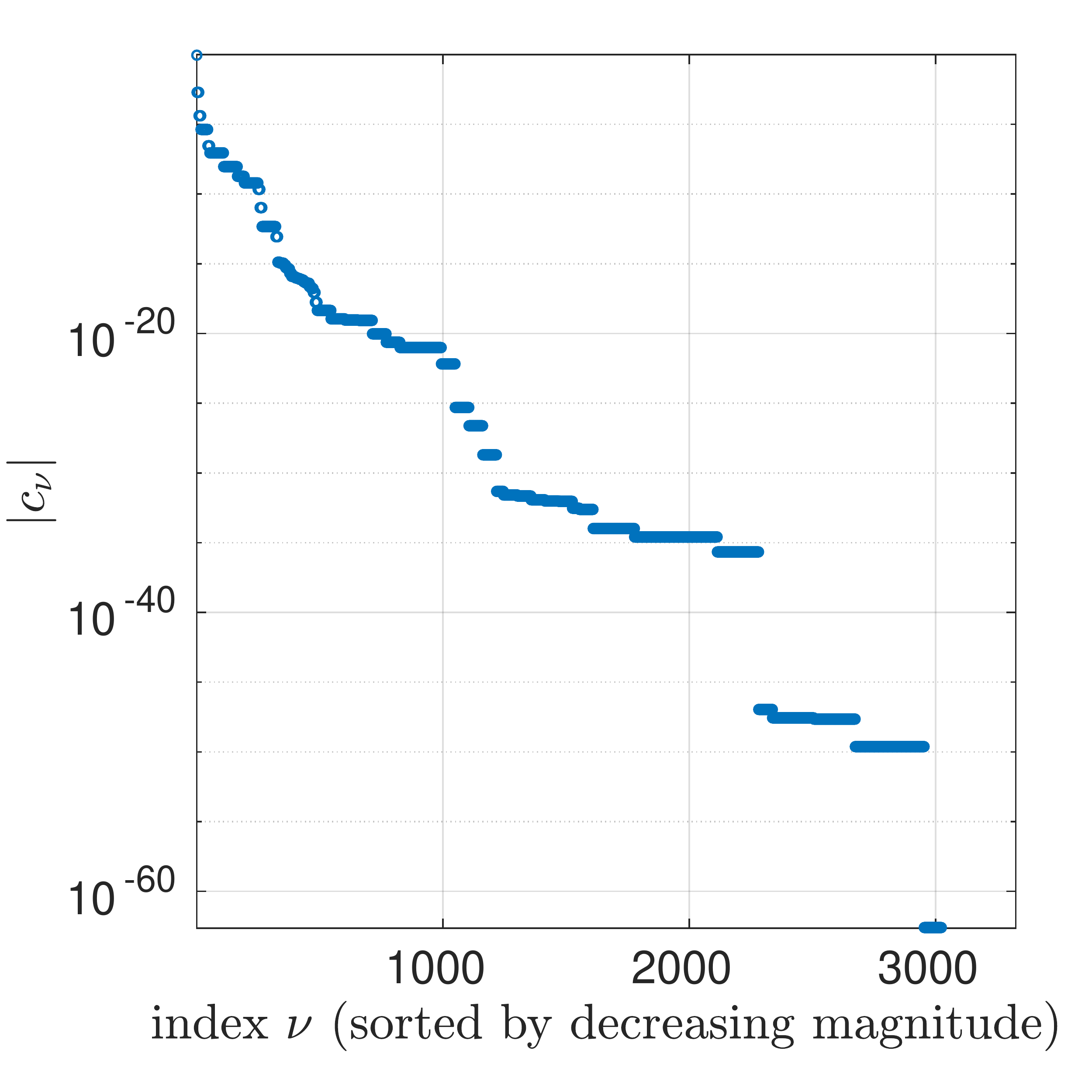}
\includegraphics[width=0.23\paperwidth,clip=true,trim=0mm 0mm 0mm 0mm]{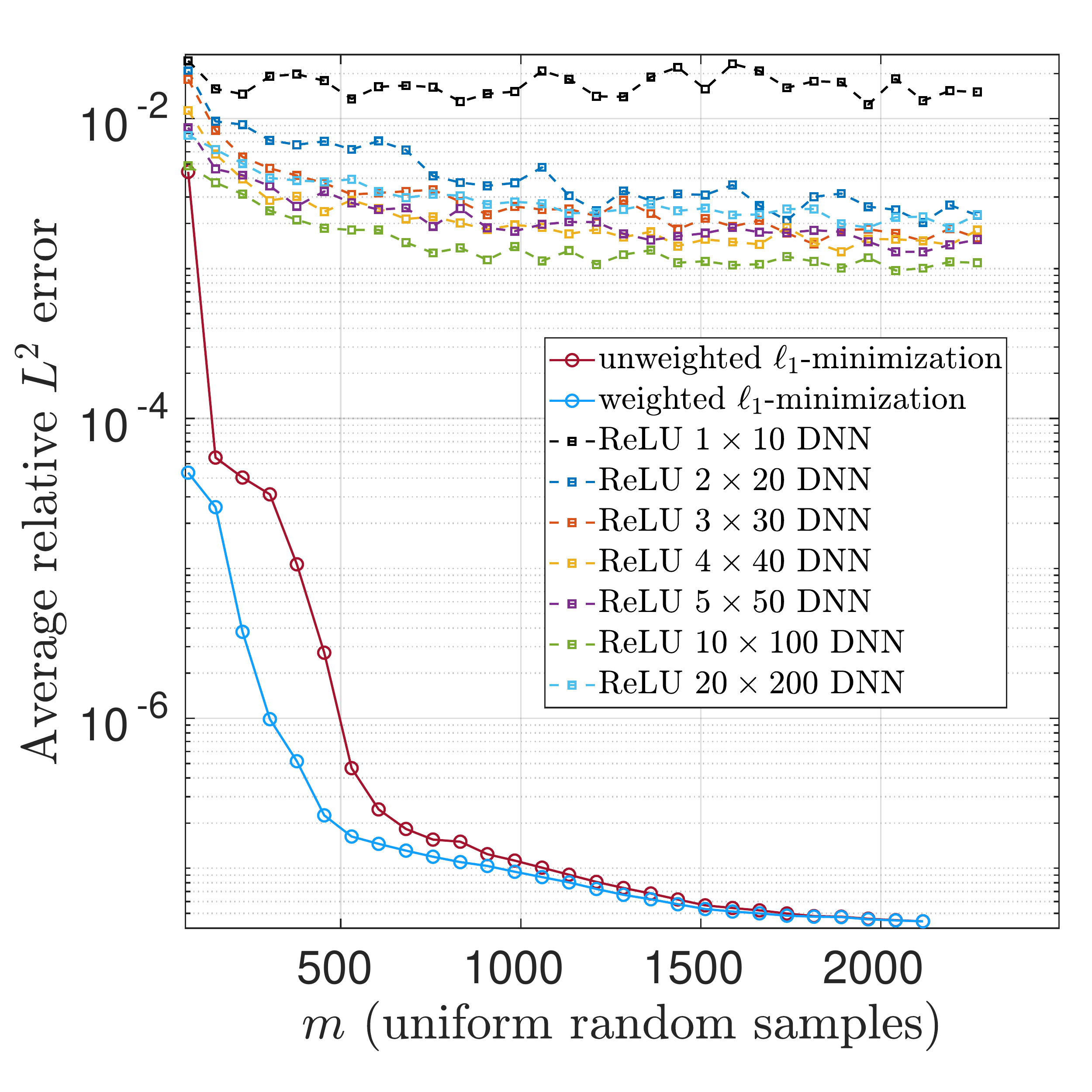}
\end{center}

\vspace{-2mm}
\caption{Legendre coefficients of $f$ from \eqref{eq:exp_cos_func} with $d=8$ sorted {\bf (left)} lexicographically and {\bf (center)} by decreasing magnitude. {\bf(right)} Average relative $L^2$ error vs. number of samples used in training. CS approximations were computed with the Legendre basis of cardinality $n=3,023$.}
\label{fig:exp_cos_func}
\end{figure}

Fig.\ \ref{fig:relu_exp_cos_func_beta_comp} displays the results of approximating $f$ from \eqref{eq:exp_cos_func} with both narrower and wider networks.
For narrower DNN architectures, corresponding to $\beta \in \{0.1,0.2,0.5\}$, we observe deeper networks perform better. For wider architectures corresponding to $\beta \in \{0.025,0.05\}$, we observe the shallower networks perform better. 
In all cases, the best performing networks had between $10^3$ and $10^5$ total trainable parameters.
Also, in contrast to the diminishing returns with increasing depth observed in the one-dimensional smooth function examples, on this smooth higher-dimensional problem we now observe divergence of wider and deeper networks.
These observations suggest the existence of fundamental stability barriers in current methods of training DNNs.

\begin{figure}[ht]
\begin{center}
\includegraphics[width=0.23\paperwidth,clip=true,trim=0mm 0mm 0mm 0mm]{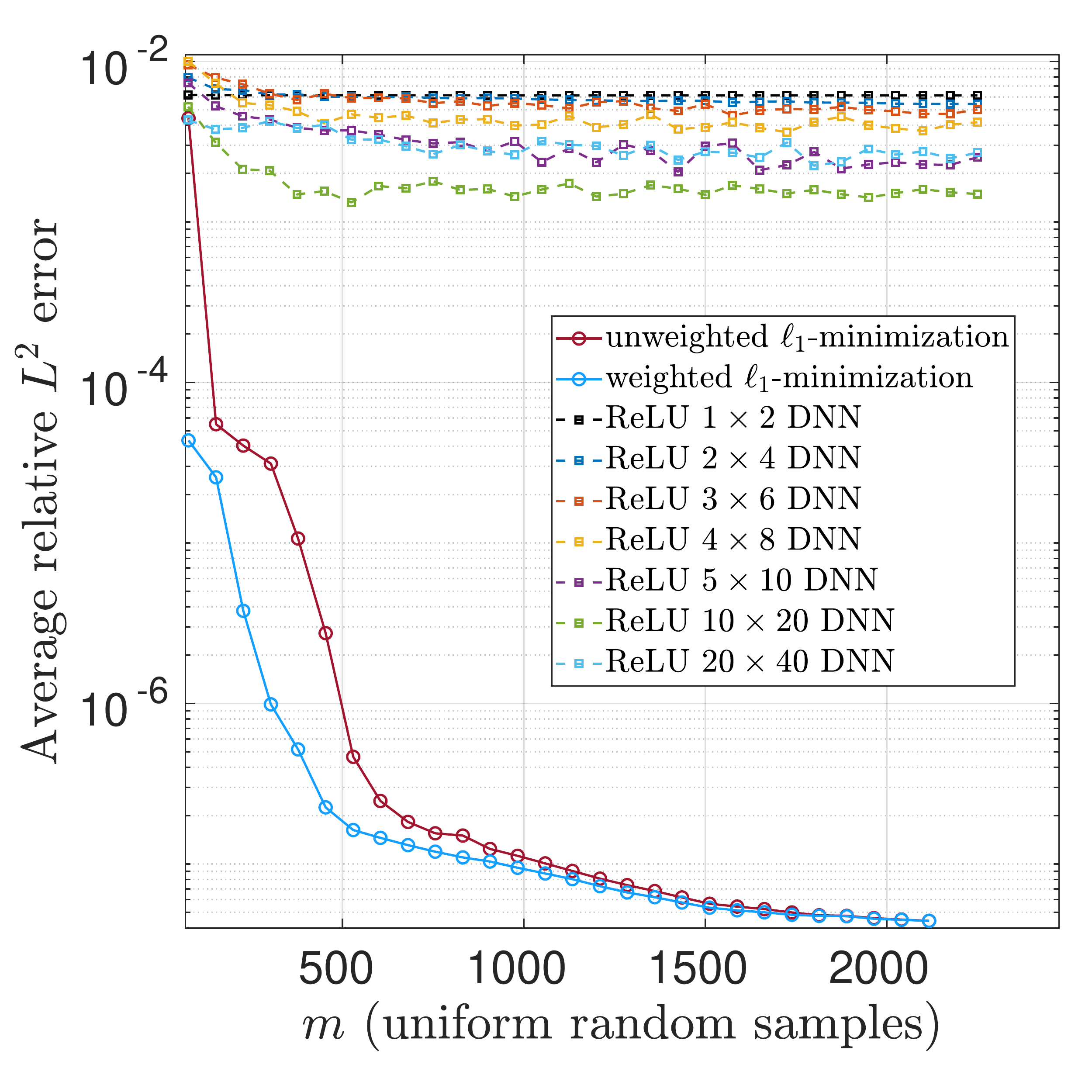}
\includegraphics[width=0.23\paperwidth,clip=true,trim=0mm 0mm 0mm 0mm]{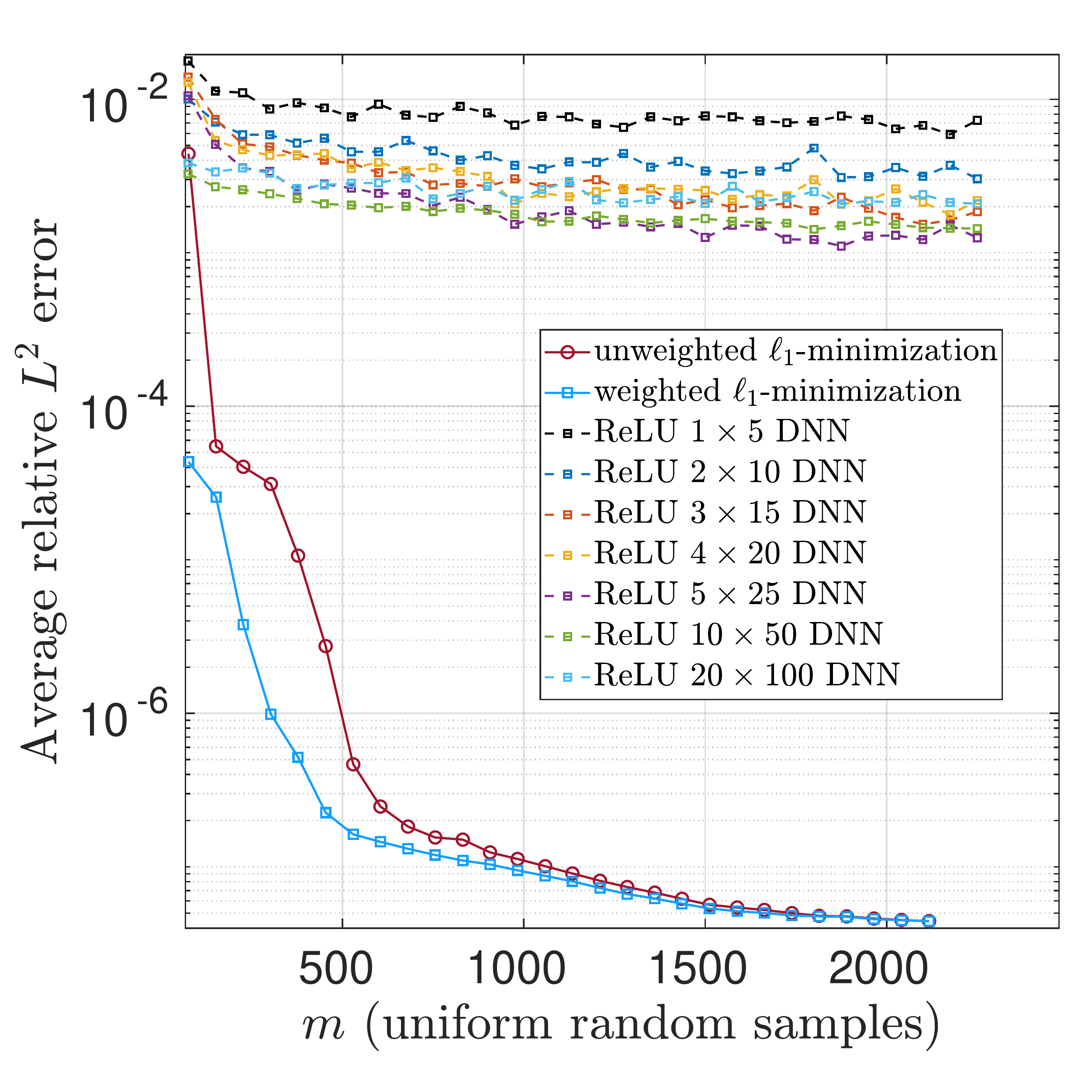}
\includegraphics[width=0.23\paperwidth,clip=true,trim=0mm 0mm 0mm 0mm]{exp_cos_d=8_10x_width_relu_DNNs_v_CS.pdf}
\includegraphics[width=0.23\paperwidth,clip=true,trim=0mm 0mm 0mm 0mm]{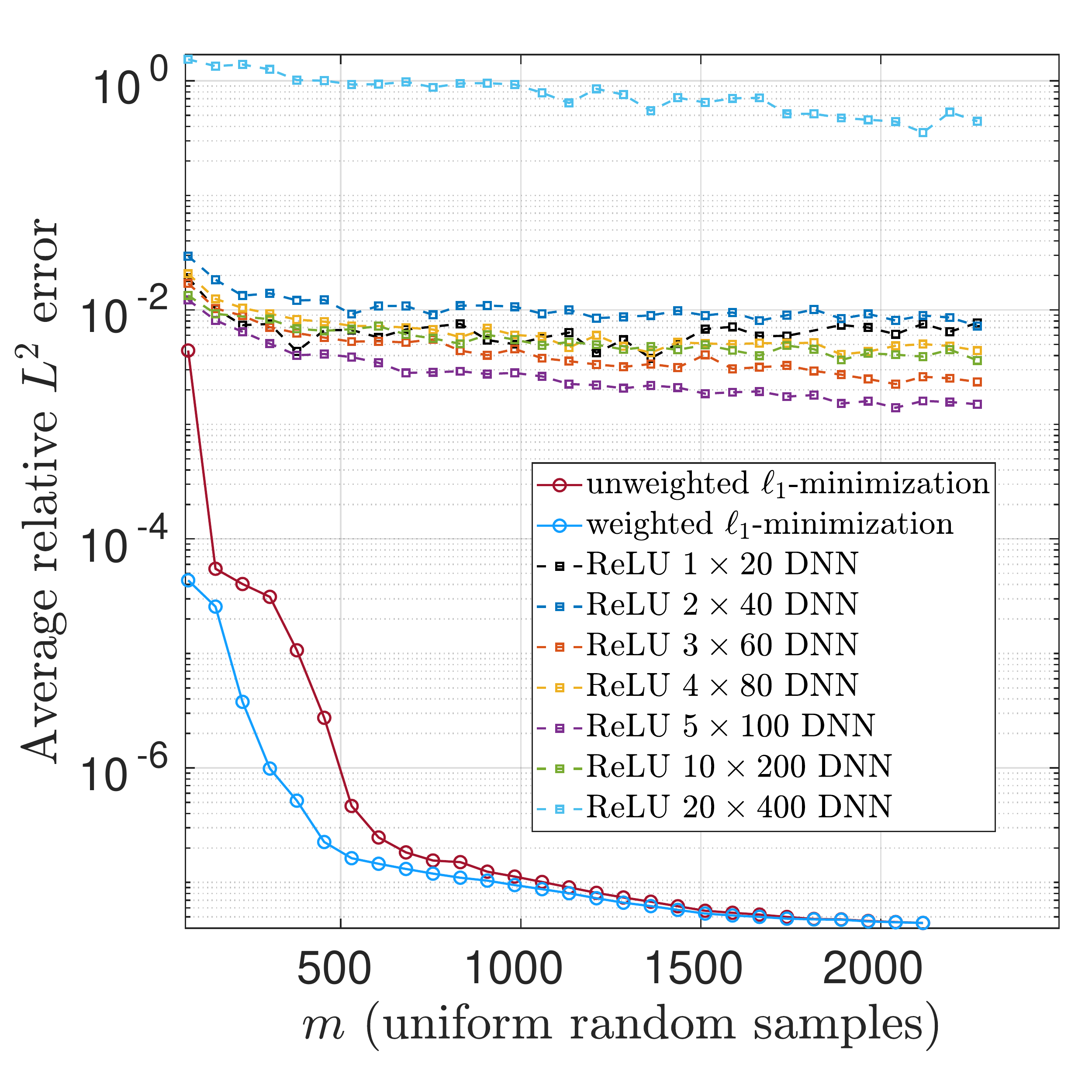}
\includegraphics[width=0.23\paperwidth,clip=true,trim=0mm 0mm 0mm 0mm]{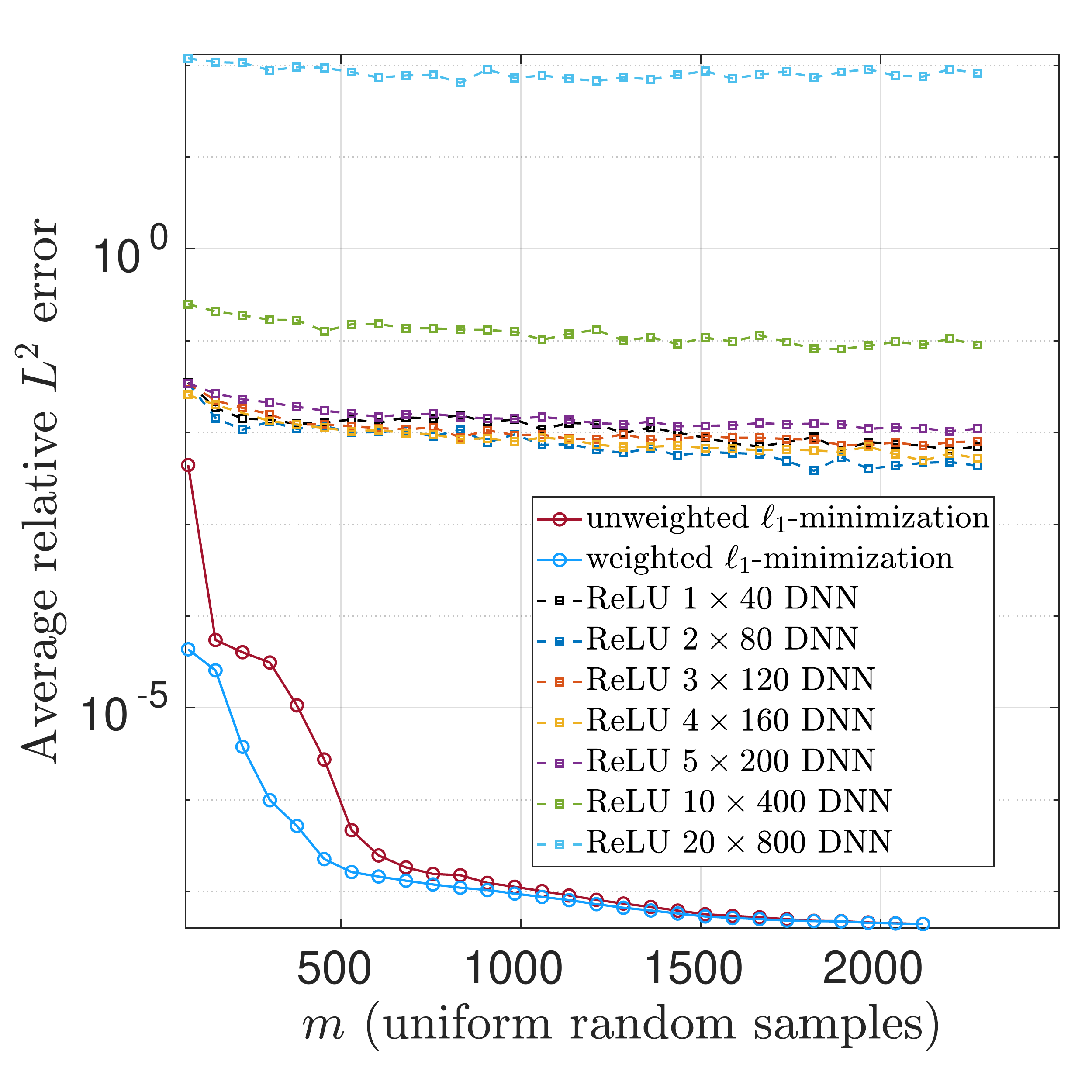}
\raisebox{2.0\height}{
\begin{adjustbox}{max width=0.218\paperwidth}
\begin{tabular}{|l|l|l|} \hline
$\beta$ & best arch. & \# params. \\ \hline
0.5 & $10\times 20$ & $\mathcal{O}(4\times 10^{3})$ \\
0.2 & $5\times 25$ & $\mathcal{O}(3\times 10^{3})$ \\
0.1 & $10\times 100$ & $\mathcal{O}(1\times 10^{5})$ \\
0.05 & $5\times 100$ & $\mathcal{O}(5\times 10^{4})$ \\
0.025 & $2\times 80$ & $\mathcal{O}(1\times 10^{4})$ \\ \hline
\end{tabular}
\end{adjustbox}
}
\end{center}

\vspace{-2mm}
\caption{Comparison of average relative $L^2$ errors w.r.t. number of samples of $f$ from \eqref{eq:exp_cos_func} with $d = 8$ used in training ReLU architectures parameterized with $\beta = L/N$ (hidden layers/nodes per hidden layer) for values {\bf(top-left)} $\beta=0.5$, {\bf(top-middle)} $\beta=0.2$, {\bf(top-right)} $\beta=0.1$, {\bf(bottom-left)} $\beta=0.05$, and {\bf(bottom-middle)} $\beta=0.025$. {\bf(bottom-right)} Table of best-performing architectures for each choice of $\beta$ and number of parameters. CS approximations were computed with the Legendre basis of cardinality $n=3,023$.
}
\label{fig:relu_exp_cos_func_beta_comp}
\end{figure}

Fig.\ \ref{fig:relu_exp_cos_func_weight_comp} displays the average absolute maximum of the weights and biases for each architecture. There we observe that the weights and biases of narrower network architectures, e.g.\ $\beta \in \{0.1,0.2,0.5\}$, remain relatively small as we increase the number of samples.
However, as we increase the width of the networks to values corresponding to $\beta=0.05$ and $\beta=0.025$, we begin to see the weights growing with depth as we train on larger sample sets. 

\begin{figure}[ht]
\begin{center}
\includegraphics[width=0.23\paperwidth,clip=true,trim=0mm 0mm 0mm 0mm]{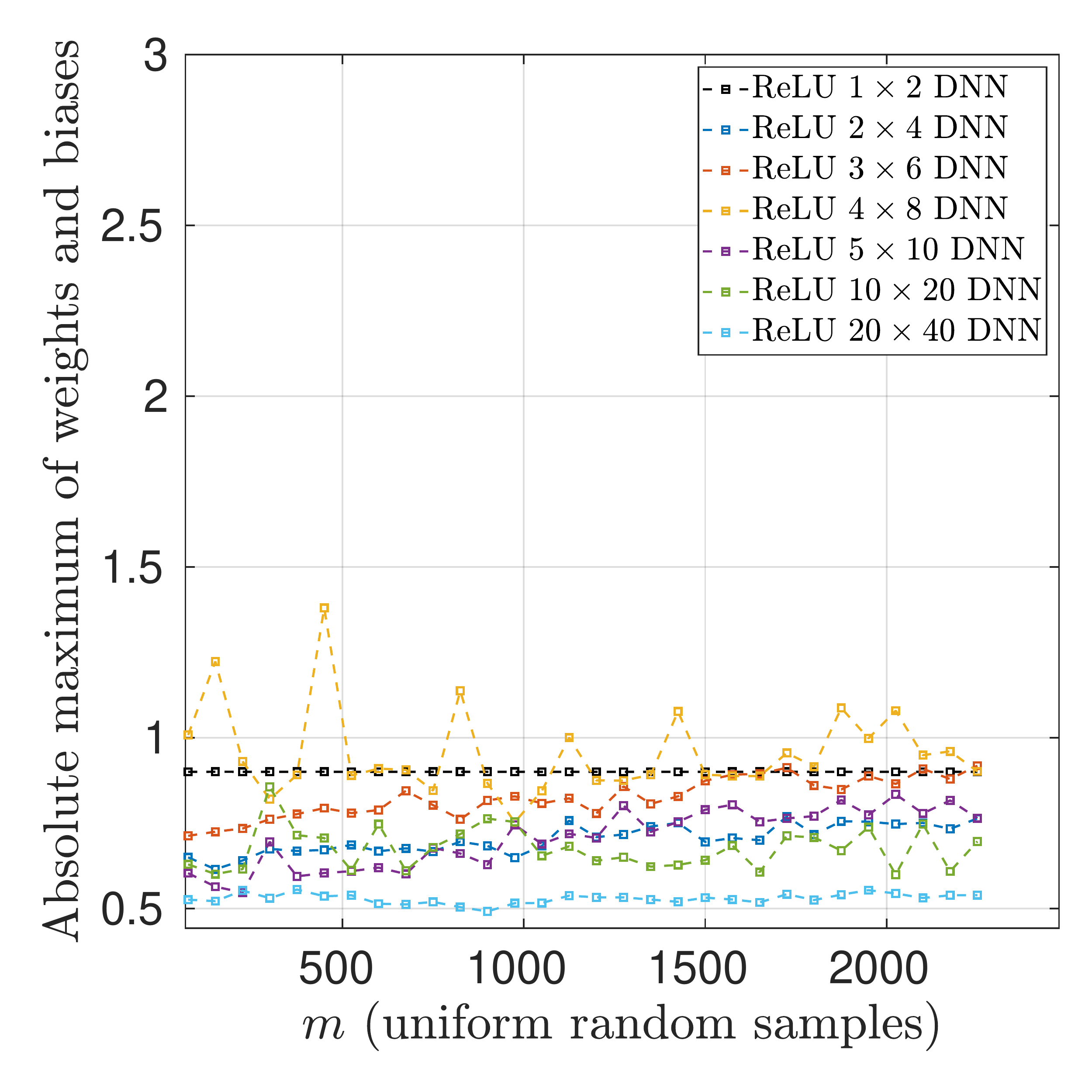}
\includegraphics[width=0.23\paperwidth,clip=true,trim=0mm 0mm 0mm 0mm]{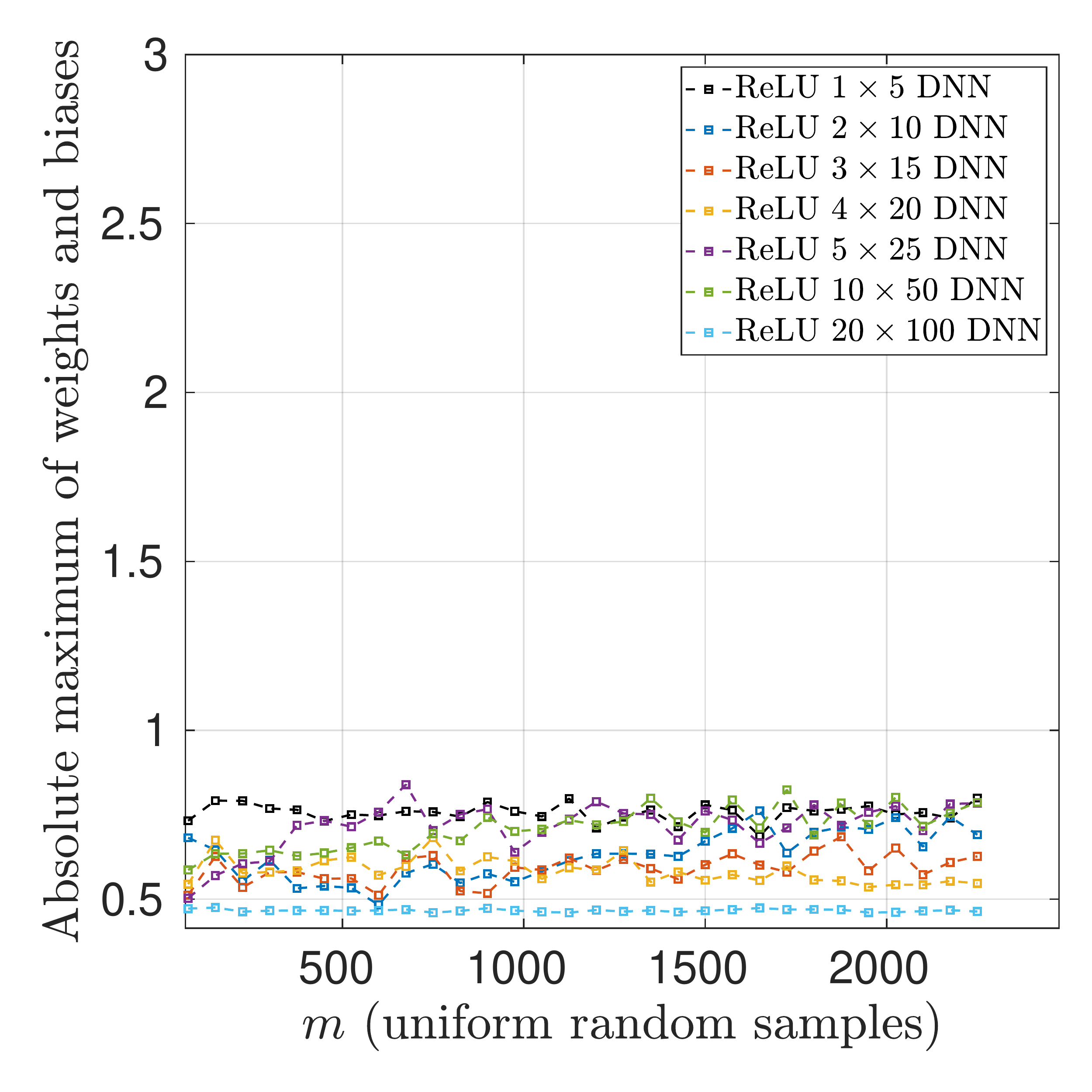}
\includegraphics[width=0.23\paperwidth,clip=true,trim=0mm 0mm 0mm 0mm]{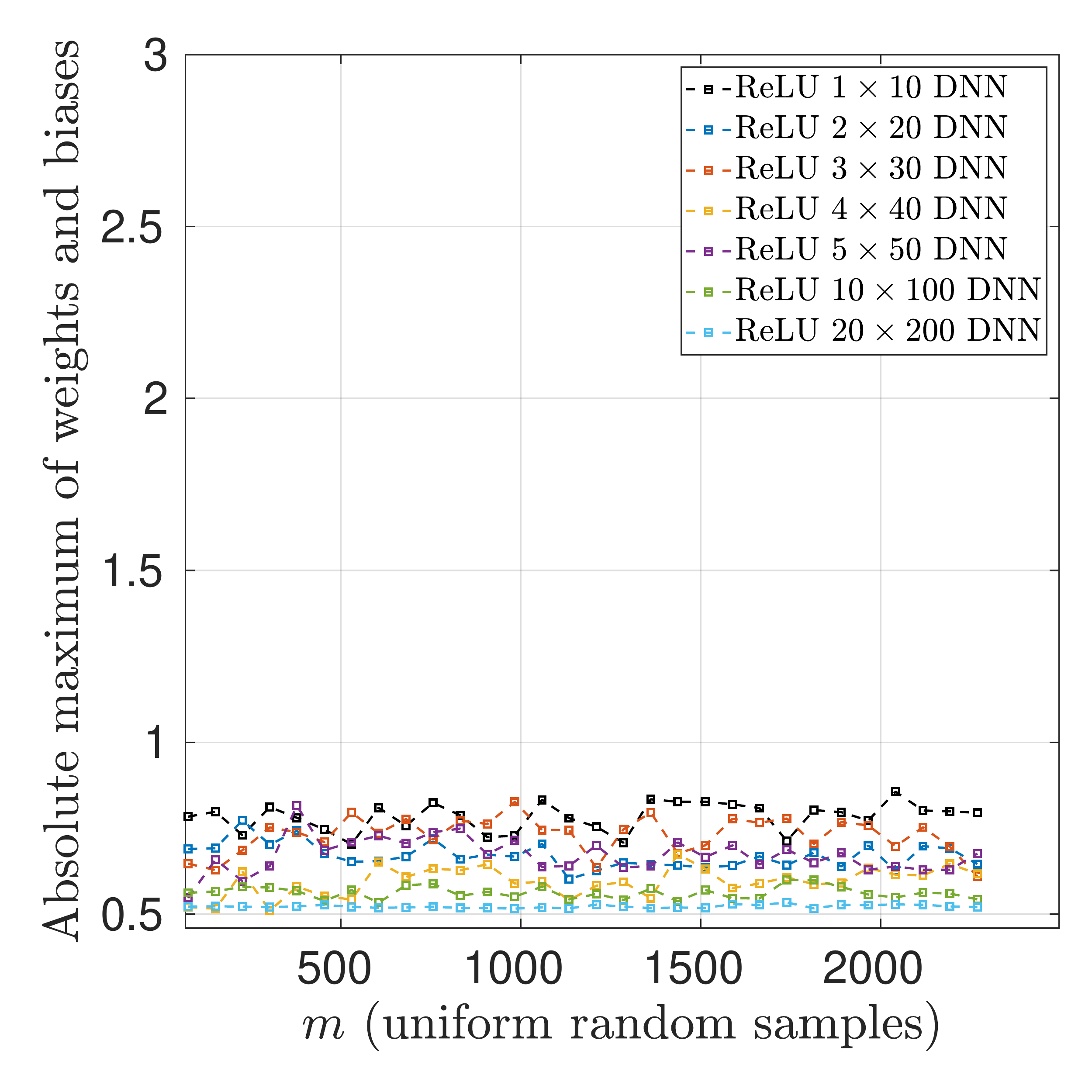}
\includegraphics[width=0.23\paperwidth,clip=true,trim=0mm 0mm 0mm 0mm]{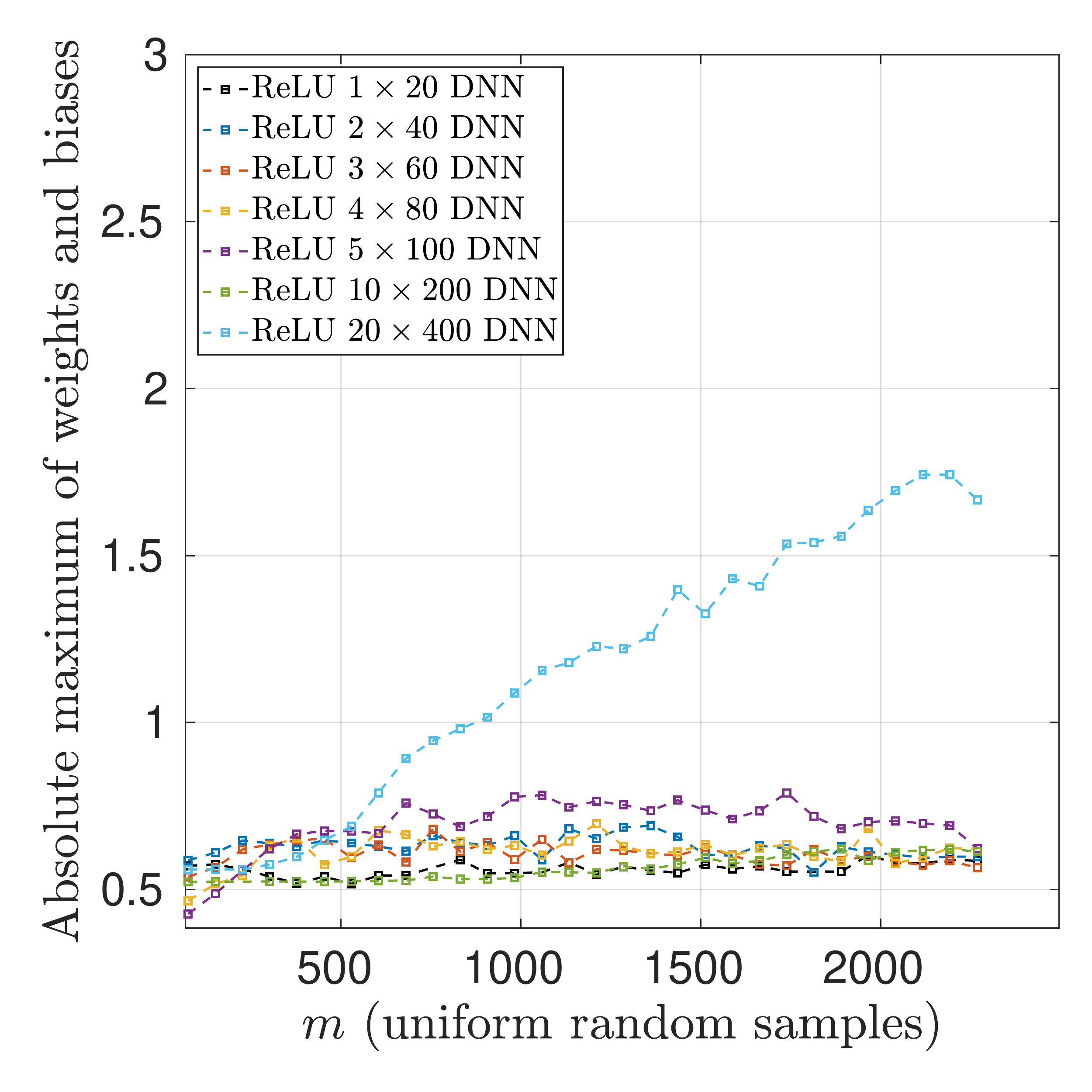}
\includegraphics[width=0.23\paperwidth,clip=true,trim=0mm 0mm 0mm 0mm]{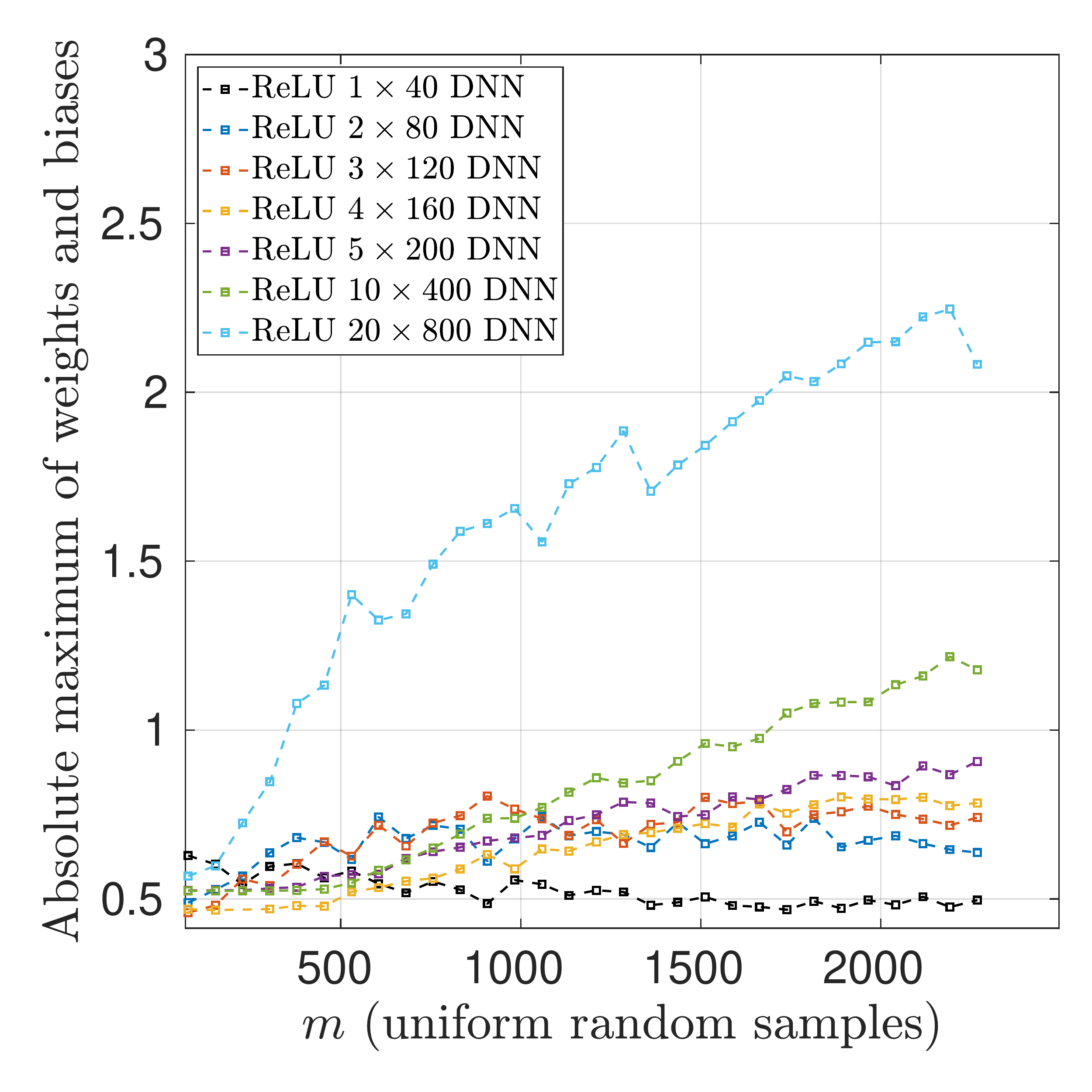}
\end{center}

\vspace{-2mm}
\caption{Comparison of average absolute maximum of weights and biases w.r.t. number of samples of $f$ from \eqref{eq:exp_cos_func} with $d=8$ used in training ReLU architectures parameterized with $\beta = L/N$ (hidden layers/nodes per hidden layer) for values {\bf(top-left)} $\beta=0.5$, {\bf(top-middle)} $\beta=0.2$, {\bf(top-right)} $\beta=0.1$, {\bf(bottom-left)} $\beta=0.05$, and {\bf(bottom-right)} $\beta=0.025$.}
\label{fig:relu_exp_cos_func_weight_comp}
\end{figure}

Fig.\ \ref{fig:relu_exp_cos_func_time_comp} displays the average run time of the {\tt Adam} optimizer in training our DNNs as we increase the number of training samples.
For the narrower and shallower DNN architectures, we see the transfer time bottleneck between the CPU and GPU is yielding an effective linear scaling with respect to samples in the average training time. This effect is also present in the deeper and wider networks, though at a diminished amount due to the fast convergence of the {\tt Adam} optimizer on larger networks.

Next we present results on the function 
\begin{align}
\label{eq:slower_decay_rational_func}
f(x) = \left( \frac{\prod_{k=1}^{\lceil d/2 \rceil}(1+4^k x_k^2)}{\prod_{k=\lceil d/2 \rceil + 1}^d (100 + 5 x_k)} \right)^{1/d},
\end{align}
also from \cite{ChkifaDexterTranWebster18}. This function has anisotropic dependence on its parameters and is less smooth than the previous example. Consequently its Legendre coefficients decay more slowly, as seen in the left and middle plots of Fig.\ \ref{fig:slower_decay_rational_func}, and the right plot shows the CS approximations only achieve an error of roughly 0.02.
The best performing ReLU DNNs achieve an error that approximately five times smaller.

\begin{figure}[ht]
\begin{center}
\includegraphics[width=0.23\paperwidth,clip=true,trim=0mm 0mm 0mm 0mm]{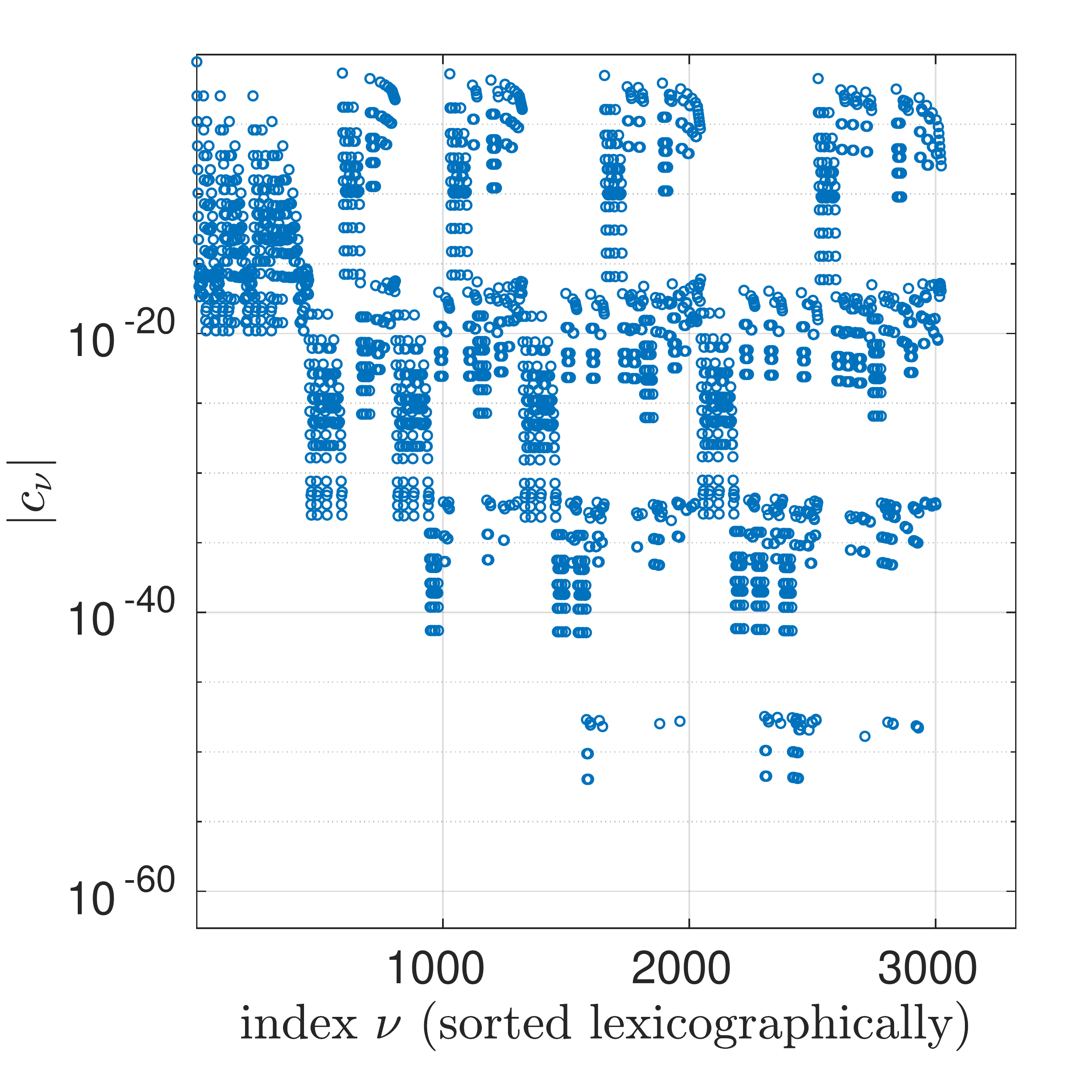}
\includegraphics[width=0.23\paperwidth,clip=true,trim=0mm 0mm 0mm 0mm]{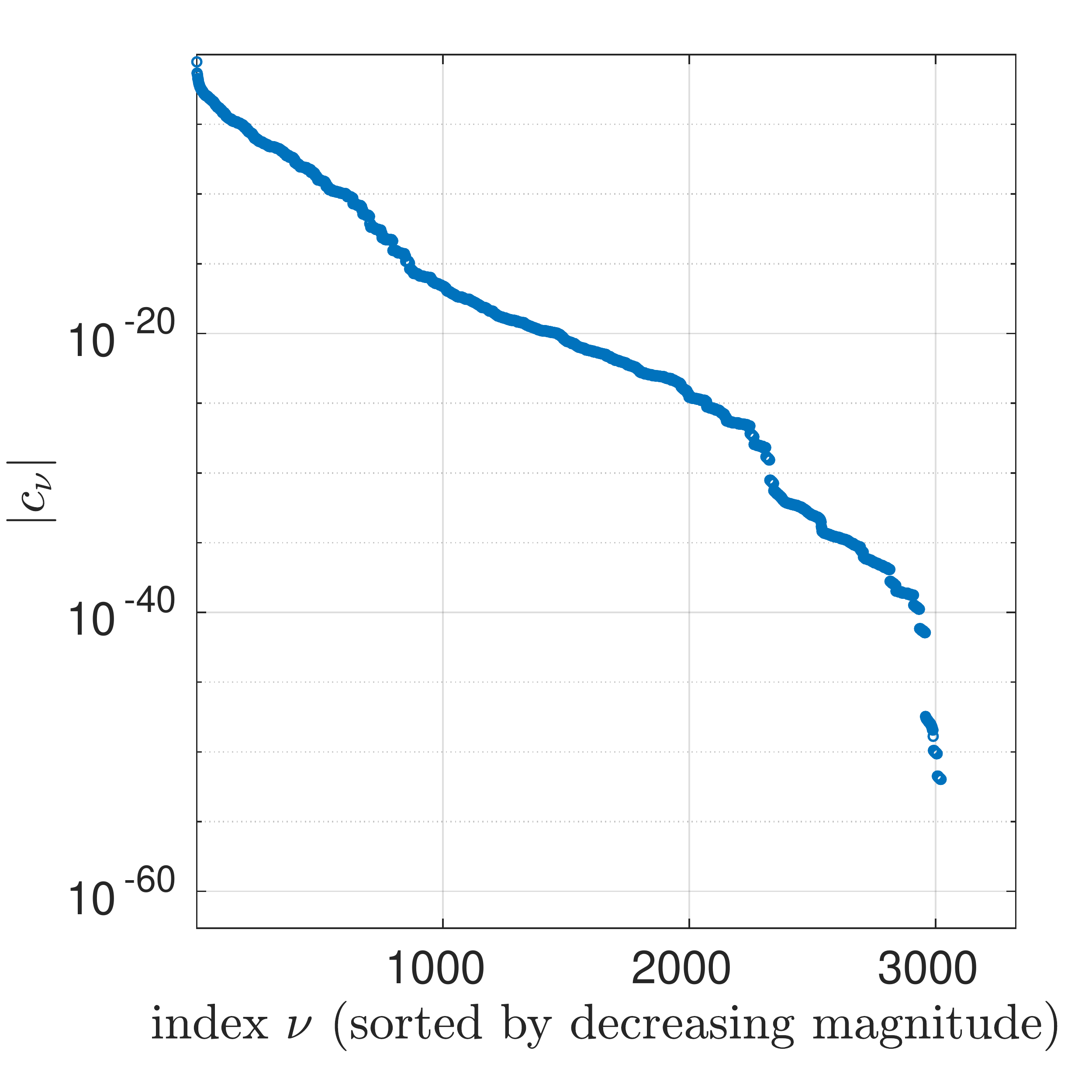}
\includegraphics[width=0.23\paperwidth,clip=true,trim=0mm 0mm 0mm 0mm]{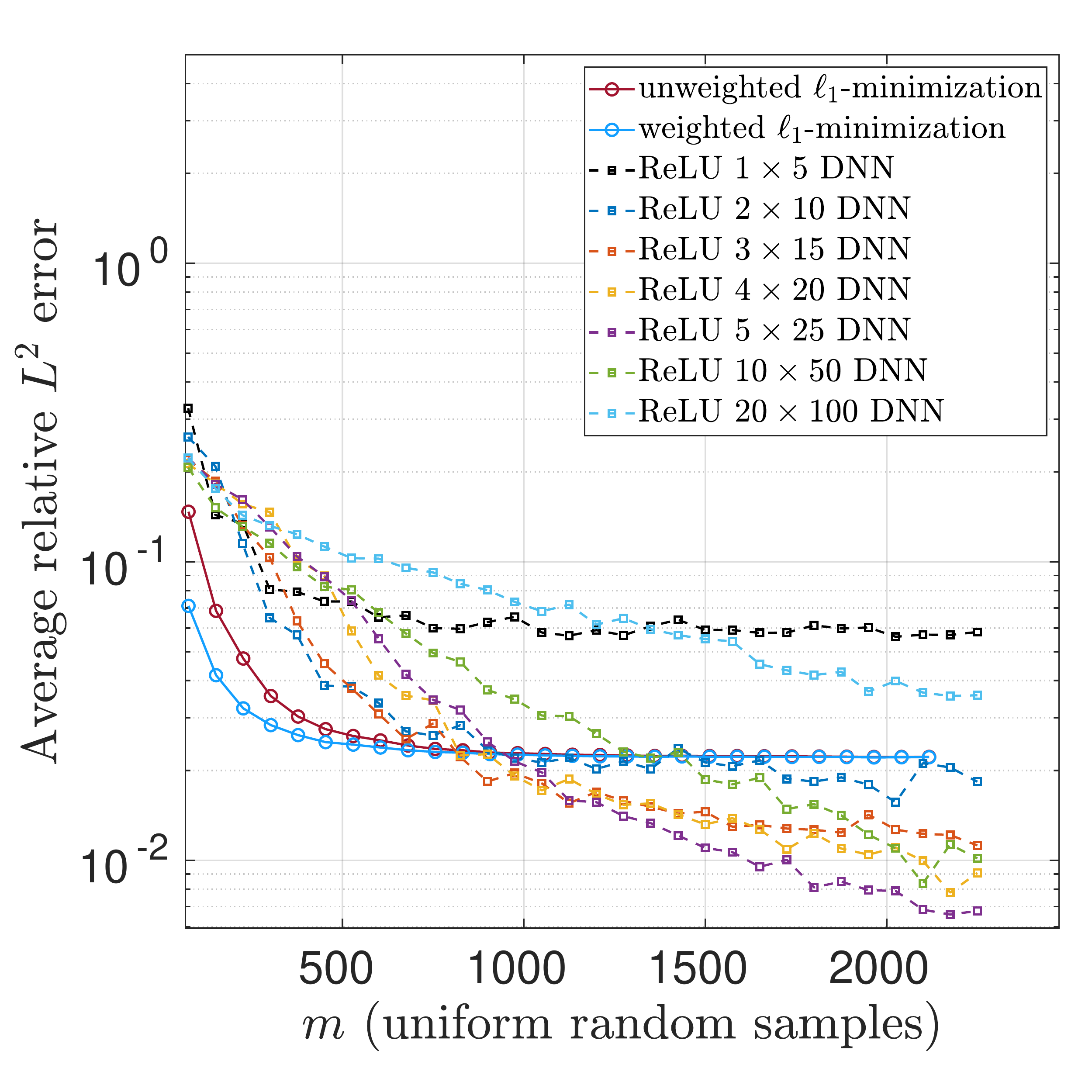}
\end{center}

\vspace{-2mm}
\caption{Legendre coefficients of $f$ from \eqref{eq:slower_decay_rational_func} with $d=8$, sorted {\bf (left)} lexicographically and {\bf (center)} by decreasing magnitude. {\bf (right)} Average relative $L^2$ error w.r.t. number of samples used in training. CS approximations were computed with the Legendre basis of cardinality $n=3,023$.}
\label{fig:slower_decay_rational_func}
\end{figure}

Fig.\ \ref{fig:slower_decay_rational_func_beta_comp} displays the effect of increasing the width of the DNNs as before. 
There, as in the previous example, we observe best performance is achieved by networks that are both wide and deep, e.g.\ the ReLU $5\times 25$ and $10\times 20$ networks. 
We also again observe that for wider networks, architectures with fewer hidden layers perform the best.
Nonetheless, comparing the results between the ReLU $1\times 20$ and $1\times 40$ DNNs and those achieved by the $10\times 20$, $5\times 25$, and $2\times 20$ DNNs, we see that far better performance is achieved with fewer samples by the deeper and narrower DNN architectures.

\begin{figure}[ht]
\begin{center}
\includegraphics[width=0.23\paperwidth,clip=true,trim=0mm 0mm 0mm 0mm]{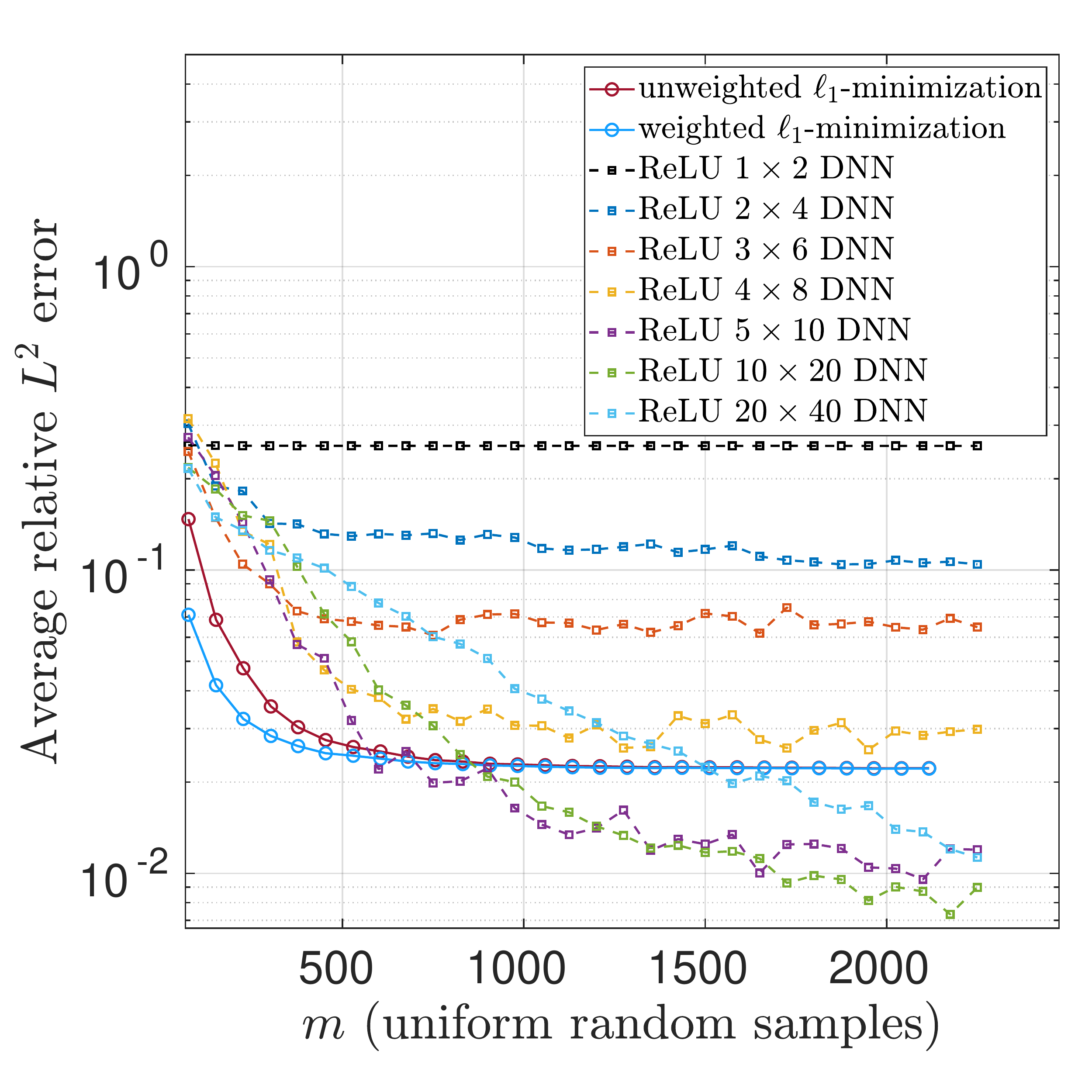}
\includegraphics[width=0.23\paperwidth,clip=true,trim=0mm 0mm 0mm 0mm]{slower_decay_rational_d=8_L2_err_CS_vs_relu_5x_width_DNNs.pdf}
\includegraphics[width=0.23\paperwidth,clip=true,trim=0mm 0mm 0mm 0mm]{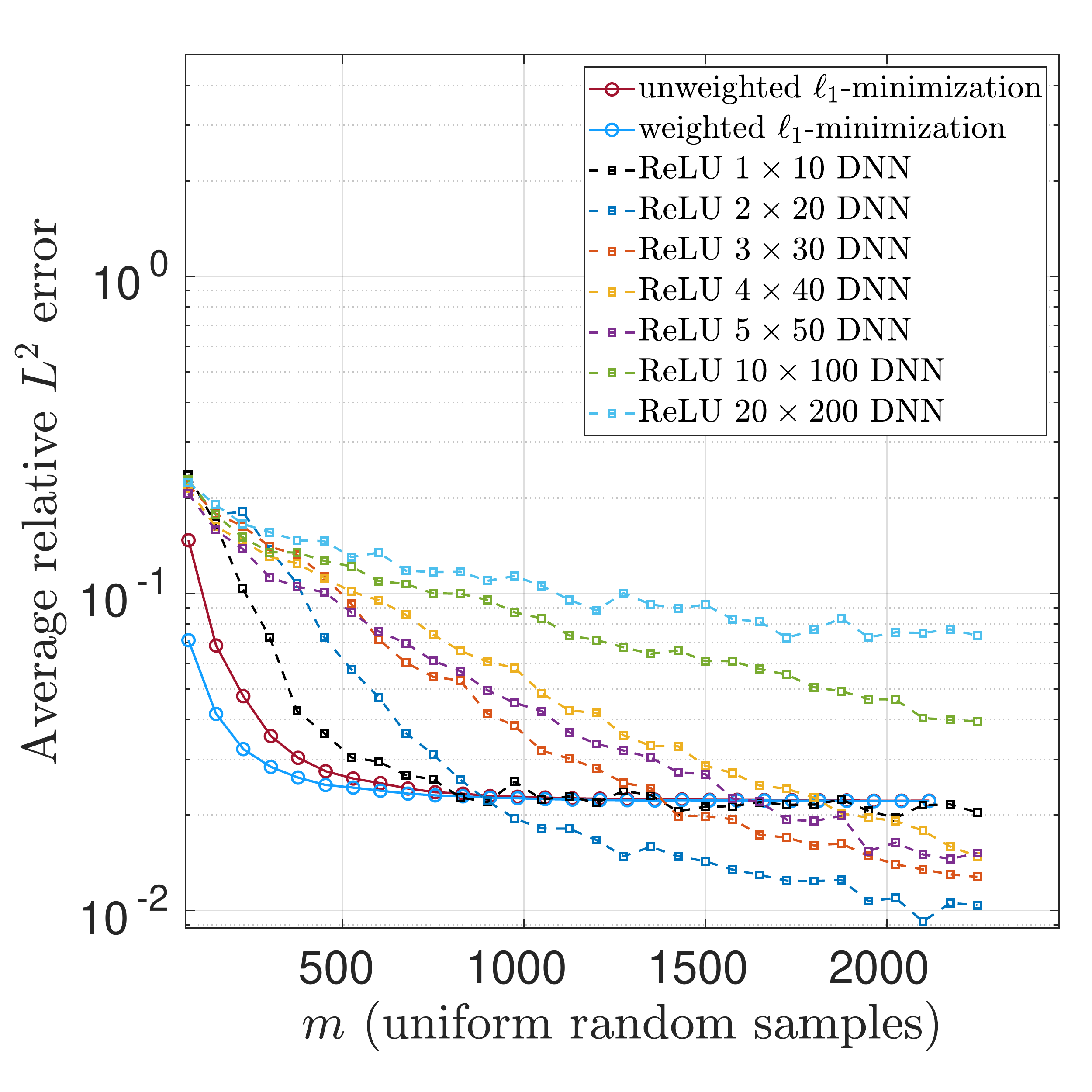}
\includegraphics[width=0.23\paperwidth,clip=true,trim=0mm 0mm 0mm 0mm]{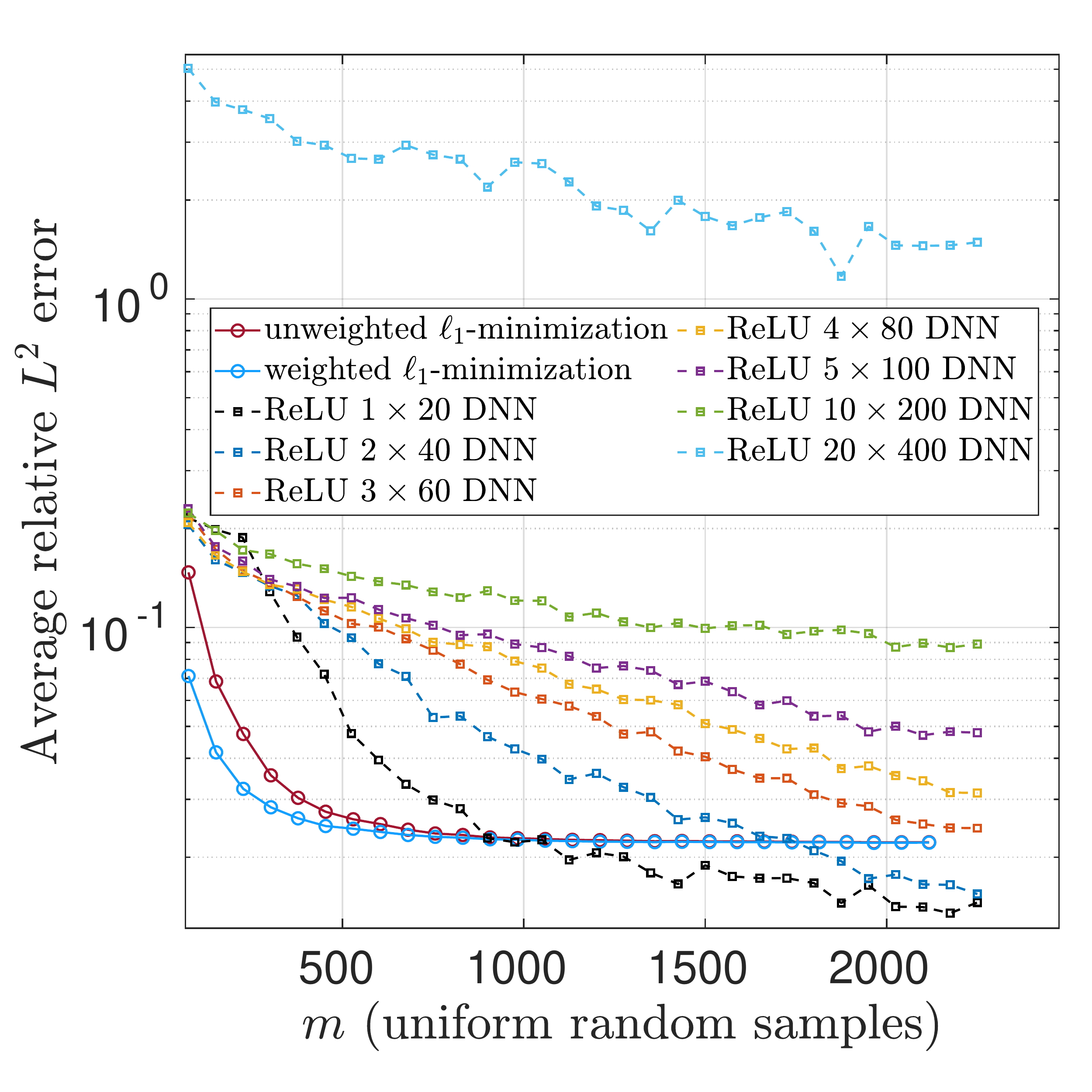}
\includegraphics[width=0.23\paperwidth,clip=true,trim=0mm 0mm 0mm 0mm]{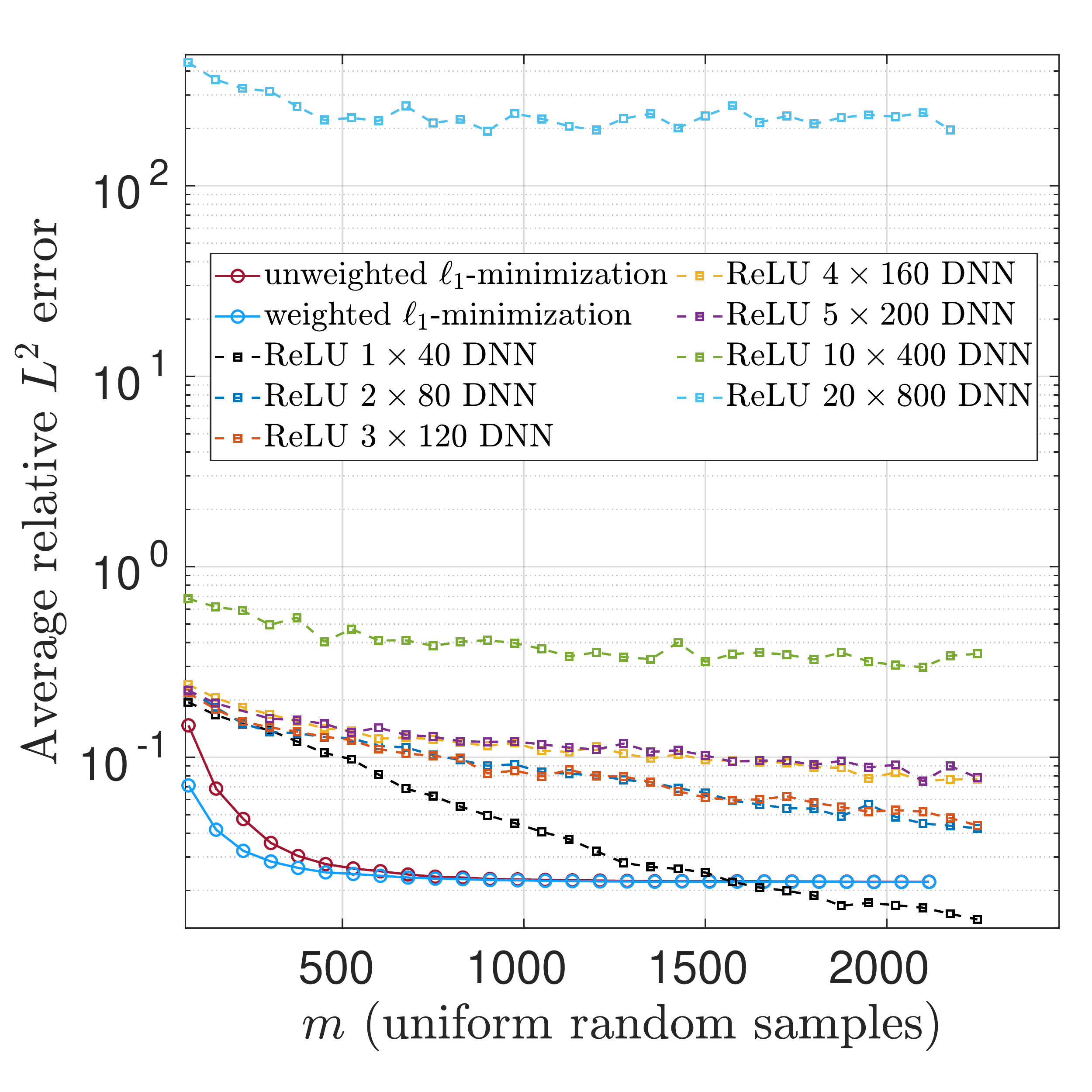}
\raisebox{2.0\height}{
\begin{adjustbox}{max width=0.218\paperwidth}
\begin{tabular}{|l|c|r|} \hline
$\beta$ & best arch. & \# params. \\ \hline
0.5 & $10\times 20$ & $\mathcal{O}(4\times 10^{3})$ \\
0.2 & $5\times 25$ & $\mathcal{O}(3\times 10^{3})$ \\
0.1 & $2\times 20$ & $\mathcal{O}(8\times 10^{2})$ \\
0.05 & $1\times 20$ & $\mathcal{O}(4\times 10^{2})$ \\
0.025 & $1\times 40$ & $\mathcal{O}(2\times 10^{3})$ \\ \hline
\end{tabular}
\end{adjustbox}
}
\end{center}

\vspace{-2mm}
\caption{Comparison of average relative $L^2$ errors w.r.t. number of samples of $f$ from \eqref{eq:slower_decay_rational_func} with $d = 8$ used in training ReLU architectures parameterized with $\beta = L/N$ (hidden layers/nodes per hidden layer) for values {\bf(top-left)} $\beta=0.5$, {\bf(top-middle)} $\beta=0.2$, {\bf(top-right)} $\beta=0.1$, {\bf(bottom-left)} $\beta=0.05$, and {\bf(bottom-middle)} $\beta=0.025$. {\bf(bottom-right)} Table of best-performing architectures for each choice of $\beta$ and number of parameters. CS approximations were computed with the Legendre basis of cardinality $n=3,023$.
}
\label{fig:slower_decay_rational_func_beta_comp}
\end{figure}

Figure \ref{fig:relu_slower_decay_rational_func_weight_comp} compares the average absolute maximum of the weights and biases of the trained DNNs.
There, as in the one-dimensional examples, we observe that the weights and biases of DNNs trained on function data from the less smooth function \eqref{eq:slower_decay_rational_func} are on average larger than those obtained after training on the smoother function \eqref{eq:exp_cos_func}.

\begin{figure}[ht]
\begin{center}
\includegraphics[width=0.23\paperwidth,clip=true,trim=0mm 0mm 0mm 0mm]{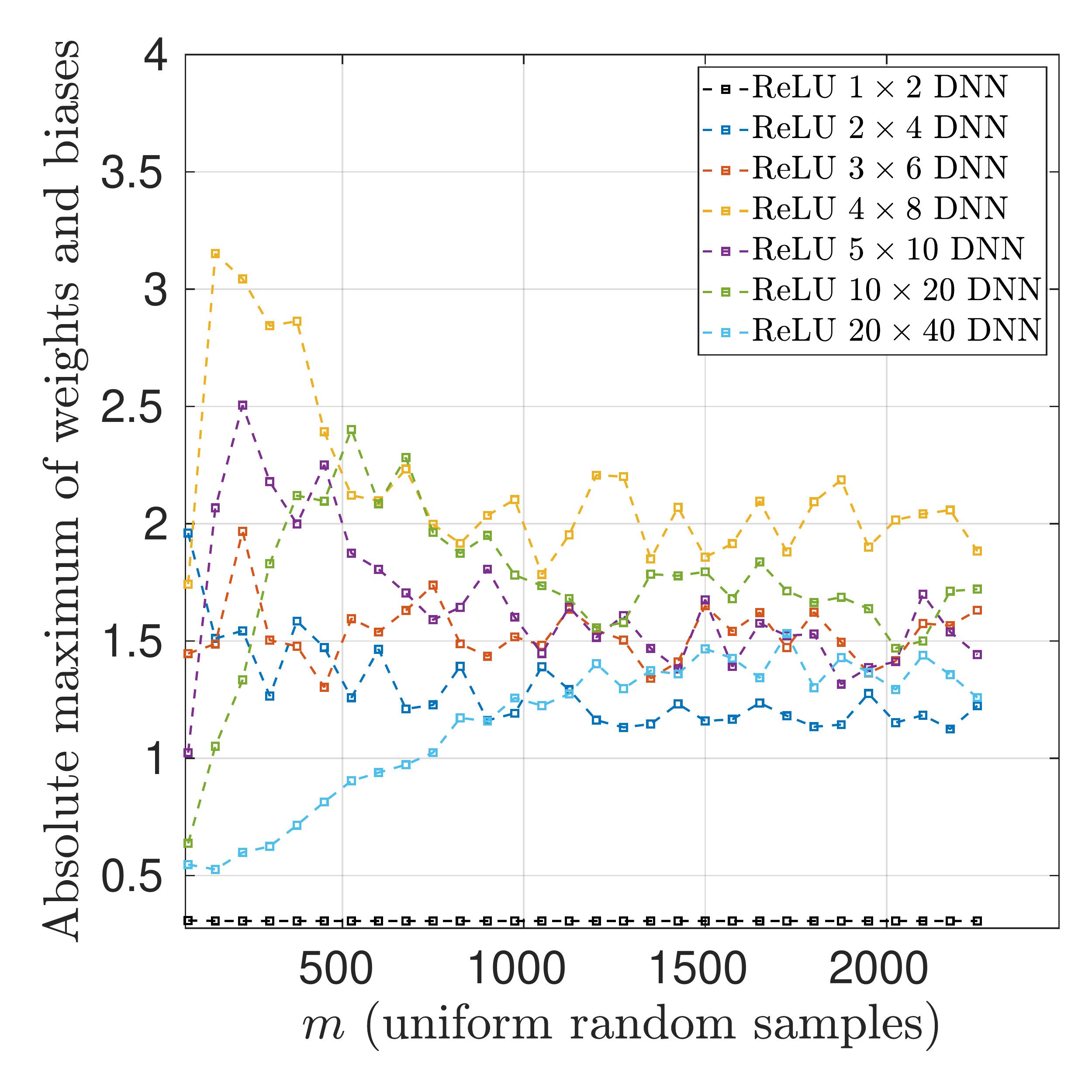}
\includegraphics[width=0.23\paperwidth,clip=true,trim=0mm 0mm 0mm 0mm]{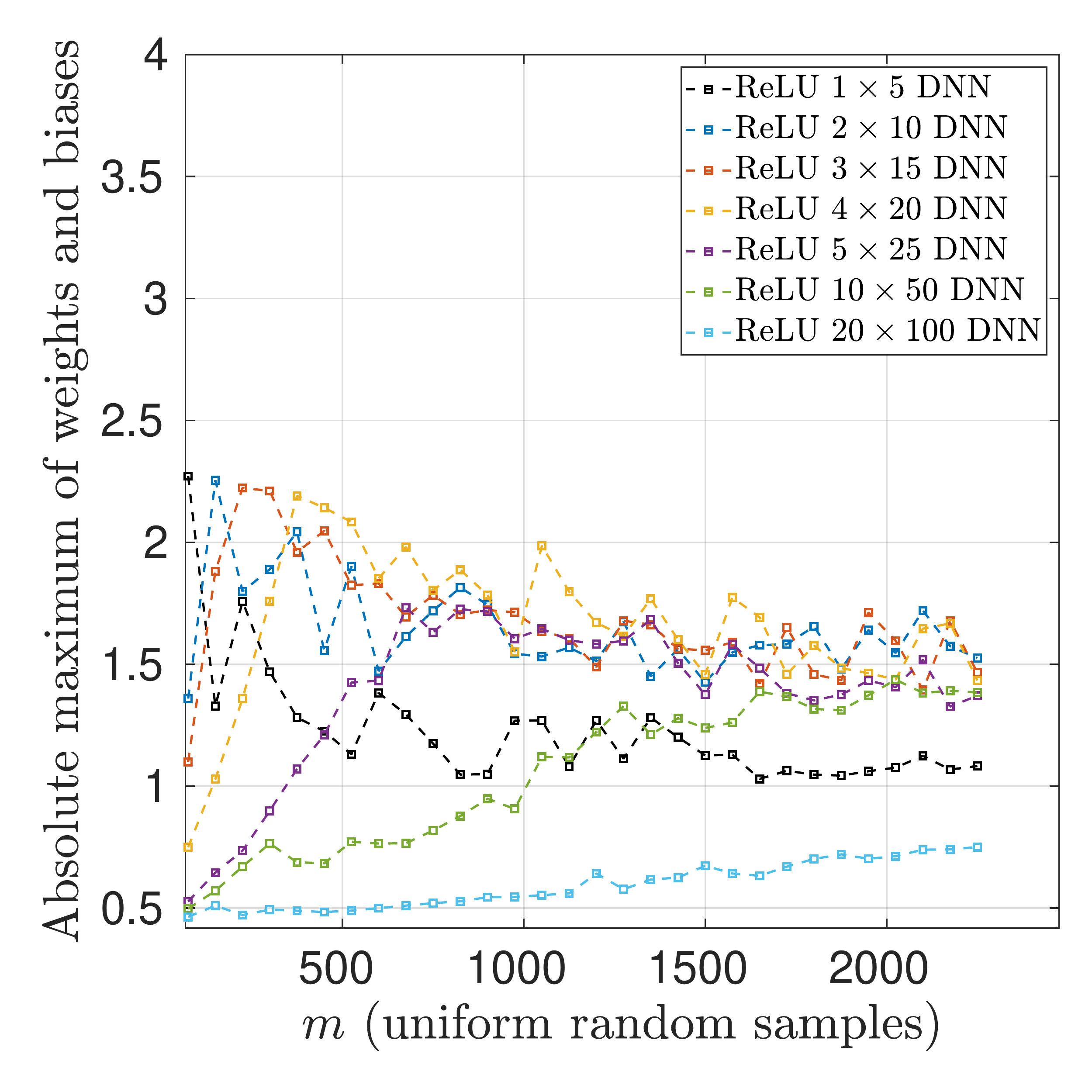}
\includegraphics[width=0.23\paperwidth,clip=true,trim=0mm 0mm 0mm 0mm]{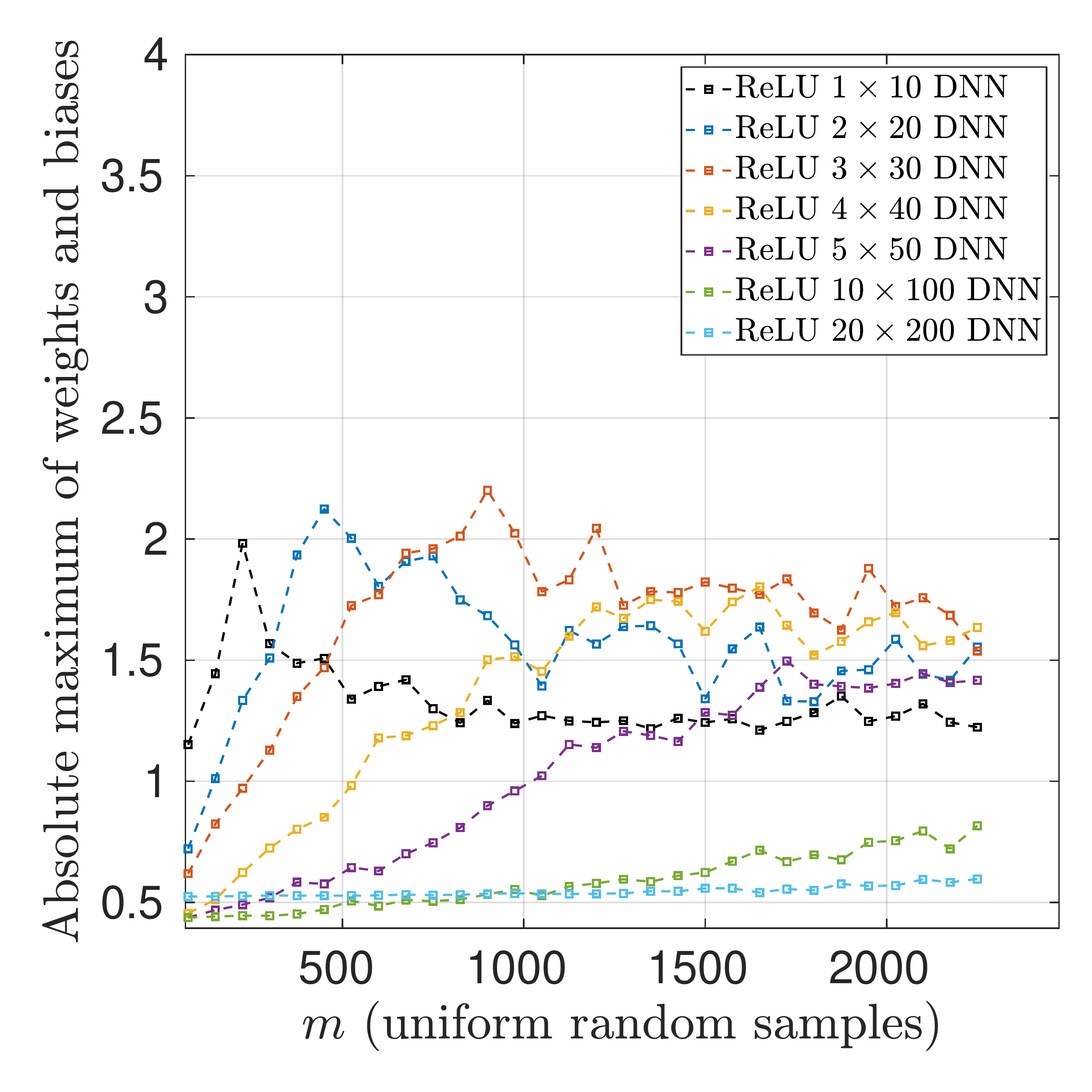}
\includegraphics[width=0.23\paperwidth,clip=true,trim=0mm 0mm 0mm 0mm]{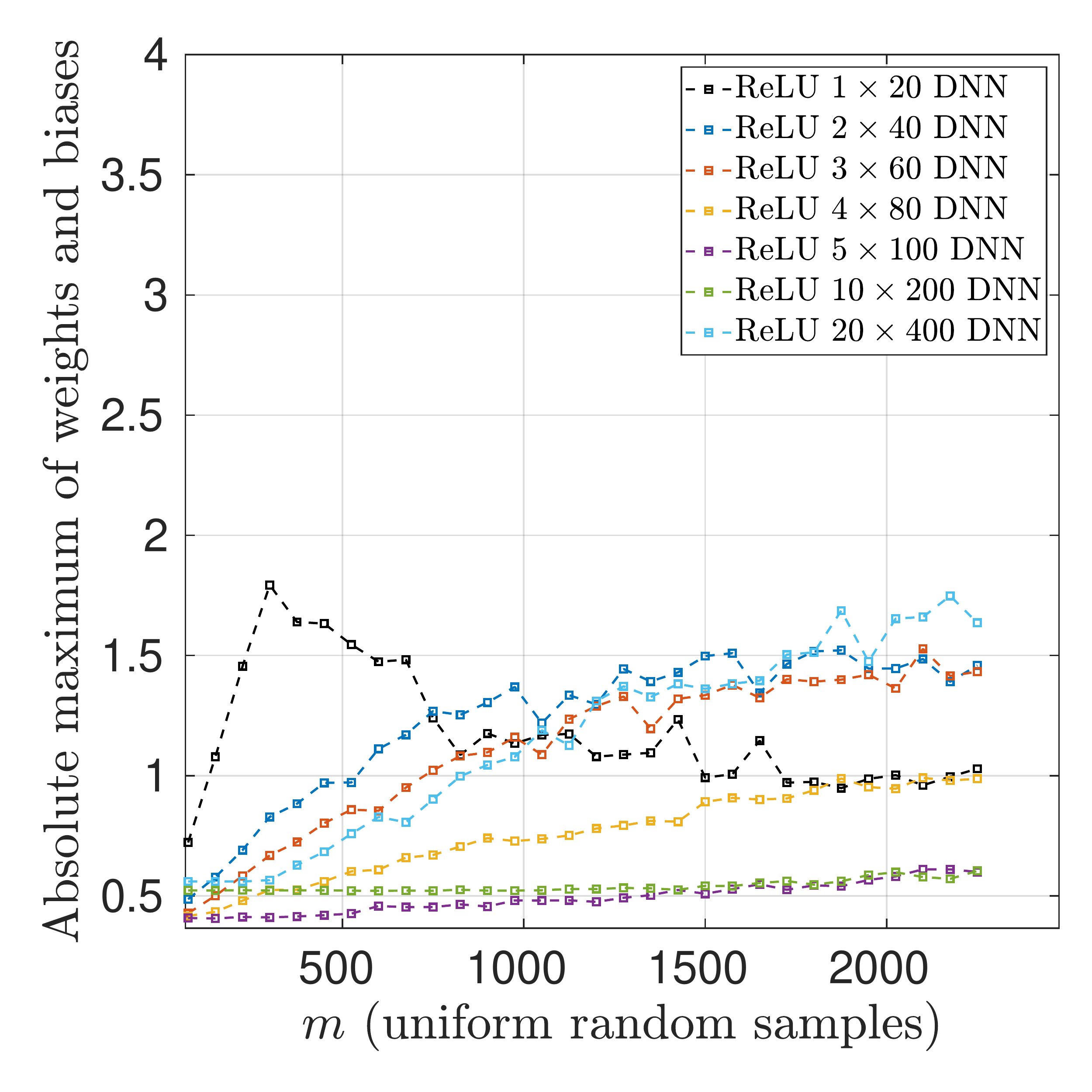}
\includegraphics[width=0.23\paperwidth,clip=true,trim=0mm 0mm 0mm 0mm]{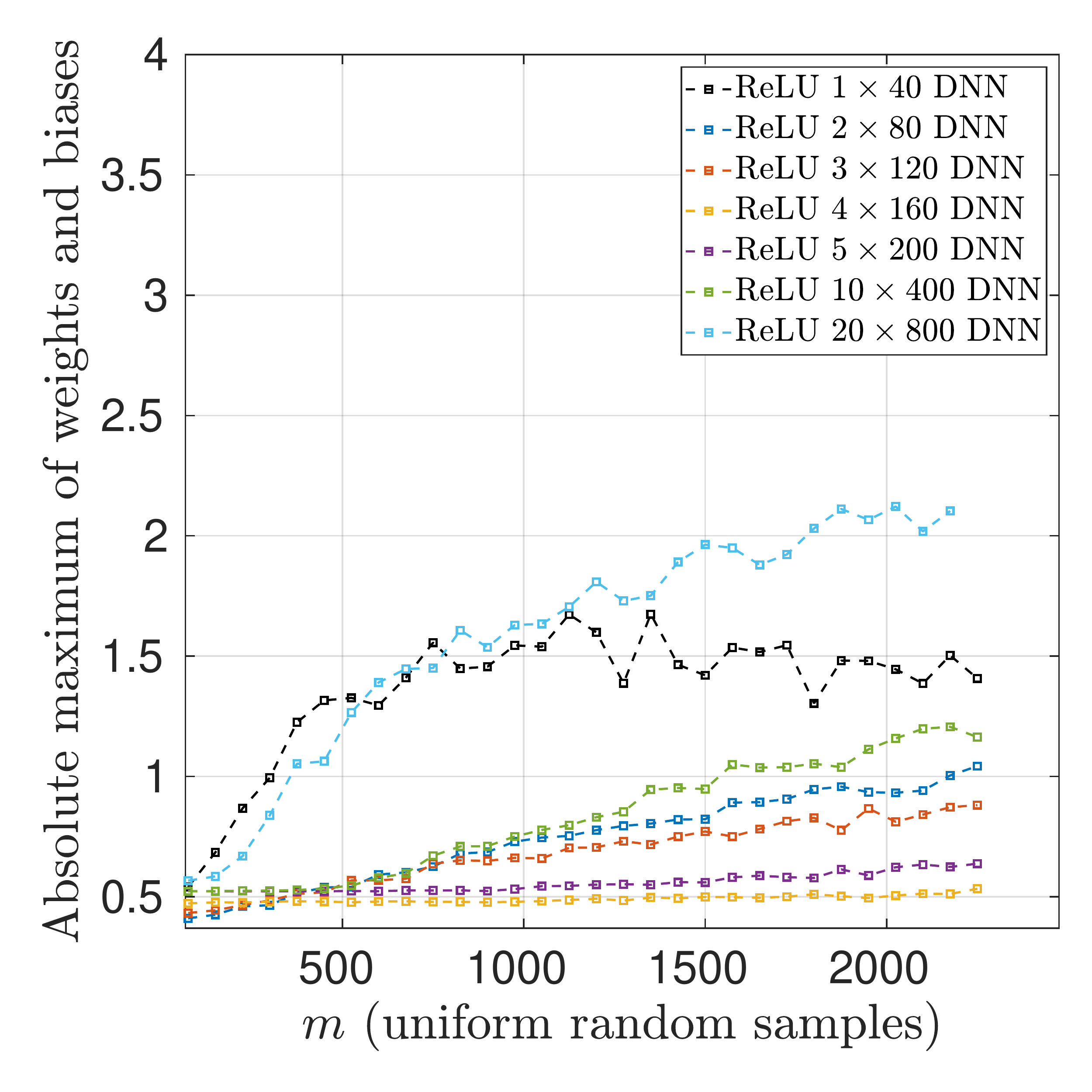}
\end{center}

\vspace{-2mm}
\caption{Comparison of average absolute maximum of weights and biases w.r.t. number of samples of $f$ from \eqref{eq:slower_decay_rational_func} with $d=8$ used in training ReLU architectures parameterized with $\beta = L/N$ (hidden layers/nodes per hidden layer) for values {\bf(top-left)} $\beta=0.5$, {\bf(top-middle)} $\beta=0.2$, {\bf(top-right)} $\beta=0.1$, {\bf(bottom-left)} $\beta=0.05$, and {\bf(bottom-right)} $\beta=0.025$.  }
\label{fig:relu_slower_decay_rational_func_weight_comp}
\end{figure}

Fig.\ \ref{fig:relu_slower_decay_rational_func_time_comp} again displays the average run time of the {\tt Adam} optimizer in training the DNNs as we increase the number of samples. There we observe similar patterns to the timing results obtained on function \eqref{eq:exp_cos_func}, e.g., linear scaling in the runtimes for narrower architectures due to the CPU-GPU transfer bottleneck. However, we also observe longer training time in general in approximating the function \eqref{eq:slower_decay_rational_func} compared to Fig.\ \ref{fig:relu_exp_cos_func_time_comp}, suggesting the difficulty of approximating less smooth functions can impact training time.

%----------------------------------------------------
\subsection{Piecewise continuous functions}
\label{subsec:piecewise_numex}
%----------------------------------------------------

In this section, we present results on approximating piecewise continuous functions. We consider the function
\begin{align}
\label{eq:halfspace_func}
f(x_1,\ldots,x_d) = \mathbbm{1}_{K}(x_1,\ldots,x_d), \qquad \mathrm{with} \qquad K = \{\bz \in \R^d: z_1+\ldots +z_d \ge 0 \}.
\end{align}
where $\mathbbm{1}_{K}(x_1,\ldots,x_d) = 1$ if $(x_1,\ldots,x_d)\in K$ and 0 otherwise.
In one dimension $K=\{x\in \R: x\ge 0\}$, so that $f(x) = \mathbbm{1}_{\{x\ge 0\}}(x)$.  
We note this simple discontinuous function equally splits the data between values in $\{0, 1\}$ in arbitrary dimension $d$, and that the separating hyperplane between the sets $K$ and $K^c$ is not aligned to any particular axis. In $d > 1$ dimensions, the CS approximation fails to converge since the coefficients are not sufficiently summable.
On the other hand, the work \cite{Petersen2018} shows the Heaviside function in $d$ dimensions given by $H(x) = \mathbbm{1}_{[0,\infty)\times \R^{d-1}}(x)$ can be approximated to $L^2$ accuracy $\varepsilon^{1/2}$ by a 2-layer ReLU network with five nonzero weights taking value in $\{\varepsilon^{-1},1,-1\}$.
As the function \eqref{eq:halfspace_func} is a rotation of the $d$-dimensional Heaviside function, the result holds in this case as well.

Fig.\ \ref{fig:halfspace_func} displays the results of approximating this function in 1, 2, and 4 dimensions.
There we observe the lack of convergence of the CS approximations when $d>1$. 
While the DNNs perform better than the CS approximations for this problem, none of the achieved results obtain more than 2 digits of accuracy.
In the right plot of Fig.\ \ref{fig:halfspace_func}, we see the \textit{double descent} behavior that has been observed in other works, see, e.g., \cite[\S 7]{Nakkiran2019}.

\begin{figure}[ht]
\begin{center}
\includegraphics[width=0.23\paperwidth,clip=true,trim=0mm 0mm 0mm 0mm]{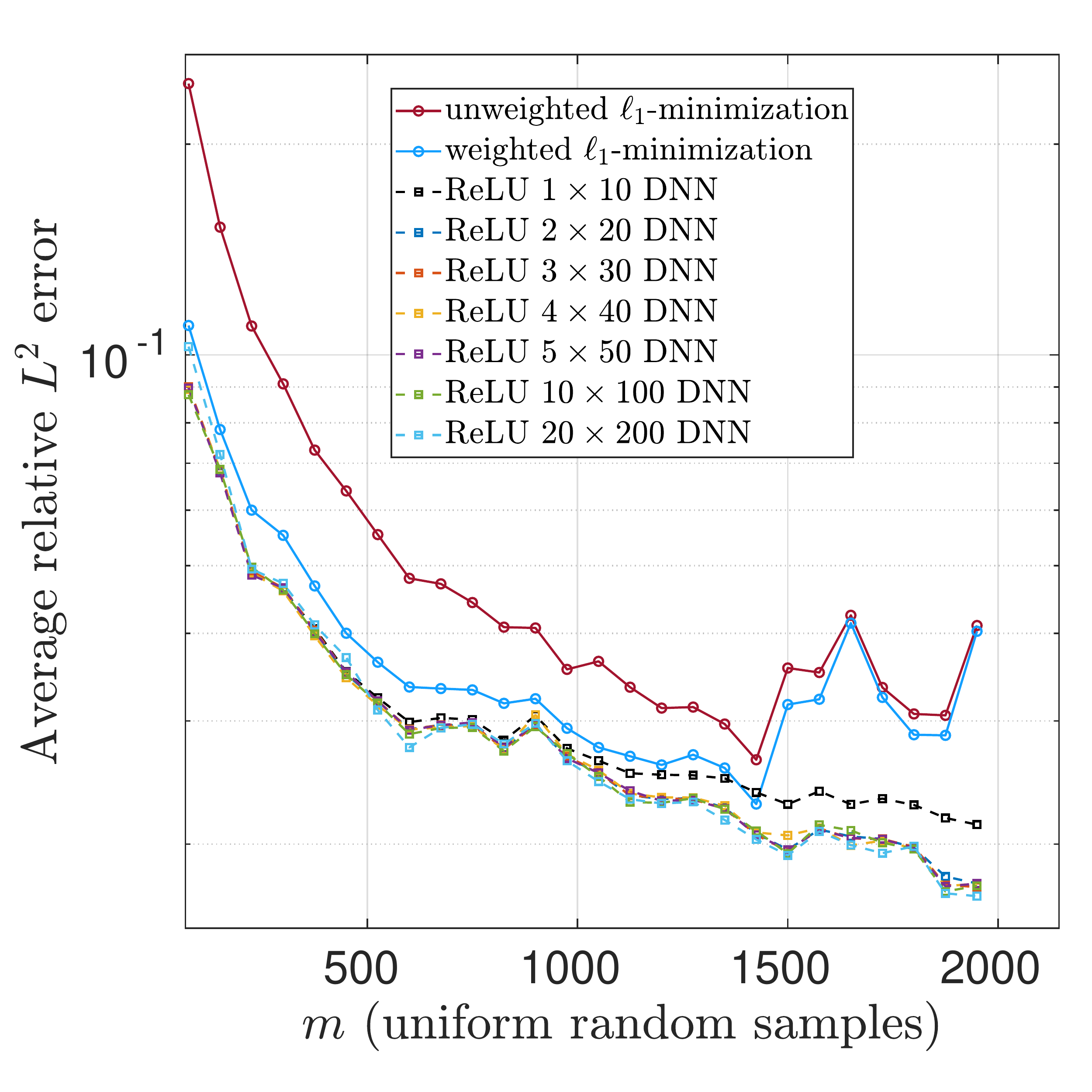}
\includegraphics[width=0.23\paperwidth,clip=true,trim=0mm 0mm 0mm 0mm]{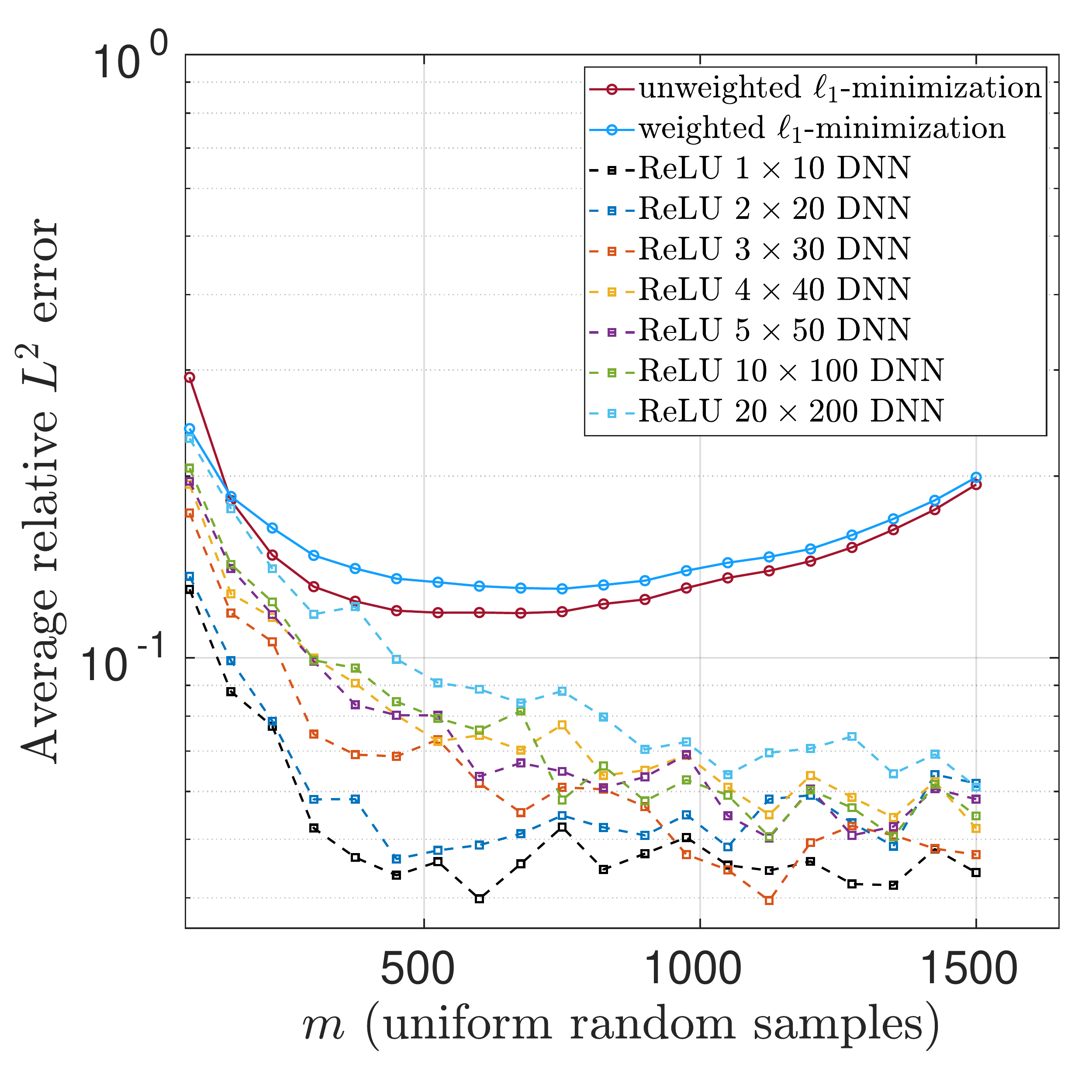}
\includegraphics[width=0.23\paperwidth,clip=true,trim=0mm 0mm 0mm 0mm]{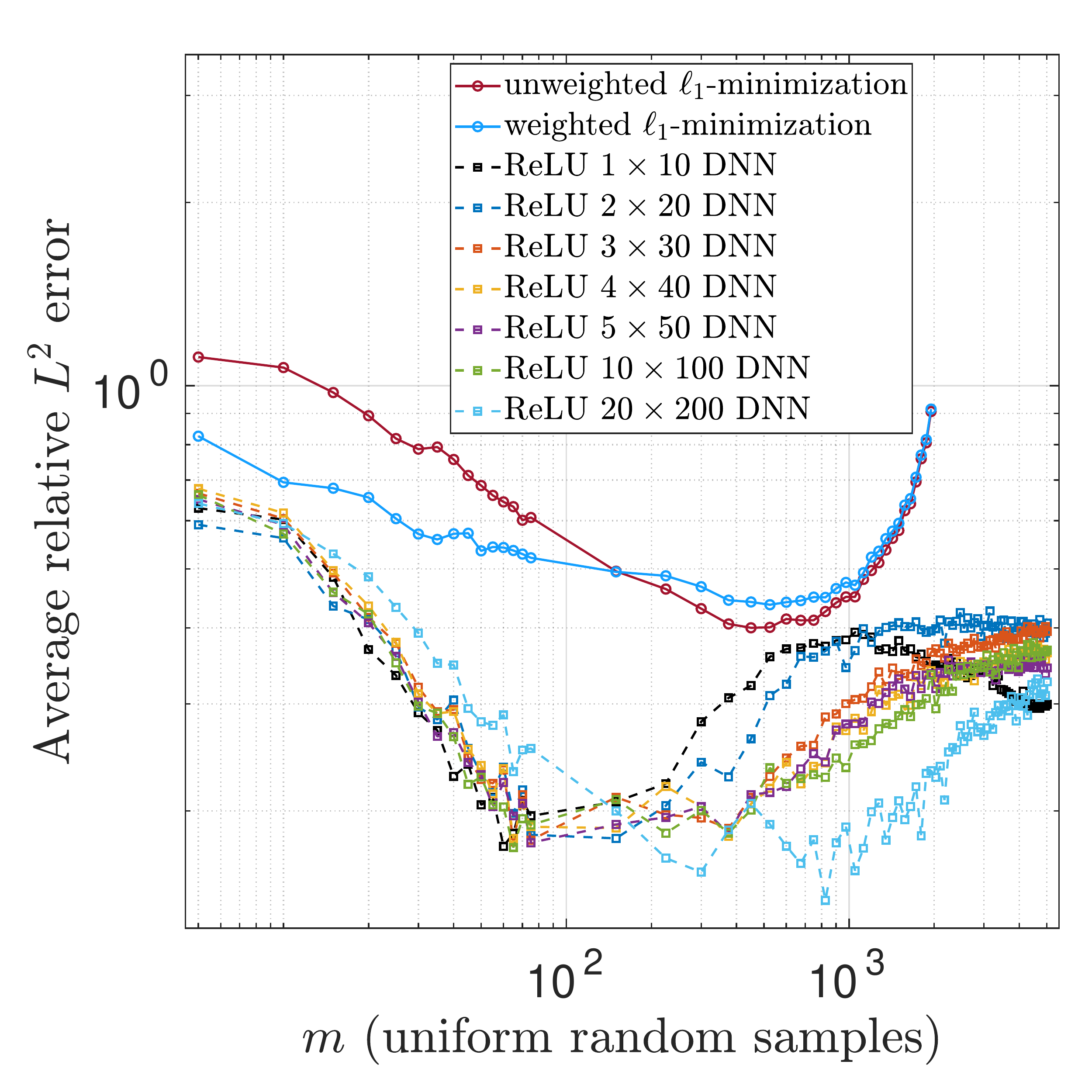}
\end{center}

\vspace{-2mm}
\caption{Average relative $L^2$ error w.r.t. number of samples used in training ReLU DNNs and solving the CS problem with Legendre basis of cardinality $n$ in $d$ dimensions with {\bf(left)} $d=1$ and $n = 3,000$, {\bf(center)} $d=2$ and $n=3,001$, and {\bf(right)} $d=4$ and $n=3,079$.}
\label{fig:halfspace_func}
\end{figure}

Fig.\ \ref{fig:halfspace_func_weight_comp} displays the average absolute maximum of the weights and biases in training the DNNs. 
There we observe that while on average the maximum weights are larger than those found for the smooth functions of the previous section, they remain bounded for the trained DNNs. 
Comparing to the aforementioned existence results, we note the weights never grow large enough in our experiments to obtain such high accuracy approximations in practice. 
Due to the initialization strategy chosen in this work to prevent failure during training, all of our networks start from an initial point corresponding to weights and biases close to 0. 
Combined with the training process which uses an initial learning rate of $10^{-3}$ which decays exponentially with the number of epochs, the algorithms may never reach a minimizer corresponding to higher-accuracy approximations of the halfspace function with larger weights.

\begin{figure}[ht]
\begin{center}
\includegraphics[width=0.23\paperwidth,clip=true,trim=0mm 0mm 0mm 0mm]{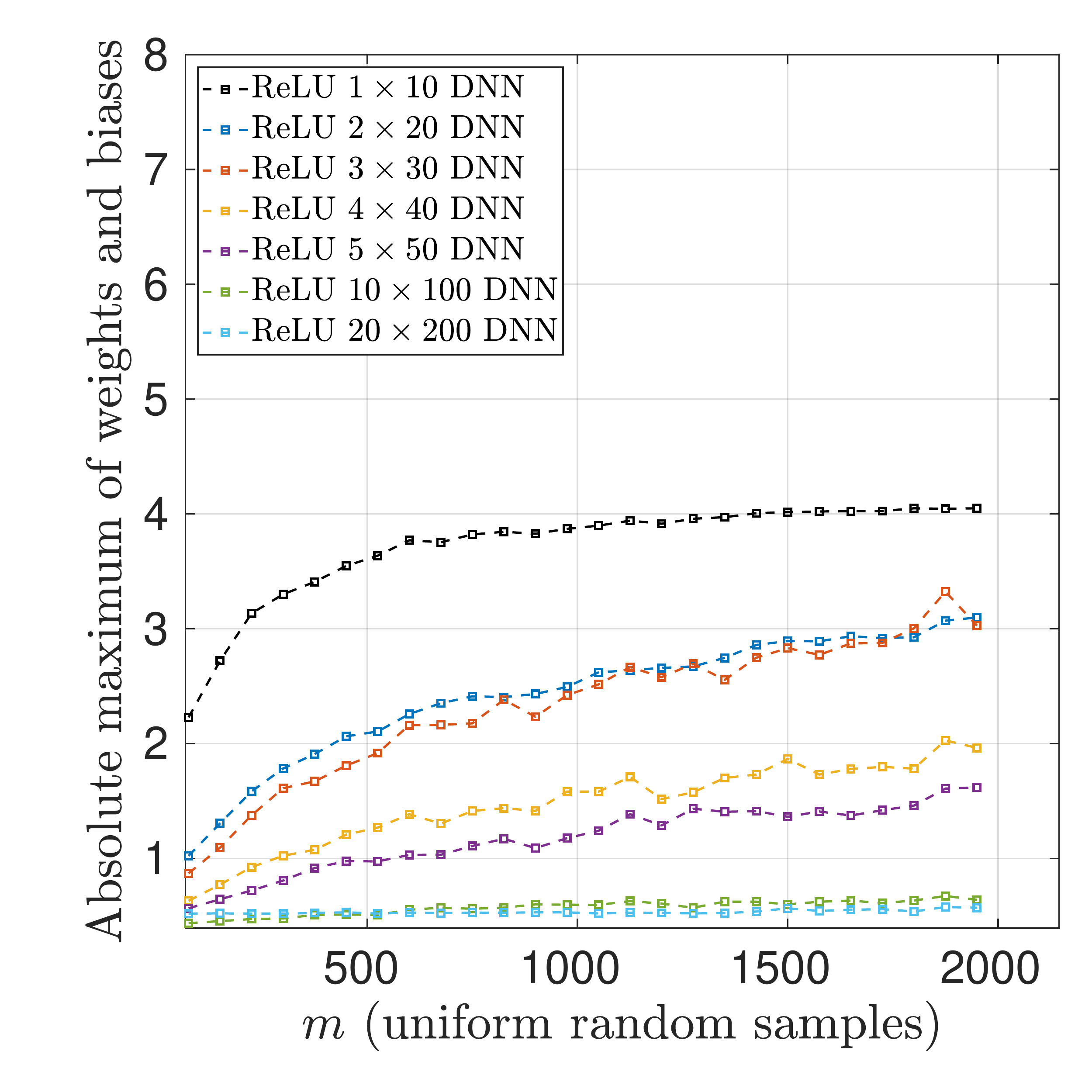}
\includegraphics[width=0.23\paperwidth,clip=true,trim=0mm 0mm 0mm 0mm]{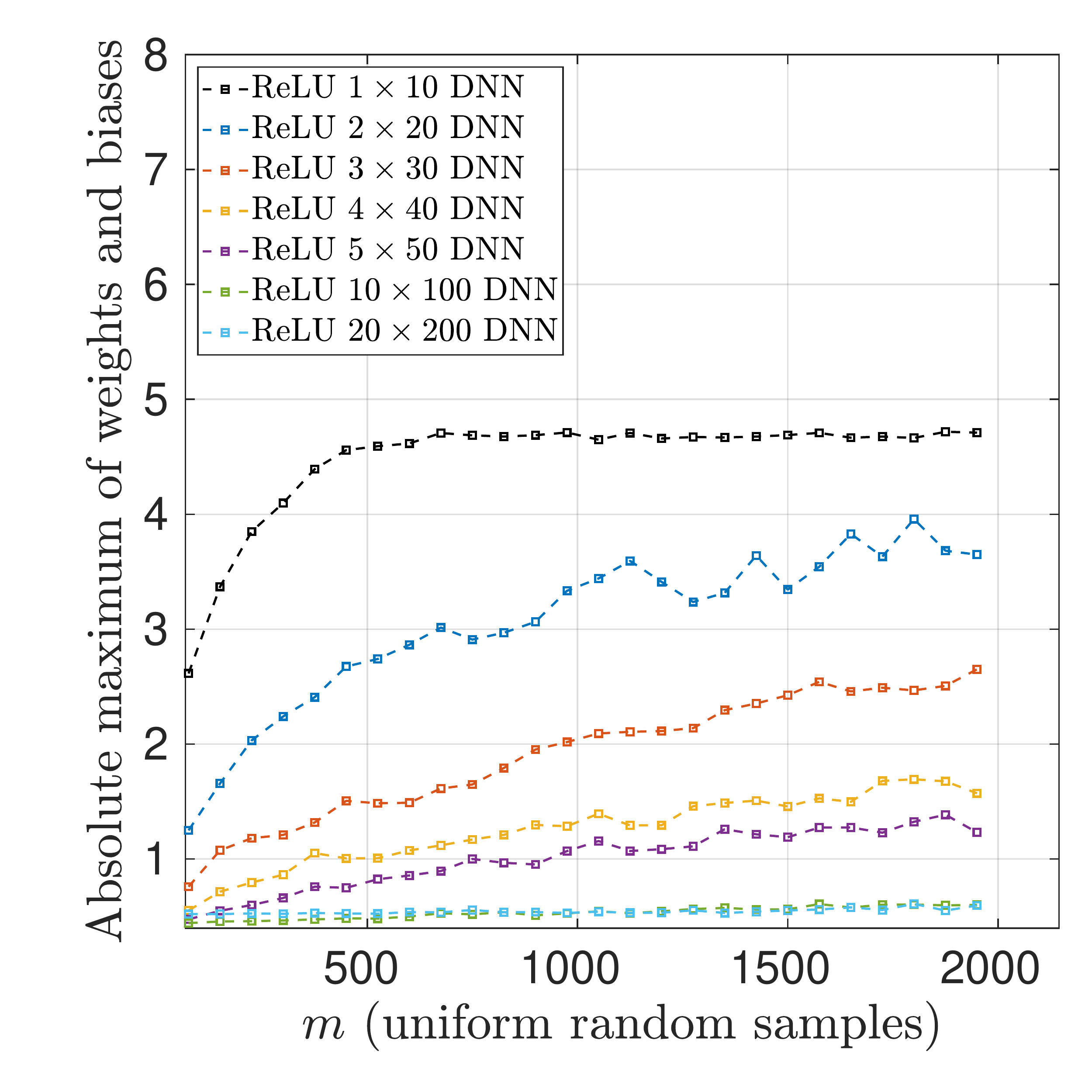}
\includegraphics[width=0.23\paperwidth,clip=true,trim=0mm 0mm 0mm 0mm]{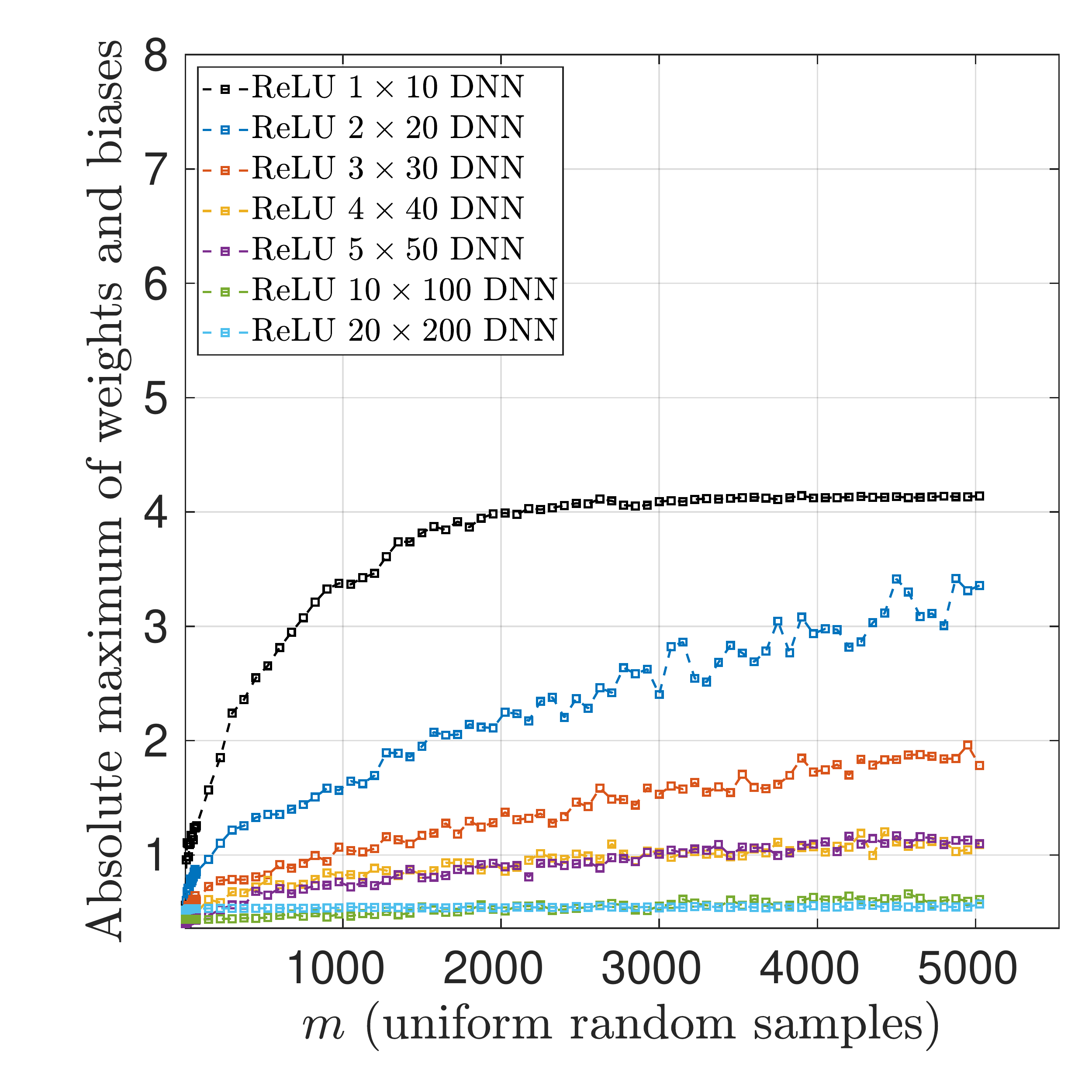}
\end{center}

\vspace{-2mm}
\caption{Comparison of average absolute maximum of weights and biases vs. number of samples of the halfspace function from \eqref{eq:halfspace_func} used in training in {\bf(left)} $d=1$, {\bf(center)} $d=2$, and {\bf(right)} $d=4$ dimensions.}
\label{fig:halfspace_func_weight_comp}
\end{figure}

% !TEX root = ./MLFA.tex

\section{Theoretical insights}\label{sec:theory1}

We conclude this paper with some theoretical insights. Specifically, we first establish convergence rates for polynomial approximation via CS (Theorem \ref{thm:CSexpapp}) and then show a practical existence theorem for DL which demonstrates the existence of an architecture and training procedure that achieves the same rates (Theorem \ref{thm:CSasNN}).
Proofs in this section can be found in the Supplementary Materials.

\subsection{Exponential convergence of polynomial approximations}\label{ss:expconvpoly}

In \S \ref{sec:polyapprox} we introduced the best $s$-term Legendre polynomial approximation.  We first establish exponential rates of convergence of such approximations.
To this end, let $\cE_{\rho} = \left \{ (z + z^{-1})/2 : z \in \bbC,\ |z| \leq \rho \right \}$
be the Bernstein ellipse with parameter $\rho > 1$ and define the Bernstein polyellipse $\cB_{\bm{\rho}} = \bigotimes^{d}_{j=1} \cE_{\rho_j}$,
where $\bm{\rho} = (\rho_1,\ldots,\rho_d) \geq \bm{1}$.  This inequality is understood componentwise, i.e. $\bm{\rho} \geq \bm{1}$ if and only if $\rho_j \geq 1$ for all $j$.  We now make the following assumption:

\begin{assumption}
\label{ass:holo}
For some $\bm{\rho} > \bm{1}$, the function $f : \cU \rightarrow \bbR$ admits a holomorphic extension to an open set $\cO$ containing $\cE_{\bm{\rho}}$.
\end{assumption}

Note that this is a reasonable assumption in practice: it is known that functions that arise as quantities of interest for a range of parametric PDEs admit holomorphic extensions \cite{CDS11,CohenDeVoreApproxPDEs}.
Under Assumption \ref{ass:holo}, it is known that the Legendre polynomial coefficients satisfy
\ben{
\label{LegCoeffBd}
|c_{\bm{\nu}} | \leq \nm{f}_{L^{\infty}(\cE_{\bm{\rho}})} \bm{\rho}^{-\bm{\nu}} \prod^{d}_{j=1} (1+2 \nu_j)^{1/2},
}
where $\nm{u}_{L^{\infty}(\cE_{\bm{\rho}})} = \sup_{z \in \cE_{\bm{\rho}}} |f(z)|$ (see, for instance, the proof of Theorem 3.5 in \cite{Opschoor2019Legendre}). When $d = 1$ this is a classical result in polynomial approximation, and guarantees convergence of the truncated expansion at an exponential rate $\rho^{-s}$. When $d > 1$, it also clarifies why best $s$-term approximation is a suitable strategy; namely, unless the parameter $\bm{\rho}$ is known, it is difficult to make an \textit{a priori} choice of coefficients which lead to a fast rate of exponential convergence.  

On the other hand, the following theorem shows favourable exponential rates of convergence for the best $s$-term approximation.  
It also reveals another important property; namely, that the prescribed rate can be achieved using an index set that is \textit{lower}.  We recall that a set $\Lambda$ of multi-indices is lower if, for any $\bm{\nu} \in \Lambda$ one has $\bm{\mu} \in \Lambda$ whenever $\bm{\mu} \leq \bm{\nu}$ (this inequality is understood componentwise, i.e.\ $\mu_j \leq \nu_j$, $\forall j$). As discussed later, the lower set property is crucial for obtaining approximations using CS that attain the same rates.

For convenience, from now on we consider the $s$-term approximation error in the $L^{\infty}$-norm on $\cU$. This is slightly more convenient for the subsequent analysis, although $L^2$-norm bounds could also be proved with some additional technicalities.

\begin{theorem}
\label{thm:LegExpOSZ}
Let $s \geq 1$ and $f$ satisfy Assumption \ref{ass:holo} for some $\bm{\rho} > \bm{1}$. Then there exists a lower set $\Lambda \subset \N^d_0$ of size at most $s$ for which
\bes{
\| f - p  \|_{L^{\infty}(\cU)} \leq  C \exp \left ( -\gamma s^{1/d} \right ),
}
where $p = \sum_{\bm{\nu} \in \Lambda} c_{\bm{\nu}}  \Psi_{\bm{\nu}}$ and $C > 0$ depends on $d$, $\bm{\rho}$, $\gamma$ and $f$ only, for any $\gamma$ satisfying
\ben{
\label{polyrategamma}
0 < \gamma <  (d+1)^{-1} \left ( d! \prod^{d}_{j=1} \log(\rho_j)  \right )^{1/d}.
}
\end{theorem}

This result asserts exponential convergence of the best $s$-term approximation at a rate depending on the product of $\log(\rho_j)$, as opposed to $\log(\rho_{\min})$, $\rho_{\min} = \min_j \{ \rho_j \}$, which would arise if some fixed isotropic index set were used.  
Importantly, it also does not require any {\em a priori} knowledge of the parameters $(\rho_j)_{j=1}^d$, and adapts to the underlying anisotropy of the function.
In the next subsection we show this rate can be achieved using CS with a sampling complexity $m$ that scales favourably with $s$.
We refer to CS as a best-in-class scheme since there are, currently, no other methods which provably achieve this property.

\begin{remark}
There are numerous results on the convergence rate of the best $s$-term polynomial approximation of holomorphic functions.  Algebraic rates of convergence can be found in, for instance, \cite{CDS10,CDS11}, for not only Legendre, but also Chebyshev and Taylor expansions.  Notably, these results also extend to the case $d = \infty$, which theoretically permits the approximation of functions of infinitely-many variables.  However, the constants in the error bounds may be large in practice \cite{TWZ17}.
The rates shown above -- which are based on \cite[Sec.\ 3.9]{CohenDeVoreApproxPDEs} and \cite{Opschoor2019Legendre} -- possess the advantage of always being attained in lower sets.  However, they may also exhibit a poor scaling with $d$ in the constant $C$.  For \textit{quasi-optimal} error bounds and rates, see \cite{BNTT14,BTNT12,TWZ17}.  The results shown in \cite{TWZ17} are asymptotically sharp as $s \rightarrow \infty$, hence they provide the optimal rate in the finite-dimensional setting.  
\end{remark}

%----------------------------------------------------
\subsection{Exponential convergence via compressed sensing}
\label{subsec:CS}
%----------------------------------------------------

Theorem \ref{thm:LegExpOSZ} prescribes exponential rates of convergence for the best $s$-term approximation.  We now show that these rates can be achieved via CS.
To do so, following \cite{ABBCorrecting}, we now suppose that the CS approximation $\hat{f} = \sum_{\bm{\nu} \in \Lambda} \hat{c}_{\bm{\nu}} \Psi_{\bm{\nu}}$ is computed by solving the so-called \textit{weighted square-root LASSO} problem
\begin{align}
\label{eq:CS_sr_min}
\textnormal{minimize}_{\bz\in\R^n} \|\bz\|_{1,\bm{u}} + \mu \|\bA \bz - \bm{f} \|_2,
\end{align}
rather than \eqref{eq:CS_BPDN_weighted}. The reasons for doing this are explained in \S \ref{sec:CSSM}.

\begin{theorem}
\label{thm:CSexpapp}
There exists a universal constant $c > 0$ such that the following holds.  Suppose that $0 < \varepsilon < 1$, $m \geq C_1 L_{m,d,\varepsilon}$, where 
\bes{
L_{m,d,\varepsilon} = \log(2 m) \left ( \log^2(2m) \log(2 d) + \log(2 \varepsilon^{-1} \log(2m) ) \right ),
}
$\bm{x}_1,\ldots,\bm{x}_m$ are drawn independently from the uniformly measure on $\cU$,
\ben{
\label{supperCSexp}
1 \leq s \leq \sqrt{\frac{m}{c L_{m,d,\varepsilon}}},
}
and $\Lambda = \Lambda^{\mathrm{HC}}_{s}$ is as in \eqref{HCindex}.
Then the following holds with probability at least $1-\varepsilon$.  Let $f : \cU \rightarrow \bbR$ satisfy Assumption \ref{ass:holo} for some $\bm{\rho} > \bm{1}$ and let $\hat{f} = \sum_{\bm{\nu} \in \Lambda} \hat{c}_{\bm{\nu}} \Psi_{\bm{\nu}}$, where $\bm{\hat{c}} = (\hat{c}_{\bm{\nu}} )_{\bm{\nu} \in \Lambda}$ is any minimizer of \eqref{eq:CS_sr_min} with $\mu = \frac{12 \sqrt{42}}{35} s$ and weights $\bm{u}$ given by \eqref{eq:uweightsLeg}.  Then
\bes{
\nmu{f - \hat{f}}_{L^{\infty}(\cU)} \leq C \cdot \exp(-\gamma s^{1/d} ),
}
for all $\gamma$ satisfying \eqref{polyrategamma}, where $C>0$ depends on $d$, $\bm{\rho}$, $\gamma$ and $f$ only.
\end{theorem}

This theorem asserts that CS achieves the same rates as those guaranteed in Theorem \ref{thm:LegExpOSZ}, with a sample complexity that is (up to the logarithmic factor) quadratic in $s$.  In particular, scaling implied by \eqref{supperCSexp} depends only on logarithmically on the dimension $d$.  Hence, this estimate scales particularly well in higher dimensions.

It is important to note the key role the lower set property plays in this result.  First, the union of all lower sets $S$ of cardinality at most $s$ is $\cup \{ S : \mbox{$S$ lower, $ |S| \leq s$} \} = \Lambda^{\mathrm{HC}}_{s}$
, i.e.\ the hyperbolic cross set of degree $s$.  This is the rationale for choosing $\Lambda$ in this way.  Second, much like how sparse sets are promoted by the $\ell^1$-norm, sparse and lower sets are promoted by the weighted $\ell^1$-norm, with the weights taken to be $\bm{u}$ (these weights penalize high indices). Had one considered the unweighted $\ell^1$-norm in the above result, the sample complexity would have scaled with a higher algebraic power of $s$ \cite{Adcock2016}.

\subsection{Existence of good training for DNN approximation}

We now give our main result:

\begin{theorem}
\label{thm:CSasNN}
Let $0 < \varepsilon < 1$, $c> 0$ and $L_{m,d,\varepsilon}$ be as in Theorem \ref{thm:CSexpapp} and suppose that $m \geq c L_{m,d,\varepsilon}$ and $s$ satisfies \eqref{supperCSexp}.  Let $\bm{x}_1,\ldots,\bm{x}_m$ be drawn independently from the uniform measure on $\cU$.  Then the following holds with probability at least $1-\varepsilon$.  Let $f : \cU \rightarrow \bbR$, $\nm{f}_{L^{\infty}(\cU)} \leq 1$, satisfy Assumption \ref{ass:holo} for some $\bm{\rho} > \bm{1}$.  Then there exists a family of neural networks $\mathcal{N}$
with $n = |\Lambda^{\mathrm{HC}}_s|$ trainable parameters and of size and depth
\bes{
\begin{split}
\mathrm{depth}(\Phi_{\Lambda,\delta}) &\leq c' (1 + d \log(d))(1+\log(s)) (s + \log(n) +\gamma s^{1/d})
\\
\mathrm{size}(\Phi_{\Lambda,\delta}) & \leq c' \left ( d^2 s^2 + (d s + d^2 n )\left ( 1 +  \log(s)+ \log(n) +\gamma s^{1/d} \right ) \right )
\end{split},\qquad \Phi \in \mathcal{N},
}
where $c'>0$ is a universal constant, and a regularization functional $\cJ : \mathcal{N} \rightarrow \bbR_{+}$ such that any minimizer $\hat{\Phi}$ of the regularized loss function
$
\cL(\Phi) : = \sqrt{\frac1m\sum^{m}_{i=1} | \Phi(\bm{y}_i) - f(\bm{y}_i) |^2} + \cJ(\Phi),
$
satisfies
\bes{
\nmu{f - \hat{\Phi}}_{L^{\infty}(\cU)} \leq C \cdot \exp(-\gamma s^{1/d}) ,
}
for all $\gamma$ satisfying \eqref{polyrategamma}, where $C >0$ depends on $d$, $\bm{\rho}$, $\gamma$ and $f$ only. The functional $\mathcal{J}$ is a weighted $\ell^1$-norm penalty of the weights in the output layer.
\end{theorem}

Note that in this result the \textit{size} of an architecture refers to the number of nonzero weights and biases.  The \textit{number of trainable parameters} in an architecture refers to the number of weights and biases that are trainable, as opposed to those that are fixed. 
The proof of this theorem uses a result of \cite{Opschoor2019Legendre} (see Proposition \ref{prop:polyNN}), which 
states that a finite set of Legendre polynomials can be uniformly approximated by a neural network of a given depth and size.  Theorem \ref{thm:CSasNN} is obtained by rewriting the CS recovery using the weighted square-root LASSO as a neural network training problem.  See \S \ref{sec:proofs} for the details.

Theorem \ref{thm:CSasNN} asserts the existence of a DNN architecture, with relatively few training parameters, and a training procedure for which the resulting DNN approximations are guaranteed to perform as well as CS, up to a constant, in terms of sample complexity and convergence rates. Of course, the specific procedure suggested by the proof would not be expected to lead to any superior performance over CS. We make no claim that this approach is either practical or numerically stable (indeed, its proof relies on monomials).
This is analogous to how standard existence theorems in DNN approximation theory (see \S \ref{sec:introduction}), while constructive, do not lead to superior approximations over the classical approximations on which they are based.  However, it does indicate that with sufficiently careful architecture design and training, one may achieve superior performance with DNNs over CS.  
The extent to which this can be done is a largely open problem, requiring further theoretical and empirical investigation.

% !TEX root = ./MLFA.tex

%----------------------------------------------------
\section{Conclusions}
\label{sec:conclusion}
%----------------------------------------------------

In this work we have presented results highlighting the key issues of accuracy, stability, sample complexity and computational efficiency of practical function approximation with DNNs. Our theoretical contribution on the existence of a DL procedure which performs as well as CS suggests that DL can, in theory, enjoy the same accuracy and sample complexity properties as CS. However, our numerical results comparing standard DNN architectures and common methods of training with CS techniques suggest that current methods in DL are generally unable to achieve these theoretical convergence rates on smooth function approximation problems. On the other hand, while DNNs perform relatively badly on highly smooth functions, their performance on more challenging problems, such as less smooth functions or functions with jump discontinuities is rather more promising. Certainly the fact that the same DNN architectures can be used on quite different problems sets them apart from traditional methods in scientific computing, e.g.\ polynomial-based methods, which are usually tied to a specific class of function (e.g.\ smooth functions). Hence there is ample scope and need for future work along the lines of investigation initiated in this paper. This includes both empirical investigations into architecture and cost function design, as well as algorithms for training, and further novel theoretical insights into \textit{practical existence theory}; that is, the existence of not only effective DNN architectures, but also training procedure and sampling strategies which realize them efficiently. The hope is that, with these further efforts, DNNs may develop into effective tools for scientific computing that can consistently outperform current
best-in-class approaches across a range of challenging problems.

%----------------------------------------------------
\section{Acknowledgements}
\label{sec:acknowledgements}
%----------------------------------------------------

The authors thank Clayton Webster for important discussions motivating our study of practical DNN approximation. We thank Simone Brugiapaglia for useful discussions related to various aspects of this work and Sebastian Moraga for his assistance with several technical steps in the proofs. We also acknowledge and thank the anonymous referees for proofreading and providing valuable comments on the manuscript. Finally, Boris Hanin and Philipp Petersen are gratefully acknowledged for their helpful and inspiring comments on the approximation theory of DNNs and practical issues related to initialization and trainability.

\appendix

% !TEX root = ./MLFA_supplement.tex

%----------------------------------------------------
\section{Testing setup -- further information} 
\label{sec:testingsetupsupp}
%----------------------------------------------------

In this section, we give further details of the testing and training setup for the DNNs and numerical experiments with CS. Section \ref{sec:singledoubleSM} describes the hardware used in training and testing our DNN models in single and double precision. 
Section \ref{sec:AdamSM} describes further aspects of training our DNNs with the {\tt Adam} solver. Section \ref{sec:AdamLearnRateSM} presents results on training DNNs with {\tt Adam} under a variety of learning rate schedules.
Section \ref{sec:batchesSM} discusses the importance of batch size on the convergence of the {\tt Adam} optimizer.
Section \ref{sec:CSSM} describes the selection of truncation parameters in our CS experiments.

\subsection{Single versus double precision}\label{sec:singledoubleSM}

Double precision calculations using GPUs a generally more computationally demanding and hence we conducted a limited number of such studies. For example, the NVIDIA Tesla P100 GPUs operate at 4.7 TFLOPS ({\bf t}rillion {\bf f}loating {\bf p}oint {\bf op}erations per {\bf s}econd) in double precision vs. 9.3 TFLOPS in single precision, implying a 1:2 ratio for single vs. double precision computation time.
On the other hand, common off-the-shelf consumer GPUs such as the NVIDIA GeForce GTX 1080 Ti operate at 0.355 TFLOPS in double precision vs. 11.5 TFLOPS in single precision, implying a 1:32 ratio for single vs. double precision computation time. 

\subsection{The \texttt{Adam} optimizer}\label{sec:AdamSM}
We use the Adam optimizer \cite{kingma2014adam} as follows. We use the default values $\beta_1 = 0.9$, $\beta_2 = 0.999$ and $\delta=10^{-7}$, and set the initial learning rate $\tau_{\mathrm{init}} = 10^{-3}$. Then at epoch $k\in \N$ we set the learning rate
\begin{align}
\label{eq:lrn_rate_update}
\tau_{k} = \tau_{\mathrm{init}} \times b^{k/K_{\textrm{uf}}},
\end{align}
where $K_{\textnormal{uf}}=10^3$ is the update frequency and $b := \left(\tau_{\mathrm{final}}/\tau_{\mathrm{init}}\right)^{K_{\mathrm{uf}}/K_{\mathrm{final}}}$ is the base, so that at the final epoch $k = K_{\mathrm{final}}$, $\tau_k = \tau_{\mathrm{final}}$ the desired final learning rate.
We allow the learning rate $\tau_k$ to vary continuously between $\tau_{\mathrm{init}}$ and $\tau_{\mathrm{final}}$, using the constants above only as scaling factors.
In single precision, the DNNs are trained for $K_{\mathrm{final}} = $ 50,000 epochs or to a tolerance of $\varepsilon_{\mathrm{tol}} = 5\times 10^{-7}$, while in double precision DNNs are trained for $K_{\mathrm{final}} = $ 200,000 epochs or to a tolerance of $\varepsilon_{\mathrm{tol}} = 5\times 10^{-16}$.
We set the final stepsize $\tau_{\mathrm{final}} = \varepsilon_{\mathrm{tol}}$, so that these choices ensure that the base $b$ is approximately 0.85, implying the learning rate is reduced by 15\% every $K_{\mathrm{uf}}$ iterations.

\subsection{Learning rate schedules for \texttt{Adam}}\label{sec:AdamLearnRateSM}

In Fig.\ \ref{fig:lrn_rate_schedule_comp} we compare learning rate schedules for the {\tt Adam} optimizer on a discontinuous 2-dimensional function. The schedules compared are the exponentially decaying learning rate from \eqref{eq:lrn_rate_update}, a constant learning rate of $10^{-3}$, and a linearly decaying learning rate, i.e., at each iteration setting $\tau_k = (1 - k/K_{\mathrm{final}})10^{-3} + k\tau_{\mathrm{final}}/K_{\mathrm{final}} $, over a range of DNN architectures.  We observe that shallower networks tend to perform marginally better under the exponentially-decaying learning rate schedule, while the effect is insignificant for deeper networks.
In testing a range of values of the base parameter $b$, we found values smaller than $b=0.85$ result in a learning rate that can decrease too quickly to achieve the desired tolerance with the error stagnating, while larger values of $b$ can result in a learning rate not decreasing fast enough to quickly train larger models.

\begin{figure}[ht]
\begin{center}
\includegraphics[width=0.17\paperwidth,clip=true,trim=0mm 0mm 0mm 0mm]{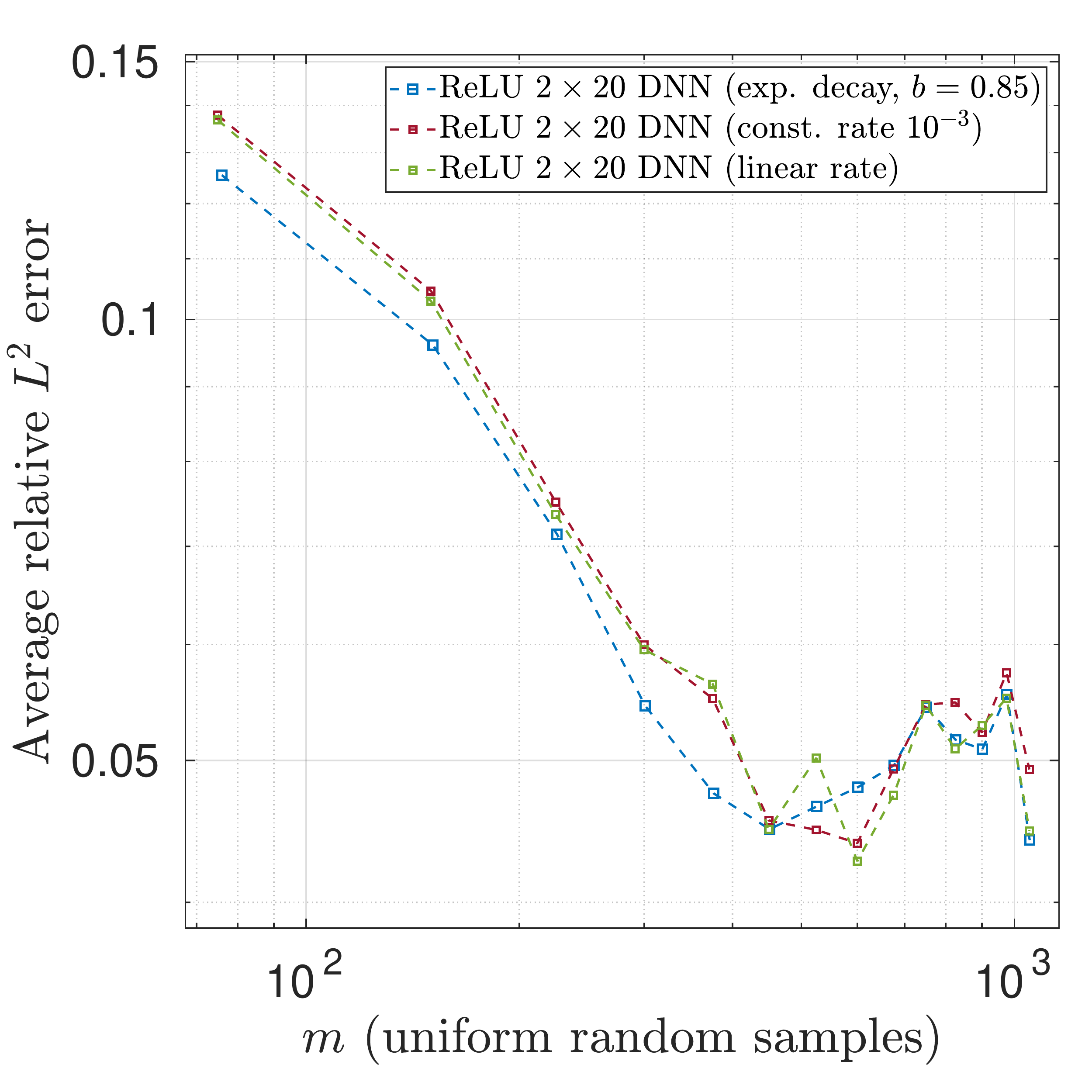}
\includegraphics[width=0.17\paperwidth,clip=true,trim=0mm 0mm 0mm 0mm]{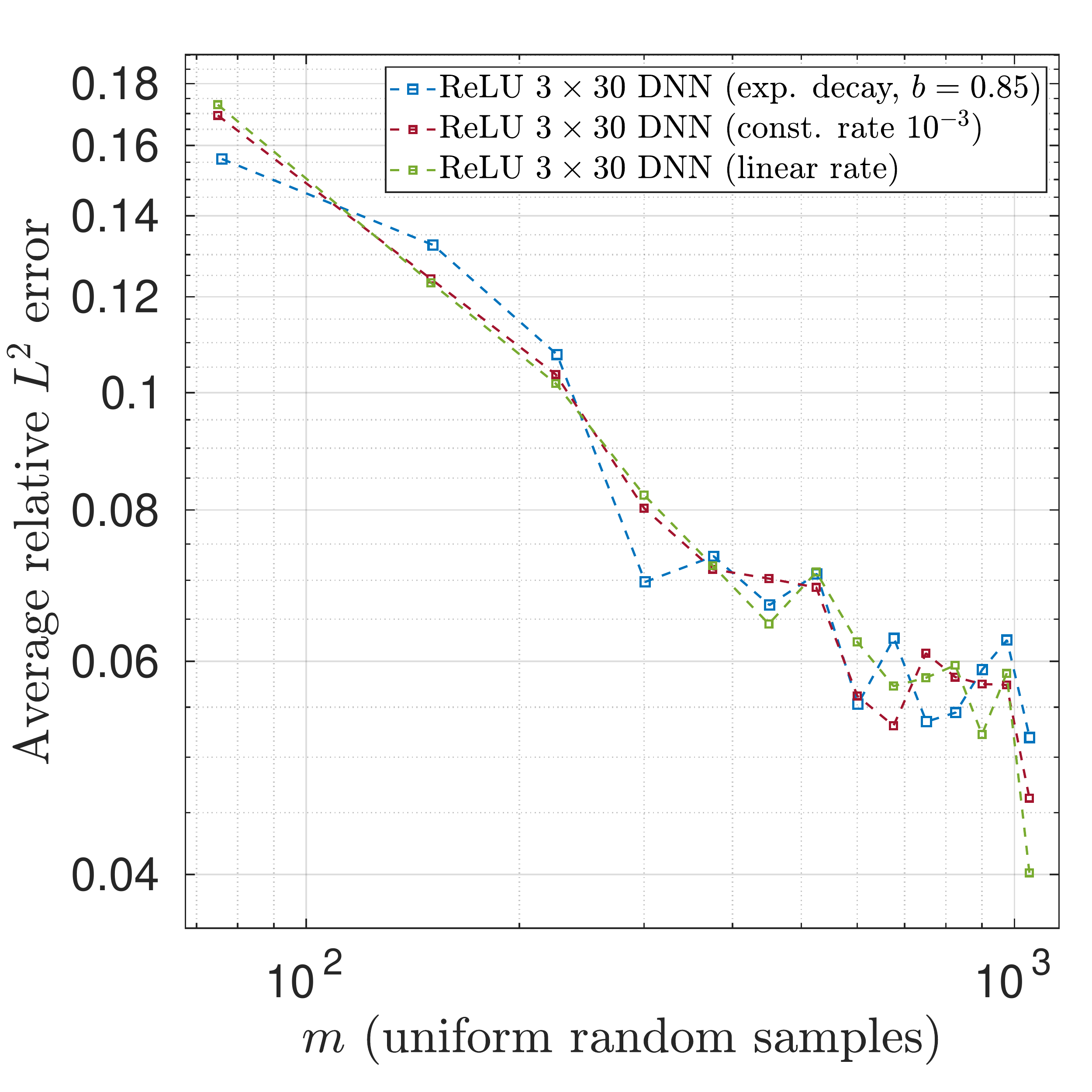}
\includegraphics[width=0.17\paperwidth,clip=true,trim=0mm 0mm 0mm 0mm]{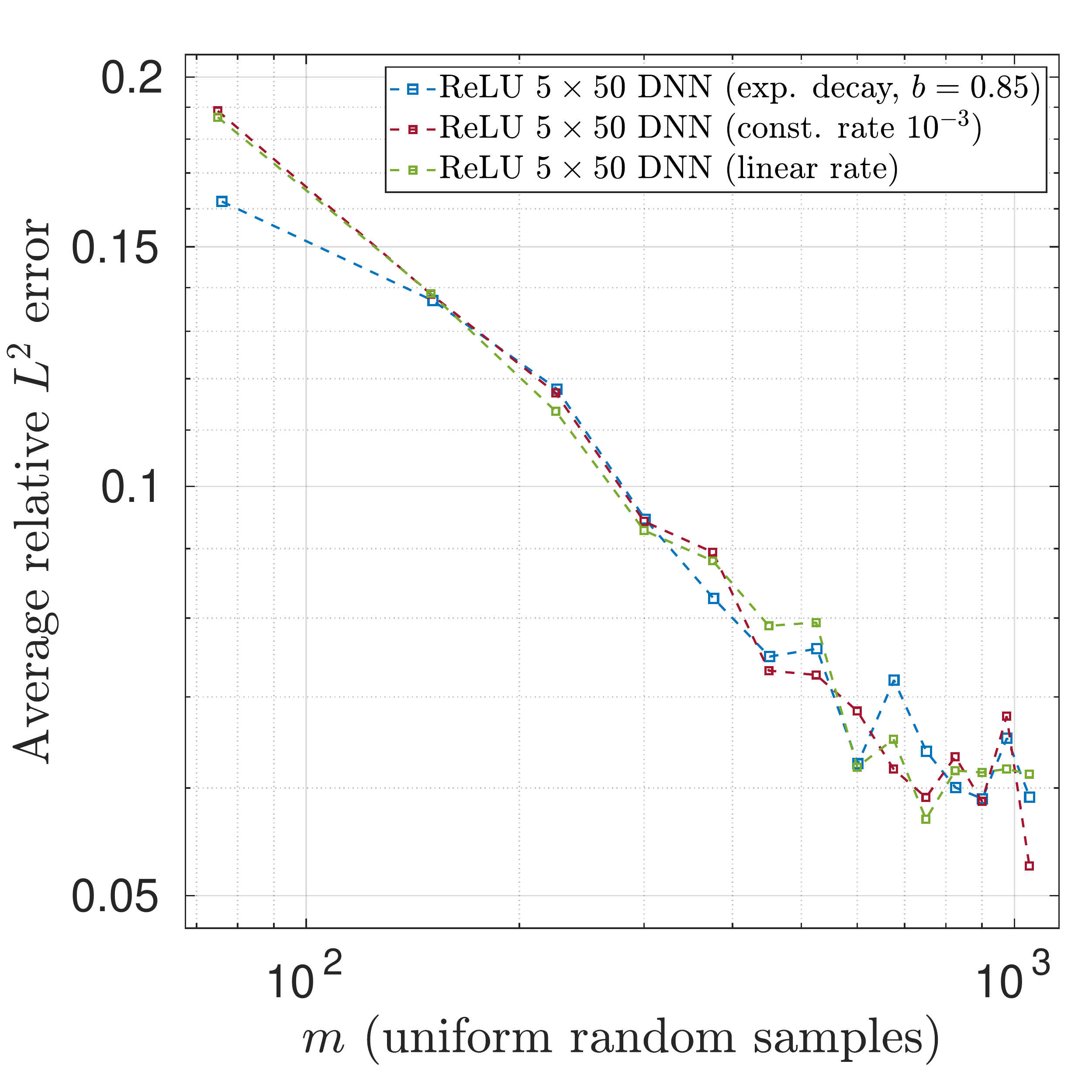}
\includegraphics[width=0.17\paperwidth,clip=true,trim=0mm 0mm 0mm 0mm]{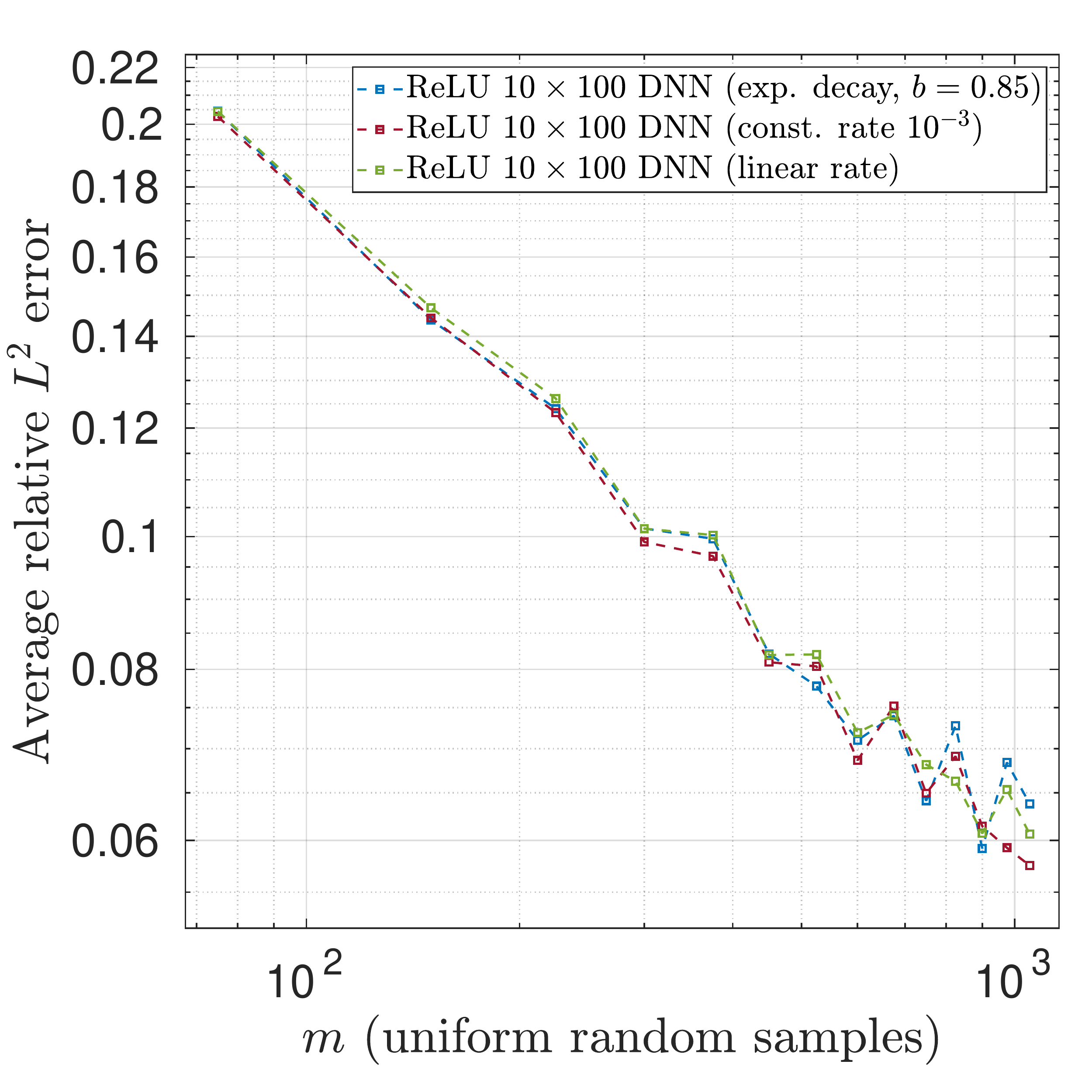}
\end{center}

\vspace{-2mm}
\caption{Comparison of learning rate schedules for the {\tt Adam} optimizer for a variety of ReLU DNN architectures on function \eqref{eq:halfspace_func} with $d = 2$.}
\label{fig:lrn_rate_schedule_comp}
\end{figure}

\subsection{Batch size comparison}\label{sec:batchesSM}

In many standard computer vision tasks, networks are trained on batches of images consisting of a small subset of the overall training set.
Due to the lower dimension of the problems and smaller data set sizes considered herein, we find empirically that setting the batch size too small can result in the transfer time between the CPU and GPU contributing to a large portion of overall run time.
Fig.\ \ref{fig:batch_size_comp} displays the effect of training four different architectures with batch sizes ranging from full-batch to quarter-batch on the generalization performance of the DNNs on a smooth 8-dimensional function. There we observe that decreasing the batch size leads to a marginal improvement in the overall error. Nonetheless for such moderate dimensional problems we find the increase in training time associated with smaller batches is not proportional to the improvement in the overall error.
We leave a more detailed study of the affects of batch size on generalization performance of trained ReLU DNNs in higher-dimensional problems to a future work.

\begin{figure}[ht]
\begin{center}
\includegraphics[width=0.17\paperwidth,clip=true,trim=0mm 0mm 0mm 0mm]{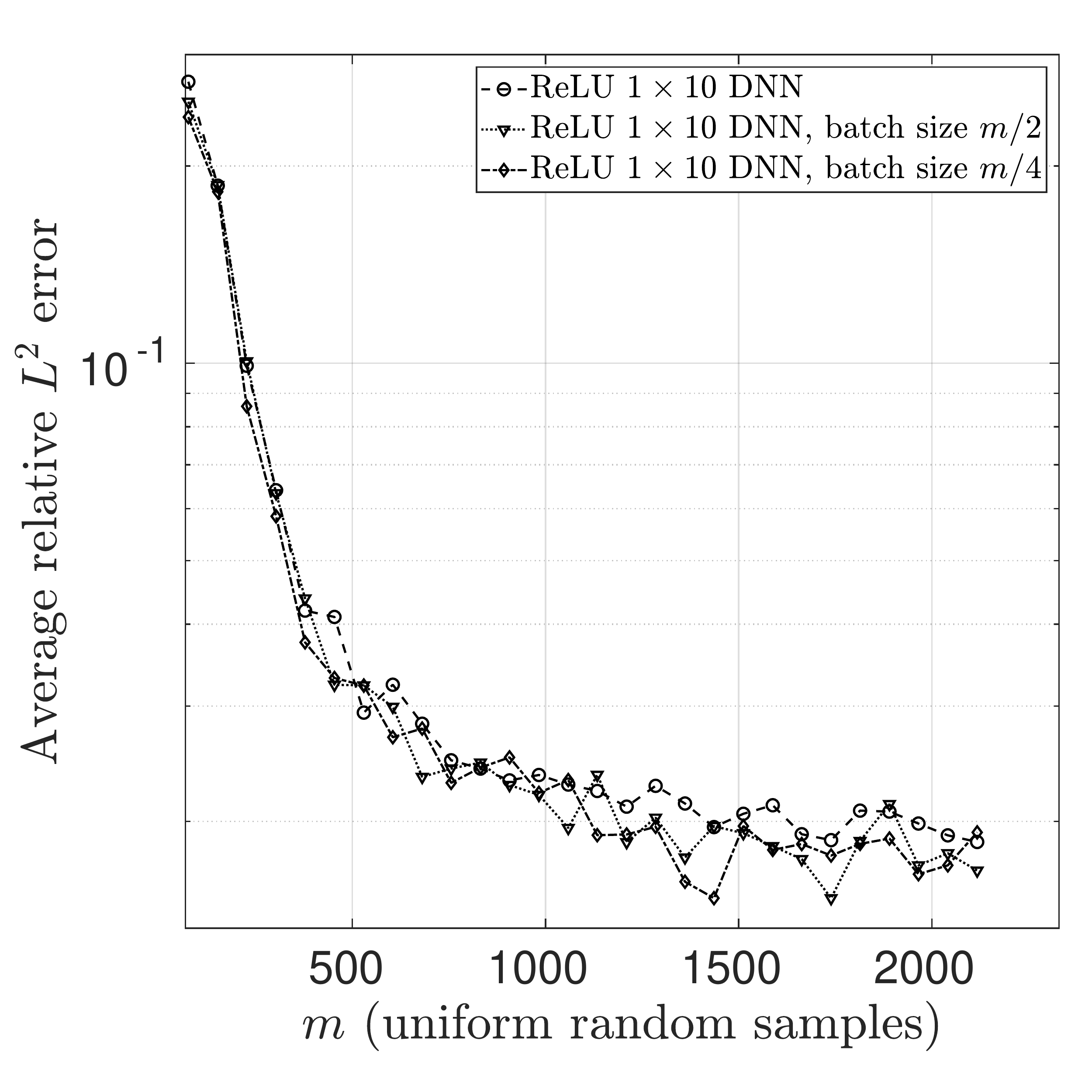}
\includegraphics[width=0.17\paperwidth,clip=true,trim=0mm 0mm 0mm 0mm]{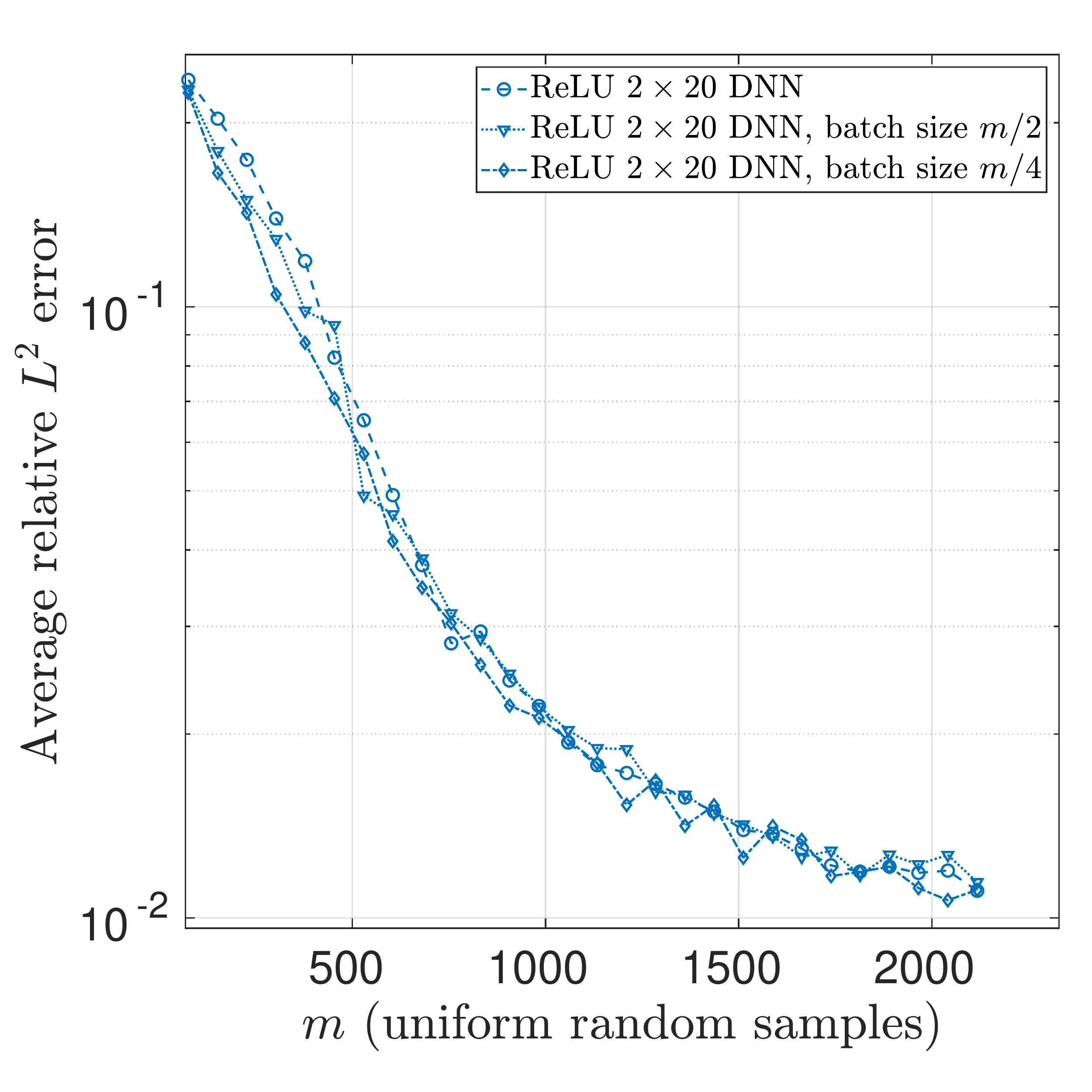}
\includegraphics[width=0.17\paperwidth,clip=true,trim=0mm 0mm 0mm 0mm]{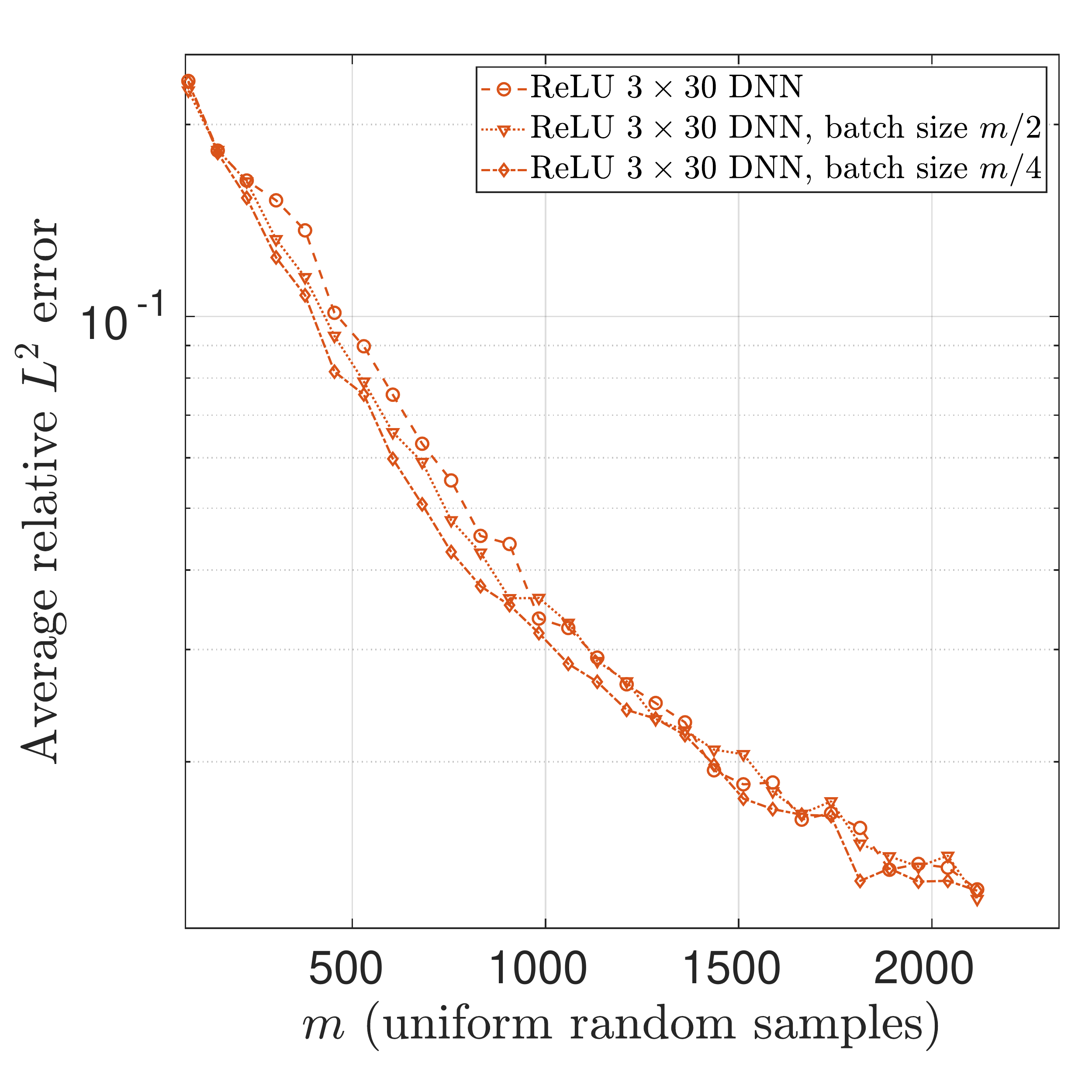}
\includegraphics[width=0.17\paperwidth,clip=true,trim=0mm 0mm 0mm 0mm]{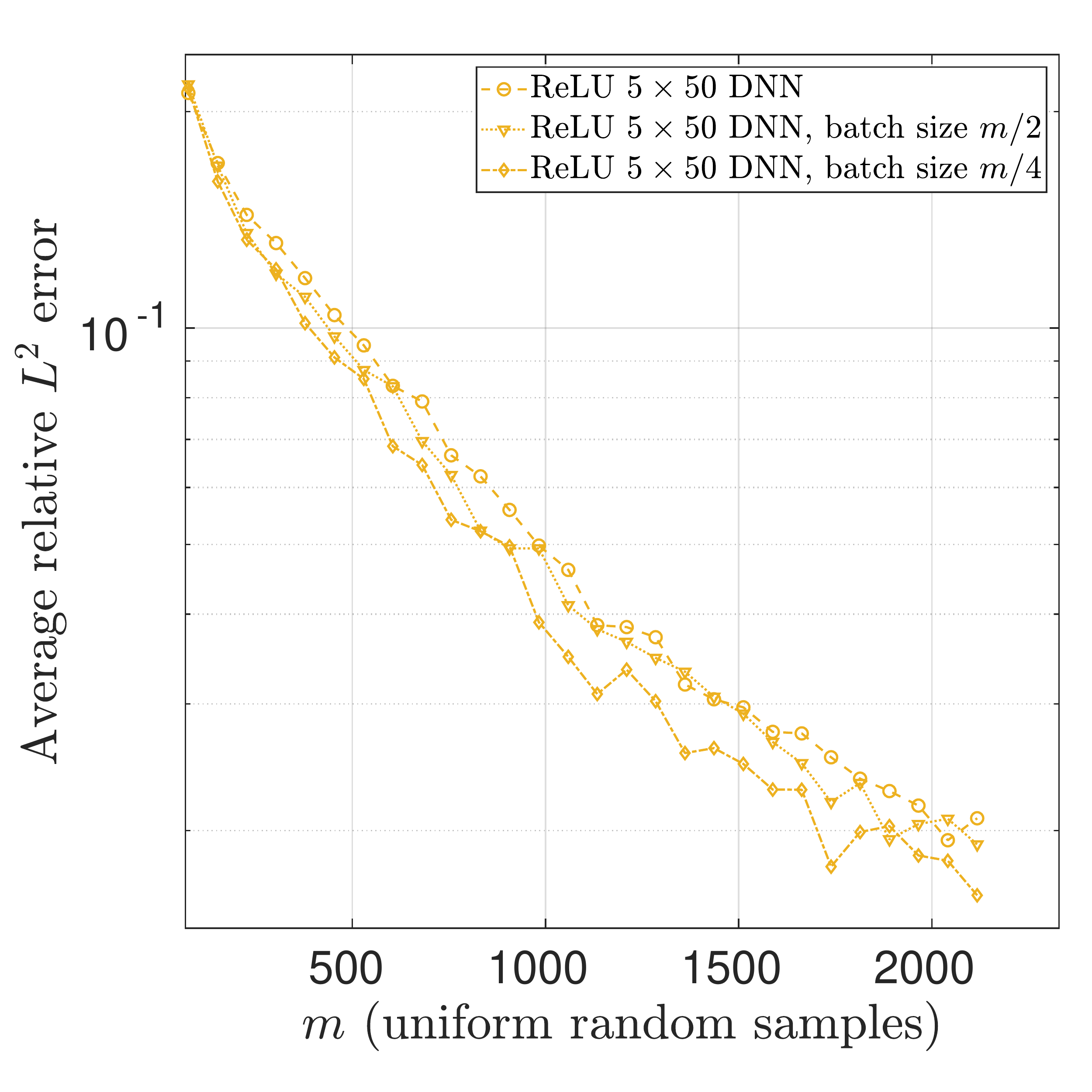}
\end{center}

\vspace{-2mm}
\caption{Comparison of batch sizes for the {\tt Adam} optimizer on a smooth function approximation problem in $d=8$ dimensions. For each of the plots, we fix an architecture and then train to 50,000 epochs in single precision or a tolerance of $5\times 10^{-7}$ when the batch size is set to {\bf(circles)} full batch {\bf(triangles)} half batch {\bf(diamonds)} quarter batch on function \eqref{eq:slower_decay_rational_func} with $d = 8$.}
\label{fig:batch_size_comp}
\end{figure}

%----------------------------------------------------
\subsection{Stability of training with respect to noise} 
\label{subsec:noise_experiments}
%----------------------------------------------------

For many problems in computational science we are provided measurements that are corrupted by some amount of noise. Methods such as compressed sensing come with theoretical recovery guarantees that assert stable and robust recovery of sparse signal vectors from such measurements. 
Therefore, an important question for the application of DL techniques to problems in science, engineering, and medicine is how robust these methods are under realistic noise scenarios. 

In Fig.\ \ref{fig:noise_experiment}, we include a selection of our experiments where the measurements given to both CS and the best-performing DNNs are corrupted by additive Gaussian noise with various noise levels, i.e., given $f(x_i) + n_i$ where $n_i \sim N(0, \sigma^2)$ for various values $\sigma \in \{10^{-1}, 10^{-2}, 10^{-3}\}$. There we observe that both solving the CS weighted and unweighted $\ell_1$-minimization problems and training the DNNs are stable to noise, i.e., the $L^2$ error is proportional to the noiseless approximation error plus the noise standard deviation. 

\begin{figure}[ht]
\begin{center}
\includegraphics[width=0.22\paperwidth,clip=true,trim=0mm 0mm 0mm 0mm]{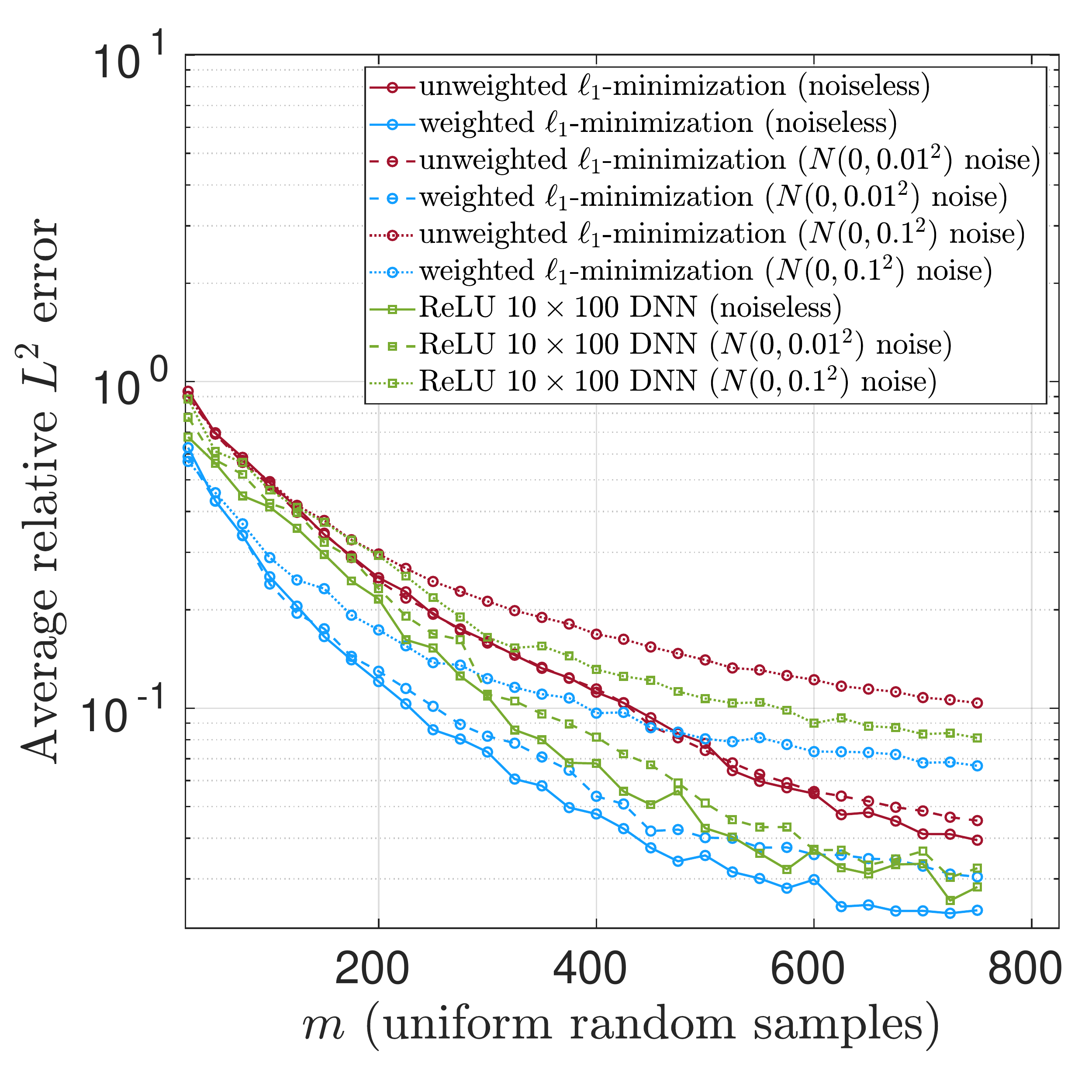}
\includegraphics[width=0.22\paperwidth,clip=true,trim=0mm 0mm 0mm 0mm]{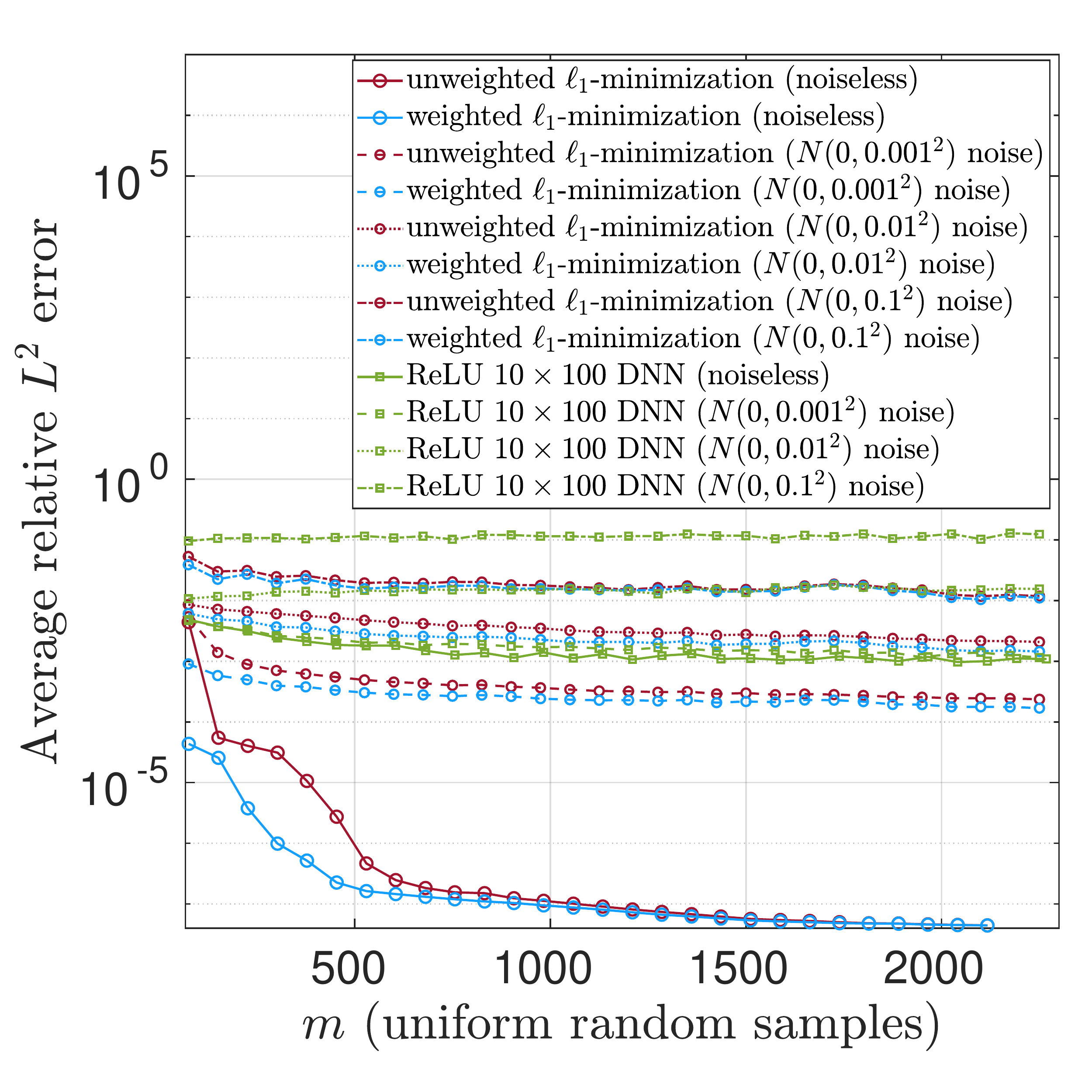}
\includegraphics[width=0.22\paperwidth,clip=true,trim=0mm 0mm 0mm 0mm]{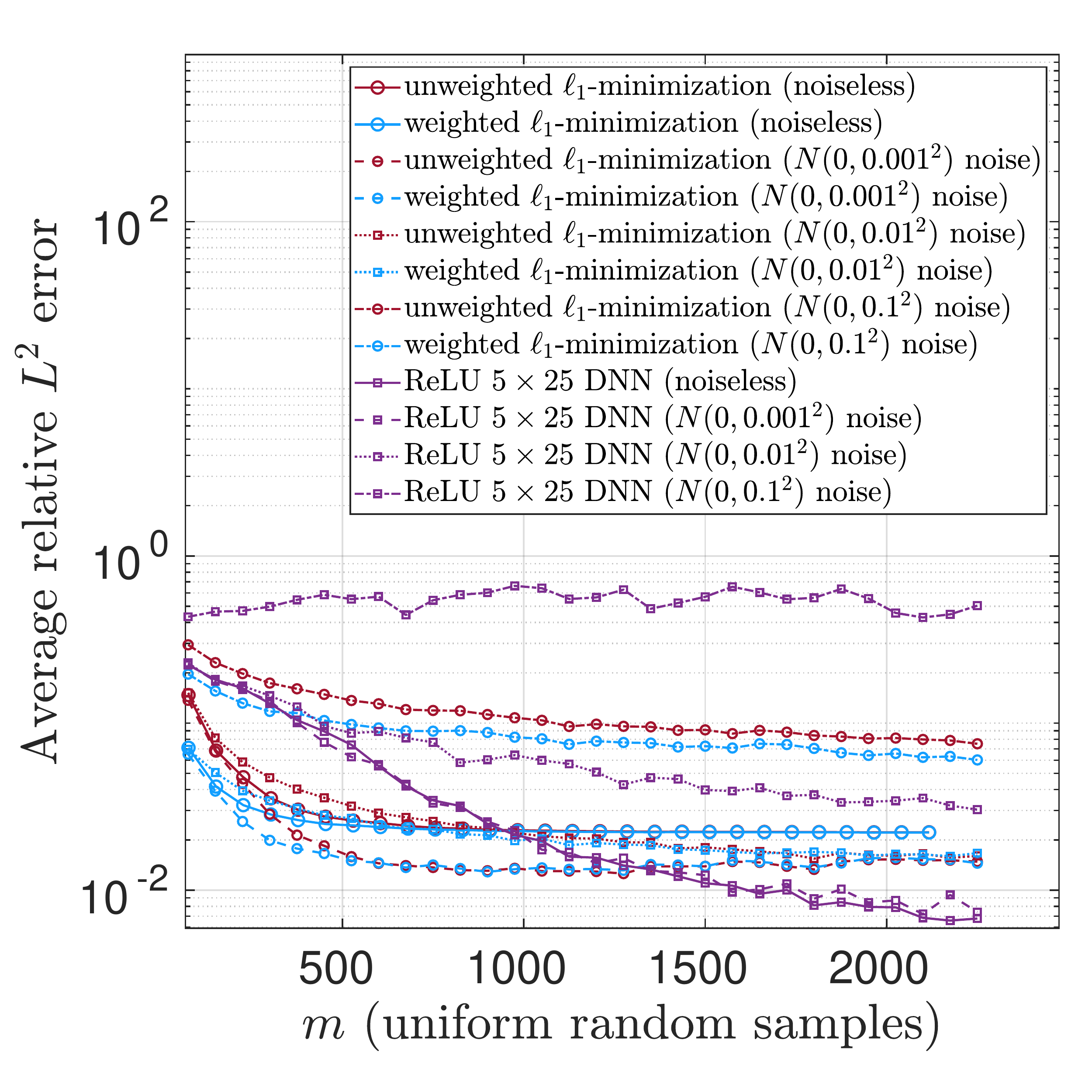}
\end{center}
\vspace{-2mm}
\caption{Average relative $L^2$ errors of ReLU $10\times 100$ DNNs in approximating {\bf(left)} \eqref{1dsmoothfnK} with $K=10$, and {\bf(middle)} \eqref{eq:exp_cos_func} and {\bf(right)} \eqref{eq:slower_decay_rational_func} in $d=8$ dimensions with and without $N(0,\sigma^2)$ noise for various $\sigma$.}
\label{fig:noise_experiment}
\end{figure}

%----------------------------------------------------
\section{Additional numerical results} 
\label{sec:add_numerical}
%----------------------------------------------------

In this section, we include additional observations on some of our experiments. Figure \ref{fig:relu_exp_cos_func_time_comp} displays the average run times for training ReLU DNNs of various sizes on function \eqref{eq:exp_cos_func} with $d=8$. Comparing errors from Fig.\ \ref{fig:relu_exp_cos_func_beta_comp} with these timing results, we note that the best performing architectures in accuracy often require less time to train than some of their shallower or narrower counterparts. We also note that some of the larger architectures for which we observe divergence also required longer training times, suggesting a careful choice of architecture is key for efficient representation. Figure \ref{fig:relu_slower_decay_rational_func_time_comp} plots the average run times for training DNNs on data from function \eqref{eq:slower_decay_rational_func}. Comparing these results with the timings on the previous function, we observe increased training times for this less smooth function.

\begin{figure}[ht]
\begin{center}
\includegraphics[width=0.23\paperwidth,clip=true,trim=0mm 0mm 0mm 0mm]{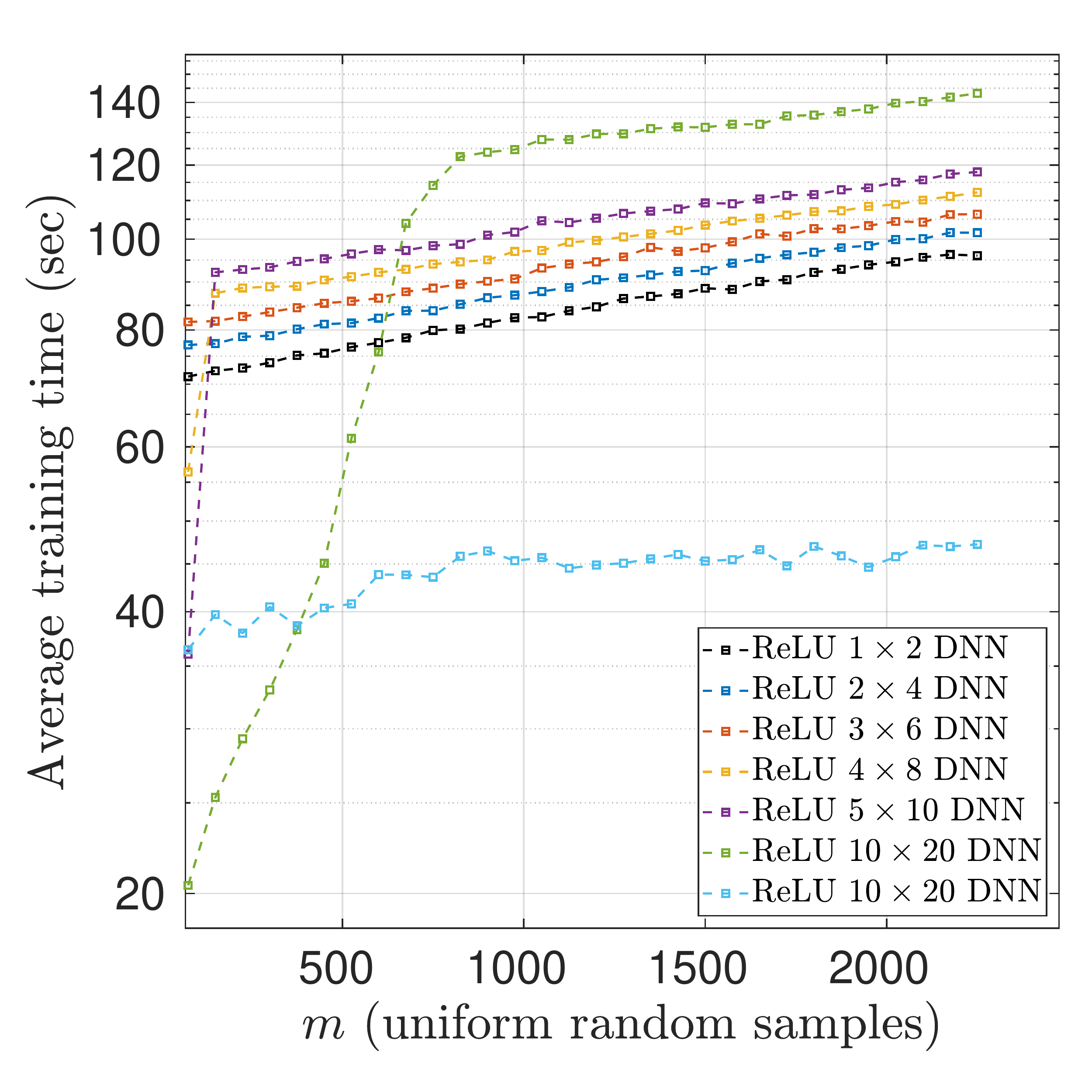}
\includegraphics[width=0.23\paperwidth,clip=true,trim=0mm 0mm 0mm 0mm]{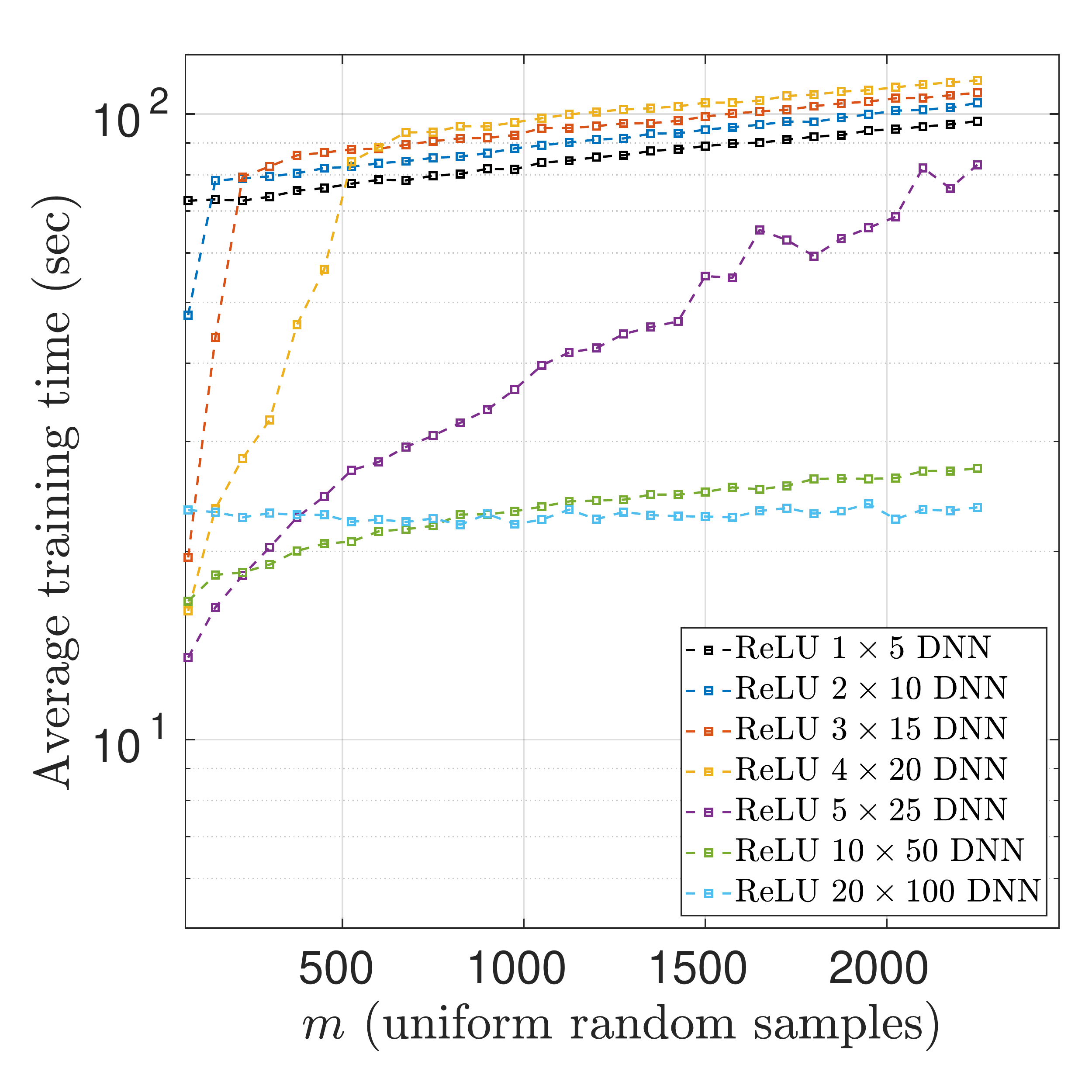}
\includegraphics[width=0.23\paperwidth,clip=true,trim=0mm 0mm 0mm 0mm]{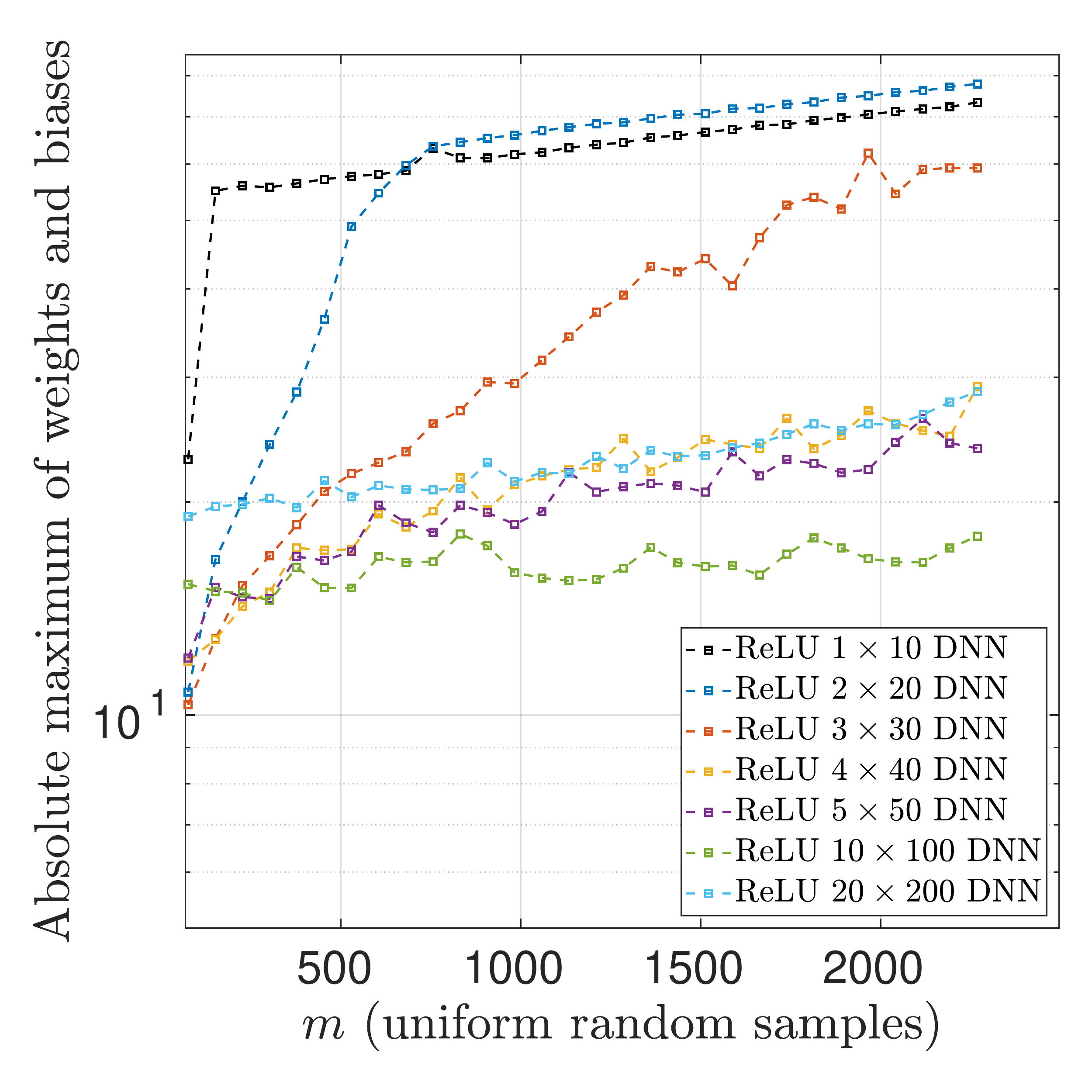}
\includegraphics[width=0.23\paperwidth,clip=true,trim=0mm 0mm 0mm 0mm]{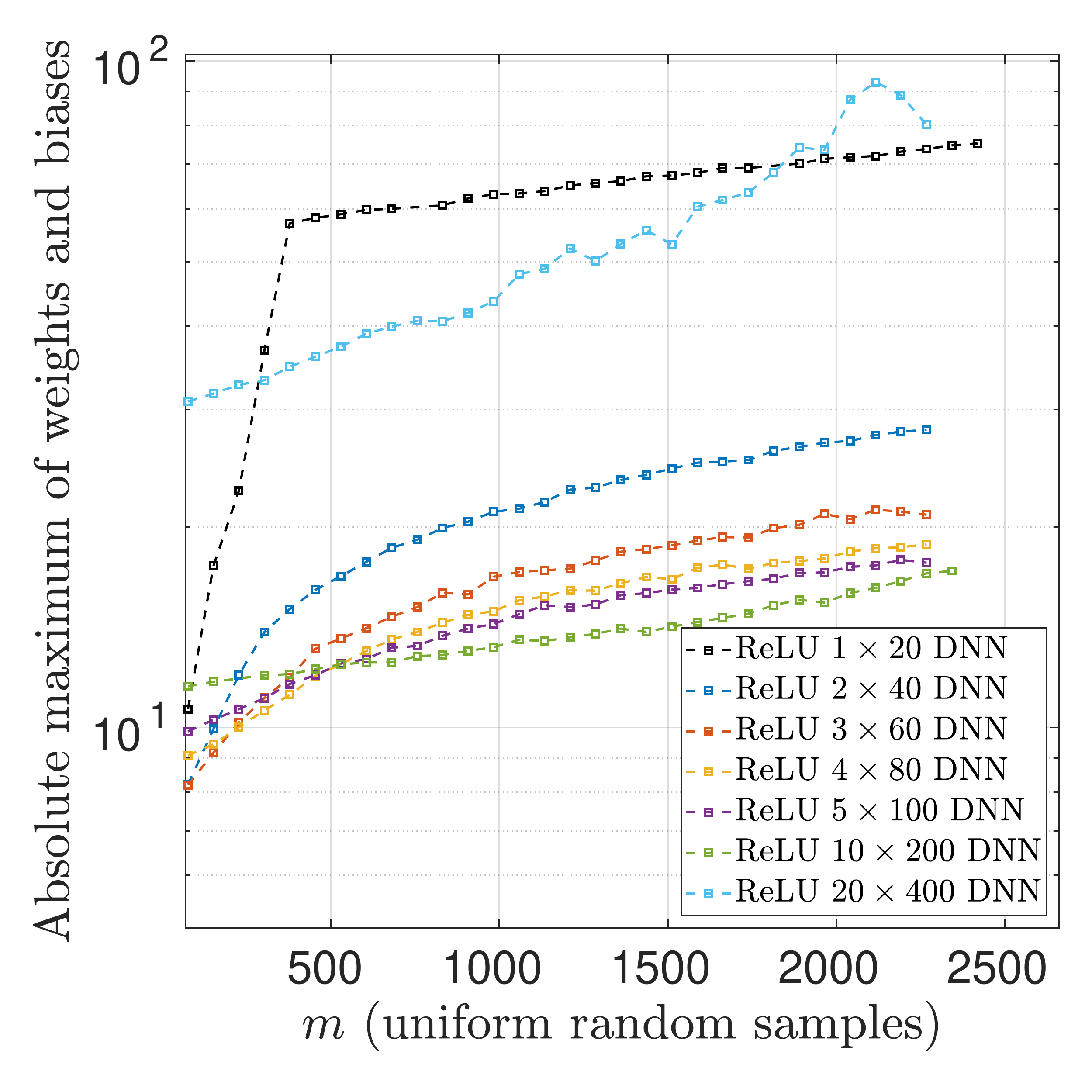}
\includegraphics[width=0.23\paperwidth,clip=true,trim=0mm 0mm 0mm 0mm]{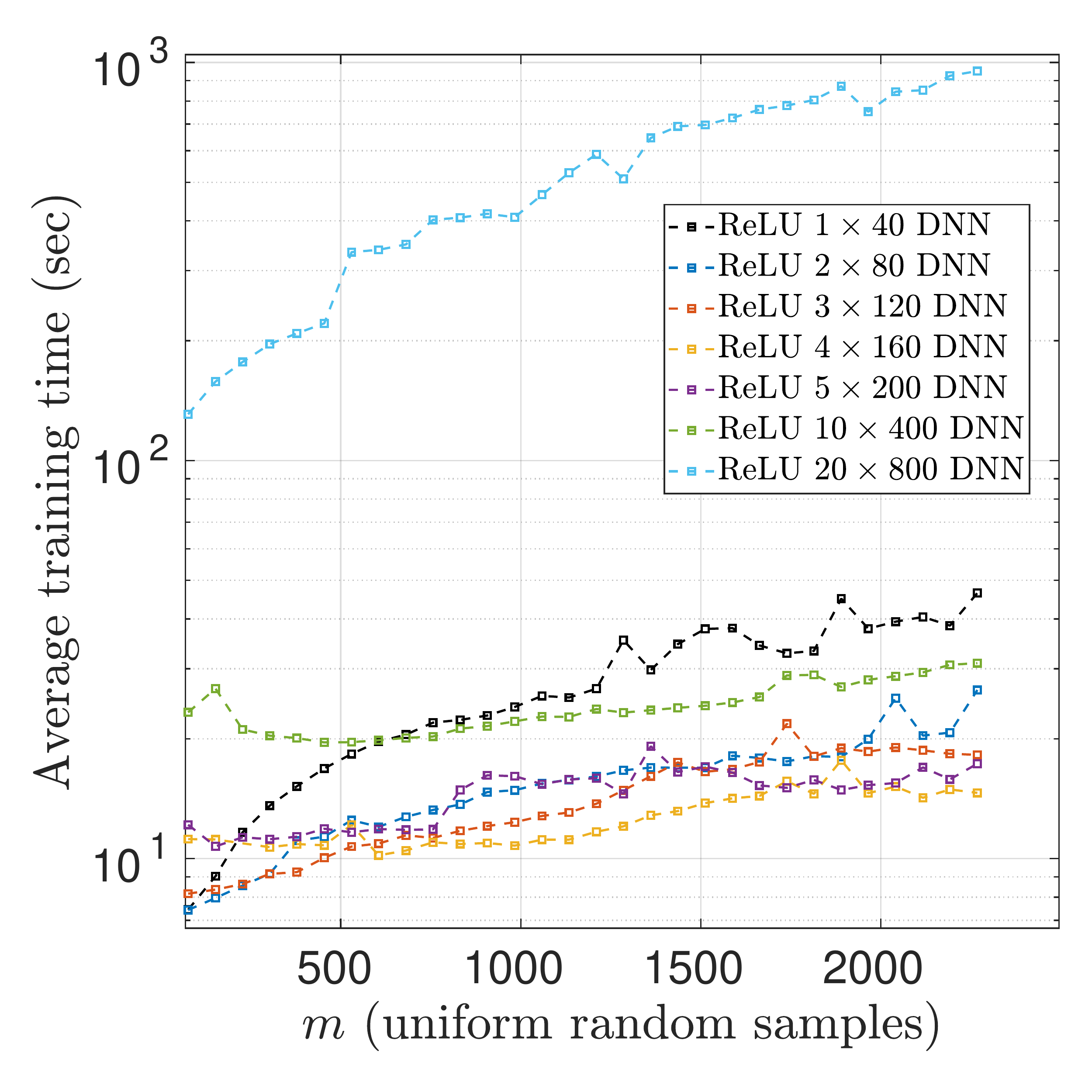}
\end{center}

\vspace{-2mm}
\caption{Comparison of training time vs. number of samples of function \eqref{eq:exp_cos_func} with $d=8$ used in training for ReLU architectures parameterized with $\beta = L/N$ (hidden layers/nodes per hidden layer) for values {\bf(top-left)} $\beta=0.5$, {\bf(top-middle)} $\beta=0.2$, {\bf(top-right)} $\beta=0.1$, {\bf(bottom-left)} $\beta=0.05$, and {\bf(bottom-right)} $\beta=0.025$.}
\label{fig:relu_exp_cos_func_time_comp}
\end{figure}

\begin{figure}[ht]
\begin{center}
\includegraphics[width=0.23\paperwidth,clip=true,trim=0mm 0mm 0mm 0mm]{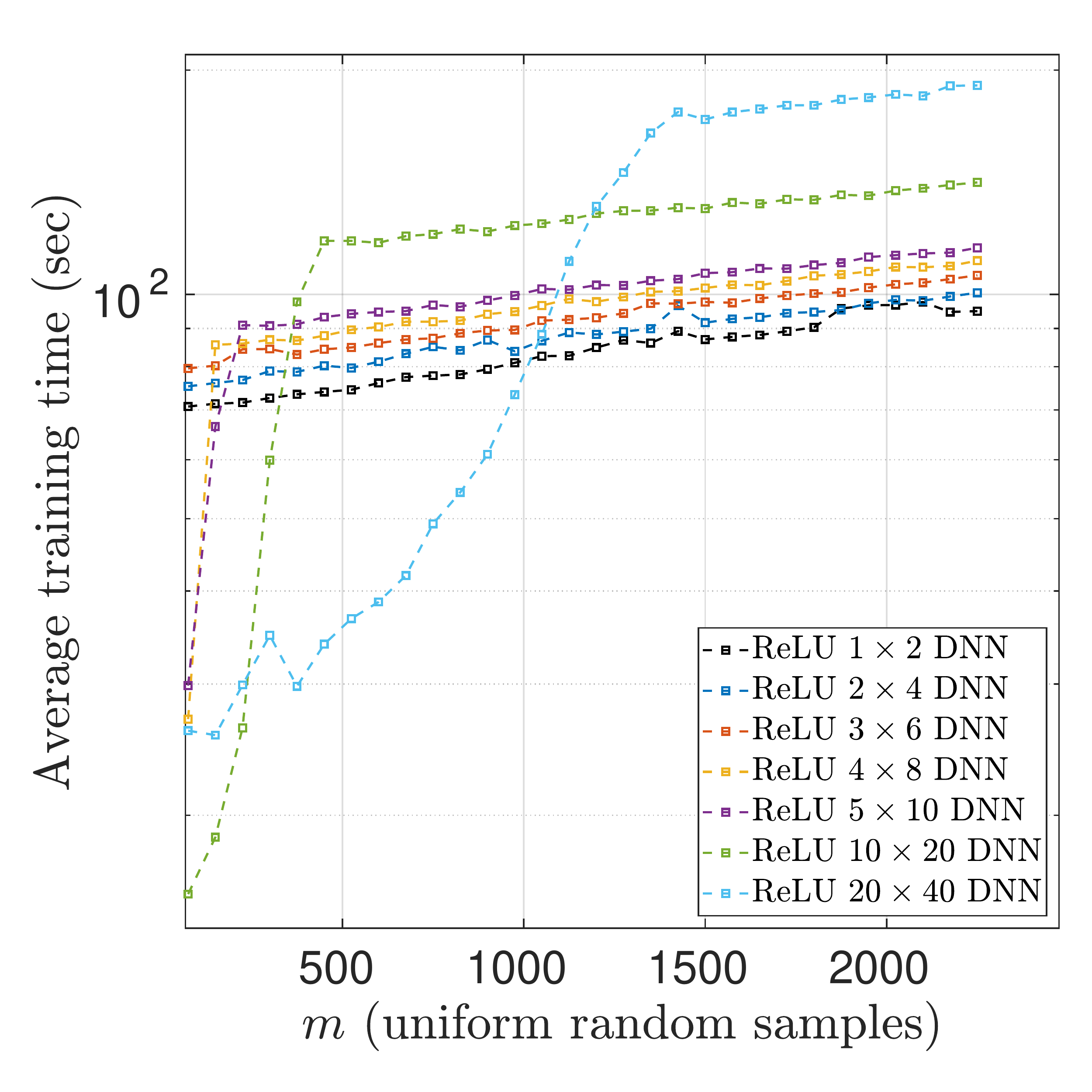}
\includegraphics[width=0.23\paperwidth,clip=true,trim=0mm 0mm 0mm 0mm]{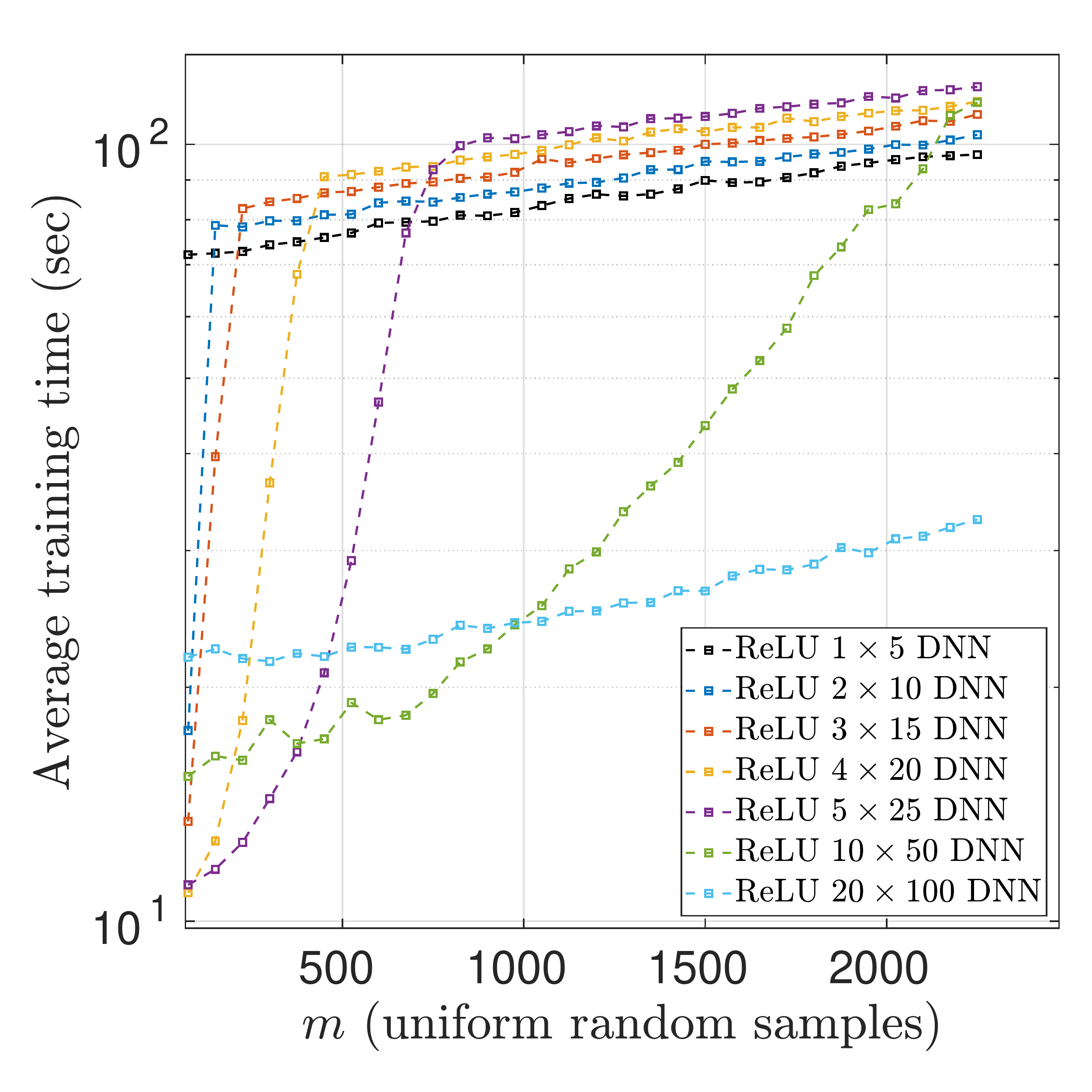}
\includegraphics[width=0.23\paperwidth,clip=true,trim=0mm 0mm 0mm 0mm]{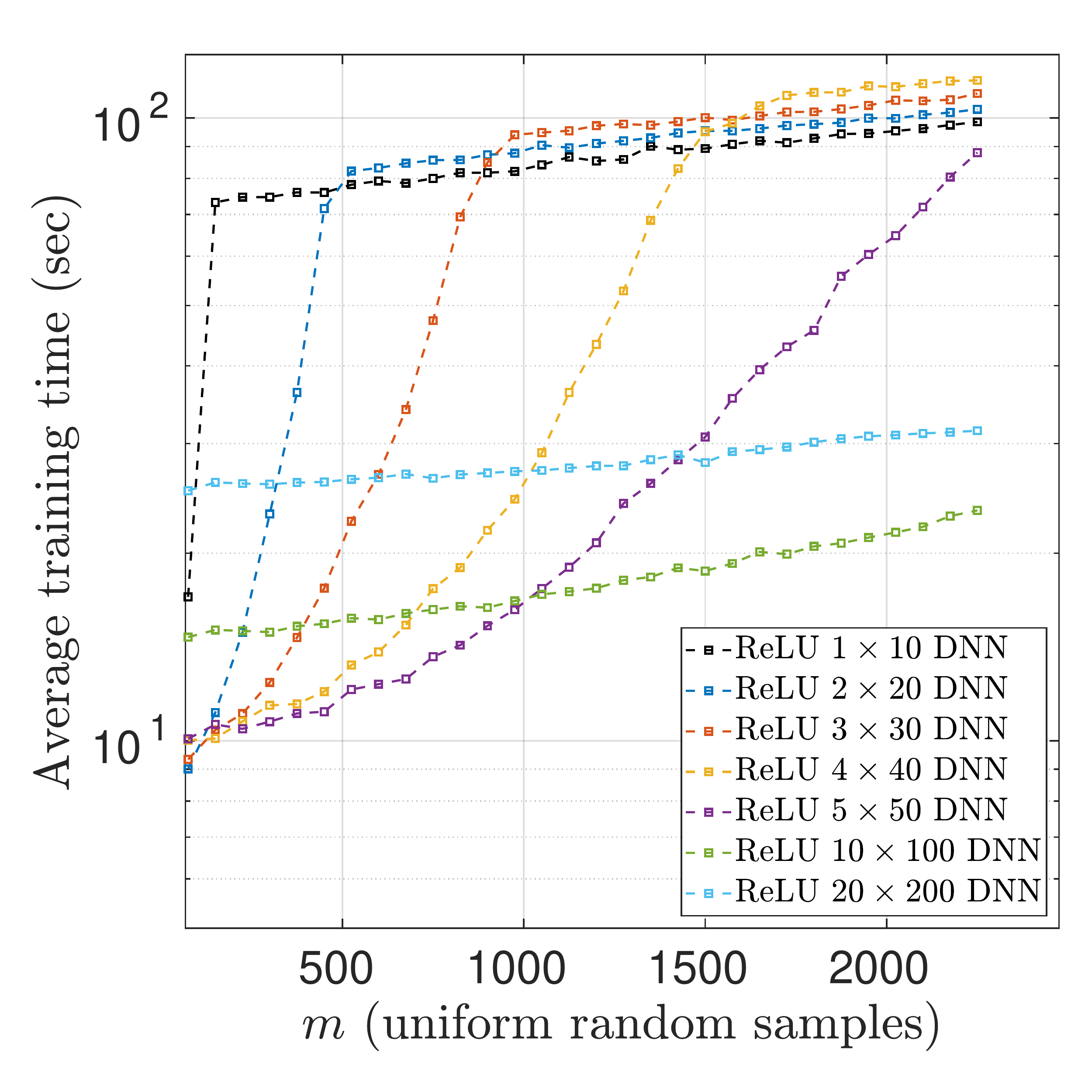}
\includegraphics[width=0.23\paperwidth,clip=true,trim=0mm 0mm 0mm 0mm]{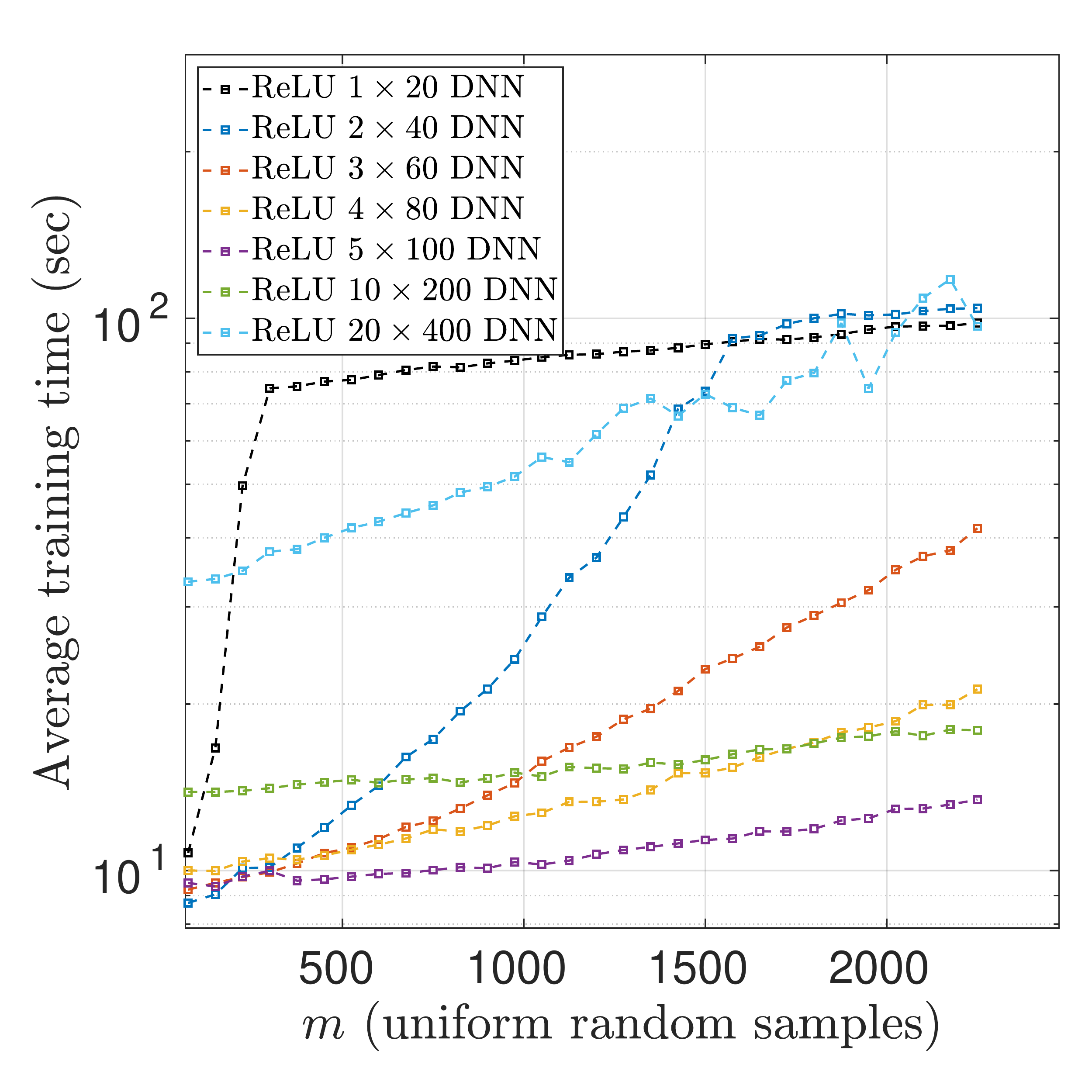}
\includegraphics[width=0.23\paperwidth,clip=true,trim=0mm 0mm 0mm 0mm]{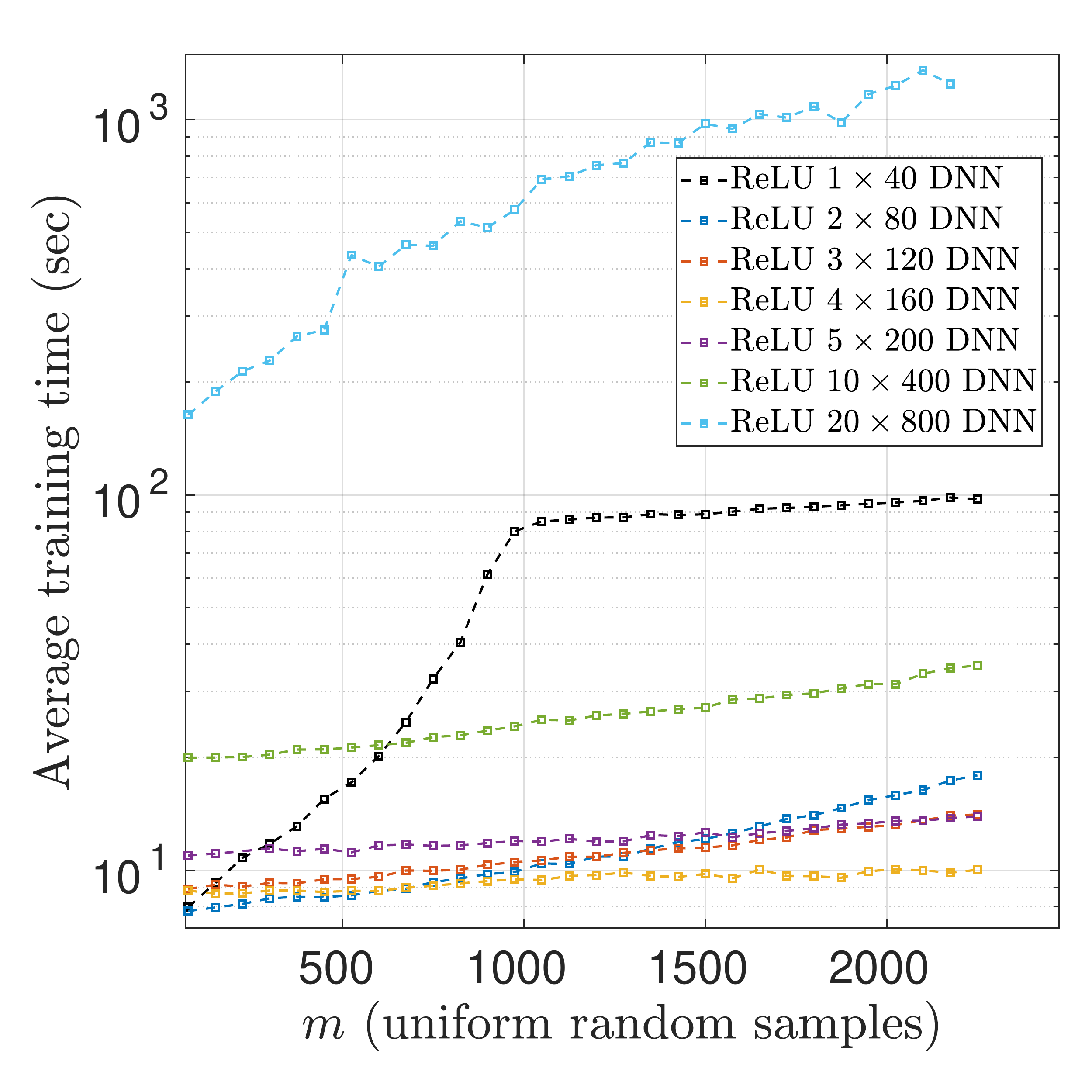}
\end{center}

\vspace{-2mm}
\caption{Comparison of training time vs. number of samples of function \eqref{eq:slower_decay_rational_func} with $d=8$ used in training for ReLU architectures parameterized with $\beta = L/N$ (hidden layers/nodes per hidden layer) for values {\bf(top-left)} $\beta=0.5$, {\bf(top-middle)} $\beta=0.2$, {\bf(top-right)} $\beta=0.1$, {\bf(bottom-left)} $\beta=0.05$, and {\bf(bottom-right)} $\beta=0.025$.}
\label{fig:relu_slower_decay_rational_func_time_comp}
\end{figure}

\subsection{Truncation parameters in compressed sensing}\label{sec:CSSM}

In our numerical experiments, we solve the weighted \eqref{eq:CS_BPDN_weighted} or unweighted \eqref{eq:CS_BPDN} quadratically-constrained basis problems with truncation parameter chosen as in \eqref{eta_opt_choice}.

It is well-known that the accuracy of the quadratically-constrained basis pursuit is affected by the choice of $\eta$, with it declining if $\eta$ is too large or too small \cite{BASBCSmodel}. In \eqref{eta_opt_choice}, we chose $\eta$ as small as possible such that the exact coefficients $\bm{c}_{\Lambda}$ are feasible for \eqref{eq:CS_BPDN}, which is in some sense an optimal choice. See \cite{ABBCorrecting} for further information.  In practice, when the coefficients $c_{\bm{\nu}}$ cannot be feasibly computed, an alternative is to use cross validation, see, e.g., \cite{DoostanOwhadiSparse}.

In our theoretical analysis in \S \ref{sec:theory1} we consider the weighted square-root LASSO problem \eqref{eq:CS_sr_min}.
As discussed above, the choice of $\eta$ in \eqref{eta_opt_choice} depends on the unknown expansion coefficients $\bm{c}_{\Lambda}$, and more precisely, the expansion tail since $\bm{f} - \bm{A} \bm{c}_{\Lambda} = \bm{e}$, where $\bm{e} =  \frac{1}{\sqrt{m}} \left ( \sum_{\bm{\nu} \notin \Lambda} c_{\bm{\nu}} \Psi_{\bm{\nu}}(\bm{x}_i) \right )^{m}_{i=1}$. Hence, previous theoretical error estimates for polynomial approximations via CS often involve unrealistic assumptions on the size of this term \cite{Adcock2016,ChkifaDexterTranWebster18}.
Much the same is true of the unconstrained LASSO \eqref{eq:CS_uncon_min}, or its weighted variant. This problem was studied in \cite{BASBCSmodel}, and in \cite{ABBCorrecting} the weighted square-root LASSO \eqref{eq:CS_sr_min} was proposed as a solution. While far less well known than the LASSO, the square-root LASSO has the beneficial property that the optimal choice of its parameter $\mu$ is independent of the noise term (i.e.\ the expansion tail), and depends only on the parameter $s$ (this can be seen in Theorem \ref{thm:CSexpapp}). Hence, as shown in Theorem \ref{thm:CSexpapp}, it allows for explicit convergence rate estimates for CS, without unrealistic assumptions being imposed on the expansion tail.

On the other hand, in our experiments we continue to use (weighted) quadratically-constrained basis pursuit because it is the standard approach in the literature.

% !TEX root = ./MLFA_supplement.tex

%----------------------------------------------------
\section{Proofs} 
\label{sec:proofs}
%----------------------------------------------------

In this section, we present the proofs of Theorems \ref{thm:LegExpOSZ}, \ref{thm:CSexpapp} and \ref{thm:CSasNN}.

\subsection{Proof of Theorem \ref{thm:LegExpOSZ}}

The proof is based on techniques from \cite[Sec.\ 3]{Opschoor2019Legendre}.

\begin{proof}[Proof of Theorem \ref{thm:LegExpOSZ}]
It is clear from \eqref{LegCoeffBd} that $\bm{c} \in \ell^1_{\bm{u}}(\N^d_0)$. Notice the following straightforward inequality
\bes{
\nm{f - \sum_{\bm{\nu} \in \Lambda} c_{\bm{\nu}} \Psi_{\bm{\nu}} }_{L^{\infty}} \leq \sum_{\bm{\nu} \notin \Lambda} u_{\bm{\nu}} | c_{\bm{\nu}} | .
}
Hence it suffices to consider the right-hand side.
Without loss of generality $s \geq 2$.  Choose $0 < \epsilon < 1$ such that
\bes{
s = \prod^{d}_{j=1} \left ( \frac{\log(\epsilon^{-1})}{\log(\rho_j)} + 1 \right ),
}
and define the lower set $\Lambda = \left \{ \bm{\nu} : \bm{\rho}^{-\bm{\nu}} \geq \epsilon \right \}$.  In the proof of Theorem 3.5 in \cite{Opschoor2019Legendre} it is shown that
\ben{
\label{OSZbd}
\sum_{\bm{\nu} \notin \Lambda} u_{\bm{\nu}} | c_{\bm{\nu}} | \leq C \exp \left ( -\beta |\Lambda |^{1/d} \right ),
}
for any $\beta$ satisfying
\bes{
0 < \beta < \left ( d! \prod^{d}_{j=1} \log(\rho_j) \right )^{1/d},
}
where $C > 0$ depends on $d$, $\bm{\rho}$, $\beta$ and $f$ only.
We now derive upper and lower bounds for $|\Lambda|$ in terms of $s$.  First, observe that
\bes{
\Lambda =  \left \{ \bm{\nu} \in \N^d_0 : \sum^{d}_{j=1} \nu_j \log(\rho_j) \leq \log(\epsilon^{-1}) \right \},
}
and therefore
\bes{
\prod^{d}_{j=1} \left ( 1 + \left \lfloor \frac{\log(\epsilon^{-1})}{d \log(\rho_j) } \right \rfloor \right ) \leq | \Lambda| \leq \prod^{d}_{j=1} \left ( 1 + \left \lfloor \frac{\log(\epsilon^{-1})}{\log(\rho_j) } \right \rfloor \right ).
}
Hence
\bes{
s \geq \prod^{d}_{j=1} \left ( \left \lfloor \frac{\log(\epsilon^{-1})}{\log(\rho_j)} \right \rfloor + 1 \right ) \geq |\Lambda |,
}
and
\bes{
s \leq \prod^{d}_{j=1} \left ( d \left \lfloor \frac{\log(1/\epsilon)}{d \log(\rho_j) } \right \rfloor + d + 1 \right ) = \prod^{d}_{j=1} \left ( \left \lfloor \frac{\log(1/\epsilon)}{d \log(\rho_j) } \right \rfloor  + 1\right )  \prod^{d}_{j=1} \left ( d + \frac{1}{\left \lfloor \frac{\log(1/\epsilon)}{d \log(\rho_j) } \right \rfloor  + 1 } \right ),
}
which gives $s  \leq |\Lambda| (d+1)^d$.  Therefore
\bes{
s (d+1)^{-d} \leq | \Lambda | \leq  s.
}
Returning to \eqref{OSZbd}, we deduce that
\bes{
\sigma_{s}(\bm{c})_{1,\bm{u}} \leq C \exp \left ( -\beta |\Lambda |^{1/d} \right ) \leq C \exp \left ( - \beta s^{1/d} / (d+1) \right ) = C \exp \left ( -\gamma s^{1/d} \right ).
}
This completes the proof.
\end{proof}

\subsection{Compressed sensing for lower set recovery}

The proofs of Theorem \ref{thm:CSexpapp} and \ref{thm:CSasNN} require some elements of compressed sensing theory, which we now introduce. 
We note that many of the constructions developed below apply more generally (for instance, to other measures $\varrho$ and other orthonormal systems).  However, for simplicity, we focus only on the case of Legendre polynomials.  What follows is based primarily on \cite{ABBCorrecting,Adcock2016,ChkifaDexterTranWebster18,RW15}.

Since our focus is on lower set recovery, the setup differs to the standard compressed sensing framework (see, e.g., \cite{FouRau13}) for the recovery of arbitrary $s$-sparse vectors.  Let $s \geq 1$ and recall that the union of all lower sets of cardinality at most $s$ is the hyperbolic cross index set $\Lambda = \Lambda^{\mathrm{HC}}_s$, defined by \eqref{HCindex}.  Write $n = |\Lambda^{\mathrm{HC}}_s|$.  Throughout this section, we consider vectors in $\bbC^n$ indexed over $\Lambda$.

Define \emph{intrinsic lower sparsity} of order $s$ by
\begin{equation}
\label{eq:defK(s)}
K(s) := \max\left\{|S|_{\bm{u}} : S \subseteq \Lambda,\ |S|\leq s,\ S \text{ lower}\right\},
\end{equation}
where 
\begin{equation}
\label{eq:def|S|_u}
|S|_{\bm{u}}:=\sum_{\bm{\nu} \in S} u_{\bm{\nu}}^2
\end{equation} 
is the \emph{weighted cardinality} of a subset $S$ with respect to the weights $\bm{u}$ \cite{RW15}.  Note that $K(s)$ is bounded, and satisfies
\ben{
\label{Ksbound}
s^2 / 4 \leq K(s) \leq s^2,
}
for weights $\bm{u}$ as in \eqref{eq:uweightsLeg}. See, for example, \cite[Lem.\ 2.2]{ABBCorrecting}.  Given this, we define the best $s$-term and lower approximation error as
\bes{
\sigma_{s,L}(\bm{c})_{1,\bm{u}} = \inf \left \{ \nm{\bm{c} - \bm{c}_{S} }_{1,\bm{u}} : S \subset \N^d_0,\ |S|_{\bm{u}} \leq K(s) \right \}.
}
Note that here and henceforth, we use $\bm{c}_{S}$ to denote either the vector $\bm{c}_{S} \in \bbC^n$ with $\bm{\nu}$th entry equal to $c_{\bm{\nu}}$ if $\bm{\nu} \in S$ and zero otherwise, or the vector $\bm{c}_{S} = (c_{\bm{\nu}})_{\bm{\nu} \in \Lambda} \in \bbC^{|\Lambda|}$.  The precise meaning will be clear from the context.

We now require the following (see \cite{ChkifaDexterTranWebster18} or \cite[Defn.\ 5.3]{ABBCorrecting}):
\begin{definition}[Lower robust null space property]
\label{def:lowerNSP}
Given $0 < \rho < 1$ and $\tau > 0$, a matrix $\bm{A} \in \mathbb{C}^{m \times n}$ is said to have the \emph{lower robust null space property (lower rNSP)} of order $s$ if
$$
\|\bm{z}_{S}\|_2 \leq \frac{\rho}{\sqrt{K(s)}} \|\bm{z}_{S^c}\|_{1,\bm{u}} + \tau \|\bm{A} \bm{z}\|_2, \quad \forall \bm{z} \in \mathbb{C}^n,
$$
for any $S \subseteq \Lambda$ such that $|S|_{\bm{u}} \leq K(s)$, where $K(s)$ is defined as in \eqref{eq:defK(s)}.
\end{definition}

The lower rNSP is sufficient to provide a recovery guarantee for the weighted square-root LASSO decoder \eqref{eq:CS_sr_min}.  In fact, although we shall not do it, this property also provides recovery guarantees for the decoders \eqref{eq:CS_BPDN} and \eqref{eq:CS_uncon_min}; see \cite{ABBCorrecting}.

\begin{theorem}
\label{thm:CSerrunderrNSP}
Suppose that $\bm{A} \in \bbC^{m \times n}$ satisfies the lower rNSP of order $s$ with constants $0 < \rho < 1$ and $\tau > 0$. Let $\bm{c} \in \mathbb{C}^n$ and $\bm{y} = \bm{A} \bm{c} + \bm{e} \in \mathbb{C}^m$ for some $\bm{e} \in \bbC^m$ and consider the the weighted square-root LASSO problem \eqref{eq:CS_sr_min} with parameter
\begin{equation*}
\mu \geq \frac{2 \tau}{1+\rho} \sqrt{K(s)}.
\end{equation*}
Then
\begin{equation*}
\begin{split}
\nmu{\bm{c} - \bm{\hat{c}}}_{1,\bm{u}} &\leq  2  \frac{1+\rho}{1-\rho} \sigma_{s,L}(\bm{c})_{1,\bm{u}} + \left ( \frac{1+\rho}{1-\rho} \mu + \frac{2 \tau \sqrt{K(s)}}{1-\rho} \right )   \nm{\bm{e}}_{2}.
\end{split}
\end{equation*}
\end{theorem}
\begin{proof}
By \cite[Thm.\ 5.6]{ABBCorrecting}, we have
\bes{
\nmu{\bm{c} - \bm{\hat{c}}}_{1,\bm{u}} \leq \frac{1+\rho}{1-\rho} \left ( 2 \sigma_{s,L}(\bm{c})_{1,\bm{u}}  + \nmu{\bm{\hat{c}}}_{1,\bm{u}} - \nmu{\bm{c}}_{1,\bm{u}} \right ) + \frac{2 \tau \sqrt{K(s)}}{1-\rho} \nmu{\bm{A} ( \bm{\hat{c}} - \bm{c} ) }_{2}.
}
Since $\bm{\hat{c}}$ is a minimizer, we obtain
\eas{
\nmu{\bm{c} - \bm{\hat{c}}}_{1,\bm{u}} \leq & 2  \frac{1+\rho}{1-\rho}\sigma_{s,L}(\bm{c})_{1,\bm{u}} + \frac{1+\rho}{1-\rho} \mu \left (\nmu{\bm{A} \bm{\hat{c}} - \bm{f}  }_{2}  -  \nmu{\bm{A} \bm{\hat{c}} - \bm{f}  }_{2} \right )
\\
& + \frac{2 \tau \sqrt{K(s)}}{1-\rho} \left (  \nmu{\bm{A} \bm{\hat{c}} - \bm{f}  }_{2}  + \nmu{\bm{A} \bm{\hat{c}} - \bm{f}  }_{2} \right ),
}
and by assumption on $\mu$ we deduce that
\bes{
\nmu{\bm{c} - \bm{\hat{c}}}_{1,\bm{u}} \leq   2  \frac{1+\rho}{1-\rho}\sigma_{s,L}(\bm{c})_{1,\bm{u}} + \left ( \frac{1+\rho}{1-\rho} \mu + \frac{2 \tau \sqrt{K(s)}}{1-\rho} \right )  \nmu{\bm{A} \bm{\hat{c}} - \bm{f}} _{2} ,
}
as requred.
\end{proof}

We note in passing one can also provide recovery guarantees in the $2$-norm.  See, for instance, \cite{ABBCorrecting}.  In practice, it is difficult to work directly with the lower rNSP.  Hence we consider the following (see \cite{ChkifaDexterTranWebster18} or \cite[Defn.\ 5.3]{ABBCorrecting}):

\begin{definition}[Lower restricted isometry property] 
\label{def:lowerRIP}
A matrix $\bm{A} \in \mathbb{C}^{m \times n}$ is said to have the \emph{lower restricted isometry property} of order $s$ if there exists a constant $0 < \delta < 1$ such that 
$$
(1-\delta) \|\bm{z}\|_2^2 
\leq \|\bm{A}\bm{z}\|_2^2 
\leq (1+\delta) \|\bm{z}\|_2^2, \quad 
\forall \bm{z} \in \mathbb{C}^n, \; |\supp(\bm{z})|_{\bm{u}} \leq K(s),
$$
where $\supp(\bm{z}) := \{ \bm{\nu} \in \Lambda: z_{\bm{\nu}} \neq 0\}$ and $|\supp(\bm{z})|_{\bm{u}}$ is its weighted cardinality defined as in \eqref{eq:def|S|_u}. The smallest constant such that this property holds is called the $s^{\text{th}}$ lower restricted isometry constant of $\bm{A}$ and it is denoted as $\delta_{s,L}$. 
\end{definition}

The following result, see \cite{ChkifaDexterTranWebster18} or \cite[Lem.\ 5.4]{ABBCorrecting}, asserts that the lower restricted isometry property is a sufficient condition for the lower rNSP:

\begin{lemma}
\label{lem:lRIPimplieslrNSP}
Let $s \geq 2$ and suppose $\bm{A} \in \bbC^{m \times n}$ satisfies the lower restricted isometry property (with $K(s)$ as in \eqref{eq:defK(s)} for weights \eqref{eq:uweightsLeg}) of order $2 s$ with constant $\delta_{2s,L} < 1/5$.  Then $\bm{A}$ has the lower rNSP of order $s$ with constants $\rho = \frac{4 \delta}{1-\delta}$ and $\tau = \frac{\sqrt{1+\delta}}{1-\delta}$.
\end{lemma}

Finally, we also need a result asserting the lower restricted isometry property for the measurement matrix \eqref{eq:A_f_def}.  The following result was first proved in \cite{ChkifaDexterTranWebster18}.  See, for example, \cite[Thm.\ 5.5]{ABBCorrecting}:
\begin{theorem}
\label{thm:AlowerRIP}
 Let $0 < \delta,\varepsilon < 1$ and suppose that
\bes{
m \geq C \cdot K(s) \cdot L(s,\delta,\varepsilon),
}
where $K(s)$ is as in \eqref{eq:defK(s)}, $C > 0$ is a universal constant,
\bes{
L(s,\delta,\varepsilon)
= \frac{1}{\delta^2}\ln\left(\frac{K(s)}{\delta^2}\right)
\max\left\{\frac{1}{\delta^4} \ln\left(\frac{K(s)}{\delta^2}\ln\left(\frac{K(s)}{\delta^2}\right)\right) \ln(n),
\frac{1}{\delta}\ln\left(\frac{1}{\delta\varepsilon} \ln\left(\frac{K(s)}{\delta^2}\right)\right)\right\}.
}
and $n = | \Lambda^{\mathrm{HC}}_{s} |$.  Let $\bm{x}_1,\ldots,\bm{x}_m$ be drawn independently according to the uniform measure on $\cU$ and consider the matrix $\bm{A}$ defined in \eqref{eq:A_f_def}, where $\{ \Psi_{\bm{\nu}} \}_{\bm{\nu} \in \N^d_0}$ is the orthonormal tensor Legendre polynomial basis.  Then, with probability at least $1 - \varepsilon$, $\bm{A}$ satisfies the lower RIP of order $s$ with constant $\delta_{s,L} \leq \delta$.
\end{theorem}

\subsection{Proof of Theorem \ref{thm:CSexpapp}}

We now give the proof of Theorem \ref{thm:CSexpapp}.  First, we recall the following inequality for the cardinality of the hyperbolic cross index set:
\ben{
\label{HCsize}
n = | \Lambda^{\mathrm{HC}}_{s} | \leq \min \left\{ 2 s^3 4^d , e^2 s^{2+\log_2(d)}, \frac{s(\log(s)+d\log(2))^{d-1}}{(d-1)!}  \right\}
}
The first and third bounds are due to Theorems 3.7 and 3.5 of \cite{chernov2016new} respectively. The second bound is due to \cite[Thm.\ 4.9]{kuhn2015approximation}

\begin{proof}[Proof of Theorem \ref{thm:CSexpapp}]
We claim that, with probability at least $1-\varepsilon$, the matrix $\bm{A}$ has the lower restricted isometry property of order $2 s$ with constant $\delta \leq \delta_{2s,L} = 1/6$.  Suppose this claim is true. Then Lemma \ref{lem:lRIPimplieslrNSP} gives that it satisfies the lower rNSP with constants $\rho = 4/5$ and $\tau = \sqrt{42}/5$. Also, by this and \eqref{Ksbound},
\bes{
\frac{2 \tau}{1+\rho} \sqrt{K(s)} \leq \frac{12 \sqrt{42}}{35} s = \mu.
}
Hence Theorem \ref{thm:CSerrunderrNSP} gives
\bes{
\nm{f - \tilde{f}}_{L^{\infty}(\cU)} \leq \nm{\bm{c} - \bm{\hat{c}}}_{1,\bm{u}} \leq  C \left ( \sigma_{s,L}(\bm{c})_{1,\bm{u}} + s \nmu{\bm{e}}_2 \right ).
}
Recall by definition that
\bes{
\nm{\bm{e}}_2 = \sqrt{\frac1m \sum^{m}_{i=1} \left| f(x_i) - \sum_{\bm{\nu} \in \Lambda} c_{\bm{\nu}} \Psi_{\bm{\nu}}(x_i) \right|^2 } \leq \nm{f - \sum_{\bm{\nu} \in \Lambda} c_{\bm{\nu}} \Psi_{\bm{\nu}}}_{L^{\infty}(\cU)} \leq \nm{\bm{c} - \bm{c}_{\Lambda}}_{1,\bm{u}} \leq \sigma_{s,L}(\bm{c})_{1,\bm{u}}.
}
Hence, by Theorem \ref{thm:LegExpOSZ},
\bes{
\nm{f - \tilde{f}}_{L^{\infty}(\cU)} \leq C (1+s)   \exp \left ( -\gamma s^{1/d} \right ).
}
Since this holds for all $\gamma$ satisfying  \eqref{polyrategamma}, and since the exponential term dominates as $s \rightarrow \infty$, we deduce (after possible change of $C$) that 
\bes{
\nm{f - \tilde{f}}_{L^{\infty}(\cU)} \leq C  \exp \left ( -\gamma s^{1/d} \right ).
}
To complete the proof, it remains to prove the claim.  Let $L(2s,\delta_{2s,L},\varepsilon)$ be the log factor in Theorem \ref{thm:AlowerRIP}.  Then, since $\delta_{2s,L} = 1/6$ and $K(s) \leq s^2$ by \eqref{Ksbound}, we have
\eas{
L(2s,\delta_{2s,L},\varepsilon) \leq c \log(2 s) \max \left \{ \log \left ( 2 s \log(2 s) \right ) \log(n) , \log(2 \varepsilon^{-1} \log(2 s) ) \right \}.
}
Note that $\log(2s \log(2s)) \leq \log(4 s^2) = 2 \log(2s)$.
Using this and the estimate \eqref{HCsize} for $n = |\Lambda^{\mathrm{HC}}_{s}|$ we obtain
\eas{
L(2s,\delta_{2s,L},\varepsilon) \leq c \log(2s) \left ( \log^2(2s) \log(2d) + \log(2 \varepsilon^{-1} \log(2 s)) \right ).
}
Observe that $L_{m,d,\varepsilon} \geq c'$ for some universal constant $c' > 0$ and therefore $s \leq c'' m$ for some universal constant $c'' > 0$.  It follows that
\bes{
L(2s,\delta_{2s,L},\varepsilon) \leq c  \left ( \log^2(2m) \log(2d) + \log(2 \varepsilon^{-1} \log(2 m)) \right ) = c L_{m,d,\varepsilon},
}
for possibly different constant $c > 0$.  Hence
\bes{
m \geq c \cdot s^2 \cdot L_{m,d,\varepsilon} \geq c \cdot s^2 \cdot L(2s,\delta_{2s,L},\varepsilon).
}
The claim now follows immediately from Theorem \ref{thm:AlowerRIP}.
\end{proof}

\subsection{Proof of Theorem \ref{thm:CSasNN}}

We make use of the following result, which can be found in \cite[Prop.\ 2.13]{Opschoor2019Legendre}:

\begin{proposition}
\label{prop:polyNN}
For every finite subset $\Lambda \subset \N^d_0$ and every $ 0 < \delta < 1$ there exists a ReLU neural network $\Phi_{\Lambda,\delta} : \bbR^d \rightarrow \bbR^{|\Lambda|}$ such that, if $\Phi_{\Lambda,\delta} = (\Phi_{\bm{\nu},\delta})_{\bm{\nu} \in \Lambda}$, then
\bes{
\nm{\Psi_{\bm{\nu}} - \Phi_{\bm{\nu},\delta} }_{L^{\infty}(\cU)} \leq \delta.
}
The depth and size of this network satisfy
\eas{
\mathrm{depth}(\Phi_{\Lambda,\delta}) &\leq c \left ( 1 + d \log(d) \right ) \left ( 1 + \log(m(\Lambda)) \right ) \left ( m(\Lambda) + \log(\delta^{-1}) \right )
\\
\mathrm{size}(\Phi_{\Lambda,\delta}) & \leq c \left ( d^2 m(\Lambda)^2 + d m(\Lambda) \log(\delta^{-1}) + d^2 | \Lambda | \left (1 + \log(m(\Lambda)) + \log(\delta^{-1}) \right ) \right ),
}
where $m(\Lambda) = \max_{\bm{\nu} \in \Lambda} \nm{\bm{\nu}}_1$ and $c>0$ is a universal constant.
\end{proposition}

The general idea of the proof is to use Proposition \ref{prop:polyNN} to approximately express matrix-vector multiplication $\bm{A} \bm{z}$ as a neural network $\Phi$ evaluated at the sample points $\bm{x}_i$ and then use the compressed sensing results to establish an error bound.  Since this process commits an error, we first require the following result, which shows that the lower rNSP is robust to small matrix perturbations:

\begin{lemma}
\label{lem:rNSPapproxmat}
Suppose that $\bm{A} \in \bbC^{m \times n}$ satisfies the lower rNSP of order $s$ with constants $0 < \rho < 1$ and $\tau > 0$ and let $\bm{A'}\in \bbC^{m \times n}$ satisfy
\begin{equation*}
\nmu{\bm{A} - \bm{A'}}_2 \leq \sigma,\qquad \sigma < \frac{1-\rho}{\tau(\sqrt{K(s)} + 1)}.
\end{equation*}
Then $\bm{A'}$ satisfies the lower rNSP of order $s$ with constants $0 < \rho' < 1$ and $\tau$, where 
\begin{equation*}
\rho' \leq \frac{\rho + \tau \sigma \sqrt{K(s)} }{1-\tau \sigma}.
\end{equation*}
\end{lemma}
\begin{proof}
Let $\bm{z} \in \mathbb{C}^n$ and $S$ satisfy $|S|_{\bm{u}} \leq K(s)$.  Then
\eas{
\|\bm{z}_{S}\|_2 & \leq \frac{\rho}{\sqrt{K(s)}} \|\bm{z}_{S^c}\|_{1,\bm{u}} + \tau \|\bm{A} \bm{z}\|_2
\\
& \leq \frac{\rho}{\sqrt{K(s)}} \|\bm{z}_{S^c}\|_{1,\bm{u}} + \tau \nmu{\bm{A'} \bm{z}}_{2} + \tau \sigma \nmu{\bm{z}}_{2}
\\
& \leq \frac{\rho}{\sqrt{K(s)}} \|\bm{z}_{S^c}\|_{1,\bm{u}}  + \tau \nmu{\bm{A'} \bm{z}}_{2} + \tau \sigma \nmu{\bm{z}_S}_{2}+ \tau \sigma \nmu{\bm{z}_{S^c}}_{2}.
}
Hence
\bes{
(1-\tau \sigma ) \|\bm{z}_{S}\|_2 \leq \left ( \frac{\rho}{\sqrt{K(s)}}  + \tau \sigma \right ) \|\bm{z}_{S^c}\|_{1,\bm{u}} + \tau \nmu{\bm{A'} \bm{z}}_{2}.
}
The result now follows immediately.
\end{proof}

\begin{proof}[Proof of Theorem \ref{thm:CSasNN}]
Let $\Lambda = \Lambda^{\mathrm{HC}}_{s}$ and $\Phi_{\Lambda,\delta}$ be as in Proposition \ref{prop:polyNN}. For the moment, we do not specify the choice of $\delta \in (0,1)$. This will be done at the end of the proof.
Now define the family of neural networks
\bes{
\mathcal{N} = \left \{ \Phi : \bm{x} \mapsto \bm{z}^{\top} \Phi_{\Lambda,\delta}(\bm{x}),\ \bm{z} \in \bbR^n \right \}.
}
Notice that this family has $n$ trainable parameters, $\mathrm{depth}(\Phi) = \mathrm{depth}(\Phi_{\Lambda,\delta}) + 1$ and $\mathrm{size}(\Phi) = \mathrm{size}(\Phi_{\Lambda,\delta})  + n$ for $\Phi \in \mathcal{N}$.  For $\Phi \in \mathcal{N}$, let $\cJ(\Phi) = \mu^{-1} \nm{\bm{z}}_{1,\bm{u}}$, where $\mu = \frac{4\sqrt{42}}{19} s$. Then observe that 
\bes{
\Phi(\bm{x}) = \sum_{\bm{\nu} \in \Lambda} z_{\bm{\nu}} \Phi_{\bm{\nu},\delta}(\bm{x}).
}
and therefore
\bes{
\cL(\Phi) = \nm{\bm{A'} \bm{z} - \bm{f}}_2 + \mu^{-1} \nm{\bm{z}}_{1,\bm{u}},\qquad \bm{A'} = \frac{1}{\sqrt{m}} \left ( \Phi_{\bm{\nu},\delta}(\bm{x}_i) \right )_{\substack{1 \leq i \leq m \\ \bm{\nu} \in \Lambda}}.  
}
Hence $\hat{\Phi} =  \bm{\hat{c}}^{\top} \Phi_{\Lambda,\delta}$ is a minimizer of $\cL$ over $\mathcal{N}$ if and only if $\bm{\hat{c}}$ is a minimizer of
\ben{\label{LASSO_Ap}
\mathrm{minimize}_{\bm{z} \in \bbR^n} \nm{\bm{z}}_{1,\bm{u}} +  \mu \nm{\bm{A'} \bm{z} - \bm{f}}_2 .
}
We seek to use Lemma \ref{lem:rNSPapproxmat}.  Observe that 
\eas{
\nm{(\bm{A} - \bm{A'}) \bm{z} }^2_2 &= \frac1m \sum^{m}_{i=1} \left | \sum_{\bm{\nu} \in \Lambda} z_{\bm{\nu}} \left ( \Psi_{\bm{\nu}}(\bm{x}_i) -  \Phi_{\bm{\nu},\delta}(\bm{x}_i) \right ) \right |^2 
\\
&\leq  \sum_{\bm{\nu} \in \Lambda} \nm{\Psi_{\bm{\nu}} - \Phi_{\bm{\nu},\delta} }^2_{L^{\infty}(\cU)} \nm{\bm{z}}^2_2 \leq n \delta^2 \nm{\bm{z}}^2_2.
}
Hence,  we require $\delta$ to satisfy
\bes{
\nm{\bm{A} - \bm{A'}}_2 \leq \sqrt{n} \delta \leq \sigma : = \frac{5/\sqrt{42}}{9 + 10 s }.
}
As in the proof of Theorem \ref{thm:CSexpapp}, the conditions on $m$ assert that $\bm{A}$ has the lower rNSP of order $s$ with constants $\rho = 4/5$ and $\tau = \sqrt{42}/5$.  Therefore, by Lemma \ref{lem:rNSPapproxmat}, $\bm{A'}$ has the lower rNSP of order $s$ with constants $\rho'$ and $\tau = \sqrt{42}/5$ with  probability at least $1 -\epsilon$, where
\bes{
\rho' \leq \frac{\rho + \tau \sigma s}{1-\tau \sigma}  \leq 9/10.
}
Now we derive an error bound for $f - \hat{\Phi}$. To do so, we write $\hat{\Phi} =  \bm{\hat{c}}^{\top} \Phi_{\Lambda,\delta}$, where $\bm{\hat{c}}$ is the corresponding minimizer  of \eqref{LASSO_Ap}, and set 
\bes{
f_{\Psi} = \sum_{\bm{\nu} \in \Lambda} \hat{c}_{\bm{\nu}} \Psi_{\bm{\nu}}.
}
Then following similar arguments as in proof of Theorem \ref{thm:CSexpapp}, we obtain
\eas{
\nmu{f - \hat{\Phi}}_{L^{\infty}(\cU)} & \leq  \nmu{f - f_{\Psi}}_{L^{\infty}(\cU)}  + \nmu{f_{\Psi} - \hat{\Phi}}_{L^{\infty}(\cU)} \\
 &\leq  \nmu{\bm{{c}} -\bm{\hat{c}}}_{1,\bm{u}}  +  \sum_{\bm{\nu} \in \Lambda} \nm{\Psi_{\bm{\nu}} - \Phi_{\bm{\nu},\delta} }_{L^{\infty}(\cU)} \nm{\bm{\hat{c}}}_1 \\
  &\leq   C_2 \left ( \sigma_{s,L}(\bm{c})_{1,\bm{u}} + s \nmu{\bm{e'}}_2 \right )  +  \delta \nm{\bm{\hat{c}}}_1,
}
where $\bm{e'}$ is such that $\bm{f} = \bm{A'}\bm{c}+ \bm{e'}$. Moreover, we have 
\bes{
\nmu{\bm{e'}}_2  = \nmu{(\bm{A}-\bm{A'})\bm{c}+ \bm{e}}_2 \leq  \sqrt{n} \delta \nmu{\bm{c}}_2 +    \sigma_{s,L}(\bm{c})_{1,\bm{u}},
}
and now using the fact that $\bm{\hat{c}}$ is a minimizer of \eqref{LASSO_Ap}, we deduce that
\bes{
\nm{\bm{\hat{c}}}_1 \leq \nm{\bm{\hat{c}}}_{1,\bm{u}}+ \mu \nm{\bm{A'} \bm{\hat{c}} -\bm{f}}_{2} \leq \mu \nm{\bm{f}}_2 \leq \mu \nm{f}_{L^{\infty}(\cU)} \leq \mu.
}
Here, in the last step, we use that fact that $\nm{f}_{L^{\infty}(\cU)} \leq 1$ by assumption. Combining these estimates, noticing that $\nm{\bm{c}}_2 \leq \nm{f}_{L^2(\cU)} \leq \nm{f}_{L^{\infty}(\cU)} \leq 1$ and using the same bound for $ \sigma_{s,L}(\bm{c})_{1,\bm{u}}$ as in Theorem \ref{thm:CSexpapp}, we deduce that
\bes{
\nmu{f - \hat{\Phi}}_{L^{\infty}(\cU)} \leq  C  \exp \left ( -\gamma s^{1/d} \right ) +  \delta \sqrt{n}  + \delta \mu .
}
Recall that  $\mu = \frac{4\sqrt{42}}{19} s$. We now set
\bes{
\delta = \min \left \{ \frac{1}{\sqrt{n}} \exp \left ( -\gamma s^{1/d} \right ) , \frac{1}{s} \exp \left ( -\gamma s^{1/d} \right ) ,  \frac{5/\sqrt{42}}{\sqrt{n}( 9 + 10 s) }  \right \},
}
to deduce that
\bes{
\nmu{f - \hat{\Phi}}_{L^{\infty}(\cU)} \leq   C  \exp \left ( -\gamma s^{1/d} \right ),  
}
after possible change of $C$.

It remains to establish the size and depth bounds for the network. First, from the definition of $\delta$, notice that
\bes{
\delta  \gtrsim   \frac{\exp \left ( -\gamma s^{1/d} \right )}{s \sqrt{n}},
}
from which, we easily obtain
\bes{
\log(\delta^{-1})  \lesssim   \log(s)+ \log(n) +\gamma s^{1/d}.
}
Now, notice that $\nm{\bm{\nu}}_{1} \leq \prod^{d}_{j=1} (\nu_j+1) - 1 \leq s,$ for all $\bm{\nu} \in \Lambda$, and therefore $m(\Lambda) \leq s$. Proposition \ref{prop:polyNN} therefore gives
\eas{
\mathrm{depth}(\Phi_{\Lambda,\delta}) &\leq c' (1 + d \log(d))(1+\log(s)) (s +  \log(s)+ \log(n) +\gamma s^{1/d})
\\
\mathrm{size}(\Phi_{\Lambda,\delta}) & \leq c' \left ( d^2 s^2 + d s ( \log(s)+ \log(n) +\gamma s^{1/d}) + d^2 n \left ( 1 +  \log(s)+ \log(n) +\gamma s^{1/d} \right ) \right ),
}
for some universal constant $c' > 0$.
Since $s \geq 1$, we deduce that
\eas{
\mathrm{depth}(\Phi_{\Lambda,\delta}) &\leq c' (1 + d \log(d))(1+\log(s)) (s + \log(n) +\gamma s^{1/d})
\\
\mathrm{size}(\Phi_{\Lambda,\delta}) & \leq c' \left ( d^2 s^2 + (d s + d^2 n )\left ( 1 +  \log(s)+ \log(n) +\gamma s^{1/d} \right ) \right ).
}
which completes the  proof.

\end{proof}

\bibliographystyle{siamplain}
\bibliography{MLFArefs}

\newpage

\end{document}